\documentclass[10pt,journal,compsoc]{IEEEtran}

\ifCLASSOPTIONcompsoc

  \usepackage[nocompress]{cite}
\else
  \usepackage{cite}
\fi
\ifCLASSINFOpdf

\else

\fi

\usepackage{amsmath,amssymb,bbm,bm,upgreek,proof,amsthm,amssymb}
\usepackage{mathtools}
\usepackage{comment}
\usepackage{upgreek}
\usepackage{multirow}
\usepackage{makecell}
\usepackage[mathscr]{eucal}
\usepackage{url}
\usepackage{graphicx} 
\usepackage{hyperref}
\usepackage{nameref}
\usepackage{cancel}
\usepackage{subcaption}
\usepackage{xcolor}
\usepackage{array, rotating}
\usepackage{multirow, makecell}

\newcommand{\iss}{i^{**}}
\newcommand{\op}{\mathcal{O}_{\mathbb{P}}}
\newcommand*{\p}{\mathbb{P}}
\newcommand{\RR}{\mathbb{R}}
\newcommand{\ee}{\end{aligned} \end{equation}}
\newcommand{\eq}{\end{quote}}
\newcommand{\diag}{\mathrm{diag}}
\newcommand{\ep}{\end{parts}}
\newcommand{\bqp}{\begin{quote}\begin{parts}}
\newcommand{\bds}{\boldsymbol}
\newcommand{\epq}{\end{parts}\end{quote}}
\newcommand{\E}{\mathbb{E}}

\newtheorem{remark}{Remark}
\newtheorem{lemma}{Lemma}
\newtheorem{theorem}{Theorem}
\newtheorem{A}{Assumption}

\newtheorem{claim}{Claim}
\newcommand{\twoinf}{2\rightarrow \infty}
\newcommand{\T}{\top}

\newcommand{\W}{\M W}
\newcommand{\off}{\mathrm{off}}

\newcommand{\tp}{^\top}
\newcommand{\HW}{\hat{\M W}}

\DeclareMathOperator*{\argmin}{arg\,min}

\newcommand {\ba}{{\bm{c}}}
\newcommand{\BL}{\Big (}
\newcommand{\BR}{\Big )}
\newcommand{\MD}{\mathbf{\Delta}}

\def\T{{ \mathrm{\scriptscriptstyle T} }} 
\newcommand{\Rom}[1]{\text{\uppercase\expandafter{\romannumeral #1\relax}}}
\newcommand{\bee}{\begin{equation}\begin{aligned}}
\newcommand{\M}[1]{{{\mathbf{\MakeUppercase{#1}}}}}

\newcommand{\emm}{\end{bmatrix}}

\newcommand{\F}{\mathrm{F}}

\numberwithin{equation}{section}

\begin{document}
\title{Exact Recovery of Community Structures Using DeepWalk and Node2vec}
\author{Yichi~Zhang and Minh~Tang
\IEEEcompsocitemizethanks{\IEEEcompsocthanksitem Yichi Zhang and Minh Tang are with the Department of Statistics, North Carolina State University, Raleigh, NC, 27606.\protect\\
E-mail: yzhan239@ncsu.edu $\&$ mtang8@ncsu.edu}}
\IEEEtitleabstractindextext{
\begin{abstract}
Random-walk based network embedding algorithms like DeepWalk and node2vec are widely used to obtain Euclidean representation of the nodes in a network prior to performing downstream inference tasks. However, despite their impressive empirical performance, there is a lack of theoretical results explaining their large-sample behavior. In this paper, we study node2vec and DeepWalk through the perspective of matrix factorization. In particular we analyze these algorithms in the setting of community detection for stochastic blockmodel graphs (and their degree-corrected variants). By exploiting the row-wise uniform perturbation bound for leading singular vectors, we derive high-probability error bounds between the matrix factorization-based node2vec/DeepWalk embeddings and their true counterparts, uniformly over all node embeddings. Based on strong concentration results, we further show the {\textit{perfect}} membership recovery by node2vec/DeepWalk, followed by $K$-means/medians algorithms. Specifically, as the network becomes sparser, our results guarantee that with large enough window size and vertices number, applying $K$-means/medians on the matrix factorization-based node2vec embeddings can, with high probability, correctly recover the memberships of all vertices in a network generated from the stochastic blockmodel (or its degree-corrected variants). The theoretical justifications are mirrored in the numerical experiments and real data applications, for both the original node2vec and its matrix factorization variant. 
\end{abstract}
}
\maketitle
\IEEEdisplaynontitleabstractindextext
\IEEEpeerreviewmaketitle
\IEEEraisesectionheading{\section{Introduction}}
\IEEEPARstart{G}{iven} a network $\mathcal{G}$, a popular approach for analyzing
$\mathcal{G}$ is to first map or embed its vertices into some low dimensional
Euclidean space and then apply machine
learning and statistical inference procedures in this space. Through this embedding process, multiple tasks could be
conducted on the network such as community detection (e.g.,
\cite{von2007tutorial,wang2017community}), link
prediction (e.g., \cite{liben2007link}), node classification (e.g.,
\cite{perozzi2014deepwalk,hamilton2017inductive}) and
network visualization (e.g., \cite{theocharidis2009network}).
There has been a large and diverse collection of network
embedding algorithms proposed in the literature, including those based
on spectral embedding \cite{rohe2011spectral,sussman2012consistent,shi_malik}, multivariate statistical dimension
reduction \cite{robinson1995typology,ye2005two}, and
neural network \cite{kipf2016semi,line_graph,pine}. See \cite{frame_work}, 
\cite{hamilton2017}, \cite{cui2018survey} and \cite{chen2020graph}
for recent surveys of network embedding and graph representation
learning.

In recent years there has been significant interest in  
network embeddings based on random-walks. The most
well-known examples include DeepWalk \cite{perozzi2014deepwalk} and node2vec
\cite{grover2016node2vec}. These algorithms are computationally
efficient and furthermore yield impressive
empirical performance in many different scientific
applications including recommendation systems \cite{palumbo2018knowledge},
biomedical natural language processing \cite{zhang2019biowordvec}, human protein identification
\cite{zhang2019identification}, traffic prediction
\cite{zheng2020gman} and city road layout modeling
\cite{chu2019neural}. Nevertheless, despite their wide-spread use,
there is still a lack of theoretical results on their large-sample properties. In
particular it is unclear what the node embeddings represents as well
as their behavior as the number of nodes increases. 

Theoretical properties for DeepWalk, node2vec, and related
algorithms had been studied previously in the computer science
community. The focus here had been mostly on the convergence of the 
{\em entries} of the co-occurrence matrix as the lengths and/or number
of random walks go to infinity. For example, motivated by the analysis in \cite{levy2014neural} for word2vec, 
the authors of \cite{qiu2018network,sussman_vec} 
showed that DeepWalk and node2vec using the skip-gram model with negative sampling is equivalent to factorizing a
matrix whose entries are obtained by taking the entry-wise logarithm of
a co-occurrence matrix, provided that the embedding dimension $d$ is
sufficiently large (possibly exceeding the number of nodes $n$). 
These authors also derived the limiting form of the entries of this matrix as the length of the
random walks goes to infinity. 
These results were further extended in \cite{qiu2020concentration} to yield finite-sample concentration bounds for
the co-occurrence entries. 
Note, however, that the above cited works focused
exclusively on the case of a fixed graph and thus do not provide results
on the large sample behavior of these algorithms as $n$ increases.   

The statistical community, in contrast, had extensively studied the
large-sample properties of graph embeddings
based on matrix factorization. However the embedding algorithms
considered are almost entirely based on singular value decomposition (SVD) of either the
adjacency matrix or the Laplacian matrix and its normalized and/or
regularized variants. For example, in the setting of the popular
stochastic blockmodel random graphs, \cite{rohe2011spectral} and
\cite{sussman2012consistent} derived consistency results for a
truncated SVD of the normalized Laplacian
matrix and the adjacency matrix. 
Subsequently 
\cite{tang2018limit,grdpg} strengthened these results by providing
central limit theorems for the components of the eigenvectors of
either the adjacency matrix or the normalized Laplacian matrix under
the more general random dot product graphs model. As DeepWalk and node2vec are based on taking the entry-wise
logarithm of a random-walk co-occurrence matrix, the techniques used in these cited results do not readily translate to this setting.

\subsection{Contributions of the current paper} The current paper
studies large-sample properties of random-walk based
embedding algorithms. We first present convergence results for
the embeddings of DeepWalk and node2vec in the case of stochastic
blockmodel graphs and their degree-correctd variant. We then show that running $K$-means or $K$-medians on the resulting
embeddings is sufficient for {\em exact} recovery of the latent community
assignments.  Our theoretical results thus provide a bridge between
previous results in the computer science community and
their statistics counterpart.

We emphasize that our focus on stochastic blockmodel graphs 
is done purely for ease of exposition. Indeed, most of our results continue to hold for the
more general inhomogeneous Erd\H{o}s-R\'{e}nyi (IER) random graphs
model \cite{bollobas2001random,hoff2002}, provided that the edge probabilities are sufficiently
homogeneous, i.e., the minimum and maximum values for the edge
probabilities are of the same order (possibly converging to $0$) as $n$
increases; recall that IER is one of the most general model for
edge independent random graphs. In particular we can show that the co-occurrence matrices constructed from the sampled networks is uniformly close (entrywise) to that for the true but unknown edge probabilities matrices. However, as IER random graphs need not possess low-dimensional structure (even when $n$ increases), it is not clear what the embeddings obtained from these co-occurrence matrices represent. See Section~\ref{sec:disc} for further
discussion. 

We now outline our approach. The original node2vec and DeepWalk
algorithms are based on optimizing a non-convex skip-gram model using
stochastic gradient descent (SGD); this optimization problem has multiple local minimima and the obtained embeddings can thus
be numerically unstable (see e.g., \cite{surprising}). We instead consider, for each embedding
dimension $d$, the optimal low-rank approximation
of an observed transformed co-occurrence matrix similar to that used in 
\cite{levy2014neural,qiu2020concentration}, and recently \cite{barot2021community,sussman_vec}. We first
show that the entries of the co-occurrence matrix computed using
the observed adjacency matrix is {\em uniformly} close to the entries
of the co-occurrence matrix computed using the true but unknown edge
probabilities matrix. This uniform bound implies that the
entry-wise logarithm of the two co-occurrence matrices are
also {\em uniformly} close 
and thus, with high probability, the
co-occurrence matrix constructed using the observed
graph is well-defined. 
%
In the case of stochastic
blockmodel graphs the true edge probabilities matrix give rise to a (transformed) co-occurrence matrix with rank at most $K$ where $K$ is the number of
blocks and thus for stochastic blockmodel graphs with $K \ll n$ blocks. By leveraging both classical (e.g., the celebrated Davis-Kahan theorem \cite{davis70}) as well as recent results on matrix perturbations in the $2 \to \infty$ norm (e.g., \cite{cape2019two,lei2019unified}), we show that the truncated
low-rank representation of both matrices are {\em uniformly} close, i.e., the
embeddings of the observed graph is, up to
orthogonal transformation, approximately the same as that for the
true edge probabilities matrix. 
Therefore, by running $K$-means or $K$-medians on the
embeddings of the observed graph, we can with high probability recover the latent community structures for {\em every} vertices. %

Our paper is organized as follows. In Section~\ref{sec:2}, we give a brief introduction of 
node2vec \cite{grover2016node2vec} and DeepWalk \cite{perozzi2014deepwalk}, and describe the matrix factorization
perspective for these algorithms. In particular, DeepWalk
can be treated as a special case of node2vec by setting the $2^{\mathrm{nd}}$-order
random-walk parameters $(p,q)$ to be $(1,1)$, which will be
assumed in Section~\ref{sec:3} for simplicity of theoretical
analysis. In Section~\ref{sec:3} we provide uniform entry-wise error
bounds for the entries of the $t$-step random-walk
transition matrix and their implications for community recovery. 
The theoretical results in Section~\ref{sec:3} hold
for both the dense and sparse regimes where the average degree grows
linearly and sublinearly in the number of nodes, respectively. 
In Section~\ref{sec:simu} we present simulations to corroborate our theoretical
results. In Section \ref{sec:real}, we apply node2vec to three real-world network datasets and show its remarkable practical performances.We conclude the paper in Section~\ref{sec:disc} with a discussion of some open
questions and potential improvements. All proofs of the stated results,  associated technical
lemmas, and additional numerical results are provided in the Supplementary File. 


\subsection{Notation}
We first introduce some general notations that are used 
throughout this paper. For a given positive integer $K$, we denote by
$[K]$ the set $\{1,2,\dots,K\}$. We denote a graph on $n$ vertices by $\mathcal{G} = (\M V, \M E)$ where $\M V =
\{v_i\}_{i = 1}^n$ and $\M E= \{e_{ii'}\}_{i,i' = 1}^n$ are the
vertices and edge sets, respectively. Unless specified otherwise, all graphs in this paper are
assumed to be undirected, unweighted and loop-free. For each node $v_i$ we denote
by $\mathcal{N}(v_i)$ the set of nodes $v_{i'}$ adjacent to $v_i$. 
If $\mathcal{G}$ is a graph on $n$ vertices then its $n \times n$ adjacency matrix is denoted
as $\M A = [a_{ii'}]$. In the subsequent discussion we
often assume that the upper triangular entries of $\M A$ are independent
Bernoulli random variables with $\mathbb{E}[a_{ii'}] = p_{ii'}$ when $i' < i$. As $\mathbf{A}$ is symmetric we also set $a_{ii'} = a_{i'i}$ for $i' > i$ and denote by
$\M P = [p_{ii'}]$ the corresponding $n \times n$ matrix of edge probabilities.

Given a graph $\mathcal{G}$ with adjacency matrix $\mathbf{A}$, let $\M D_{\M A} = \mathrm{diag}(d_1,\dots,d_n)$ be a diagonal matrix with $d_i
= \sum_{i' = 1}^{n}a_{ii'} $ as its $i$th diagonal element. Assuming $\mathcal{G}$ is connected, we define a
random walk on $\mathcal{G}$ with a $1$-step transition matrix $\hat{\M
  W} = \M A \M D_{\M A}^{-1}$. 
Correspondingly, when appropriate, we also define $\mathbf{W} =
\mathbf{P} \mathbf{D}_{\mathbf{P}}^{-1}$ where $\M D_{\M P}
= \diag(p_1,\dots,p_n)$ is the diagonal matrix with $p_i = \sum_{i' =
1}^{n}p_{ii'}$. 

We use $\|\cdot\|$, $\|\cdot \|_\F$, $\|\cdot\|_{\infty}$ and $\|\cdot\|_{\mathrm{max}}$ to denote the
spectral norm, Frobenius norm, maximum absolute row sum, and maximum entry-wise value of a matrix, respectively. 
We also use $\|\cdot\|_{\max,\mathrm{off}}$ and
$\|\cdot\|_{\max,\mathrm{diag}}$ to denote the maximum value for the
off-diagonal and diagonal entries of a matrix, i.e., for a square matrix $\M M =[m_{ii'}]$,
\bee
\|\M M\|_{\max,\mathrm{off}} = \max_{i\neq i'}
|m_{ii'}|, \,\, \|\M M\|_{\max,\mathrm{diag}} = \max_{i} |m_{ii}|.
\ee
 We use $|\cdot|$ to denote the absolute value of a real number as
 well as the cardinality of a finite set. The vectors $\M 0_{d}$ and
 $\M 1_{d} \in \RR^{d}$ are $d$ dimensional vectors with all elements
 equal to $0$ and $1$, respectively. The set of $d \times d'$
 matrices with orthonormal columns is denoted as
 $\mathbb{O}_{d,d'}$ while the set of $d \times d$ orthogonal
 matrices is denoted as $\mathbb{O}_{d}$. 

 For two terms $a$ and $b$, let $a\wedge b :=
 \min\{a,b\}$. We write $a \precsim b$ and $a \succsim b$ if there
 exists a constant $c$ not depending on $a$ and $b$ such that $a \leq
 c b$ and $a \geq c b$, respectively. If $a \precsim b$ and $a \succsim
 b$ then $a \asymp b$. We say an event $\mathcal{A}$ depending on $n$ happens with high probability (whp) if $\p(\mathcal{A}) \geq 1 - \mathcal{O}(n^{-c})$ for some constant $c > 3$. Finally, for random sequences $A_n, B_n$,
 we write $A_n = O_{\mathbb{P}}(B_n)$ if $A_n/B_n$ is bounded whp and
 $A_n = o_{\mathbb{P}}(B_n)$ if $A_n/B_n \rightarrow 0$ whp.
 
\section{Summary of node2vec and SBM}\label{sec:2}
In this section we first provide a brief overview of the node2vec
algorithm. We then discuss the popular stochastic blockmodel (SBM) for
random graphs. Finally we discuss a matrix factorization perspective
to node2vec and show that, for a graph $\mathcal{G}$ generated from
a stochastic blockmodel, this
matrix factorization approach leads to a low-rank approximation of an elementwise non-linear transformation of the random walk transition matrix for $\mathcal{G}$. 
\subsection{Node2vec with negative sampling}\label{sec:node2vec}
First introduced in \cite{grover2016node2vec}, node2vec is a
computationally efficient and widely-used algorithm for network
embedding. Motivated by the ideas behind word2vec for text
documents \cite{mikolov2013distributed}, node2vec generates sequences of nodes using random walks which are then feed into a skip-gram
model \cite{Tomas2013ICLR} to yield the node embeddings. The original
skip-gram model is quite computationally demanding for large networks
and hence, in practice, usually replaced by a skip-gram with
negative sampling (SGNS). The resulting algorithm is 
summarized below.
\begin{enumerate}
\item {\bf (Sampling Random Paths):} First
  generates $r$ independent $2^{\mathrm{nd}}$ order random walks on $\mathcal{G}$ with each
  having a fixed length $L$. 
  A $2^{\mathrm{nd}}$ order
random walk of length $L$ starting at $v_i$ with parameters $p$ and $q$ is generated
as follows. First let $v_1^{(i)} = v_i$. 
Next sample $v^{(i)}_2$ from
$\mathcal{N}(v_1^{(i)})$ uniformly at random.
Then for $3 \leq \ell \leq L$, sample $v^{(i)}_{\ell} \in \mathcal{N}(v^{(i)}_{\ell-1})$ with probability,
\begin{equation*}
\p(v^{(i)}_{\ell} = v_0) 
= 
\begin{cases}
\frac{1}{p}J(v_0) & \text{if $v_0 = v^{(i)}_{\ell-2}$},\\
J(v_0)& \text{if $v_0  \in \mathcal{N}(v^{(i)}_{\ell-2})$},\\
\frac{1}{q}J(v_0)  & \text{if $v_0  \not \in \mathcal{N}(v^{(i)}_{\ell-2})$},
\end{cases} 
\end{equation*}
where $J(v_0)$ is given by
\bee
\frac{1}{J(v_0)} = 
 p^{-1} &+ |\mathcal{N}(v^{(i)}_{\ell-2}) \cap
\mathcal{N}(v^{(i)}_{\ell-1})| \\ & +
q^{-1}|\mathcal{N}(v^{(i)}_{\ell-2})^{\mathrm{c}} \cap
\mathcal{N}(v^{(i)}_{\ell-1})|\Bigr.
\ee
The form of $J(v_0)$ allows for $v_{\ell}^{(i)}$ to have possibly unbalanced
  probabilities of reaching three different types of nodes in the
  neighborhood of $v_{\ell-1}^{(i)}$, namely (1) the previous node
  $v_{\ell-2}^{(i)}$; (2) nodes belonging to both the neighborhoods of
  $v_{\ell-2}^{(i)}$ and $v_{\ell-1}^{(i)}$;  (3) nodes belonging only to the
  neighborhood of $v_{\ell-1}^{(i)}$ but not the neighborhood of
  $v_{\ell-2}^{(i)}$. The parameters $p > 0$ and $q > 0$ provide weights for these three different type of nodes and hence control the speed at
  which the random walk leaves the neighborhood of the original node
  $v_i$. 
In this paper we assume that the starting
  vertex $v_i$ of any random walk is sampled according to a
  stationary distribution $\mathbf{S} = (S_1, \dots, S_n)$ on $\mathcal{G}$ with
\bee
\p\big(\text{Starting Vertex is } v_i\big) = S_i =\frac{d_{i}}{2|\mathbf{E}|}
\ee 
for all $v_i \in \mathbf{V}$. For a given $i \in [n]$ we denote by $r_i$ the number of random walks starting from $v_i$, $\bds{\ell}^{(i)}_{j}$ as the $j$th random walk starting from $v_i$ and $\mathcal{L}_{i} = \{\bds{\ell}^{(i)}_{j} , j\in [r_i]\}$ as the set of all random walks starting from $v_i$. 

\begin{remark}
{\upshape We consider only the case of $p = q = 1$ for our
theoretical analysis. The choice $p = q = 1$ is
the default setting for node2vec as suggested in
the original paper \cite{grover2016node2vec} and leads to a sampling scheme
equivalent to that of DeepWalk \cite{perozzi2014deepwalk}; 
the subsequent analysis thus also applies to DeepWalk.}
\end{remark}

\item {\bf (Calculating $\M C$) :} Borrowing ideas from word2vec
\cite{mikolov2013distributed}, node2vec creates a $n \times n$ node-context matrix $\M C = [C_{ii'}]_{n \times
    n}$ whose $ii'$th entry records
  the number of times the pair $(v_i,v_{i'})$ appears among
  all random paths in $\bigcup_{i = 1}^{n}\mathcal{L}_i$. More
  specifically, for a given window size $(t_L, t_U)$, 
  $C_{ii'}$ is the number of times that $(v_i,v_{i'})$ appears within a sequence 
\begin{equation}
    \label{node2vec:ii'}
    \begin{aligned}
\dots,v_i,&\underbrace{\dots\dots}_{t - 1 \text{vertices}},v_{i'},\dots \quad \text{ or } 
\\
\dots,v_{i'}, &\underbrace{\dots\dots}_{t - 1\text{vertices}}, v_i,\dots
\end{aligned}
\end{equation}
among all random paths in $\bigcup _{i = 1}^{n}\mathcal{L}_i$; here $t$ is any integer satisfying $t_L \leq t \leq t_U \leq L-1$   
\begin{remark}
{\upshape
The original node2vec algorithm fixed $t_L = 1$ while in this paper we
allow for varying $t_L$ for a more flexible theoretical analysis. In Section~\ref{sec:T} we show that different values for $(t_L,t_U)$ could
lead to different convergence rates for the embedding and furthermore appropriate values for $(t_L, t_U)$ depend
intrinsically on the sparsity of the network.}
\end{remark}

\item {\bf (Skip-gram model with negative sampling) :}
Given the $n \times n$ matrix $\M C$ and an embedding dimension $d$,
node2vec uses the SGNS model to learn the node embedding matrix $\M F
\in \mathbb{R}^{n \times d}$ and the context embedding matrix
$\mathbf{F}' \in \mathbb{R}^{n \times d}$. The $i$th row of $\M F$ is
the $d$-dimensional embedding vector of node $v_i$.
In slight contrasts to the original node2vec, in this paper we do not
require the constraint $\M F = \M F'$. The objective function of SGNS
model for a given $\M C$ is defined as
\bee\label{SKIP_NS} \mathrm{g}(\M F,\M F') =& \sum_{ij}
C_{ij}\Big[\log\big\{\sigma( {\boldsymbol f}_i\tp 
{\boldsymbol f}'_{j})\big\} 
\\
&+ \kappa \mathbb{E}_{{\boldsymbol
f}'_\mathcal{N} \sim \M P_{\text{ns}}}\big[\log\big\{\sigma(-{\boldsymbol
f}_i\tp {\boldsymbol f}'_\mathcal{N})\big\}\big]\Big].\ee
Here $\bm{f}_i$ (resp. $\bm{f}'_{j}$) are the $i$ (resp. $j$) row of $\mathbf{F}$ (resp. $\mathbf{F}'$), $\kappa$ is the ratio of negative to positive samples,
$$\M P_{\text{ns}}({\boldsymbol f}'_\mathcal{N}) = \frac{\sum_{i' = 1}^n
C_{\mathcal{N}i'}}{\sum_{i,i'} C_{ii'}}$$ is the empirical unigram
distribution for the negative samples, and $\sigma$ is the logistic function. The original node2vec algorithm solves
for $(\hat{\M F}, \hat{\M F}')$ by minizing Eq.~\eqref{SKIP_NS}  over 
$(\mathbf{F}, \mathbf{F'})$ using SGD. In
this paper we use a matrix
factorization approach, described in section~\ref{sec:MF}, to find $(\hat{\M F}, \hat{\M F}')$. 
\end{enumerate}
\subsection{Stochastic blockmodel}\label{sec:sbm}
The stochastic blockmodel (SBM) of \cite{holland1983stochastic} is one of the
  most popular generative model for network data. It often serves as a benchmark for evaluating community detection algorithms \cite{abbe2017community}. Our
  theoretical analysis of node2vec/DeepWalk is situated in the context of this model. We parametrize 
  a $K$-blocks SBM in terms of two parameters $(\M B, \M
  Z )$ where $\M B = [b_{uu'}]$ is a
  symmetric matrix of blocks connectivity and $\M Z
  \in \{0,1\}^{n\times K}$ is a matrix whose rows denote the block
  assignments for the nodes; we use $\tau(i)\in [K]$ to represent
  the community assignment for node $i$, i.e., the $i$th row of $\M
  Z$ contains a single $1$ in the $\tau{k}(i)$th element and $0$
  everywhere else. Given $\mathbf{B}$ and $\M Z$, the edges
  $a_{ii'}$ of $\mathcal{G}$ are {\em independent} Bernoulli random
  variables with $\mathbb{P}[a_{ii'} = 1] = B_{\tau(i),\tau(i')}$, i.e., the probability of
  connection between $i$ and $i'$ depends only on the communities
  assignment of $i$ and $i'$. Denote by 
\bee\label{SBM:P}
\M P = [p_{ii'}] = \M
  Z \M B \M Z\tp
\ee the matrix of edge probabilities. We
  denote a graph with adjacency matrix $\mathbf{A}$ sampled from a
  stochastic blockmodel as $\mathbf{A} \sim \mathrm{SBM}(\mathbf{B},
  \M Z)$, and, for any stochastic blockmodel graph, we denote
  by $n_{k}$ the number of vertices assigned to block $k$. 
  We shall also assume, without
loss of generality, that $\M Z$ is ordered by blocks: 
\bee\label{theta}
\M Z := \begin{pmatrix}
\bds 1_{n_1} & \bds 0 & \dots
 & \bds 0
\\
\bds 0 & \bds 1_{n_2} & 
\dots & \bds 0
\\
\vdots & \vdots & \vdots & \vdots
\\
\bds 0 & \bds 0 & \dots & \bds 1_{n_K}
\end{pmatrix}.
\ee
{\color{black}In real-world applications the average degree of a networks usually grows at a slower rate than $\Theta(n)$. To model this phenomenon we introduce a sparse parameter $\rho_n$ that can vanish as $n\rightarrow \infty$. For ease of exposition we use the following parametrization of $\M B$ that is commonly used in the literature (see  e.g., \cite{bickel2009nonparametric}). 
\begin{A}
  \label{am:s} There exists a fixed $K \times K$ matrix $\M B_0$ such that $\M B = \rho_n\M B_0$ with $\rho_n\succsim n^{-\beta}$ for some $\beta \in [0,1)$.
\end{A}
The parameter $\rho_n$ scales the edge probabilities in $\mathbf{B}$. As $\rho_n \succsim n^{-\beta}$, the average degree of the nodes in $\mathcal{G}$ grows at rate $n^{1 - \beta}$ so that larger values of $\beta$ lead to sparser network. 
It is well known that, for sufficiently large $n$, if $\mathcal{G}$ satisfies Assumption \ref{am:s} then $\mathcal{G}$ is connected
with high probability (see e.g. Section~7.1 of \cite{bollobas2001random}). Then $\M  P = \M Z \mathbf{B} \M Z^{\top}$ has a $K \times K$ block structure and thus has rank at most $K$. 
}
\subsection{Node2vec and matrix factorization}\label{sec:MF}
{\color{black}
In general, for a fixed given embedding dimension $d < n$, 
minimization of the objective function in Eq.~\eqref{SKIP_NS} leads to
a non-convex optimization problem and the potential convergence of SGD
into local minima makes the asymptotic analysis of $\hat{\M F}$ quite
complicated. Indeed, almost all existing results for non-convex optimization using gradient descent or SGD only guarantees convergence to a local minima provided that the initial estimate is sufficiently close to this local minima, see e.g., \cite[Section 5]{chi2019nonconvex} and \cite{sgd_convergence}.
We thus desire a different approach for finding $\hat{\M
  F}$, namely one for which the form of $\hat{\M F}$ is more readily
apparent. One such approach is the use of matrix factorization.  
For example, in the context of word2vec embedding,
\cite{levy2014neural} showed that minimization of Eq.~\eqref{SKIP_NS}
when $\mathbf{C}$ is a word-context matrix is equivalent to a matrix
factorization problem on some {\em elementwise} non-linear transformation
of $\mathbf{C}$ and that this transformation can be related to the notion of pointwise
mutual information between the words. Motivated by this line of inquiry,
we consider a formulation of node2vec wherein $\hat{\M
  F} \hat{\M F}'^{\top}$ is a low-rank approximation of
some elementwise transformation $\tilde{\mathbf{M}}$ of $\hat{\mathbf{W}}$;
recall that $\hat{\mathbf{W}}$ is the $1$-step transition matrix for
the canonical random walk on $\mathcal{G}$. 
We emphasize that this approach had
been considered previously in \cite{qiu2018network} and recently by \cite{barot2021community,sussman_vec}. The main
contribution of our paper is in showing that this matrix
factorization leads to consistent community recovery for stochastic blockmodel graphs. 

We now describe the matrix $\tilde{\mathbf{M}}$. 
 In the context of the word2vec algorithm,
\cite{levy2014neural} showed that there exists some embedding dimension $d$ such that the minimizer of Eq.~\eqref{SKIP_NS} over $\M F \in
\mathbb{R}^{n \times d}$ and $\M F' \in \mathbb{R}^{n \times d}$ satisfies
\bee\label{Fact}
\hat{\M F}\hat{\M F}'^\top \!&=\tilde{\M M}(\M C,\kappa) 
\!:= \left[\log \frac{C_{ij}(\sum_{ij}C_{ij})}{\kappa\bigl(\sum_{i}C_{ij}\bigr) \bigl( \sum_{j}C_{ij}\bigr)}\right]_{n\times n}
\ee

Using the same idea for our analysis of node2vec, we first fixed $n$
and show that if the number of sampled random paths increases then $\tilde{\M
  M}(\M C,k)$ converges, elementwise, to a limiting matrix
$\tilde{\mathbf{M}}_0$ defined below. Note that the entries of $\tilde{\M M}_0$ can be interpreted as 
point-wise mutual information (PMI) between the nodes. 
\begin{theorem}
\label{T1}
Let $n$ be fixed but arbitrary. Suppose $\mathcal{G}$ is a connected graph on $n$ vertices and $t_U$ is large enough such 
that the entries of $\sum_{t = t_L}^{t_U}\hat{\M W}^t$ are all
positive. Applying the node2vec sampling strategy introduced in Section~\ref{sec:node2vec} on $\mathcal{G}$ we have
\bee\label{T1:1}
\tilde{\M M}(\M C, \kappa) \xrightarrow{\text{a.s.}}& \tilde{\M M}_0(\mathcal{G},t_L,t_U,\kappa,L)
\\
&:= \log\Bigg\{\frac{2|\M A|}{\kappa \gamma }\sum_{t = t_L}^{t_U}(L - t)\M D_{\M A}^{-1}\hat{\M W}^t\Bigg\} 
\ee
as the number of random paths $r = \sum_{i = 1}^n r_i \rightarrow \infty$; recalling that $\hat{\M
  W} = \M A \M D_{\M A}^{-1}$. The convergence of $\tilde{\mathbf{M}}(\mathbf{C}, \kappa)$ to $\tilde{\mathbf{M}}_0$
is element-wise and uniform over all entries of
$\tilde{\mathbf{M}}(\mathbf{C}, \kappa)$. Here $|\mathbf{A}|$ denote the
sum of the entries in $\mathbf{A}$ and the constant $\gamma$ is defined as
$$\gamma := \frac{1}{2}(L - t_L - t_U)(t_U - t_L + 1).$$
\end{theorem}

To reduce notation clutter, we will henceforth drop the dependency of
$\tilde{\mathbf{M}}_0$ on the parameters
$\mathcal{G}, t_L, t_U, \kappa, L$. 
As the value of $r$ is chosen purely
for computational expediency, i.e., smaller values of $r$
require sampling fewer random walks, we will thus take the conceptual view that $r \rightarrow \infty$ so that $\tilde{\mathbf{M}}(\mathbf{C}, \kappa) \rightarrow \tilde{\mathbf{M}}_0$; note that $\tilde{\mathbf{M}}_0$ can be constructed explicitly from $\mathbf{A}$ without needing to sample any random walk. 
Combining Eq.~\eqref{Fact} and Theorem~\ref{T1}, we have that, for any 
fixed $n$, there exists an embedding dimension $d$ such that for $r\rightarrow \infty$, the matrices $\hat{\M F}$ and $\hat{\M F}'$ are
exact factors for factorizing $\tilde{\M
  M}_0$. 
 Note that $\mathbf{D}_{\mathbf{A}}^{-1}
\hat{\mathbf{W}}^{t}$ is symmetric for any $t \geq 1$ and hence $\tilde{\mathbf{M}}_0$ is {\em symmetric}. 

In practice one usually chooses $d \ll n$ to reduce the noise in the embeddings as well as combat the curse of dimensionality in downstream inference. Obviously if $d < n$ then
exact factors $(\hat{\M F}, \hat{\M
  F'})$ for factorizing
$\tilde{\mathbf{M}}_0$ might no longer exist (see e.g., \cite{levy2014neural}). The requirement that $\hat{\M F} \hat{\M F}'^{\top} = \tilde{\mathbf{M}}_0$ is, however, both misleading and unnecessary. Indeed,
as the observed graph is but a single {\em noisy}
sample generated from some true but unobserved edge probabilities matrix
$\mathbf{P}$, what we really want to recover is the factorization induced by $\mathbf{P}$. More specifically, replacing $\hat{\M W}^t$ and $|\M A|$ with ${\M W}^t$ and $|\M P|$ in $\tilde{\M M}_0$, we define
\bee\label{truth}
\M M_0
&= \log\Bigg\{\frac{2|\M P|}{\kappa \gamma}\sum_{t = t_L}^{t_U}(L - t)\M D_{\M P}^{-1}{\M W}^t\Bigg\} 
\ee
as the underlying-truth counterpart of $\tilde{\M M}_0$; 
note that, similar to $\tilde{\mathbf{M}}_0$, we had dropped the parameters associated with $\M M_0$ for simplicity of notations. Under the  SBM setting, the true signal matrices $\M P$ and $\M M_0$ are both low-rank and hence an embedding dimension of $d = \text{rk}(\M M_0) \ll n$ is sufficient to recover the factorization induced by $\M M_0$. 
  
  To be more precise, recall from Eq.~\eqref{theta} that for stochastic blockmodel graphs, the matrix $\M P$ has a $K\times K$
block structure. Thus both $\M W^{t}$ and  $\M D_{\M P}^{-1}\M W^t$ also have $K \times
K$ block structures. Eq.~\eqref{truth} then implies that $\M M_0$  also
has a $K \times K$ block structure and hence $\mathrm{rank}(\M M_0)
\leq K$. Most importantly, the $K \times K$ block structure of ${\M M}_0$
is also sufficient for recovering the community structure in
$\mathcal{G}$. {\color{black}
We will show in Section~\ref{sec:3} that the relative
error, in the \textit{row-wise maximum} norm, between
$\tilde{\M M}_0$ and ${\M M}_0$ converges to $0$ as
$n\rightarrow \infty$. This convergence, together with results for perturbation of eigenspaces, implies the existence of an embedding
dimension $d \leq K$ for which the $n \times d$ matrices $\hat{\M F}$
and $\hat{\M F'}$ obtained by factorizing $\tilde{\M M}_0$  lead to exact recovery of the community
structure in $\mathcal{G}$. }
}

\begin{remark}
{\upshape
If $\M P$ does
not arise from a stochastic blockmodel graph then ${\M M}_0$ need not
have a low-rank structure. Nevertheless we can still consider a
rank-$d$ approximation to ${\M M}_0$ for some $d <
\mathrm{rk}(\mathbf{M}_0)$. Furthermore, as we will clarify in
Section~\ref{sec:disc}, the bound for
$\|\tilde{\mathbf{M}}_0 - \mathbf{M}_0\|_{\max}$ in Section~\ref{sec:3}
also holds for general edge independent random graphs, provided that
the entries of $\M P$ is reasonably homogeneous. Hence $\tilde{\mathbf{M}}_0$ has an approximate low-rank
structure if and only if $\mathbf{M}_0$ also has an approximate low-rank
structure.
}
\end{remark}

In summary, motivated by the low-rank structure of  $\mathbf{M}_0$ in the case of SBM graphs, we view 
the matrix factorization approach for node2vec as
finding the best rank $d < n$ approximation $\hat{\bds{\mathcal{F}}}\cdot
\hat{\bds{\mathcal{F'}}}\tp$ to $\tilde{\M M}_0$ under Frobenius norm,
i.e.,  
\bee\label{propose}
(\hat{\bds{\mathcal{F}}},\hat{\bds{\mathcal{F}'}}) = \argmin_{({\bds{\mathcal{F}}},{\bds{\mathcal{F}}}')\in \mathbb{R}^{n\times d}\cdot\mathbb{R}^{n\times d}} \ \|\tilde{\M M}_0 - {\bds{\mathcal{F}}}\cdot {\bds{\mathcal{F}'}}\tp\|_{\F}.
\ee 
The minimizer of Eq.~\eqref{propose} is obtained by
truncating the SVD of $\tilde{\M M}_0$. More specifically, let
\bee\label{svdtrue}
\tilde{\M M}_0 = \hat{\M U} \hat{\M \Sigma} \hat{\M V}^{\top}
\ee 
with a decreasing order of singular values in $\hat{\M \Sigma}$. Then
for a given $d \leq \mathrm{rk}(\mathbf{M}_0)$, let 
\bee\label{def:f0}
\hat{\bds{\mathcal{F}}} = \hat{\M U}_d, \,\hat{\bds{\mathcal{F}}}' =
\hat{\M V}_d\hat{\M \Sigma}_d^{}
\ee
where $\hat{\M U}_d \in \mathbb{R}^{n\times d}, \hat{\M V}_d \in
\mathbb{R}^{n\times d}$ are the first $d$ columns of $ \hat{\M U}$ and
$\hat{\M V}$, respectively, and $\hat{\M \Sigma}_d \in
\mathbb{R}^{d\times d}$ is the diagonal matrix containing the $d$
largest singular values in $\hat{\M \Sigma}$.

{\remark\label{rk:dcsbm}
  {\upshape The appropriate embedding
  dimension $d$ for factorizing $\tilde{\M M}_0$ depends on knowing
  $\mathrm{rank}(\mathbf{M}_0)$.
  but the convergence of
  $\tilde{\M M}_0$ to that of ${\M M}_0$ does not require knowing $\mathrm{rank}(\mathbf{M}_0)$. For ease of exposition we will assume that
  $\mathrm{rank}(\mathbf{M}_0)$ is known; in practice it can be estimated consistently using an eigenvalue
  thresholding procedure provided that $\mathbf{M}_0$ has a low-rank
  structure. Finally, in the context of SBM graphs 
  and their degree-corrected variant, community recovery using $\hat{\M F}$ also depends on knowing $K$. For
  simplicity we also assume that $K$ is known, noting that consistent
  estimates for $K$
  are provided in \cite{bickel13:_hypot,lei2014}.}
}

\section{Theoretical analysis}\label{sec:3}
\label{sec:T}
\subsection{Entry-wise  concentration of $\hat{\M W}^t$ and $\tilde{\M M}_0$}
\label{sec:dense}
Recall that $\hat{\bds{\mathcal{F}}}$ is obtained from the
eigendecomposition of $\tilde{\M M}_0$ while the true embedding is
obtained from the eigendecomposition of $\mathbf{M}_0$ (see Eq.~\eqref{truth}). Therefore, before studying the community recovery using $\hat{\bds{\mathcal{F}}}$, we first study the
convergence of $\tilde{\M M}_0$ to $\mathbf{M}_0$. In particular we derive concentration bounds for $\tilde{\M M}_0 - \M M_0$ in both Frobenius and infinity norms. 
These bounds are
facilitated by the following Theorem~\ref{T2} which provides a precise uniform
bound for the entry-wise difference between the $t$-step
transition matrix $\hat{\M W}^{t}$ and $\M W^{t}$ defined using the 
adjacency matrix $\mathbf{A}$ and the edge probabilities matrix $\mathbf{P}$, respectively. 
\begin{theorem}\label{T2}
Let $\mathcal{G}\sim \mathrm{SBM}(\mathbf{B},
  \mathbf{Z})$ where $\mathbf{B}$ satisfies Assumption~\ref{am:s}. We then have the following bounds.
\begin{enumerate}
\item\textbf{(Dense regime)} Suppose $\rho_n\asymp 1$. Then
\begin{align}
&\|\hat{\M W} - \M W\|_{\max} = \mathcal{O}_{\mathbb{P}}( n^{-1}), \label{T2:oneorder}
\\ &
\|\hat{\M W}^2 - {\M W}^2\|_{\max,\mathrm{diag}} = 
\mathcal{O}_{\mathbb{P}}(n^{-1}), \label{T2main30}
\\ &
\|\hat{\M W}^2 - {\M W}^2\|_{\max,\mathrm{off}} = 
 \mathcal{O}_{\mathbb{P}}\Big(\frac{\log^{1/2}{n}}{n^{3/2}}\Big), \label{T2main3}
 \end{align}
 Furthermore, for $t \geq 3$, 
\begin{align}
     &\label{T2main4}
\|\hat{\M W}^{t} - \M W^{t}\|_{\max} =
                \mathcal{O}_{\mathbb{P}}\Big(\frac{\log^{1/2}{n}}{n^{3/2}}\Big), 
\end{align}
\item\textbf{(Sparse regime)} Let $\rho_n\rightarrow 0$ with $\rho_n \succsim n^{-\beta}$ for some $\beta \in [0,1)$. Then for $t \geq 4$ satisfying $\tfrac{t-3}{t-1} > \beta$ we have
\bee\label{T5:6:main}
\|\hat{\M W}^{t} - \M W^{t}\|_{\max} =
\op\Big(\frac{\log^{1/2}{n}}{n^{3/2} \rho_n^{1/2}}\Big).
\ee
\par
In addition if $0 \leq \beta < 1/2$ then
\bee\label{T10:extend}
&\|\hat{\M W}^{2} - \M W^{2}\|_{\max,\off} = \op\Big(\frac{\log^{1/2}n}{n^{3/2} \rho_n} \Big),
\\
&\|\hat{\M W}^{3} - \M W^{3}\|_{\max} = \op\Big(\frac{\log^{1/2}n}{n^{3/2} \rho_n} \Big).
\ee
\end{enumerate}
\end{theorem}
\begin{remark}
{\upshape Throughout this paper we assume that $t_L \geq 2$ instead of
  $t_{L} \geq 1$ as used in the original node2vec formulation. The rationale for this assumption is as follows. Recall
  the definition of $\tilde{\M M}_0$ in Eq.~\eqref{T1:1}. 
 If we allow $t$ to starts from $1$ in the sum $
  \sum_{t = t_L}^{t_U}(L - t)\cdot\big(\M D_{\M A}^{-1}\hat{\M
    W}^t\big)$ then the term $\hat{\M W}$ might lead to 
  a convergence rate of $\tilde{\M M}_0$ to ${\M M}_0$ that is slower
  than that given in Eq.~\eqref{T3:dense}. For example in the dense regime Eq.~\eqref{T2:oneorder} and Eq.~\eqref{T2main3} show that the entries of $\hat{\M W} - \M W$
  are of larger magnitude than the entries of $\hat{\M W}^{t} - \M
  W^{t}$ for $t \geq 2$.}
\end{remark}


Before discussing the convergence rate of   $\tilde{\M M}_0$ to ${\M
M}_0$ we first find a value of $t_U$ such that, for large values of $n$, $\tilde{\M M}_0$ is well defined with high
probability. 
We note that the entries of $\{\M W^t\}_{t\geq 1}$ are
uniformly of order $\Theta(n^{-1})$. Then, under the dense regime, $t = 2$ is sufficient to guarantee that all the off-diagonal entries of $\hat{\mathbf{W}}^{t}$ are uniformly of order 
$\Omega(n^{-1} - n^{-3/2} \log^{1/2}{n}) = \Omega(n^{-1})$ with high probability (c.f. Eq.~\eqref{T2main30}) while $t = 3$ is sufficient to guarantee that all entries of $\hat{\mathbf{W}}^{t}$ are of order $\Omega(n^{-1})$ with high probability (c.f. Eq.~\eqref{T2main3}). If we are under the sparse regime with $\beta < 1/2$ then these same values of $t \geq 2$ are still sufficient to guarantee that the entries of $\hat{\mathbf{W}}^{t}$ are of order $\Omega(n^{-1})$ (c.f. Eq.~\eqref{T5:6:main} and Eq.~\eqref{T10:extend}). Finally, if we are under the sparse regime with $\beta \geq 1/2$ then choosing $t \geq 4$ with $\tfrac{t-3}{t-1} > \beta$ is sufficient to guarantee that the entries $\hat{\M
  W}^{t}$ are uniformly of order $\Omega(n^{-1} - n^{-3/2}
\rho_n^{-1/2} \log^{1/2}{n}) = \Omega(n^{-1})$ with high probability. Now recall that the
matrix $\tilde{\M M}_0$ is of the form
$$\log\Bigg\{\frac{2|\M A|}{\kappa \gamma} \sum_{t = t_L}^{t_U}(L - t)\M D_{\M A}^{-1}\hat{\M W}^t\Bigg\}$$
We therefore have, for $t_U \geq 3$ in the dense regime, $t_U \geq 2$ in the not too sparse regime of $\beta < 1/2$, or for $\tfrac{t_U - 3}{t_U - 1} > \beta$ in general,  that the entries of
the inner sum are bounded away from $0$ with high
probability. For the dense regime, the condition can be further relaxed to $t_U \geq 2$, as a dense graph has a diameter of $2$ and thus all entries of $\hat{\M W}^2$ are uniformly larger than $0$ with high probability; see Theorem 10.10 in \cite{bollobas2001random}.  Therefore, with high probability, the elementwise logarithm is well-defined for all
entries of $\tilde{\M M}_0$. 
Given the existence of $\tilde{\M M}_0$, the following result shows the convergence rate of $\tilde{\M
M}_0$ to ${\M M}_0$.
\begin{theorem}\label{T3}
Suppose $\mathcal{G}\sim \mathrm{SBM}(\mathbf{B},
  \bm{\Theta})$ satisfies Assumption~\ref{am:s}, and $t_U\geq t_L \geq 2$ where $t_L$ is chosen as described above. Then
$\tilde{\M M}_0$ is well-defined with high probability. Denote 
$$\Delta = \max\{\|\tilde{\mathbf{M}}_0 - \mathbf{M}_0\|_{\F}, \|\tilde{\mathbf{M}}_0 - \mathbf{M}_0 \|_{\infty}\}.$$
We then have the following bounds.
\begin{enumerate}
    \item \textbf{(Dense regime)} Let $\rho_n\asymp 1$. 
    Then for $t_L \geq 2$ we have \bee\label{T3:dense}
    \Delta = \mathcal{O}_{\mathbb{P}}(n^{1/2}\log^{1/2}{n}).
\ee
    \item \textbf{(Sparse regime)} Let $\rho_n\rightarrow 0$ with $\rho_n \succsim n^{-\beta}$ for some $\beta \in [0,1)$. Then for $t_L$ satisfying $\tfrac{t_L - 3}{t_L - 1} > \beta$ we have
    \bee\label{T3:e}
     \Delta = \op\big(n^{1/2}\rho_n^{-1/2} \log^{1/2}{n} \big).
    \ee
    In addition if $0 \leq \beta < 1/2$ then for $t_L \geq 2$ we have
    \bee\Delta = \op\big(n^{1/2}\rho_n^{-1} \log^{1/2}{n}\big).
 \ee
\end{enumerate}
In both regimes we have $\|\M M_0\|_{\F} = \Theta(n)$ and 
$\|\M M_0\|_{\infty} = \Theta(n)$.
\end{theorem}
Theorem~\ref{T3} indicates that as $\beta$ increases
(equivalently, as $\rho_n$ decreases) so that the graph $\mathcal{G}$
becomes sparser, we could (1) still guarantee the existence of $\tilde{\M M}_0$ when
 $t_U$ is sufficiently large, and (2) control the convergence rate of
$\|\tilde{\M M}_0 - {\M M}_0\|_{\F}$ relative to $\|\M M_0\|_{\F}$ by increasing $t_L$.
\subsection{Subspace perturbations and exact recovery}
Theorem~\ref{T3} implies that $\tilde{\mathbf{M}}_0$ is close to
$\mathbf{M}_0$ under both Frobenius and infinity norms, i.e., 
$\|\tilde{\mathbf{M}}_0 - \mathbf{M}_0\|_{\star}/\|\mathbf{M}_0\|_{\star} =
o_{\mathbb{P}}(1)$ for $\star \in \{F, \infty\}$ and sufficiently large $n$. 
Now, by Eq.~\eqref{theta}, $\mathbf{M}_0$ has a $K \times K$
block structure and hence $\mathrm{rk}(\mathbf{M}_0) \leq K$. 
Furthermore the
eigenvectors of $\mathbf{M}_0$ associated with its non-zero eigenvalues is
sufficient for recovering the community assignments induced by
$\M Z$. The following result, which follows from bounds for $\|\tilde{\M M}_0 - \mathbf{M}_0\|_{\infty}$ given in Theorem~\ref{T3} together with perturbations bounds for invariant subspaces using $2\rightarrow \infty$ norm \cite{cape2019two}, shows that  
the embedding $\hat{\bds{\mathcal{F}}}$ given by the leading
eigenvectors of $\tilde{\mathbf{M}}_0$ is {\em uniformly} close to
that of the leading eigenvectors of $\mathbf{M}_0$. Therefore $K$-means or $K$-medians clustering on the rows of $\hat{\bds{\mathcal{F}}}$ will recover the community membership for \textit{every} nodes, i.e, attain strong or exact recovery of $\M Z$.
\begin{theorem}\label{c1}
Under the condition of Theorem~\ref{T3}, let $\hat{\mathbf{U}}
\hat{\bm{\Sigma}} \hat{\mathbf{U}}^{\top}$ and $\mathbf{U} \bm{\Sigma}
\mathbf{U}^{\top}$ be the eigen-decomposition of
$\tilde{\mathbf{M}}_0$ and $\mathbf{M}_0$, respectively. Let $d =
\mathrm{rk}(\mathbf{M}_0)$ and note that $\mathbf{U}$ is a $n \times
d$ matrix. Let $\hat{\bds{\mathcal{F}}} = \hat{\mathbf{U}}_{d}$ be the matrix formed by the columns of
$\hat{\mathbf{U}}$ corresponding to the $d$ largest-in-magnitude eigenvalues of
$\tilde{\M M}_0$. For a $n \times d$ matrix $\mathbf{Z}$ with rows $Z_1, Z_2, \dots, Z_n$ let $\|\mathbf{Z}\|_{2 \to \infty}$ denote the maximum $\ell_2$ norms of the $\{Z_i\}$, i.e.,
$$\|\mathbf{Z}\|_{2 \to \infty} = \max_{i} \|Z_i\|_2.$$
We then have the following results.
\begin{itemize}
\item[(i)]\textbf{(Dense regime)} Let $\rho_n\asymp 1$. Then for $t_L \geq 2$ we have
\bee\label{T3:dense_res1}
\min_{{\M T} \in \mathbb{O}_d}\|\hat{\bds{\mathcal{F}}} \mathbf{T} -
\mathbf{U}\|_{\F} &= \mathcal{O}_{\mathbb{P}}\Big(\frac{\log^{1/2}{n}}{n^{1/2}}\Big)
\\
\min_{{\M T} \in \mathbb{O}_d}\|\hat{\bds{\mathcal{F}}} \mathbf{T} -\mathbf{U}\|_{\twoinf} &= 
\mathcal{O}_{\mathbb{P}}\Big(\frac{\log^{1/2}{n}}{n}\Big).
\ee
\item[(ii)] \textbf{(Sparse regime)} Let $\rho_n \rightarrow 0$ with $\rho_n \succsim n^{-\beta}$ for some $\beta \in [0,1/2)$. If $t_L \geq 2$, we have
\bee
\min_{\M T \in \mathbb{O}_{d}} \|\hat{\bds{\mathcal{F}}} \mathbf{T} - \mathbf{U}\|_{\F} &=
\op\Big(\frac{\log^{1/2}{n}}{n^{1/2} \rho_n}\Big)
 \\
\min_{\M T \in \mathbb{O}_{d}} \|\hat{\bds{\mathcal{F}}}  \mathbf{T} - \mathbf{U}\|_{\twoinf} &=
\op\Big(\frac{\log^{1/2}{n}}{n \rho_n}\Big) 
\ee
\item[(iii)] \textbf{(Sparse regime)} Let $\rho_n \rightarrow 0$ with $\rho_n \succsim n^{-\beta}$ for some $\beta \in [0,1)$. If $\frac{t_L - 3}{t_L - 1} > \beta$, we have,
\bee\label{rateF:spare:general}
\min_{\M T \in \mathbb{O}_{d}} \|\hat{\bds{\mathcal{F}}} \mathbf{T} - \mathbf{U}\|_{\F} &=
\op\Big(\frac{\log^{1/2}{n}}{(n \rho_n)^{1/2}}\Big) \\ 
\min_{\M T \in \mathbb{O}_{d}} \|\hat{\bds{\mathcal{F}}} \mathbf{T} - \mathbf{U}\|_{\twoinf} &=
\op\Big(\frac{\log^{1/2}{n}}{n \rho_n^{1/2}}\Big) .
\ee
\end{itemize}
Given the above convergence rates, clustering the rows of $\hat{\bds{\mathcal{{F}}}}$ using either $K$-means or $K$-medians will, with high probablity, recover the memberships of {\em every} nodes in $\mathcal{G}$.
\end{theorem}
\begin{remark}\label{rk:improved}
{\upshape \color{black} Settings (ii) and (iii) in Theorem \ref{c1} both consider the sparse regime but setting (ii) focuses on the case where $\rho_n  = \omega(n^{-1/2})$ and exact recovery is achieved whenever $t_L \geq 2$ while setting (iii) considers the more general scenario of $\rho_n = \omega(n^{-\beta})$ for any fixed but arbitrary $\beta < 1$. We note that for ease of exposition we had impose $\tfrac{t_L - 3}{t_L - 1} > \beta$ for setting (iii) but this condition can be relaxed to 
\bee\label{requirement}
\frac{t_L - 2}{t_L} > \beta,
\ee
under which we still have $\tilde{\M M}_0$ is well-defined with high probability, and have a more complicated bound of
$$\min_{\M T \in \mathbb{O}_{d}} \|\hat{\bds{\mathcal{F}}} \mathbf{T} - \mathbf{U}\|_{\twoinf} \precsim \op\Bigg\{\frac{\log^{1/2}{n}}{ n^{3/2}\rho_n^{1/2}} + (n\rho_n)^{-t_L/2}\Bigg\}$$
(see \eqref{eq:hat_F_alternative}). The above bound is still sufficient to guarantee that running $K$-means or $K$-medians on the rows of $\hat{\bds{\mathcal{F}}}$ will recover the memberships of every nodes in $\mathcal{G}$ with high probability; see Section \ref{sec:rk6:pf} in the Supplementary File for a rigorous proof.
\par A recent preprint \cite{barot2021community} which appeared on arXiv after the first version of our paper also studied community recovery using SVD-based DeepWalk/node2vec and they have a similar requirement for $t_L$ as Eq.~\eqref{requirement}; see (3.1) in \cite{barot2021community}. For comparison we note that \cite{barot2021community} only derived the convergence rate of $\hat{\bds{\mathcal{F}}}$ under Frobenius norm, and thereby prove a {\em weak} recovery result which allows at most $o(n^{1/2})$ nodes to be misclassified. In contrast the max-norm concentration of $\hat{\M W}^t$ in Theorem \ref{T2} helps us derive a $2\to\infty$ norm convergence for $\hat{\bds{\mathcal{F}}}$, based on which we achieved the much stronger exact recovery (i.e., there are no mis-classified nodes). Finally we conjecture that Eq.~\eqref{requirement} for $t_L$ is sufficient but not necessary. Our simulation results in Section~\ref{sec:simu} agree with this conjecture and we leave its verification for future work.} 
\end{remark}
\begin{remark}
{\upshape The exact recovery results in Theorem \ref{c1} can also be extended to the case of degree-corrected SBM graphs \cite{karrer2011stochastic,zhao2012consistency,gao2018community}. Recall that the edge probabilities for a DCSBM is $\M P = \M \Theta \M Z\M B \M Z^\T \M \Theta$ where $\M \Theta = \text{diag}(\theta_1,\dots,\theta_n)$ are the degree-correction parameters. DCSBM allows heterogeneous edge probabilities within each community and thus yields a more flexible model in comparison with SBM. Section \ref{sec:dcsbmsr} and \ref{sec:sparsedcebm} in the Supplementary File demonstrates how to extend the technical derivations for Theorem \ref{c1} to the DCSBM case provided that the $\{\theta_i\}$ are sufficiently homogeneous, i.e., that $\max_i\theta_i/\min_i \theta_i = \mathcal{O}(1)$.
}
\end{remark}
\section{Simulation}\label{sec:simu}
We now present simulations experiments for the matrix factorization perspective of node2vec/DeepWalk. These experiments complement our theoretical
results in Section~\ref{sec:3} and illustrate the interplay between
the sparsity of the graphs, the choice of window sizes, and their
combined effects on the nodes embedding. 
\subsection{Error bounds for \texorpdfstring{${\|\tilde{\M M}_0 - \M M_0\|_{\F}}$}{TEXT}}
\label{sec:error_bound_simulations}
We first compare the 
large-sample empirical behavior of $\|\tilde{\M M}_0 - \M M_0\|_{\F}$ against the theoretical bounds given in  Theorem~\ref{T3}. We shall simulate undirected graphs
generated from a $2$-blocks SBM with parameters
\begin{equation}\label{Bdef}
  \M B(\rho_n) := \begin{pmatrix} 0.8\rho_n  & 0.3\rho_n
\\ 0.3\rho_n  & 0.8\rho_n \end{pmatrix}, \quad \bds{\pi} = (0.4, 0.6),
\end{equation}
and sparsity $\rho_n \in
\{1, 3n^{-1/3}, 3n^{-1/2}, 3n^{-2/3}\}$.
While this two blocks setting is quite simple it
nevertheless displays the effect of the sparsity
$\rho_n$ and the window size $(t_L, t_U)$ on the upper
bound for ${\|\tilde{\M M}_0 - \M M_0\|_{\F}}$.
\begin{figure*}[htbp]
\begin{subfigure}{.49\columnwidth}
\includegraphics[width=\columnwidth]{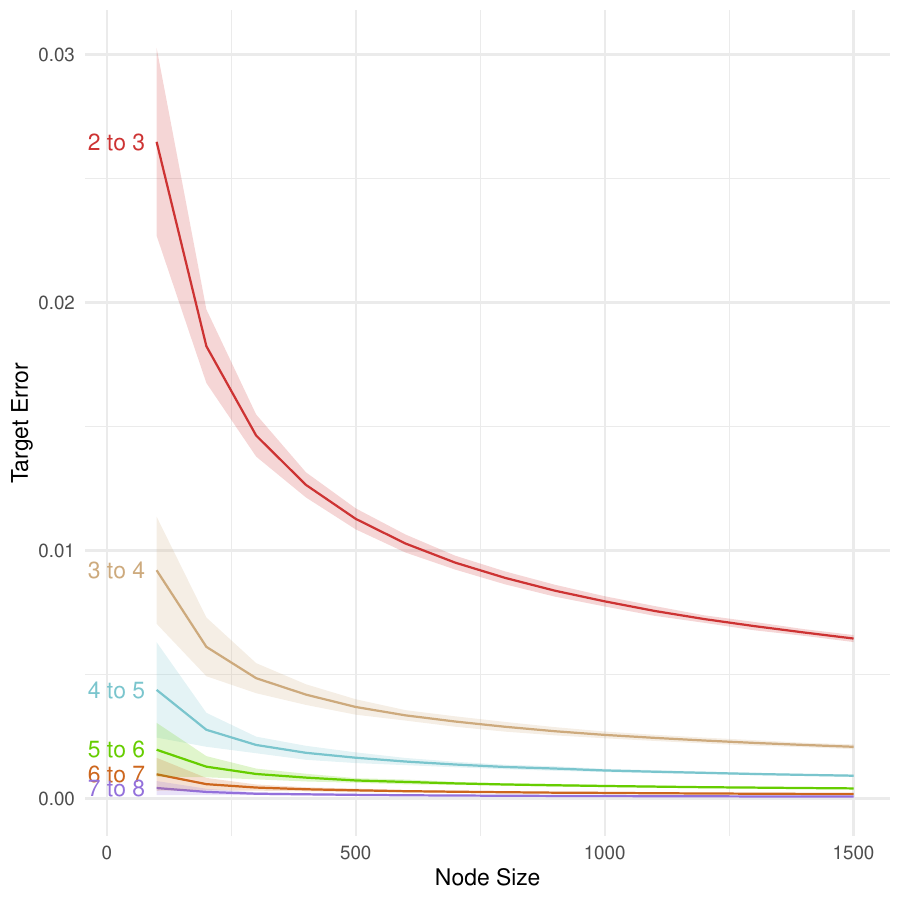}%
\caption{$\rho_n = 1$}
\end{subfigure}
\begin{subfigure}{.49\columnwidth}
\includegraphics[width=\columnwidth]{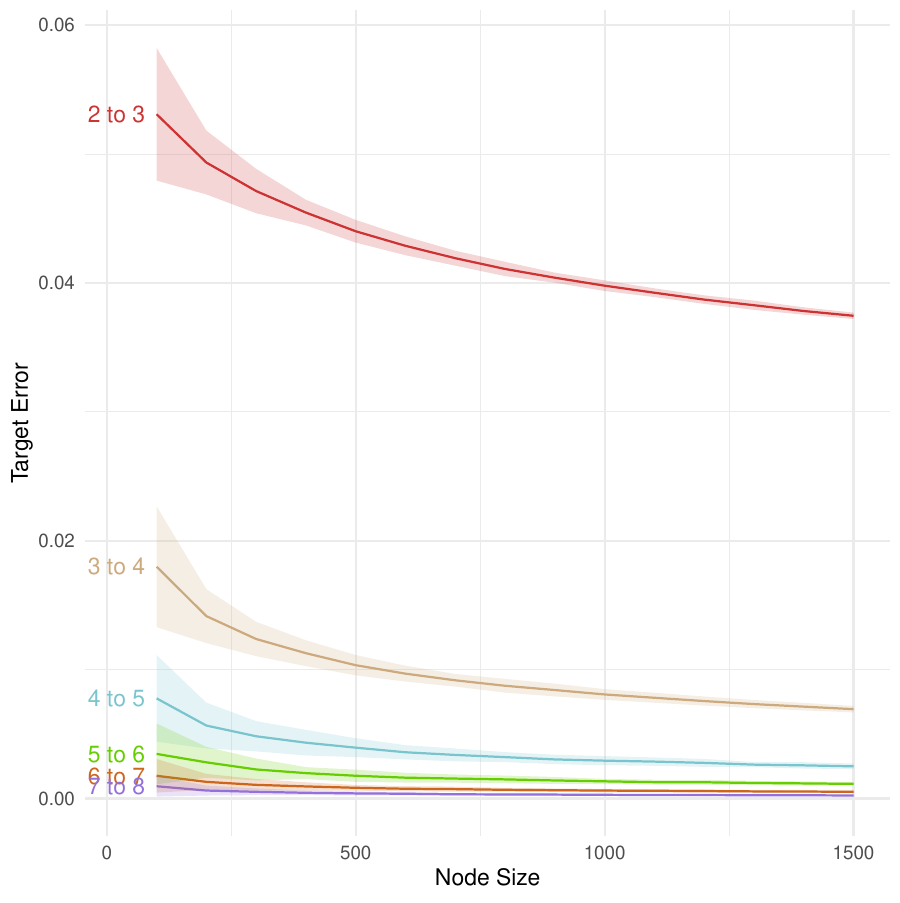}%
\caption{$\rho_n = 3n^{-1/3}$}
\end{subfigure}
\begin{subfigure}{.49\columnwidth}
\includegraphics[width=\columnwidth]{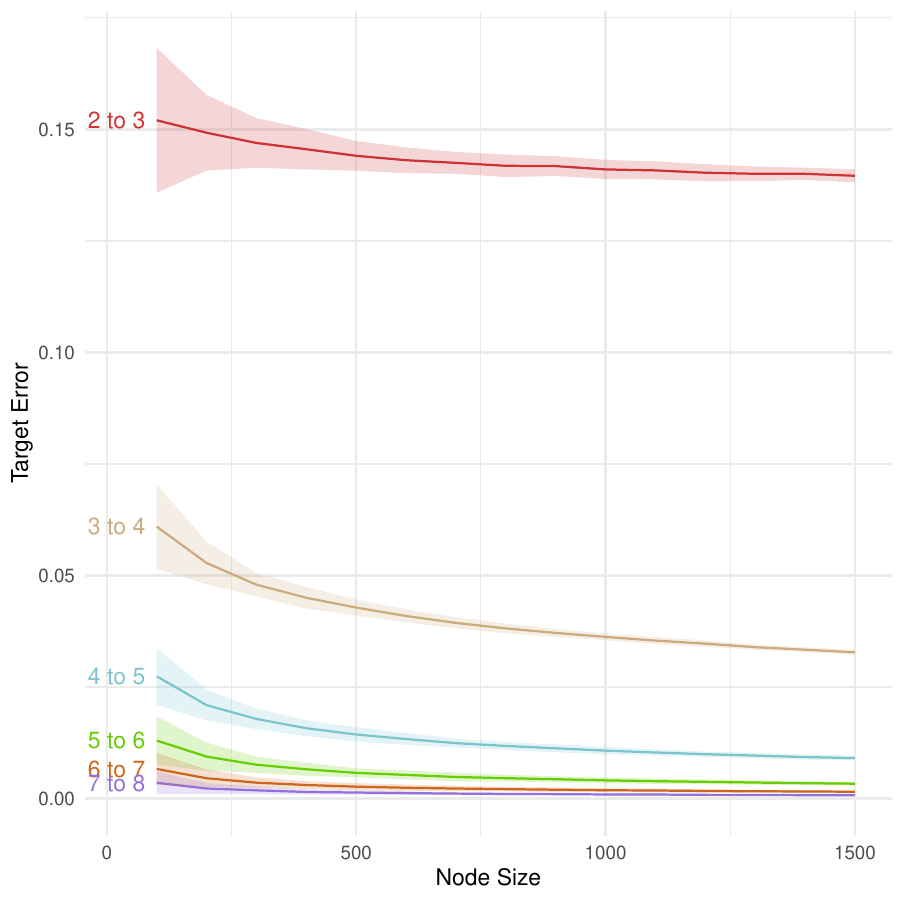}%
\caption{$\rho_n = 3n^{-1/2}$}
\end{subfigure}
\begin{subfigure}{.49\columnwidth}
\includegraphics[width=\columnwidth]{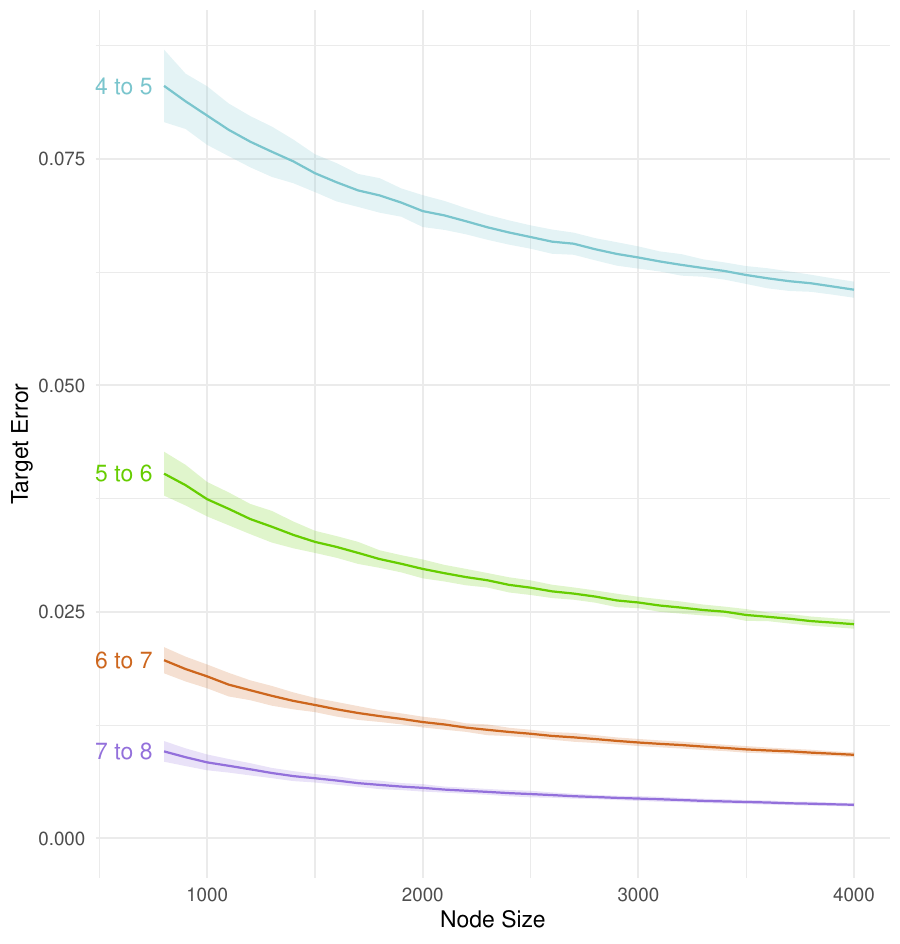}%
\caption{$\rho_n = 15n^{-2/3}$}
\end{subfigure}
\caption{Sample means and $95\%$ empirical confidence
    intervals for $\varepsilon_1(\tilde{\mathbf{M}}_0)$
    based on $100$ Monte Carlo replicates for different values of $n, \rho_n$ and 
    $(t_L, t_U)$ with $t_U - t_L = 1$.}
\label{f:rate:1}
\end{figure*}
\begin{figure*}[htbp]
\begin{subfigure}{.49\columnwidth}
\includegraphics[width=\columnwidth]{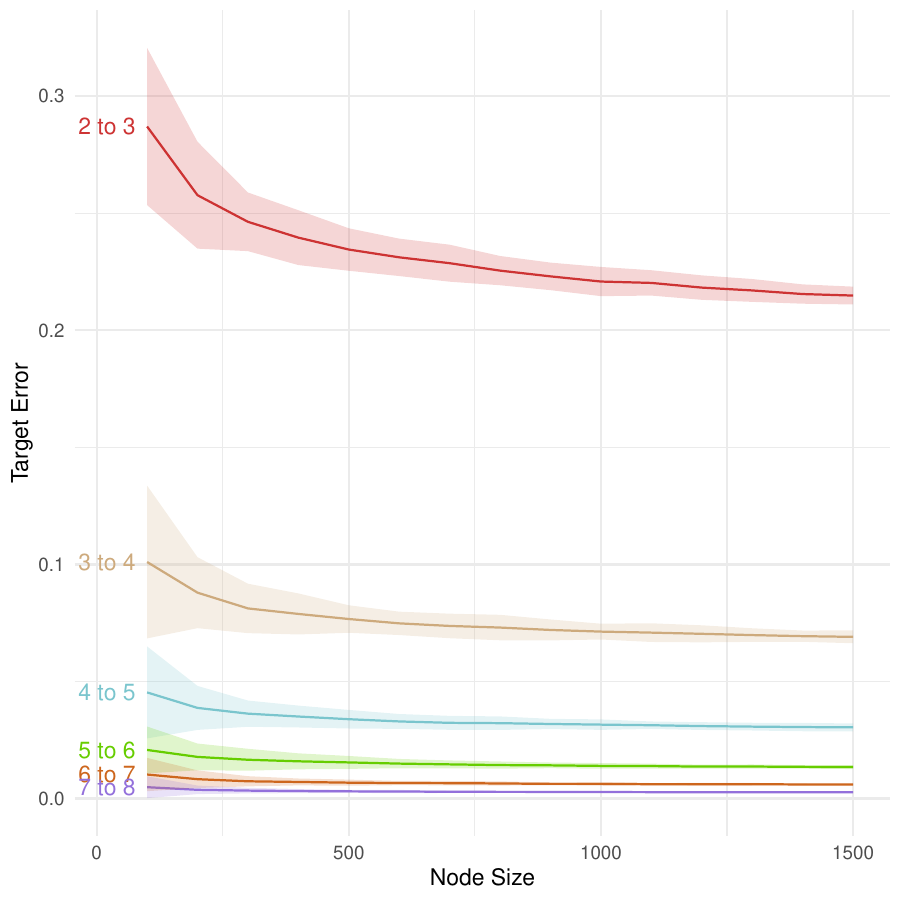}%
\caption{$\rho_n = 1$}
\end{subfigure}
\begin{subfigure}{.49\columnwidth}
\includegraphics[width=\columnwidth]{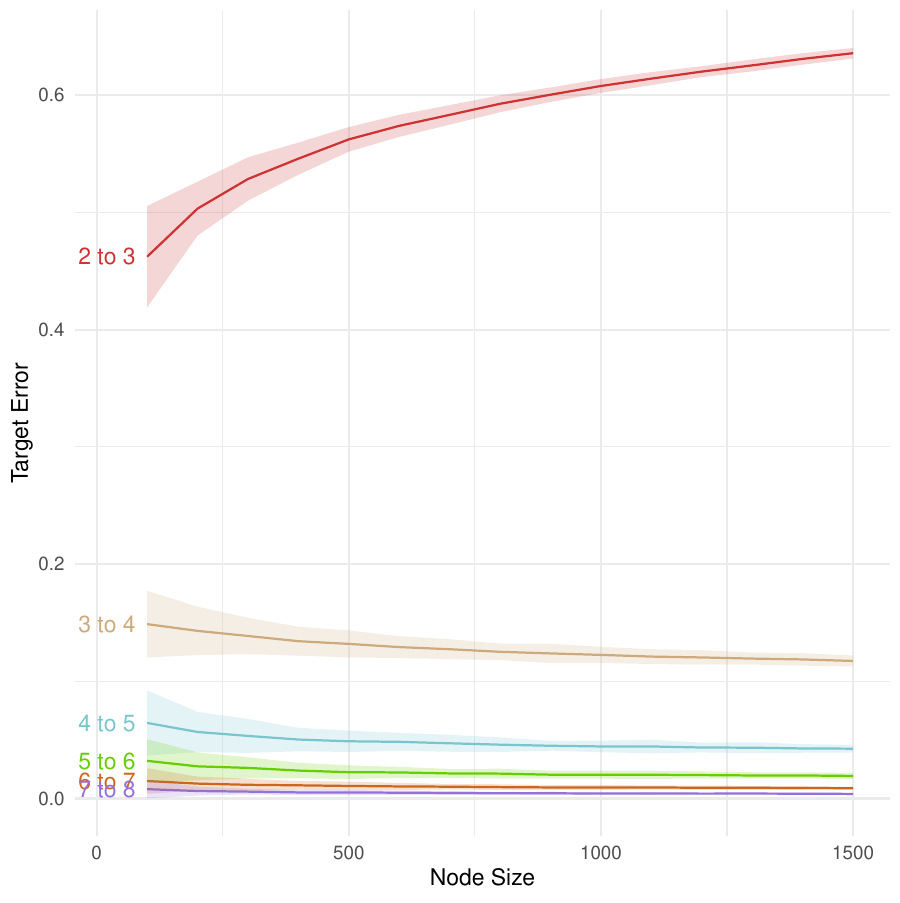}%
\caption{$\rho_n = 3n^{-1/3}$}
\end{subfigure}
\begin{subfigure}{.49\columnwidth}
\includegraphics[width=\columnwidth]{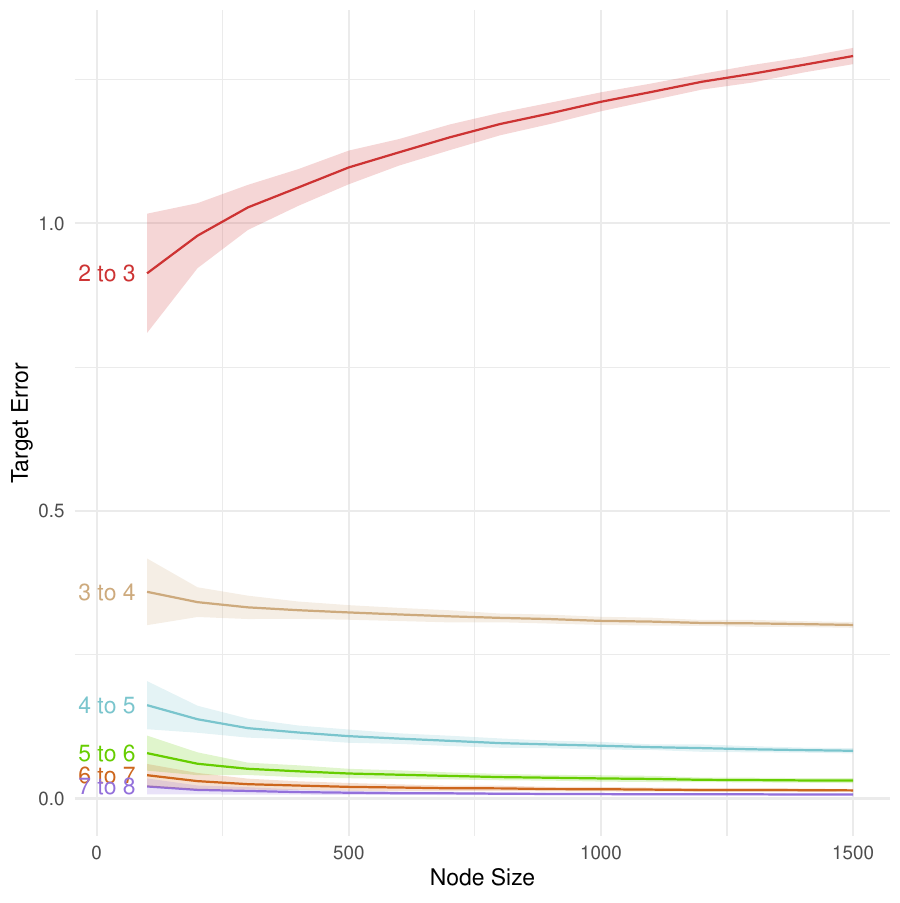}%
\caption{$\rho_n = 3n^{-1/2}$}
\end{subfigure}
\begin{subfigure}{.49\columnwidth}
\includegraphics[width=\columnwidth]{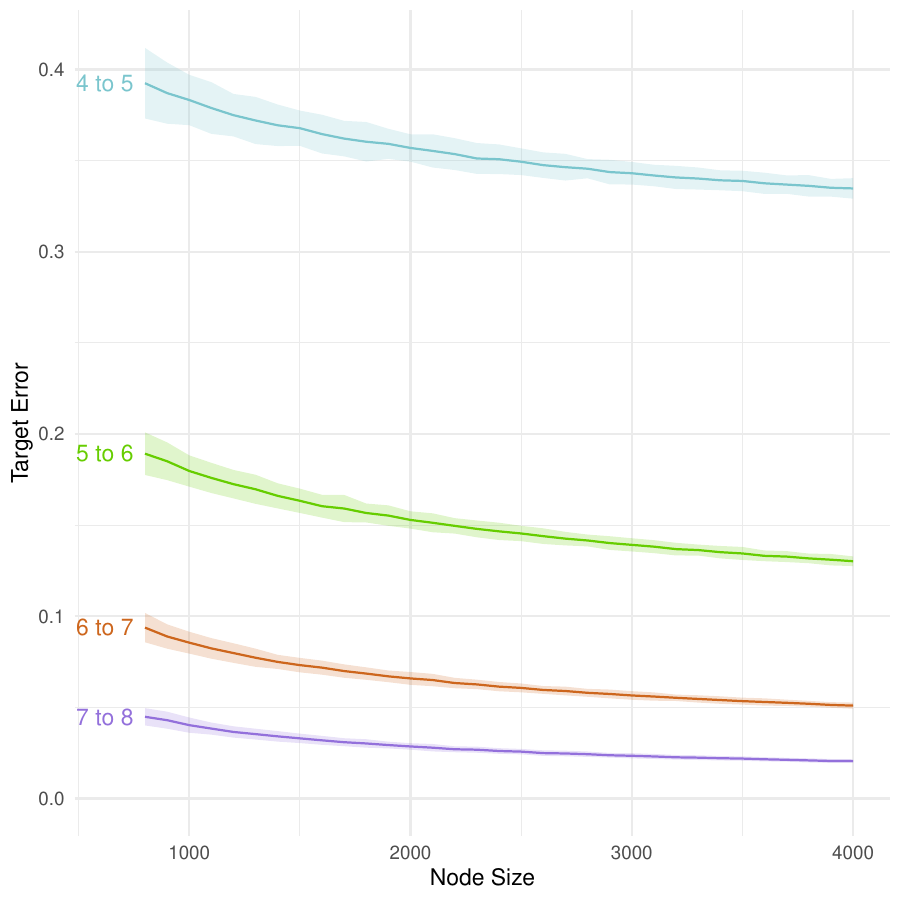}%
\caption{$\rho_n = 15n^{-2/3}$}
\end{subfigure}
\caption{Sample means and $95\%$ empirical confidence
    intervals for $\varepsilon_2(\tilde{\mathbf{M}}_0)$ based on $100$
    Monte Carlo replicates for different values of  of $n,\rho_n$, and
    $(t_L, t_U)$ with $t_U - t_L = 1$.}
\label{f:rate:3}
\end{figure*}

For each value of $n$ and sparsity  $\rho_n$ we run $100$
independent replications where, in each replicate, we generate $\mathcal{G} \sim \mathrm{SBM}(\M B(\rho_n),\M \Theta_n)$
and calculate $\tilde{\M M}_0$ for different choices of $(t_L,
t_U)$. In particular we consider two types of window size, namely $t_U = t_L + 1$ and $t_U = t_L + 3$. While $t_U = t_L + 1$ is not commonly used in practice, for simulation purpose this
choice clearly show the effects of the random walks' length $t$ on 
the error $\|\tilde{\M M}_0 - \M M_0\|_\F$. In
contrast the choice $t_U = t_L + 3$ is more realistic but also partially obfuscate the effect of
$t$ on $\|\tilde{\M M}_0 - \M M_0\|_\F$.  Recall that, from
the discussion prior to Theorem~\ref{T3}, sparser values of $\rho_n$ requires larger values of $t_U$ to guarantee that
$\tilde{\M M}_0$ is well-defined. The choices for $\big(\rho_n, n, (t_L, t_U)\big)$ in
the simulations are summarized below.
\begin{itemize}
\item If $\rho_n \geq 3n^{-1/2}$ then $n \in \{100,
200, 300, \dots1500\}$. We chose $2 \leq t_L \leq 7$ when $t_U = t_L + 1$ and chose $2 \leq t_L \leq 5$ when $t_U = t_L + 3$.
\item If $\rho_n = 3n^{-2/3}$ then $n \in \{800, 900, 
\dots, 4000\}$. We chose $4 \leq t_L \leq 7$ when $t_U = t_L + 1$ and
$3 \leq t_L \leq 5$ when
$t_U = t_L + 3$.
\end{itemize}
We calculate two relative error criteria for $\tilde{\mathbf{M}}_0$, namely
\begin{equation*}
  \varepsilon_1(\tilde{\mathbf{M}}_0) = \frac{\|\tilde{\M M}_0 - \M M_0\|_{\F}}{\|\M M_0\|_\F} \,\, \text{and} \,\,
   \varepsilon_2(\tilde{\mathbf{M}}_0) = \frac{\|\tilde{\M M}_0 - \M
    M_0\|_{\F}}{n^{1/2}\rho_n^{-1/2} \log^{1/2}{n}}.
\end{equation*}
We expect that, as $n$ increases, the first criteria converges to $0$ 
while the second criteria remains bounded.

\noindent{\bf Relative Error 1:} We first confirm the convergence of $\varepsilon_1(\tilde{\mathbf{M}}_0)$ to $0$. Figures~\ref{f:rate:1} and~\ref{f:rate:2}
shows the means and $95\%$ confidence intervals for $\varepsilon_1(\tilde{\mathbf{M}}_0)$
based on $100$
Monte Carlo replicates for different values of $\rho_n, (t_L, t_U)$.
These figures indicate the following
general patterns as predicted by the theoretical results in Theorem~\ref{T3}.
\begin{itemize}
\item The error $\varepsilon_1(\tilde{\mathbf{M}}_0)$ is smallest in the dense case and deteriorates as the sparsity factor $\rho_n$ decreases.
\item The error also depends on $(t_L, t_U)$ with larger values of $t_U - t_L$ leading to smaller $\varepsilon_1(\tilde{\mathbf{M}}_0)$ 
\item 
  If the window size
  is too small, e.g., $(t_L, t_U) = (2,3)$ or $(t_L, t_U) = (2,5)$, then
  $\tilde{\M M}_0$ is often times not well-defined.
\end{itemize}
{\bf Relative Error 2:} Figures~\ref{f:rate:1} and Figure~\ref{f:rate:2} (see the Supplementary File)
corroborate our theoretical results in Section
\ref{sec:T}. Nevertheless there are two additional questions we
should consider. The first is whether or not the bound
$\|\tilde{\mathbf{M}}_0 - \M M_0\|_\F =
\op(n^{1/2}\rho_n^{-1/2}\log^{1/2}{n})$ in Theorem~\ref{T3} is
tight and, if it is tight, the second  is whether or not the
condition $\tfrac{t_L -
3}{t_L - 1} > \beta$ is necessary to achieve this rate.
Analogous to the previous two figures, Figures~\ref{f:rate:3} and
\ref{f:rate:4} show the means and $95\%$ empirical confidence intervals for
the relative error $\varepsilon_2(\tilde{\mathbf{M}}_0)$ over $100$
Monte Carlo replicates for different values of $\rho_n$ and $(t_L, t_U)$. 
From these simulations we can answers the above questions as follows. 
\begin{itemize}
\item If $\rho_n \succsim n^{-\beta}$ is such that $\beta \leq \frac{t_L
  -3}{t_L-1}$ then $\varepsilon_2(\tilde{\mathbf{M}}_0)$ appears to converge to a constant as $n$ increases. There
  is thus evidence that the rate
  $n^{1/2}\rho_n^{-1/2} \log^{1/2}{n}$  for $\|\tilde{\mathbf{M}}_0 - \mathbf{M}_0\|_{\F}$ is optimal. 
  Nevertheless if $t_L$ is large relative to $\rho_n$, e.g.,
 $\rho_n \in \{3n^{-1/3}, 3n^{-1/2}\}$ and $t_L \geq 6$, then
  $\varepsilon_2(\tilde{\mathbf{M}}_0)$ appears to converges to $0$ which suggests that for a fixed $\beta$ the error rate for $\|\tilde{\mathbf{M}}_0 - \mathbf{M}_0\|_{\F}$ can be smaller than $n^{1/2}\rho_n^{-1/2} \log^{1/2}{n}$; this might be due to the convergence of $\hat{\M W}^{t}$ and $\M W^{t}$ towards the stationary distributions as $t$ increases. 
\item For cases such as $(t_L,t_U) \in \{(3,4), (3,6)\}$ and $\rho_n = 3n^{-1/2}$
  or $(t_L,t_U) \in \{(4,5), (3,6)\}$ and $\rho_n = 15n^{-2/3}$, 
  the $t_L$'s do not satisfy $\tfrac{t_L
  -3}{t_L-1} > \beta$. Nevertheless $\varepsilon_2(\tilde{\mathbf{M}}_0)$ still appears to converge to a constant as $n$ increases. This suggests that $\tfrac{t_L
  -3}{t_L-1} > \beta$ is sufficient but possibly not necessary for
  the bound in Eq.~\eqref{T3:e} to hold. On the other hand,
  for fixed $n$ and $\rho_n$, the error $\|\tilde{\mathbf{M}}_0 - \M M_0\|_{\F}$ generally decreases as $t_U - t_L$ increases. 
  \item Finally if $(t_L, t_U) \in \{(2,3), (2,5)\}$ and $\rho_n
\in \{3n^{-1/3}, 3n^{-1/2}\}$ then $\varepsilon_2(\tilde{\mathbf{M}}_0)$ increases with $n$. This supports the claim in Theorem~\ref{T3} of a phase transition for the error rate of $\|\tilde{\M M}_0 - \M M_0\|_{\F}$ as $t_L$ increases.
\end{itemize}
In summary Figure~\ref{f:rate:1} through Figure~\ref{f:rate:4} supports the conclusion of Theorem~\ref{T3}. In particular the error
rate in Theorem~\ref{T3} is sharp and the condition $\tfrac{t_L - 3}{t_L - 1}>\beta$ 
is sufficient but perhaps not necessary. 
\begin{table*}[htbp]
\centering\sffamily
\renewcommand{\theadfont}{\normalsize\bfseries}
\setcellgapes{0.8ex}\makegapedcells
\begin{tabular}{|*{6}{c|}}
\hline
\multirowthead{1}{$n$}  & \multicolumn{2}{c|}{SVD-based node2vec} & 
\multicolumn{2}{c|}
{Original node2vec}  \\
\cline{2-5}
 &$t_U = 5$& $t_U = 8$ &$t_U = 5$&$t_U = 8$\\
\hline
600&1.00&1.00&1.00&1.00
\\
900&1.00&1.00&1.00&1.00
\\
1500&1.00&1.00&1.00&1.00
\\
\hline
\end{tabular}
\quad\quad\quad\quad\quad\quad
\begin{tabular}{|*{6}{c|}}
\hline
\multirowthead{1}{$n$}  & \multicolumn{2}{c|}{SVD-based node2vec} & 
\multicolumn{2}{c|}
{original node2vec}  \\
\cline{2-5}
 &$t_U = 5$& $t_U = 8$ &$t_U = 5$&$t_U = 8$\\
\hline
600&0.30&0.32&0.01&0.05
\\
900&0.57&0.55&0.07&0.11
\\
1500&0.86&0.86&0.57&0.28
\\
\hline
\end{tabular}
\caption{Proportions of times that SGD-based and SVD-based node2vec variants perfectly recover all nodes memberships. The graphs are generated  from $\M B(\rho_n)$ with sparsity $\rho_n = 3n^{-1/3}$ (left table) and $\rho_n = 3n^{-1/2}$ (right table).
}\label{table1}
\end{table*}

\begin{table*}[htbp]
\centering\sffamily
\renewcommand{\theadfont}{\normalsize\bfseries}
\setcellgapes{0.8ex}\makegapedcells
\begin{tabular}{|*{6}{c|}}
\hline
\multirowthead{1}{$n$}  & \multicolumn{2}{c|}{SVD-based node2vec} & 
\multicolumn{2}{c|}
{original Node2vec}  \\
\cline{2-5}
 &$t_U = 5$& $t_U = 8$ &$t_U = 5$&$t_U = 8$\\
\hline
600&1.00&1.00 &1.00& 0.45
\\
900&1.00&1.00&1.00& 0.95
\\
1500&1.00&1.00&1.00&1.00
\\
\hline
\end{tabular}
\quad\quad\quad\quad\quad\quad
\begin{tabular}{|*{6}{c|}}
\hline
\multirowthead{1}{$n$}  & \multicolumn{2}{c|}{SVD-based node2vec} & 
\multicolumn{2}{c|}
{original Node2vec}  \\
\cline{2-5}
 &$t_U = 5$& $t_U = 8$ &$t_U = 5$&$t_U = 8$\\
\hline
600&0.25&0.40&0.00&0.00
\\
900&0.58&0.61&0.00&0.00
\\
1500&0.82&0.83&0.13&0.37
\\
\hline
\end{tabular}
\caption{Proportions of times that SGD-based and SVD-based node2vec variants perfectly recover all nodes memberships. The graphs are generated from $\M B^{\natural}(\rho_n)$ with sparsity $\rho_n = 3n^{-1/3}$ (left table) and  $\rho_n = 3n^{-1/2}$ (right table).
}\label{table2}
\end{table*}
\subsection{Exact recovery of community structure}\label{sec:erp}
Theorem~\ref{c1} together with Remark \ref{rk:improved} showed that
$\hat{\bds {\mathcal{F}}}$ combined with $K$-means/medians can correctly recover the memberships of all nodes in a SBM with high probability. We demonstrate this result for two-blocks SBMs with block probabilities being either $\M B(\rho_n)$ as given in Eq.~\eqref{Bdef} or 
\begin{equation}\nonumber
  \M B^\natural(\rho_n) := \begin{pmatrix} 0.3\rho_n  & 0.8\rho_n
\\ 0.8\rho_n  & 0.3\rho_n \end{pmatrix}.
\end{equation}
Note that $\M B(\rho_n)$ and $\M B^\natural(\rho_n)$ corresponds to 
an assortative and a dis-assortative structure, respectively. Given specific setting of $\M B, n, \rho_n$, we randomly sample $100$ graphs where each vertex is randomly assigned to one of the two blocks with equal probability and evaluate the membership recovery performances of the original node2vec \cite{grover2016node2vec} (based on SGD) and node2vec using matrix factorization (as described \cite{qiu2018network,sussman_vec,barot2021community} and this paper) followed by clustering using $K$-means. We set the window sizes to $t_U \in \{5,8\}$ and choose $\kappa = 5$ and $L = 200$. For the original node2vec we also set $t_L = 1$ as the default and $r_1 =\cdots = r_n = 200$, while for the SVD-based node2vec we set $t_L = t_U - 3$. We report in Tables \ref{table1} and \ref{table2} the proportions of times for the $100$ simulated graphs that these two variants of the node2vec algorithm correctly recover the memberships of all nodes. 
\par
The numerical results show that as $n$ increases, both the original and SVD-based node2vec are more likely to perfectly recover memberships of all nodes in the graph, under all different settings of $\rho_n,\M B,t_U$. Furthermore the accuracy when $\rho_n = 3n^{-1/3}$ is considerably higher than that for $\rho_n = 3n^{-1/2}$. 
This is consistent with the results in Theorem~\ref{c1} as a smaller magnitude for $\rho_n$ results in a slower convergence rate for $\hat{\bds{\mathcal{F}}}$ under both the Frobenius and $2\rightarrow  \infty$ norms. In addition the exact recovery performance of SVD-based node2vec when $\rho_n\asymp n^{-1/2}$ and $(t_L,t_U ) = (2,5)$ suggests that the $t_L$ threshold for Theorem \ref{c1} in  Eq.~\eqref{requirement} is possibly not sharp as $\frac{t_L - 2}{t_L} = 0 < \beta = 1/2$.{ \color{black} Finally we note that the SVD-based node2vec has better empirical performance than the original node2vec in these experiments as well as in the experiments for three-blocks SBMs and DCSBMs in Section \ref{sec_simulation_embedding}. This is consistent with the discussion in Section~\ref{sec:2}. Indeed, the entries of $\tilde{\mathbf{M}}_0$ are the limit of those for the original node2vec when the number of sampled paths $r \rightarrow \infty$ and furthermore $\tilde{\mathbf{M}}_0$ has an approximately low-rank structure as $n$ increases. In other words, at least for SBM and DCSBM graphs, we can view the original node2vec as a computationally efficient approach to approximate the embeddings based on SVD of $\tilde{\mathbf{M}}_0$. 
}

\subsection{Embedding performance}\label{sec_simulation_embedding}
In this section we perform more numerical experiments to take a closer look at the
finite-sample performance of community detection, using both the original and SVD-based node2vec embeddings.
We consider both three-blocks SBM and three-blocks DCSBM. We will vary the sample size $n$,
window sizes $t_U$, and sparsity $\rho_n$ in these simulations
and investigate the effect of these parameters on the community detection accuracy.

More specifically, for each simulation with a specified value of $n$
and $\rho_n$, we run $100$ Monte Carlo replications where, for each
replicate, we apply both the original and SVD-based node2vec algorithms with different window
sizes on the simulated random graph to obtain the
embeddings followed by community detection using $K$-means on these
embeddings. Let the true and estimated cluster labels be denoted by $\{\tau(i)\}_{i
= 1}^n$ and
$\{\hat{\tau}(i)\}_{i = 1}^n$. We calculate the accuracy of $\hat{\tau}$ as (here
$\xi(\cdot)$ denotes an arbitrary permutation of $\{1,2,\dots,K\}$),
\begin{equation}\label{def:acc}
  \text{Accuracy} =
\min_{\xi(\cdot)}\frac{\#\{i|\xi(\hat{\tau}(i))\neq \tau(i)\}}{n}.
\end{equation}

Before presenting the formal numerical results, we first  fix $\rho_n = 3n^{-1/2}$, $n = 600$
and sample one random realization from both the SBM and the
DCSBM to illustrate the node2vec embedding performances. These
visualizations, which are depicted in Figure~\ref{f:embd:sbm} and
Figure~\ref{f:embd:dcsbm} in the Supplementary File, provide us with some intuitions, namely that (i) the original and SVD-based node2vec variants yield similar embeddings (ii) for SVD-based node2vec, increasing the 
window size could help separate nodes from different communities
and thereby improve the community detection accuracy; (iii) although the embeddings appear similar, $K$-means clustering yields more accurate membership recovery for the SVD-based node2vec compared to the original SVD-based node2vec embeddings. 
We now describe the settings of the network generation models used in these simulations.  

\noindent\textbf{Stochastic Blockmodel:}
We consider three-blocks SBMs with block
probabilities being either
\begin{equation}
\label{eq:Bmat_simulations}
\M B_1 = \begin{pmatrix}
0.8 & 0.4 & 0.3 
\\
0.4 & 0.7 & 0.5
\\
0.3 & 0.5 & 0.9
\end{pmatrix} \,\, \text{or} \,\, \M B_2 =  \begin{pmatrix}
0.8 & 0.5 & 0.5 
\\
0.5 & 0.8 & 0.5
\\
0.5 & 0.5 & 0.8
\end{pmatrix},
\end{equation}
and block assignment probabilities $\bds\pi = (0.3,0.3,0.4)$. 
\begin{figure*}[htbp]
\centering
\begin{subfigure}{.49\columnwidth}
\includegraphics[width=\columnwidth]{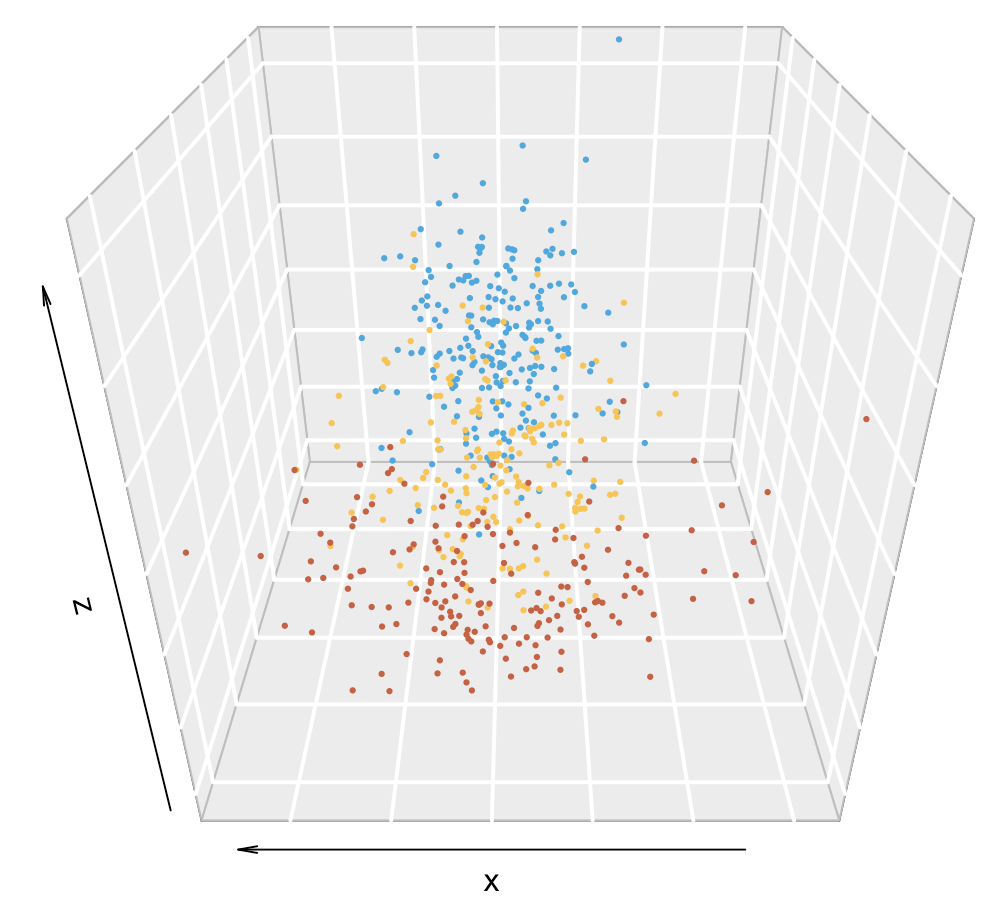}%
\caption{$t_U = 5$, $\text{accuracy} = 0.61$}
\end{subfigure}
\begin{subfigure}{.49\columnwidth}
\includegraphics[width=\columnwidth]{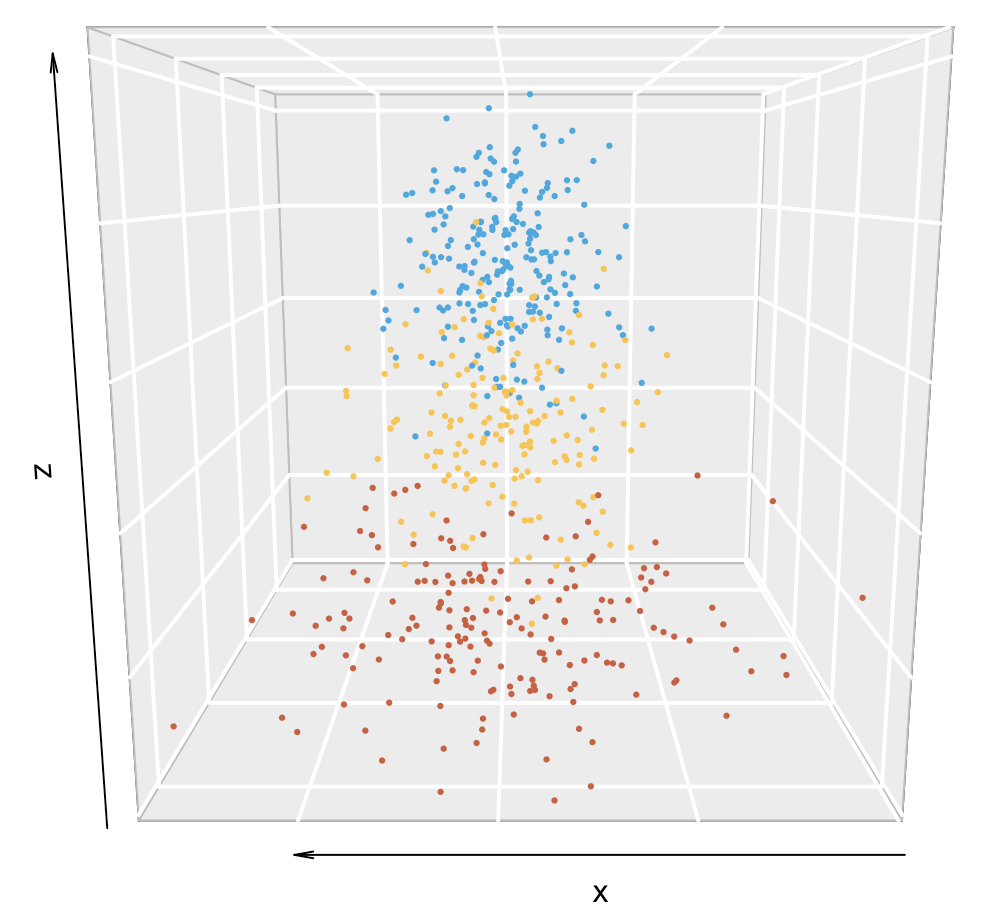}%
\caption{$t_U = 6$, $\text{accuracy} = 0.87$}
\label{12:3}
\end{subfigure}
\begin{subfigure}{.49\columnwidth}
\includegraphics[width=\columnwidth]{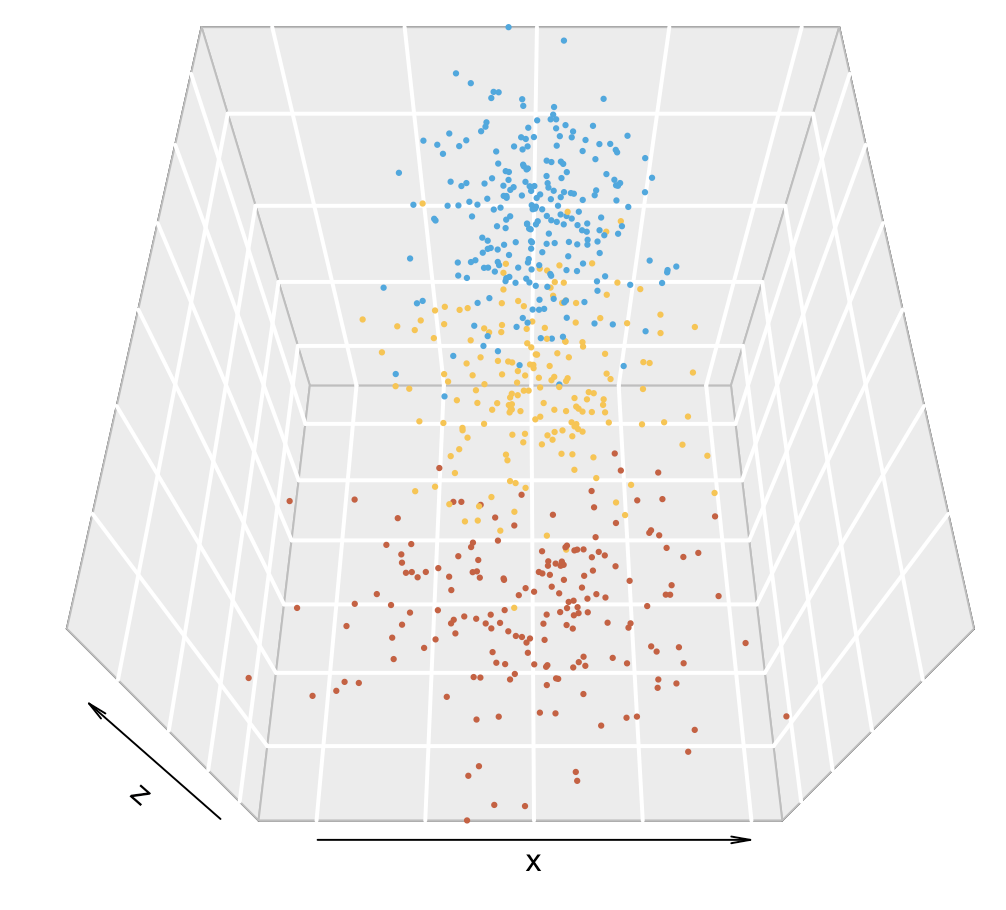}%
\caption{$t_U = 8$, $\text{accuracy} = 0.88$}
\end{subfigure}
\begin{subfigure}{.49\columnwidth}
\includegraphics[width=\columnwidth]{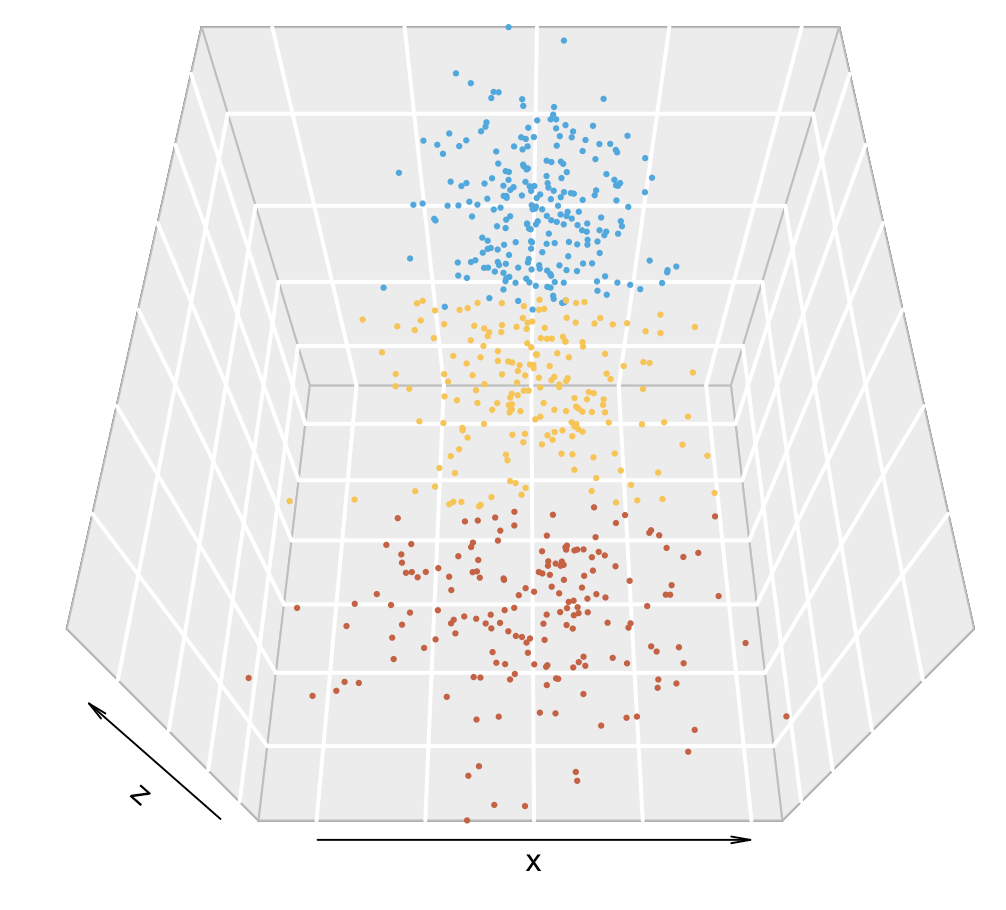}%
\caption{$t_U = 8$, recovery result}
\end{subfigure}
\par
\centering
\begin{subfigure}{.49\columnwidth}
\includegraphics[width=\columnwidth]{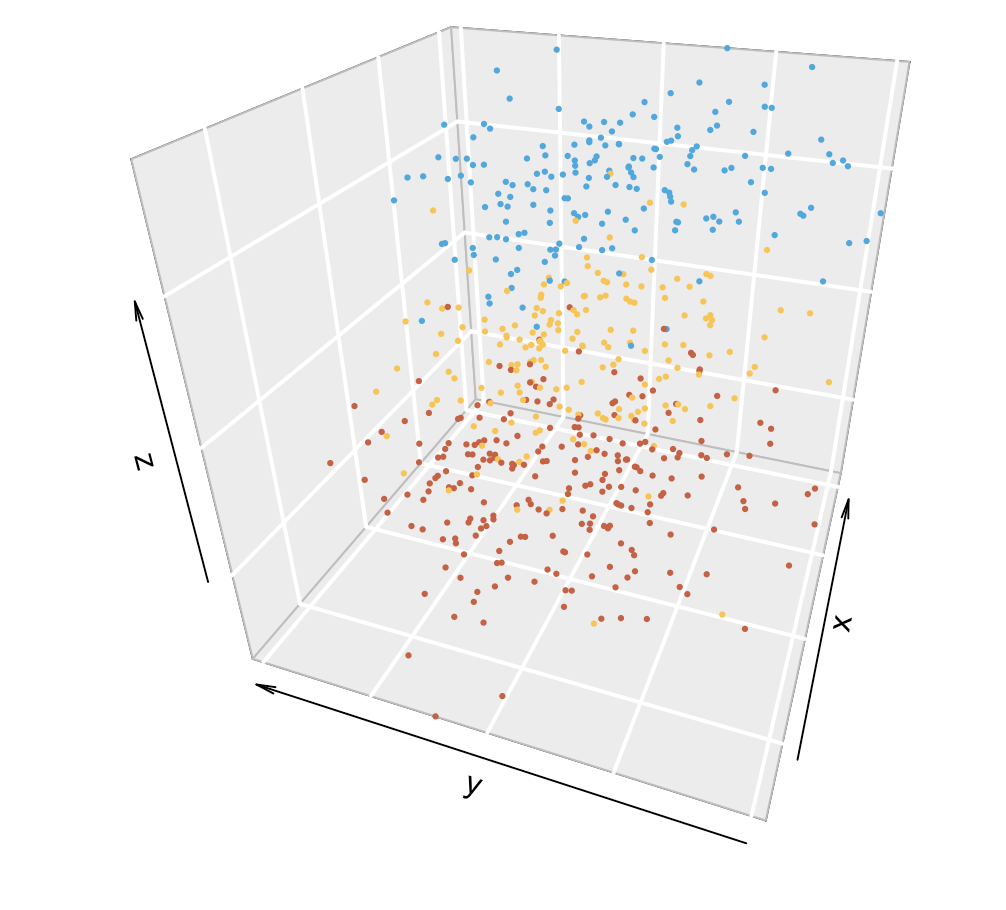}%
\caption{$t_U = 5$, $\text{accuracy} = 0.58$}
\end{subfigure}
\begin{subfigure}{.49\columnwidth}
\includegraphics[width=\columnwidth]{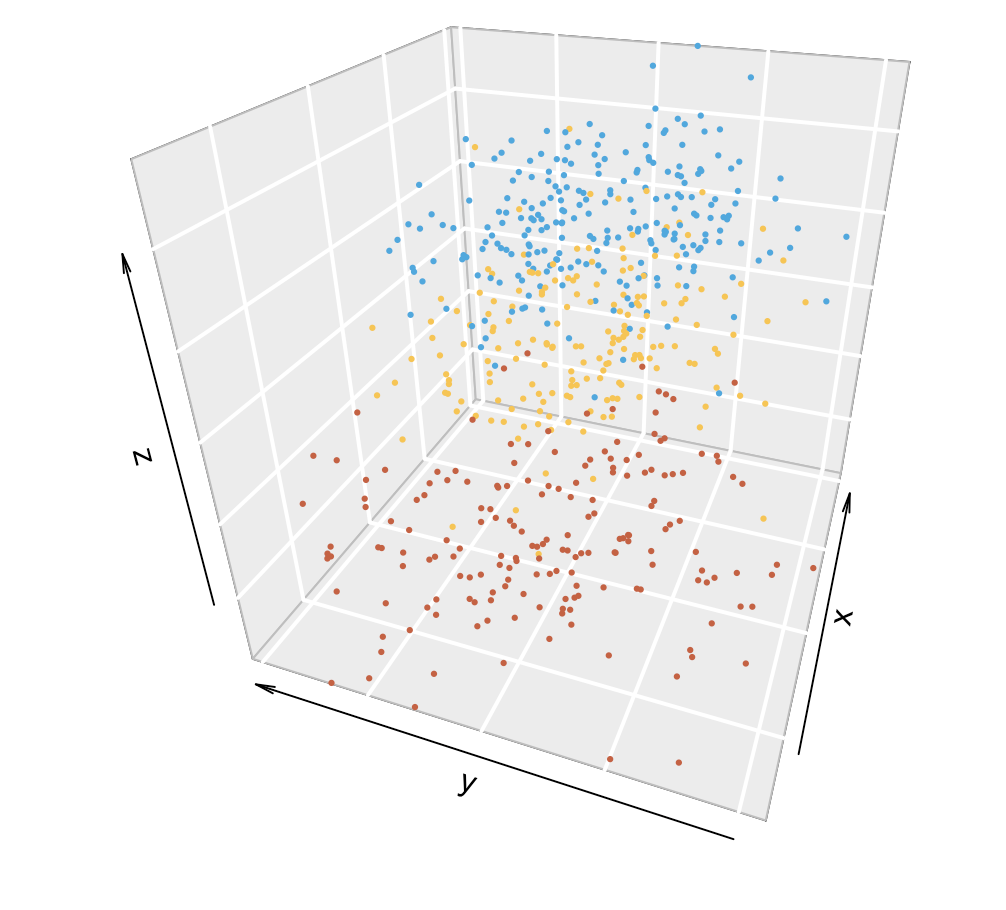}%
\caption{$t_U = 6$, $\text{accuracy} = 0.60$}
\label{12:3}
\end{subfigure}
\begin{subfigure}{.49\columnwidth}
\includegraphics[width=\columnwidth]{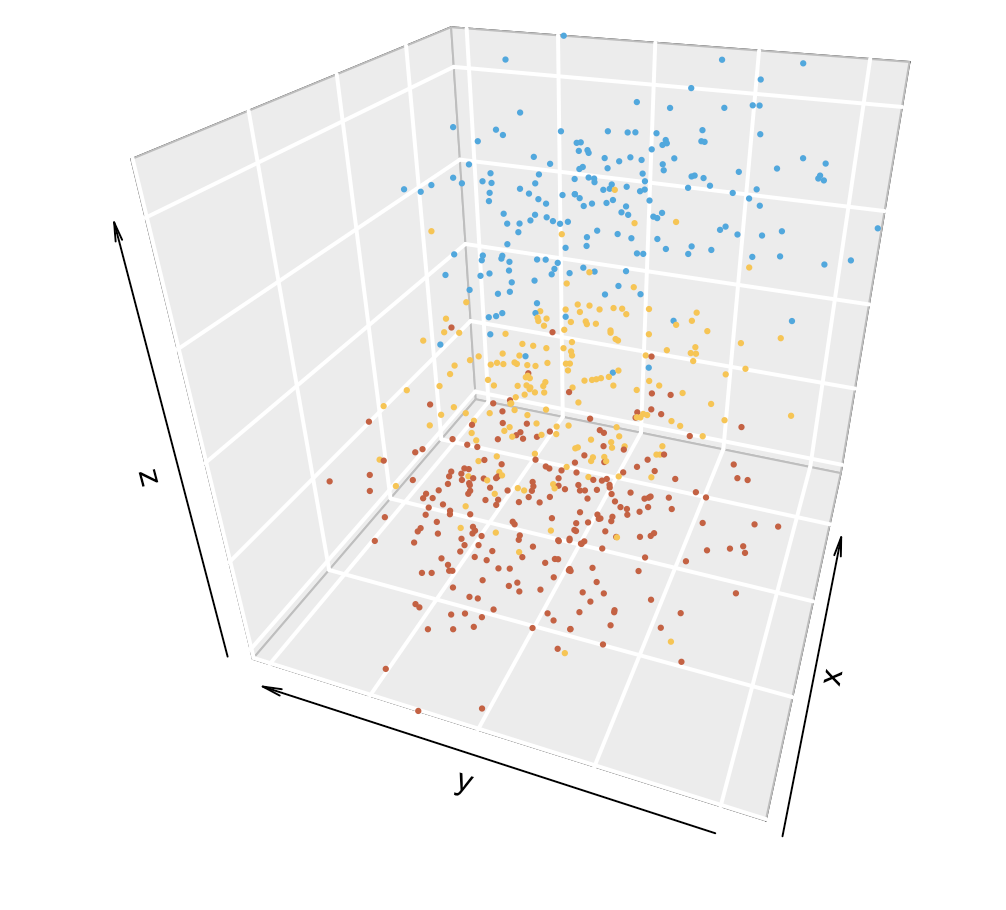}%
\caption{$t_U = 8$, $\text{accuracy} = 0.58$}
\end{subfigure}
\begin{subfigure}{.49\columnwidth}
\includegraphics[width=\columnwidth]{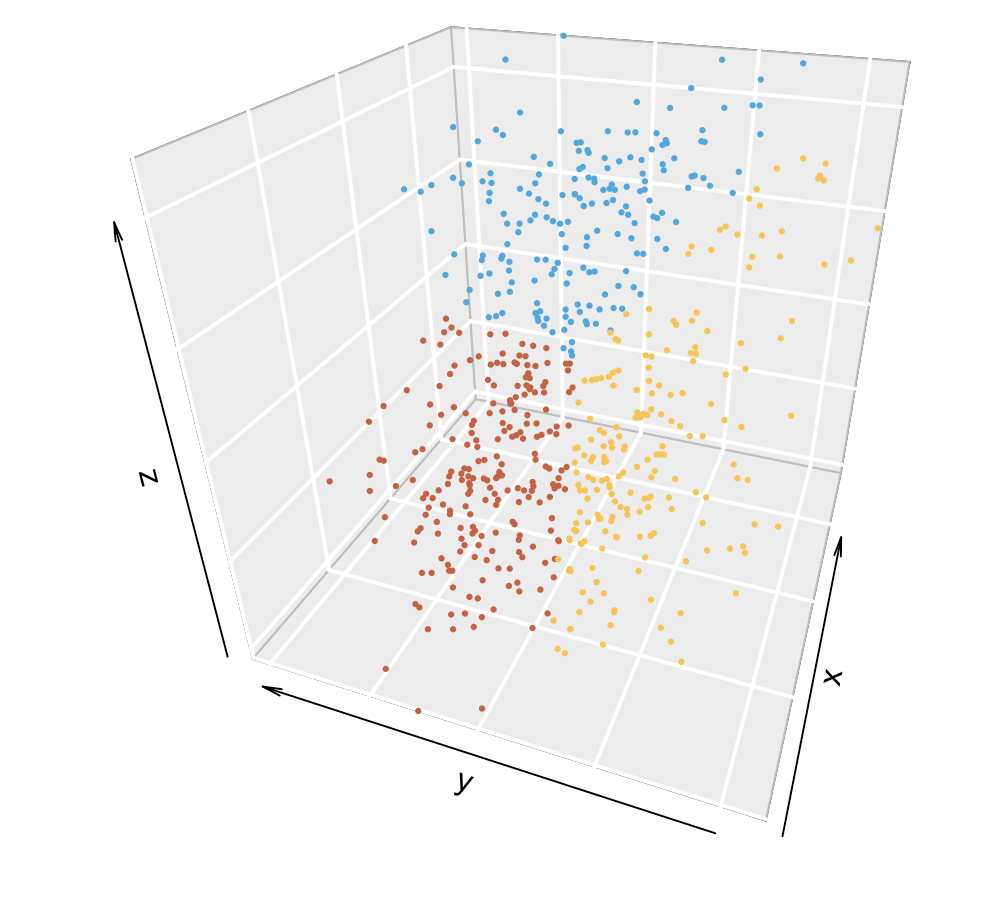}%
\caption{$t_U = 8$, recovery result}
\end{subfigure}
\caption{Visualizations of the SVD-based node2vec embeddings (first row) and original node2vec embeddings (second row) for different choices of $t_U$. The plots are for a single realization of a SBM graph on $n = 600$ vertices with block probabilities matrix $\M B_1$ (see eq.~\eqref{eq:Bmat_simulations}), sparsity $\rho_n = 3n^{-1/2}$, and block assignment probabilities $\bm{\pi} = (0.3, 0.3, 0.4)$. The embeddings in panels (a)--(c) and (e)--(g) are colored using the true membership assignments while the embeddings in panels (d) and (h) are colored using the $K$-means clustering.
Accuracy of the recovered memberships (by $K$-means clustering) are also reported for panels 
(a)--(c) and (e)--(g).}
\label{f:embd:sbm}
\end{figure*}

%

\noindent\textbf{Degree-Corrected Stochastic Blockmodel:} DCSBMs are direct
generalization of SBMs with the only difference being that each
node $i$ has a degree-correction parameter $\theta_i$ and that the
probability of connection between nodes $i$ and $j$ is
\begin{equation*}
p_{ij} = \theta_i \theta_j B_{\tau(i)\tau(j)}
\end{equation*}
instead of $p_{ij} = B_{\tau(i)\tau(j)}$ as in the case of
SBMs.  For more on DCSBMs and their inference, see
\cite{karrer2011stochastic,zhao2012consistency,gao2018community}. We generate the degree correction parameters $\theta_i$
as \bee\label{dg:cor}\theta_i = |Z_i| + 1 - (2\pi)^{-1/2}, \,\, Z_1,  \dots, Z_n \overset{\mathrm{iid}}{\sim}
\mathcal{N}(0,0.25)\ee
This procedure for generating $\theta_i$ is the same as that in \cite{gao2018community}. 

For each simulated graph, we test both the original node2vec and the SVD-based node2vec with $t_U = 5,6,7$. Other settings of the node2vec algorithms are similar to Section \ref{sec:erp}. The simulation results for the SBMs and the DCSMBs are presented in
Figure \ref{box:1} and Figure \ref{box:2} in the Supplementary File. We now summarize the main trend in these figures.

\begin{itemize} 
\item {The box plots when $\rho_n = 1$ (dense regime) have large
  interquartile ranges because there are a few replicates where, due to sampling variability, the simulated graphs are quite noisy and the community detection algorithm has low accuracy while for most of the remaining graphs we achieved exact recovery using both the original and SVD-based node2vec algorithms. Furthermore if $\rho_n = 1$ then increasing the
  window size from $t_L = 2$ to $t_L > 2$ does not yields noticeable improvement in 
  accuracy for the SVD-based node2vec. The condition $t_L \geq 2$ in Theorem \ref{c1} (i) is thus sufficient for exact recovery in the dense regime. On the other hand, when $\rho_n\neq 1$, we see that the accuracy increases with $n$ as indicated by the large-sample results in Theorem \ref{c1}.
\item When $\rho_n \rightarrow 0$ faster (i.e., the network is more sparse), we need a larger $n$ to achieve the same level of accuracy. This is consistent with Theorem \ref{c1} as the convergence rate for $\hat{\bds{\mathcal{F}}}$ depends on $n \rho_n$. 
\item When $\M B = \M B_2$ the original node2vec and SVD-based node2vec have very similar accuracy and thus our theoretical analysis of SVD-based node2vec closely reflects the performance of the original node2vec. 
\item When $\M B = \M B_1$ the SVD-based node2vec has higher accuracy compared to the original node2vec. However the embeddings generated by these algorithms are still quite similar. A plausible reason for why the original node2vec has lower accuracy is because the downstream $K$-means clustering is sub-optimal for these embeddings. For example, comparing panels (c) and (d) in Fig.~\ref{f:embd:sbm} we can see that $K$-means clustering correctly recovers most of the membership assignments for embeddings from the SVD-based node2vec. In contrast, panels (g) and (h) in Fig.~\ref{f:embd:sbm} show that $K$-means clustering is less accurate for embeddings from the original node2vec. Indeed, when replacing $K$-means with Gaussian mixtures model (GMM) \cite{scrucca2016mclust,fraley2002model} in panels (g) and $(h)$ of Fig.~\ref{f:embd:sbm} we increase the clustering accuracy from $0.58$ to $0.84$ which is close to that of $0.88$ for the SVD-based node2vec (see Fig.~\ref{fig:improve} of the Supplementary File). }

\end{itemize}
\begin{figure*}[htbp]
\centering
\begin{subfigure}{.6\columnwidth}
\includegraphics[width=\columnwidth]{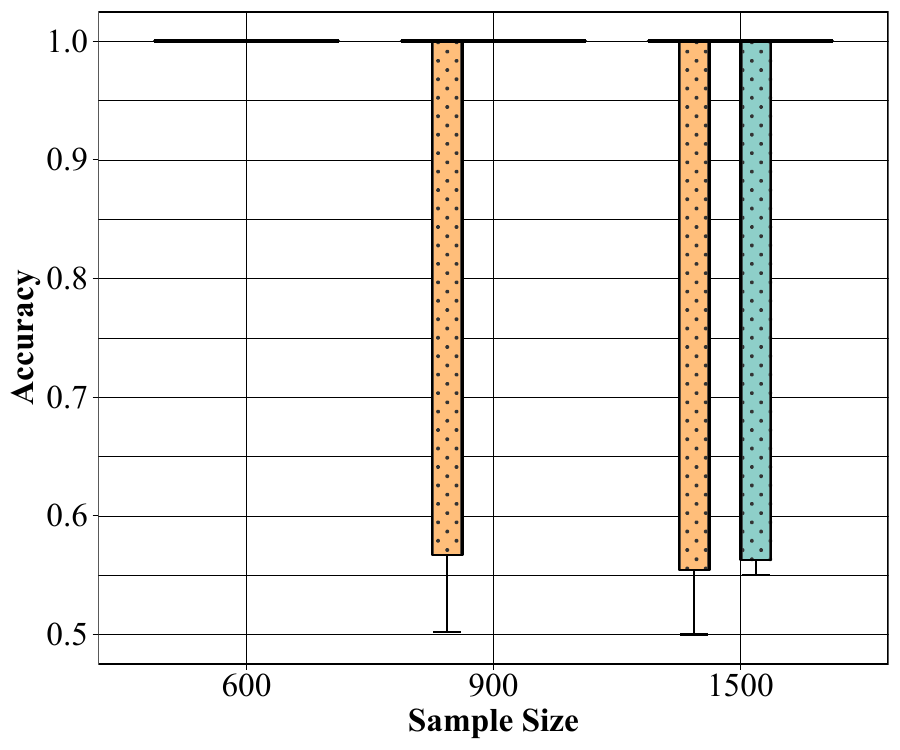} %
\caption{$\rho_n = 1$}
\label{dense:1}
\end{subfigure}
\begin{subfigure}{.6\columnwidth}
\includegraphics[width=\columnwidth]{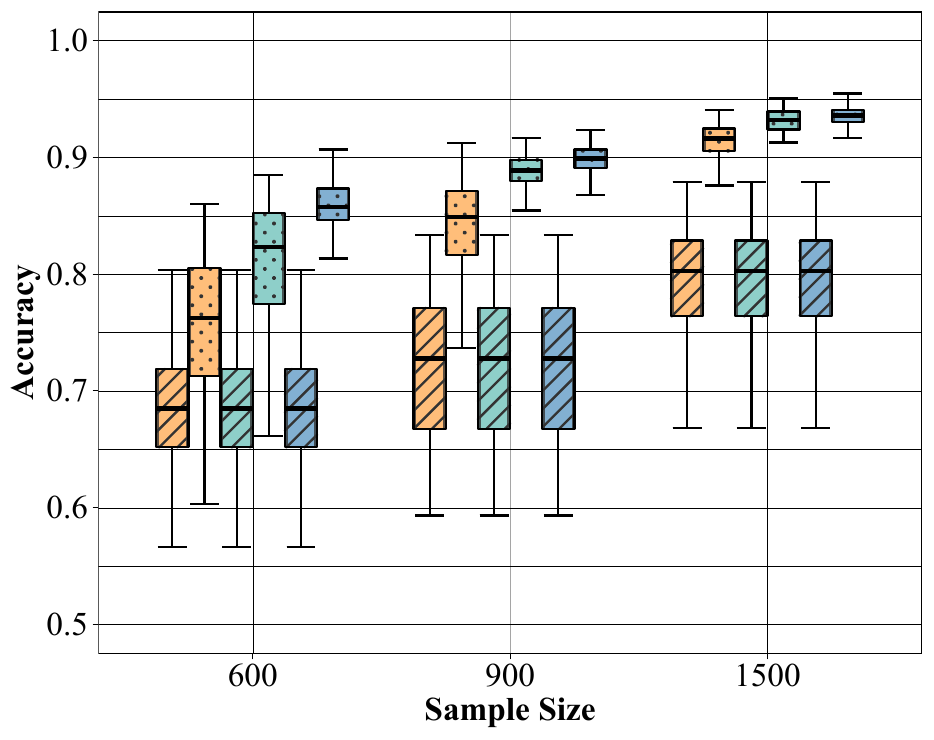}%
\caption{$\rho_n = 3n^{-1/2}$}
\label{12:1}
\end{subfigure}
\begin{subfigure}{.6\columnwidth}
\includegraphics[width=\columnwidth]{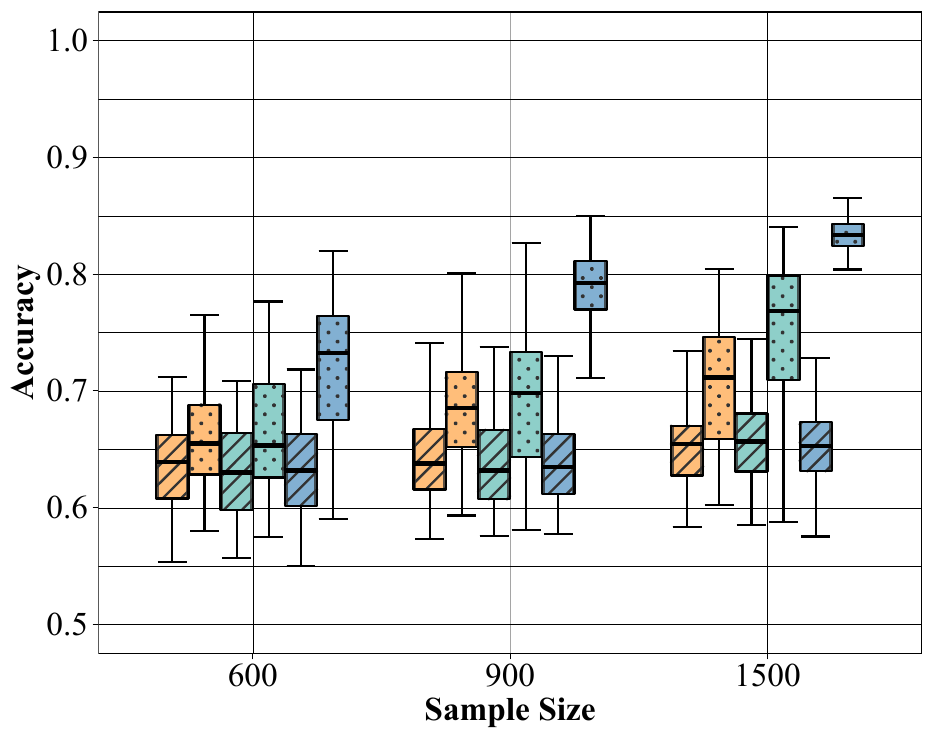}%
\caption{$\rho_n = 6n^{-2/3}$}
\end{subfigure}
\par
\centering
\begin{subfigure}{.6\columnwidth}
\includegraphics[width=\columnwidth]{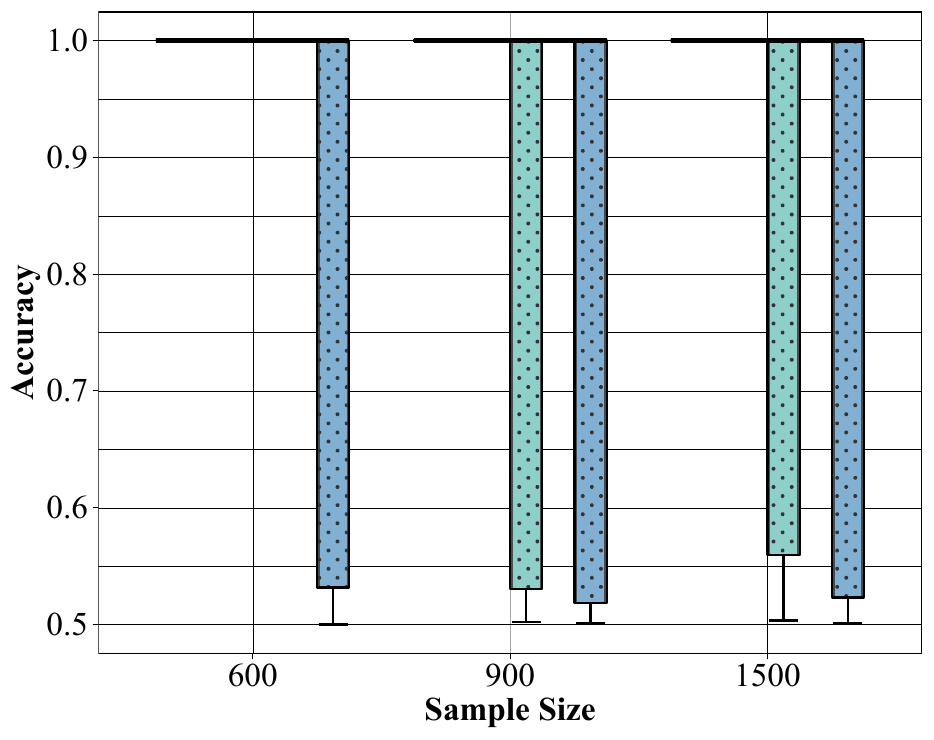}%
\caption{$\rho_n = 1$}
\label{dense:2}
\end{subfigure}
\begin{subfigure}{.6\columnwidth}
\includegraphics[width=\columnwidth]{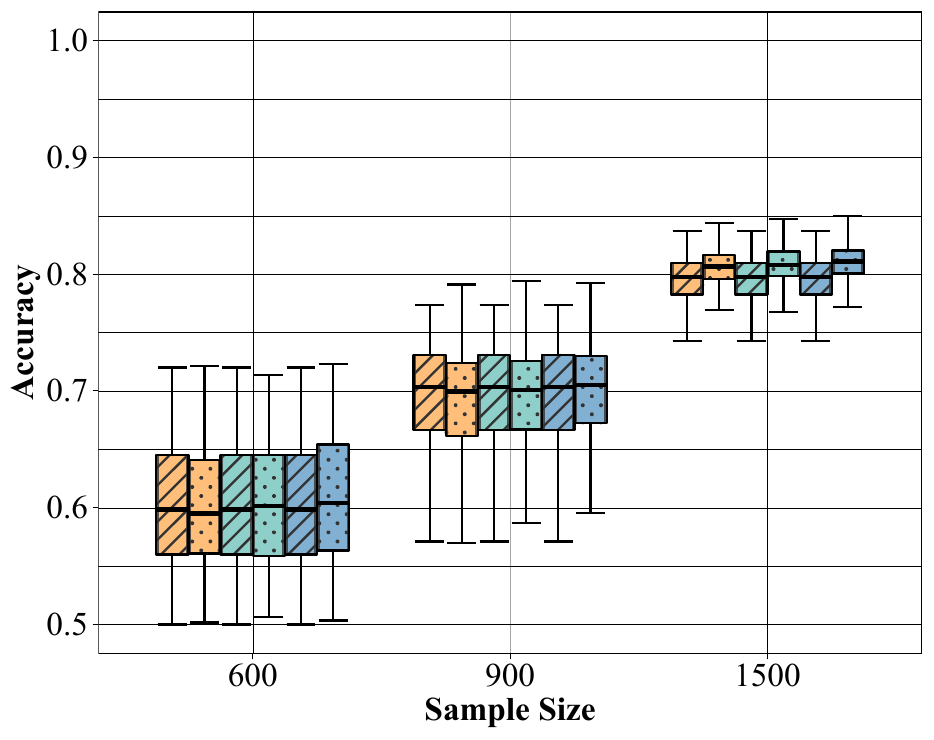}%
\caption{$\rho_n = 3n^{-1/2}$}
\label{12:2}
\end{subfigure}
\begin{subfigure}{.6\columnwidth}
\includegraphics[width=\columnwidth]{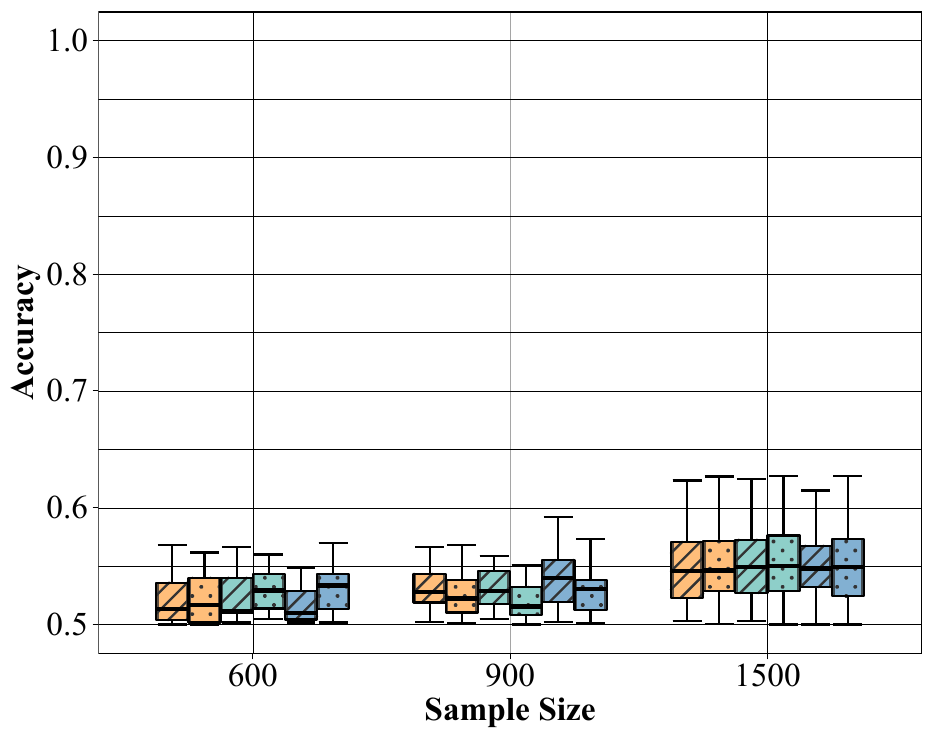}%
\caption{$\rho_n = 6n^{-2/3}$}
\label{subfig:special2}
\end{subfigure}
\caption{Community detection accuracy of node2vec followed by
    $K$-means for SBM graphs. The
    boxplots of the accuracy for each value of $n, \rho_n$ and $
    t_U$ are based on $100$ Monte Carlo replications. Boxplots with the slash pattern (resp. dot pattern) summarized the results for the original (resp. SVD-based) node2vec. Different colors (yellow, green, blue) represent the algorithms implemented for different choices of $t_U \in \{5,6,8\}$. The first and second row
    plot the results when the block probabilities for the SBM is $\M B_1$ and $\M
    B_2$, respectively.}
\label{box:1}
\end{figure*}

\section{Applications to real-world networks}\label{sec:real}
We test the membership recovery performance of node2vec  on three real-world networks, namely, the Zachary's karate graph (henceforth, ZK) \cite{zachary1977information}, political blogs graph (henceforth, PB) \cite{adamic2005political}, and  Wikipedia graph (henceforth, WIKI) \cite{sussman2012consistent}.  In each of the three graphs, the memberships of all vertices have been assigend baed on specific real-world meanings without missing. Both ZK and PB contain 2 communities, while WIKI contains 6 communities. ZK is connected with 34 vertices. By conventions \cite{karrer2011stochastic,sussman2012consistent}, we ignore the directions of edges and focus on the largest connected components of PB and WIKI, which contain 1222 and 1323 vertices, respectively. We refer interested readers to the references above for more detailed information about the three real-world network datasets.
\par
For each network dataset, we embed the vertices  into the $K$-dimensional Euclidian space through both the SVD-based and original node2vec, and then cluster the embeddings by $K$-means to estimate the memberships of each vertex;  $K$ is chosen as the exact number of memberships in each graph. We test three window sizes $t_U \in \{ 10,15,20\}$. Similar to Section \ref{sec:simu}, we set $t_L = t_U -5$ for the SVD-based node2vec and $t_L = 1$ for the original node2vec by default. To measure the membership recovery performances, we calculate the accuracies between the estimated memberships and the real memberships for ZK and PB; see the definition of accuracy in Eq.~\eqref{def:acc}. For WIKI, because the criteria of accuracy becomes computationally inflexible, we alternatively use the adjusted rand index (ARI). Similar to the accuracy, $\text{ARI} = 1$ indicates the estimated memberships perfectly recover the real memberships, while  $\text{ARI} = 0$ indicates the estimated memberships are assigned randomly. We also compare  performances of node2vec algorithms with other popular spectral embedding algorithms, including the spectral clustering based on adjacency and normalized Laplacian \cite{von2007tutorial,rohe2011spectral,sussman2012consistent}, and the spectral clustering with projection onto the sphere \cite{modell2021spectral}; for all methods we use $K$-means for the downstream clustering. \par
 The recovery results are summarized in Table \ref{tablereal}. The SVD-based and original node2vec algorithms have similar performances, which are generally better than or equivalently to other methods in all three datasets. In addition, we note the PB dataset is better modeled as a DCSBM \cite{karrer2011stochastic}. Recall that, as shown in Remark \ref{rk:dcsbm}, node2vec can theoretically attain exact recovery for DCSBMs and hence the high-accuracy of node2vec on the PB dataset is expected. Similarly, \cite{modell2021spectral} shows a valid theoretical guarantee of the spectral clustering with a spherical projection, when applying to the DCSBM graph. This can also be verified by the high accuracy of ASE+SP on PB  as shown in Table \ref{tablereal}.
 \begin{table*}[htbp]
\centering\sffamily
\renewcommand{\theadfont}{\normalsize}
\setcellgapes{0.8ex}\makegapedcells
\begin{tabular}{|*{10}{c|}}
\hline
\multirowthead{1}{Network}  & \multicolumn{3}{c|}{SVD-based node2vec} & 
\multicolumn{3}{c|}{Original node2vec}
&\multicolumn{1}{c|}{ASE}&\multicolumn{1}{c|}{LSE} & \multicolumn{1}{c|}{ASE+SP} \\
\cline{2-7}
 &$t_U = 10$& $t_U = 15$ & $t_U = 20$ &$t_U = 10$&$t_U = 15$&$t_U = 20$&&&\\
\hline
\text{ZK}&0.97&0.97&0.97&0.97&0.97&0.97&1.00&0.97& 0.97
\\
\hline
\text{PB}&0.96&0.95&0.95&0.96&0.95&0.95&0.61&0.51&0.95
\\
\hline
\end{tabular}
{\vskip 5mm} 
\centering
 \begin{tabular}{|*{10}{c|}}
\hline
\multirowthead{1}{Network}  & \multicolumn{3}{c|}{SVD-based node2vec} & 
\multicolumn{3}{c|}{Original node2vec}
&\multicolumn{1}{c|}{ASE}&\multicolumn{1}{c|}{LSE}& \multicolumn{1}{c|}{ASE+SP}  \\
\cline{2-7}
 &$t_U = 10$& $t_U = 15$ & $t_U = 20$ &$t_U = 10$&$t_U = 15$&$t_U = 20$&&&\\
\hline
\text{WIKI}&0.09&0.10&0.08&0.09&0.10&0.11&0.03&0.09&0.09
\\
\hline
\end{tabular}
\caption{The upper table reports the membership recovery accuracy of different embedding methods on the ZK and PB network datasets. The lower table reports the ARI  of different embedding methods on the WIKI network dataset. ASE and LSE denote spectral clusterings using the truncated eigendecomposition of the adjacency and normalized Laplacian matrix \cite{von2007tutorial,rohe2011spectral,sussman2012consistent}, respectively. ASE+SP denote spectral clustering using the truncated eigendecomposition of the adjacency matrix together with a spherical projection step \cite{modell2021spectral,ng_jordan_weiss}.}
\label{tablereal}
\end{table*}
\section{Discussion}\label{sec:disc}
In this paper we derive perturbation bounds and show exact recovery for the DeepWalk and node2vec
(with $p = q = 1$) algorithms under the assumption that the observed graphs are instances
of the stochastic blockmodel graphs. Our results are valid under both the dense
and sparse regimes for sufficient large $t_L$ and $n$.
The simulation results corroborate our theoretical findings; 
in particular they show that increasing the sample size and window size
can improve the community detection accuracy for both sparse SBM and
DCSBM graphs. 

We emphasize that our paper only include real data analysis on simple graphs with a small number of nodes 
just to illustrate the agreement between our theoretical results and the empirical performance of DeepWalk/node2vec. This is intentional as DeepWalk
and node2vec are widely-used algorithms with numerous
papers demonstrating their uses for analyzing real
graphs in diverse applications. In contrast, our paper is one of a few
that addresses the theory underpinning these algorithms and is, to the best of our knowledge, the
first paper to establish consistency and exact recovery for SBMs and DCSBMs
using these random-walk based embedding algorithms. Note that exact recovery for SBMs can also be achieved using other algorithms such as those based on semidefinite programming, variational Bayes, and spectral embedding; see \cite{abbe2017community,gao_optimal,lyzinski2014perfect} for a few examples.  

There are several open questions for future research: 
\begin{enumerate}
  \item
  In this paper we only consider the case of $p=q=1$ for node2vec
  embedding (recall that $p = q = 1$ is the default parameter values for
  node2vec). If $p \not = 1$ and/or $q \not = 1$ then the transformed co-occurrence matrix 
  $\tilde{\M M}_0$ can no longer be expressed in
  terms of the adjacency matrix $\M A$ or the transition matrix
  $\hat{\M W}^t$;
  this renders the theoretical analysis for general values of $p$ and
  $q$ substantially more involved. 
  One potential approach to this problem is to consider,
  similar to the notion of the {\em non-backtracking matrix} in
  community detection for sparse SBM \cite{bordenave2015non}, a
  transition matrix associated with the edges of $\mathcal{G}$ as
opposed to the transition matrix associated with the vertices in
  $\mathcal{G}$.  Indeed, if $p \not = q$ then the transition probability from a vertex $v$ to another vertex $w$ depends also on the vertex, say $u$, preceding $v$ in the random walk. i.e.,
  the transition probability for $(v,w)$ depends on the choice of $(u,v)$.
\item In this paper we focus on error bounds (in Frobenius and infinity norms) of node2vec/DeepWalk
  embedding for stochastic blockmodel graphs and their degree-corrected variant. An important question is whether
  or not stronger limit results are available for these
  algorithms. For example spectral embeddings of stochastic
  blockmodel graphs obtained via eigendecompositions of either the adjacency
  or the normalized Laplacian matrices are well-approximated by mixtures of
  multivariate Gaussians; see \cite{grdpg,tang2018limit} for more
  precise statements of these results and their implications for
  statistical inference in networks. It is thus natural to inquire if
  normal approximations also holds node2vec/Deepwalk. 
%
We ran several one-round simple simulations to visualize the embeddings of
node2vec/DeepWalk when the graphs are sampled from a SBM with 
\bee\label{Bpidef}
\M B = \Bigl(\begin{smallmatrix}0.42 & 0.42
  \\ 0.42 & 0.5\end{smallmatrix}\Bigr)\text{ and }\bds{\pi} =
(0.4,0.6).
\ee The results are summarized in Fig.~\ref{f:embd:sbm} in the Supplementary File. In particular when $n$ is large these embeddings are also well-approximated by a mixture of multivariate Gaussians.
We leave the theoretical justification of this phenomenon for future work.
\item
As we allude to in the introduction, for simplicity we only consider (degree-corrected) stochastic blockmodel graphs
    in this paper. For the more general inhomogeneous Erd\H{o}s-R\'{e}nyi random graphs
model, we expect that Theorem~\ref{T2} and Theorem~\ref{T3} 
still
    hold, provided that the edge probabilities are sufficiently
homogeneous, i.e., the minimum and maximum values for the edge
probabilities values are of the same order as $n$
increases. However, the error bounds in Theorem~\ref{c1} 
    might no longer apply since 
    the entry-wise logarithmic transformation of the co-occurrence matrices can lead
    to the setting wherein $\M M_0$ is no longer low-rank, e.g., the rank of $\M M_0$ can be as large as $n$ the number of vertices. Furthermore, even when $\M M_0$ have an approximate low-rank structure, due to the logarithmic transformation there is still the question of how the embedding of $\M M_0$ relates to the underlying latent structure in $\M P$.
%

  \item Finally, in this paper we mainly focus on the node2vec and DeepWalk embedding through
 matrix factorization (SVD-based node2vec), but also compare the SVD-based node2vec with the original node2vec in the numerical experiments. As we mentioned in the introduction the original node2vec algorithm
 uses (stochastic) gradient descent (GD/SGD) to optimize Eq.~\eqref{SKIP_NS}
 and obtain the embeddings. As Eq.~\eqref{SKIP_NS} is non-convex there can be a large number of
 local-minima, thereby making the theoretical analysis intractable unless we assume that the initial estimates for GD/SGD are sufficiently close to the global minima; see
 e.g., \cite{chi2019nonconvex,sgd_convergence} for some examples of results relating the closeness of the initial estimates and the convergence rate of GD/SGD. One popular initialization scheme for GD/SGD is via spectral methods and thus we can consider using the SVD-based embedding $\hat{\bds{\mathcal{F}}}$ as a ``warm-start`` for Eq.~\eqref{SKIP_NS}. We leave the precise convergence analysis of the resulting GD/SGD iterations to the interested reader. We note, however, that while this is certainly an interesting technical problem, the practical benefits might be limited. Indeed, the theoretical results in Section~\ref{sec:3} guaranteed perfect recovery using $\hat{\bds{\mathcal{F}}}$ while the empirical evaluations in Sections~\ref{sec:erp} and Section~\ref{sec_simulation_embedding} suggest that clustering based on $\hat{\bds{\mathcal{F}}}$ is comparable or even better than that of the original node2vec. In other words as the main objective is to recover the structure in $\mathbf{M}_0$ induced by $\M P$, it is certainly possible that optimizing Eq.~\eqref{SKIP_NS} does not lead to better inference performance due to the noise in using $\M A$ as a replacement for $\M P$. 
 \end{enumerate}
\bibliographystyle{IEEEtran}
\bibliography{ref.bib}

\begin{thebibliography}{10}
\providecommand{\url}[1]{#1}
\csname url@samestyle\endcsname
\providecommand{\newblock}{\relax}
\providecommand{\bibinfo}[2]{#2}
\providecommand{\BIBentrySTDinterwordspacing}{\spaceskip=0pt\relax}
\providecommand{\BIBentryALTinterwordstretchfactor}{4}
\providecommand{\BIBentryALTinterwordspacing}{\spaceskip=\fontdimen2\font plus
\BIBentryALTinterwordstretchfactor\fontdimen3\font minus
  \fontdimen4\font\relax}
\providecommand{\BIBforeignlanguage}[2]{{%
\expandafter\ifx\csname l@#1\endcsname\relax
\typeout{** WARNING: IEEEtran.bst: No hyphenation pattern has been}%
\typeout{** loaded for the language `#1'. Using the pattern for}%
\typeout{** the default language instead.}%
\else
\language=\csname l@#1\endcsname
\fi
#2}}
\providecommand{\BIBdecl}{\relax}
\BIBdecl

\bibitem{von2007tutorial}
U.~von Luxburg, ``A tutorial on spectral clustering,'' \emph{Statistics and
  Computing}, vol.~17, pp. 395--416, 2007.

\bibitem{wang2017community}
X.~Wang, P.~Cui, J.~Wang, J.~Pei, W.~Zhu, and S.~Yang, ``Community preserving
  network embedding.'' in \emph{AAAI}, vol.~17, 2017, pp. 203--209.

\bibitem{liben2007link}
D.~Liben-Nowell and J.~Kleinberg, ``The link-prediction problem for social
  networks,'' \emph{Journal of the American Society for Information Science and
  Technology}, vol.~58, pp. 1019--1031, 2007.

\bibitem{perozzi2014deepwalk}
B.~Perozzi, R.~Al-Rfou, and S.~Skiena, ``Deepwalk: Online learning of social
  representations,'' in \emph{Proceedings of the 20th ACM SIGKDD International
  Conference on Knowledge Discovery and Data Mining}, 2014, pp. 701--710.

\bibitem{hamilton2017inductive}
W.~Hamilton, Z.~Ying, and J.~Leskovec, ``Inductive representation learning on
  large graphs,'' in \emph{Advances in Neural Information Processing Systems},
  2017, pp. 1024--1034.

\bibitem{theocharidis2009network}
A.~Theocharidis, S.~Van~Dongen, A.~J. Enright, and T.~C. Freeman, ``Network
  visualization and analysis of gene expression data using biolayout express
  3d,'' \emph{Nature Protocols}, vol.~4, p. 1535, 2009.

\bibitem{rohe2011spectral}
K.~Rohe, S.~Chatterjee, and B.~Yu, ``Spectral clustering and the
  high-dimensional stochastic blockmodel,'' \emph{Annals of Statistics},
  vol.~39, pp. 1878--1915, 2011.

\bibitem{sussman2012consistent}
D.~L. Sussman, M.~Tang, D.~E. Fishkind, and C.~E. Priebe, ``A consistent
  adjacency spectral embedding for stochastic blockmodel graphs,''
  \emph{Journal of the American Statistical Association}, vol. 107, pp.
  1119--1128, 2012.

\bibitem{shi_malik}
J.~Shi and J.~Malik, ``Normalized cuts and image segmentation,'' \emph{IEEE
  Transactions on Pattern Analysis and Machine Intelligence}, vol.~8, pp.
  888--905, 2000.

\bibitem{robinson1995typology}
S.~L. Robinson and R.~J. Bennett, ``A typology of deviant workplace behaviors:
  A multidimensional scaling study,'' \emph{Academy of Management Journal},
  vol.~38, pp. 555--572, 1995.

\bibitem{ye2005two}
J.~Ye, R.~Janardan, and Q.~Li, ``Two-dimensional linear discriminant
  analysis,'' in \emph{Advances in Neural Information Processing Systems},
  2005, pp. 1569--1576.

\bibitem{kipf2016semi}
T.~N. Kipf and M.~Welling, ``Semi-supervised classification with graph
  convolutional networks,'' in \emph{International Conference on Learning
  Representations}, 2017.

\bibitem{line_graph}
L.~Cai, J.~Li, J.~Wang, and S.~Ji, ``Line graph neural networks for link
  prediction,'' \emph{IEEE Transactions on Pattern Analysis and Machine
  Intelligence}, vol.~44, pp. 5103--5113, 2022.

\bibitem{pine}
S.~Gui, X.~Zhang, P.~Zhong, S.~Qiu, M.~Wu, J.~Ye, Z.~Wang, and J.~Liu,
  ``{PINE}: Universal deep embedding for graph nodes via partial permutation
  invariant set functions,'' \emph{IEEE Transactions on Pattern Analysis and
  Machine Intelligence}, vol.~44, p. 770–782, 2022.

\bibitem{frame_work}
H.~Yan, D.~Xu, B.~Zhang, H.-J. Zhang, Q.~Yang, and S.~Lin, ``Graph embedding
  and extensions: A general framework for dimensionality reduction,''
  \emph{IEEE Transactions on Pattern Analysis and Machine Intelligence},
  vol.~29, pp. 40--51, 2007.

\bibitem{hamilton2017}
W.~L. Hamilton, R.~Ying, and J.~Leskovec, ``Representation learning on graphs:
  Methods and applications,'' \emph{IEEE Data Engineering Bulletin}, vol.~40,
  pp. 52--74, 2017.

\bibitem{cui2018survey}
P.~Cui, X.~Wang, J.~Pei, and W.~Zhu, ``A survey on network embedding,''
  \emph{IEEE Transactions on Knowledge and Data Engineering}, vol.~31, pp.
  833--852, 2018.

\bibitem{chen2020graph}
F.~Chen, Y.-C. Wang, B.~Wang, and C.-C.~J. Kuo, ``Graph representation
  learning: a survey,'' \emph{APSIPA Transactions on Signal and Information
  Processing}, vol.~9, p. e15, 2020.

\bibitem{grover2016node2vec}
A.~Grover and J.~Leskovec, ``node2vec: Scalable feature learning for
  networks,'' in \emph{Proceedings of the 22nd ACM SIGKDD International
  Conference on Knowledge Discovery and Data Mining}, 2016, pp. 855--864.

\bibitem{palumbo2018knowledge}
E.~Palumbo, G.~Rizzo, R.~Troncy, E.~Baralis, M.~Osella, and E.~Ferro,
  ``Knowledge graph embeddings with node2vec for item recommendation,'' in
  \emph{European Semantic Web Conference}.\hskip 1em plus 0.5em minus
  0.4em\relax Springer, 2018, pp. 117--120.

\bibitem{zhang2019biowordvec}
Y.~Zhang, Q.~Chen, Z.~Yang, H.~Lin, and Z.~Lu, ``Biowordvec, improving
  biomedical word embeddings with subword information and mesh,''
  \emph{Scientific Data}, vol.~6, pp. 1--9, 2019.

\bibitem{zhang2019identification}
X.~Zhang, L.~Chen, Z.-H. Guo, and H.~Liang, ``Identification of human membrane
  protein types by incorporating network embedding methods,'' \emph{IEEE
  Access}, vol.~7, pp. 140\,794--140\,805, 2019.

\bibitem{zheng2020gman}
C.~Zheng, X.~Fan, C.~Wang, and J.~Qi, ``Gman: A graph multi-attention network
  for traffic prediction,'' in \emph{AAAI}, vol.~34, 2020, pp. 1234--1241.

\bibitem{chu2019neural}
H.~Chu, D.~Li, D.~Acuna, A.~Kar, M.~Shugrina, X.~Wei, M.-Y. Liu, A.~Torralba,
  and S.~Fidler, ``Neural turtle graphics for modeling city road layouts,'' in
  \emph{Proceedings of the IEEE International Conference on Computer Vision},
  2019, pp. 4522--4530.

\bibitem{levy2014neural}
O.~Levy and Y.~Goldberg, ``Neural word embedding as implicit matrix
  factorization,'' in \emph{Advances in Neural Information Processing Systems},
  2014, pp. 2177--2185.

\bibitem{qiu2018network}
J.~Qiu, Y.~Dong, H.~Ma, J.~Li, K.~Wang, and J.~Tang, ``Network embedding as
  matrix factorization: Unifying deepwalk, line, pte, and node2vec,'' in
  \emph{Proceedings of the Eleventh ACM International Conference on Web Search
  and Data Mining}, 2018, pp. 459--467.

\bibitem{sussman_vec}
C.~Lin, D.~L. Sussman, and P.~Ishwar, ``Ergodic limits, relaxations, and
  geometric properties of random walk node embeddings,'' 2021, arXiv preprint
  \#2109.04526.

\bibitem{qiu2020concentration}
J.~Qiu, C.~Wang, B.~Liao, R.~Peng, and J.~Tang, ``Concentration bounds for
  co-occurrence matrices of markov chains,'' 2020, arXiv preprint \#2008.02464.

\bibitem{tang2018limit}
M.~Tang and C.~E. Priebe, ``Limit theorems for eigenvectors of the normalized
  laplacian for random graphs,'' \emph{Annals of Statistics}, vol.~46, pp.
  2360--2415, 2018.

\bibitem{grdpg}
P.~Rubin-Delanchy, J.~Cape, M.~Tang, and C.~E. Priebe, ``A statistical
  interpretation of spectral embedding: the generalised random dot product
  graph,'' \emph{Journal of the Royal Statistical Society, Series B}, vol.~84,
  pp. 1446--1473, 2022.

\bibitem{bollobas2001random}
B.~Bollob{\'a}s, \emph{Random graphs}.\hskip 1em plus 0.5em minus 0.4em\relax
  Cambridge university press, 2001.

\bibitem{hoff2002}
P.~D. Hoff, A.~E. Raftery, and M.~S. Handcock, ``Latent space approaches to
  social network analysis,'' \emph{Journal of the American Statistical
  Association}, vol.~97, pp. 1090--1098, 2002.

\bibitem{surprising}
C.~Hacker and B.~Rieck, ``On the suprising behavior of node2vec,'' 2022, {ICML}
  Workshop on Topology, Algebra, and Machine Learning. arXiv preprint
  \#2206.082525.

\bibitem{barot2021community}
A.~Barot, S.~Bhamidi, and S.~Dhara, ``Community detection using low-dimensional
  network embedding algorithms,'' 2021, arXiv preprint \#2111.05267.

\bibitem{davis70}
C.~Davis and W.~Kahan, ``The rotation of eigenvectors by a pertubation.
  {III}.'' \emph{Siam Journal on Numerical Analysis}, vol.~7, pp. 1--46, 1970.

\bibitem{cape2019two}
J.~Cape, M.~Tang, and C.~E. Priebe, ``The two-to-infinity norm and singular
  subspace geometry with applications to high-dimensional statistics,''
  \emph{Annals of Statistics}, vol.~47, pp. 2405--2439, 2019.

\bibitem{lei2019unified}
L.~Lei, ``Unified two-to-infinity eigenspace perturbation theory for symmetric
  random matrices,'' 2019, arXiv preprint \#1909.04798.

\bibitem{mikolov2013distributed}
T.~Mikolov, I.~Sutskever, K.~Chen, G.~S. Corrado, and J.~Dean, ``Distributed
  representations of words and phrases and their compositionality,'' in
  \emph{Advances in Neural Information Processing Systems}, 2013, pp.
  3111--3119.

\bibitem{Tomas2013ICLR}
T.~Mikolov, K.~Chen, G.~S. Corrado, and J.~Dean, ``Efficient estimation of word
  representations in vector space,'' 2013, arXiv preprint \#1301.3781.

\bibitem{holland1983stochastic}
P.~W. Holland, K.~B. Laskey, and S.~Leinhardt, ``Stochastic blockmodels: First
  steps,'' \emph{Social Networks}, vol.~5, pp. 109--137, 1983.

\bibitem{abbe2017community}
E.~Abbe, ``Community detection and stochastic block models: recent
  developments,'' \emph{Journal of Machine Learning Research}, vol.~18, pp.
  6446--6531, 2017.

\bibitem{bickel2009nonparametric}
P.~J. Bickel and A.~Chen, ``A nonparametric view of network models and
  newman--girvan and other modularities,'' \emph{Proceedings of the National
  Academy of Sciences}, vol. 106, pp. 21\,068--21\,073, 2009.

\bibitem{chi2019nonconvex}
Y.~Chi, Y.~M. Lu, and Y.~Chen, ``Nonconvex optimization meets low-rank matrix
  factorization: An overview,'' \emph{IEEE Transactions on Signal Processing},
  vol.~67, pp. 5239--5269, 2019.

\bibitem{sgd_convergence}
B.~Fehrman, B.~Gess, and A.~Jentzen, ``Convergence rates for the stochastic
  gradient descent method for non-convex objective function,'' \emph{Journal of
  Machine Learning Research}, vol.~21, 2020.

\bibitem{bickel13:_hypot}
P.~Sarkar and P.~J. Bickel, ``Hypothesis testing for automated community
  detection in networks,'' \emph{Journal of the Royal Statistical Association,
  Series B}, vol.~78, pp. 253--273, 2016.

\bibitem{lei2014}
J.~Lei, ``A goodness-of-fit test for stochastic block models,'' \emph{Annals of
  Statistics}, vol.~44, pp. 401--424, 2016.

\bibitem{karrer2011stochastic}
B.~Karrer and M.~E. Newman, ``Stochastic blockmodels and community structure in
  networks,'' \emph{Physical Review E}, vol.~83, no.~1, p. 016107, 2011.

\bibitem{zhao2012consistency}
Y.~Zhao, E.~Levina, and J.~Zhu, ``Consistency of community detection in
  networks under degree-corrected stochastic block models,'' \emph{Annals of
  Statistics}, vol.~40, pp. 2266--2292, 2012.

\bibitem{gao2018community}
C.~Gao, Z.~Ma, A.~Y. Zhang, and H.~H. Zhou, ``Community detection in
  degree-corrected block models,'' \emph{Annals of Statistics}, vol.~46, pp.
  2153--2185, 2018.

\bibitem{scrucca2016mclust}
L.~Scrucca, M.~Fop, T.~B. Murphy, and A.~E. Raftery, ``Mclust 5: clustering,
  classification and density estimation using gaussian finite mixture models,''
  \emph{The R Journal}, vol.~8, no.~1, p. 289, 2016.

\bibitem{fraley2002model}
C.~Fraley and A.~E. Raftery, ``Model-based clustering, discriminant analysis,
  and density estimation,'' \emph{Journal of the American Statistical
  Association}, vol.~97, no. 458, pp. 611--631, 2002.

\bibitem{zachary1977information}
W.~W. Zachary, ``An information flow model for conflict and fission in small
  groups,'' \emph{Journal of Anthropological Research}, vol.~33, no.~4, pp.
  452--473, 1977.

\bibitem{adamic2005political}
L.~A. Adamic and N.~Glance, ``The political blogosphere and the 2004 us
  election: divided they blog,'' in \emph{Proceedings of the 3rd international
  workshop on Link discovery}, 2005, pp. 36--43.

\bibitem{modell2021spectral}
A.~Modell and P.~Rubin-Delanchy, ``Spectral clustering under degree
  heterogeneity: a case for the random walk laplacian,'' 2021, arXiv preprint
  \#2105.00987.

\bibitem{ng_jordan_weiss}
A.~Ng, M.~I. Jordan, and Y.~Weiss, ``On spectral clustering: analysis and an
  algorithm,'' in \emph{Advances in Neural Information Processing Systems},
  2001, pp. 849--856.

\bibitem{gao_optimal}
C.~Gao, Z.~Ma, A.~Y. Zhang, and H.~H. Zhou, ``Achieving optimal
  mis-classification proportion in stochastic blockmodels,'' \emph{Journal of
  Machine Learning Research}, vol.~18, pp. 1--45, 2017.

\bibitem{lyzinski2014perfect}
V.~Lyzinski, D.~L. Sussman, M.~Tang, A.~Athreya, and C.~E. Priebe, ``Perfect
  clustering for stochastic blockmodel graphs via adjacency spectral
  embedding,'' \emph{Electronic Journal of Statistics}, vol.~8, pp. 2905--2922,
  2014.

\bibitem{bordenave2015non}
C.~Bordenave, M.~Lelarge, and L.~Massouli{\'e}, ``Non-backtracking spectrum of
  random graphs: community detection and non-regular ramanujan graphs,'' in
  \emph{Proceedings of the 56th IEEE Annual Symposium on Foundations of
  Computer Science}, 2015, pp. 1347--1357.

\bibitem{yu2015useful}
Y.~Yu, T.~Wang, and R.~J. Samworth, ``A useful variant of the {D}avis-{K}ahan
  theorem for statisticians,'' \emph{Biometrika}, vol. 102, pp. 315--323, 2015.

\bibitem{lu2013spectra}
L.~Lu and X.~Peng, ``Spectra of edge-independent random graphs,''
  \emph{Electronic Journal of Combinatorics}, vol.~20, p. P27, 2013.

\end{thebibliography}
\newpage
\begin{IEEEbiography}[{\includegraphics[width=1in,height=1.25in,clip,keepaspectratio]{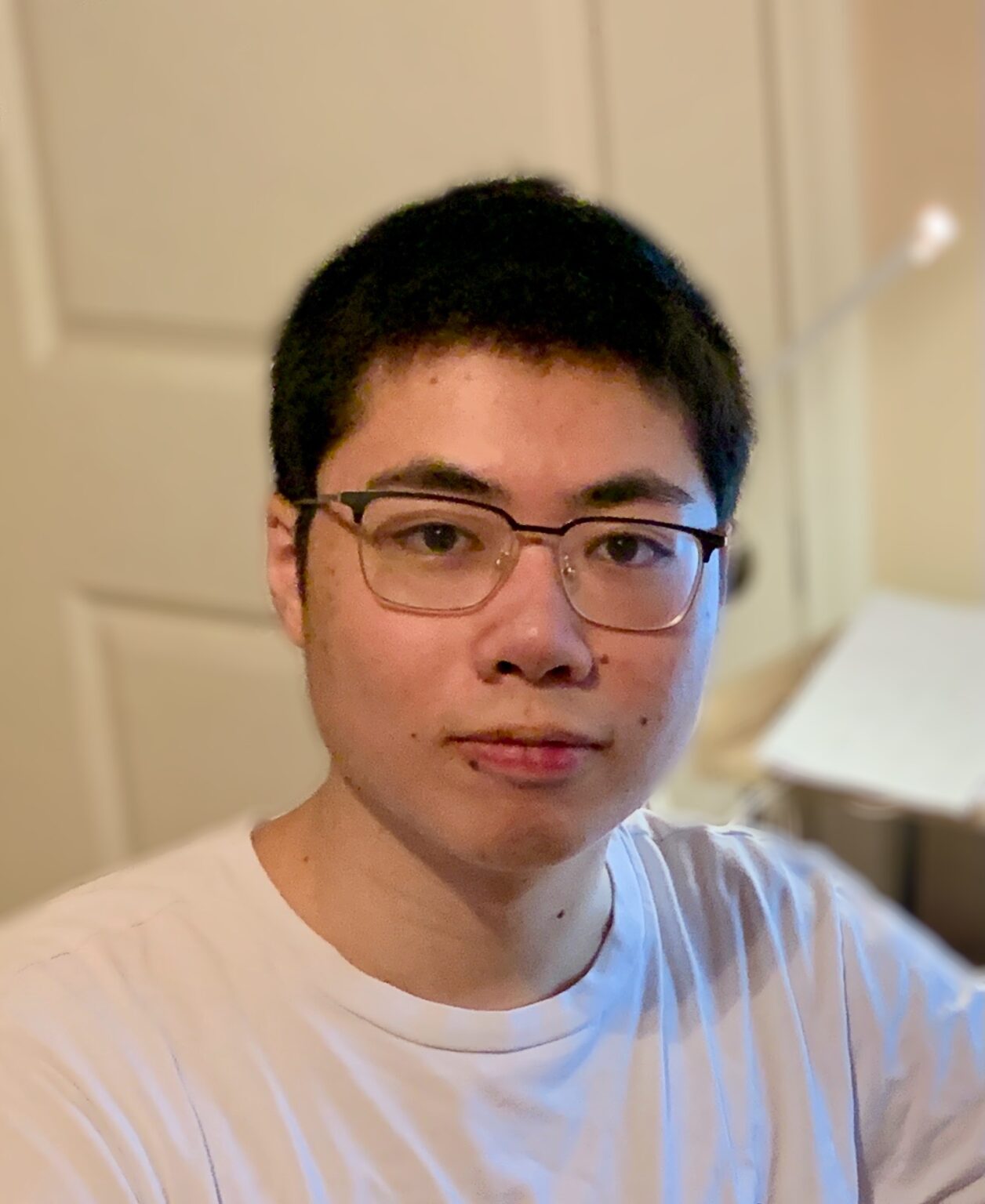}}]{Yichi Zhang}
Yichi Zhang received the BS degree in mathematics from Sichuan University in 2018. He is currently a fifth year PhD student in the Department of Statistics at North Carolina State University. His research interests include statistical inference on random graphs, high-dimensional statistics, and causal inference.
\end{IEEEbiography}
\begin{IEEEbiography}[{\includegraphics[width=1in,height=1.25in,clip,keepaspectratio]{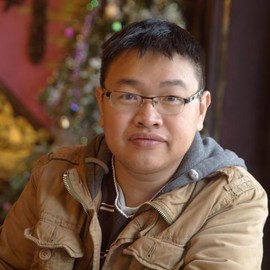}}]{Minh Tang}
    Minh Tang received the BS degree from Assumption University, Thailand, in 2001, the MS degree from the University of Wisconsin, Milwaukee, in 2004, and the PhD degree from Indiana University, Bloomington, in 2010, all in computer science. He is currently an assistant professor in the Department of Statistics at North Carolina State University. His research interests include statistical pattern recognition, dimensionality reduction, and high-dimensional data analysis.
\end{IEEEbiography}
\enlargethispage{-5in}
\newpage

\appendices
\onecolumn
\renewcommand\thesection{\Alph{section}}
\setcounter{section}{0}

\renewcommand{\thefigure}{\Alph{section}\arabic{figure}}
\setcounter{figure}{0}

\renewcommand{\thetable}{\Alph{section}\arabic{table}}
\setcounter{table}{0}

\renewcommand{\theequation}{\Alph{section}.\arabic{equation}}
\setcounter{equation}{0}

\renewcommand{\thelemma}{\Alph{section}\arabic{lemma}}
\setcounter{lemma}{0}

\renewcommand{\thetheorem}{\Alph{section}\arabic{theorem}}
\setcounter{theorem}{0}

\renewcommand{\theremark}{\Alph{section}\arabic{remark}}
\setcounter{remark}{0}

\makeatletter
\@addtoreset{theorem}{section}
\@addtoreset{lemma}{section}
\@addtoreset{remark}{section}
\makeatother

\section{Proofs under the dense regime}\label{app:densepf}
In this section we provide proofs of our main theorems under the dense regime of $\rho_n = \Theta(1)$. This is done for ease of exposition as the proofs for the sparse regime follows the same conceptual ideas and proof structure but with substantially more involved and tedious technical derivations.  
We first list two basic lemmas that will be used repeatedly in the subsequent proofs. Proofs of these lemmas are deferred to
Appendix~\ref{app:lemma}.
\begin{lemma}
  \label{l1}
Let $\bds d = (d_1,\dots,d_n)$ and $\bds p = (p_1,\dots,p_n)$. Then for any integer $t>0$ we have 
\bee
\M 1_{n}\tp \hat{\M W}^t  = \M 1_{n}\tp  {\M W}^t  = \M 1_{n}\tp, \,\, \hat{\M W}^t \bds{d} = \bds{d}, \,\, \M W^t \bds{p} = \bds{p}.
\ee
\end{lemma}

\begin{lemma}
  \label{l:dense:rate}
Under Assumption~\ref{am:s}, we have 
\begin{eqnarray}\nonumber
\ &\|\M D_{\M A}\|_{} = \max_{i} d_i = \op(n),
\\\nonumber
&\|\M D_{\M A}^{-1}\|_{} = \max_{i} 1/d_i = \op(1/n),
\\\label{l4:res:dp}
&\|\M D_{\M A} - \M D_{\M P}\|_{} = \max_{i}{|d_i - p_i|} = \op
                      (n^{1/2} \log^{1/2}{n}),
\\\nonumber
&\|\M D^{-1}_{\M A} - \M D^{-1}_{\M P}\|_{} = \max_{i}{|d_i^{-1} -
              p_i^{-1}|} = \op\Bigl(\frac{\log^{1/2}{n}}{n^{3/2}}\Bigr),
\\\nonumber
&\label{l4:res}
\|\hat{\M W}^{t}\|_{\max} = \op(n^{-1}),  
\\\nonumber
&\max_{i,{i'}} w^{(t)}_{ii'} \asymp 1/n, \, \min_{i,{i'}} w^{(t)}_{ii'} \asymp 1/n,
\end{eqnarray}
for any fixed $t \geq 1$. Here $w^{(t)}_{ii'}$ is the $ii'$th entry of $\M W^t$.
\end{lemma}

\subsection{Proof of Theorem~\ref{T2} (Dense regime)}\label{a:p:t2}
We use the following three steps to bound Eq.~\eqref{T2:oneorder}-Eq.~\eqref{T2main4} in turn.

\noindent\textbf{Step 1 (Bounding $\|\hat{\M W} - \M W\|_{\max}$):}
We start with the decomposition
\bee\label{T1:order_1_0}
\hat{\M W} - \M W &= \M A \M D_{\M A}^{-1} - \M P \M D_{\M P}^{-1} 
= \underbrace{\M A \M D_{\M P}^{-1} \M D_{\M A}^{-1}(\M D_{\M P} - \M D_{\M A})}_{\M \Delta^{(1)}_1} + \underbrace{(\M A - \M P)\M D_{\M P}^{-1}}_{\M \Delta^{(1)}_2}.
\ee
For the first term, we have
\begin{equation*}
\M \Delta^{(1)}_1 = \M A \M D_{\M P}^{-1} \M D_{\M A}^{-1}(\M D_{\M P} - \M D_{\M A}) = \Big[ \Big| \frac{a_{ii'}}{d_{i'} \cdot p_{i'}}(d_{i'} - p_{i'})\Big|\Big]_{n\times n}
\end{equation*}
and hence
\begin{equation*}
\| \M \Delta^{(1)}_1 \|_{\max} = \max_{i,i'} \Big| \frac{a_{ii'}}{d_{i'} \cdot p_{i'}}(d_{i'} - p_{i'})\Big| \precsim \max_{i'}\frac{1}{n}\Big| \frac{d_{i'} - p_{i'}}{d_{i'}}\Big|
\end{equation*}
by $|a_{ii'}| \leq 1$ and $c_0 <p_{ii'}<c_1$. Lemma~\ref{l:dense:rate}
then implies
\begin{equation}
\label{T1:order_1_1}
\begin{split}
\| \M \Delta^{(1)}_1 \|_{\max} &\precsim \frac{1}{n}\max_{i'}|d_{i'} - p_{i'}| \cdot \max_{i'}\frac{1}{d_{i'}} =
\op\big(n^{-3/2}\log^{1/2}{n}\big).
\end{split}
\end{equation}
For the second term we have, by Assumption~\ref{am:s}, that
\bee
\label{T1:order_1_2}
\|\M \Delta^{(1)}_2\|_{\max}  
&= \max_{i,{i'}}\Big| \frac{a_{ii'} - p_{ii'}}{p_{i'}}\Big| \leq \max_{i'} \frac{1}{p_{i'}} = \op\big(n^{-1}\big).
\ee
Combining Eq.~\eqref{T1:order_1_1} and Eq.~\eqref{T1:order_1_2} yields
\bee
\|\hat{\M W} - \M W\|_{\max}& \leq\!\|\M \Delta^{(1)}_1\|_{\max} \!+ \|\M
\Delta^{(1)}_2\|_{\max} 
= \op\big(n^{-1}\big).
\ee

\noindent\textbf{Step 2 (Bounding $\|\hat{\M W}^2 - \M
  W^2\|_{\max,\mathrm{diag}}$, $\|\hat{\M W}^2 - \M
  W^2\|_{\max,\mathrm{off}}$):}
We first decompose $\hat{\M W}^2 - \M W^2$ as
\bee\label{T1:order_2_1}
\hat{\M W}^2 - \M W^2 = \underbrace{(\hat{\M W} - \M W )\M W}_{\M\Delta^{(2)}_1} + \underbrace{\hat{\M W} (\hat{\M W} - \M W)}_{\M\Delta^{(2)}_2}.
\ee
As with Eq~\eqref{T1:order_1_0} we have,
\bee
\M\Delta^{(2)}_1 &= (\M A \M D_{\M A}^{-1} - \M P \M D_{\M P}^{-1})\M W 
\\
&= \{\M A \M D_{\M P}^{-1} \M D_{\M A}^{-1}(\M D_{\M P} - \M D_{\M A}) + (\M A - \M P)\M D_{\M P}^{-1}\}\M W 
\\
&= \underbrace{\{\M A \M D_{\M P}^{-1} \M D_{\M A}^{-1}(\M D_{\M P} - \M D_{\M A}) - \M A \M D_{\M P}^{-2}(\M D_{\M P} - \M D_{\M A})\}\M W}_{\M \Delta^{(2,1)}_{1}} 
+ \underbrace{\M A \M D_{\M P}^{-2}(\M D_{\M P} - \M D_{\M A})}_{\M \Delta^{(2,2)}_{1}}\M W +  \underbrace{(\M A - \M P)\M D_{\M P}^{-1} \M W}_{\M \Delta^{(2,3)}_{1}}.
\ee
The $ii'$th element of $\M \Delta^{(2,1)}_{1}$ is given by
\bee
&\sum_{i^* = 1}^{n} \frac{a_{ii^*}(p_{i^*} - d_{i^*})^2}{p_{i^*}^2 d_{i^*}}\cdot w_{i^{*}{i'}}.
\ee
We therefore have, by Assumption~\ref{am:s} and Lemma~\ref{l:dense:rate},
\bee
\|\M \Delta^{(2,1)}_{1}\|_{\max}
&\leq \max_{i,i'} \sum_{i^* = 1}^{n} \frac{a_{ii^*}(p_{i^*} - d_{i^*})^2}{p_{i^*}^2 d_{i^*}} w_{i^{*}{i'}} 
\leq n \max_{i}|p_i - d_i|^2 \BL\max_{i}\frac{1}{p_i^2 d_i}\BR(\max_{i,{i'}}w_{ii'})
= \op(n^{-2}\log n).
\ee
Similarly, for $\M \Delta^{(2,2)}_{1}$ we have
\bee\label{T1:d12}
\|\M \Delta^{(2,2)}_{1}\|_{\max} 
&= \max_{i,{i'}} \Big|\sum_{i^* = 1}^{n}\frac{a_{ii^*}}{p_{i^*}^2}(p_{i^*} - d_{i^*})w_{i^{*}{i'}}\Big|
\leq n(\max_{i} |p_i - d_i|) \cdot \max_{i,{i'}} w_{ii'} \cdot
\BL \max_{i}\frac{1}{p_i^2}\BR
=\op\big(n^{-3/2}\log^{1/2}{n}\big).
\ee
We now consider the term $\M \Delta^{(2,3)}_{1}$. We have
\bee
\|\M \Delta^{(2,3)}_{1}\|_{\max} 
&= \max_{i,{i'}}\Big|\sum_{i^* = 1}^n\frac{(a_{ii^*} - p_{ii^*})}{p_{i^*}}w_{i^{*}{i'}} \Big| 
\ee
Assumption~\ref{am:s} and Lemma~\ref{l:dense:rate} then imply
\begin{equation*}
\max_{i',i^*} w_{i^{*}{i'}}/{p_{i^*}} \asymp n{-2},\,\, \text{and} \,\, \min_{i'i^*}w_{i^{*}{i'}}/{p_{i^*}} \asymp n^{-2}.
\end{equation*}
and hence, by Bernstein inequality, we have
\bee\label{T1:d22}
\max_{i,{i'}}\Big|\sum_{i^* = 1}^n\frac{(a_{ii^*} - p_{ii^*})}{p_{i^*}}w_{i^{*}{i'}}\Big| = \op(n^{-3/2} \log^{1/2}{n}).
\ee
Combining the above bounds we obtain
\bee\label{T1:order_2_2}
\|\M \Delta_{1}^{(2)}\|_{\max} &\leq \|\M \Delta_{1}^{(2,1)}\|_{\max} + \|\M \Delta_{1}^{(2,2)}\|_{\max} + \|\M \Delta_{1}^{(2,3)}\|_{\max} 
= \op
\big(n^{-3/2}\log^{1/2}{n}\big).
\ee
We now consider the term $\M \Delta^{(2)}_2 = \hat{\M W}(\hat{\M W} -
\M W)$. We have
\bee
\|\M \Delta^{(2)}_2\|_{\max} &= \|(\M W - \hat{\M W})^2 - \M W(\M W - \hat{\M W})\|_{\max} 
\leq   \|(\M W - \hat{\M W})^2\|_{\max} + \|\M W(\M W - \hat{\M W})\|_{\max}.
\ee
The same argument for bounding $\|\M \Delta_2^{(1)}\|_{\max}$ as given above also yields
$$\|\M W(\M W - \hat{\M W})\|_{\max}  = \op(n^{-3/2}\sqrt{\log
  n}).$$
We then bound $\|(\M W - \hat{\M W})^2\|_{\max}$ through
the following expansion
\begin{equation}
\label{t:w:w-w2}
\begin{split}
\|(\M W - \hat{\M W})^2\|_{\max}
&= \max_{i,i'}\Big|\sum_{i^* = 1}^{n}\Big(\frac{p_{ii^*}}{p_{i^*}} -
\frac{a_{ii^*}}{d_{i^*}}\Big)
\Big(\frac{p_{i^*i'}}{p_{i'}} - \frac{a_{i^{*}{i'}}}{d_{i'}}\Big)\Big|
\\
&=\max_{i,i'}\Big|\sum_{i^* = 1}^{n}\Big\{\BL\frac{p_{ii^*}}{p_{i^*}} - \frac{a_{ii^*}}{p_{i^*}} \BR + \BL \frac{a_{ii^*}}{p_{i^*}} - \frac{a_{ii^*}}{d_{i^*}}
\BR\Big\}
\Big\{\BL\frac{p_{i^*i'}}{p_{i'}} - \frac{a_{i^{*}{i'}}}{p_{i'}} \BR + \BL \frac{a_{i^{*}{i'}}}{p_{i'}} - \frac{a_{i^{*}{i'}}}{d_{i'}}\BR\Big\}\Big|
\\
&\leq  \underbrace{\max_{i,{i'}}\Big|\sum_{i^{*}= 1}^n \BL \frac{a_{ii^*}}{p_{i^*}} - \frac{a_{ii^*}}{d_{i^*}}
\BR\BL\frac{p_{i^*i'}}{p_{i'}} - \frac{a_{i^{*}{i'}}}{p_{i'}} \BR\Big|}_{\delta^{(2,1)}_{2}} 
+\underbrace{\max_{i,{i'}}\Big|\sum_{i^* = 1}^{n}\Big(\frac{p_{ii^*}}{p_{i^*}} - \frac{a_{ii^*}}{d_{i^*}}\Big)\BL \frac{a_{i^{*}{i'}}}{p_{i'}} - \frac{a_{i^{*}{i'}}}{d_{i'}}\BR\Big|}_{ \delta^{(2,2)}_{2}}
\\
&+\underbrace{\max_{i,{i'}}\Big|\sum_{i^{*}= 1}^{n}\BL\frac{p_{ii^*}}{p_{i^*}} - \frac{a_{ii^*}}{p_{i^*}} \BR\BL\frac{p_{i^*i'}}{p_{i'}} - \frac{a_{i^{*}{i'}}}{p_{i'}} \BR\Big|}_{ \delta^{(2,3)}_{2}}.
\end{split}
\end{equation}
Since $|a_{ii'}| \leq 1$ and $|a_{ii'}-p_{ii'}|\leq 1$, we have
\begin{equation*}
  \begin{split}
\delta^{(2,1)}_2 &\leq  \max_{i^{'}}\BL\sum_{i^* =
  1}^{n}\Big|\frac{1}{p_{i^*}} - \frac{1}{d_{i^*}}\Big|
\frac{1}{p_{i'}} \BR
\precsim n\cdot \frac{1}{n}\max_{i^*}\BL\frac{|p_{i^*} -
  d_{i^*}|}{p_{i^*} d_{i^*}}
\BR 
= \op(n^{-3/2}\log^{1/2}{n}).
\end{split}
\end{equation*}
Similar reasoning also yields
\begin{equation*}
  \begin{split}
 \delta^{(2,2)}_{2}&\leq \max_{i,{i'}}\sum_{i^* = 1}^{n}\Big|\frac{p_{ii^*}}{p_{i^*}} - \frac{a_{ii^*}}{d_{i^*}}\Big| \cdot \Big| \frac{1}{p_{i'}} - \frac{1}{d_{i'}}\Big|
\\
&\leq n \BL \max_{i^*}\frac{1}{p_{i^*}} + \max_{i^*}\frac{1}{d_{i^*}}\BR\cdot \max_{i^{'}}\Big|\frac{1}{p_{i'}} - \frac{1}{d_{i'}} \Big|
\precsim n\cdot \big\{n^{-1} + \op(1/n)\big\} \cdot \op\big(n^{-3/2} \log^{1/2}{n}\big)
  = \op\big(n^{-3/2}\log^{1/2}{n}\big).
\end{split}
\end{equation*}
We now bound $\delta^{(2,3)}_{2}$ by considering the diagonal and
off-diagonal terms respectively. For the diagonal terms we have
\bee\label{t:d:d223}
\max_{i}\sum_{i^{*}= 1}^{n}\!\frac{(a_{ii*} - p_{ii*})^2}{p_{i*} p_i}
& \leq \max_{i} \!\sum_{i^* = 1}^{n}
\frac{1}{p_ip_{i^*}} 
\!=\mathcal{O}(n^{-1})\!
\ee
Eq.~\eqref{T1:order_2_2} together with the above bounds for $\delta_2^{(2,1)}$ through
$\delta_{2}^{(2,3)}$ yield
\bee\label{T22:res1}
\|\hat{\M W}^2 - \M W^2\|_{\max,\text{diag}} 
\leq \|\MD^{(2)}_1\|_{\max} + \|\M W(\M W - \hat{\M W})\|_{\max} 
+ \delta^{(2,1)}_{2}  + \delta^{(2,2)}_{2}  
+\max_{i} \sum_{i^{*}= 1}^{n}\frac{(p_{ii*} - a_{ii*})^2}{p_{i*}
  p_{i}} = \op(n^{-1}).
\ee
We now consider the off-diagonal terms with $i \not = i'$ for
$\delta_{2}^{(2,3)}$. First define
$$\zeta_{i^*}^{ii'} = \frac{1}{p_{i^*} p_{i'}} (p_{ii^*} -
a_{ii^*})(p_{i^*i'} - a_{i^*i'}).$$
We now make an important observation that if $i \not = i'$ then the collection of random variables
$\zeta_{i^*}^{ii'}$ for $i^{*} = 1,2\dots,n$ are {\em independent} mean
$0$ random variables. Indeed, when $i \not = i'$ then $a_{ii*}$ and
$a_{i*i'}$ are independent and hence
\begin{equation*}
\E\zeta^{ii'}_{i^{*}} = \frac{1}{p_{i^*} p_{i'}}\E(p_{ii^*} -
a_{ii^*})
\cdot \E(p_{i^{*}i'} - a_{i^*i'}) = 0.
\end{equation*}
We thus have
\bee\label{T1:off-d}
\max_{i \neq {i'}}\Big|\sum_{i^{*}= 1}^{n}\frac{(p_{ii^*} - a_{ii^*})(p_{i^*i'} - a_{i^{*}{i'}})}{p_{i^*}p_{i'}} \Big|
&= \max_{i \neq
  {i'}}\Big|\sum_{i^{*}= 1}^{n} \zeta_{i*}^{ii'} \Big|
 = \max_{i \not
  = i'} \Bigl|
\sum_{i^* = 1\atop i^* \neq i,{i'}}^{n}\zeta^{ii'}_{i^{*}}
+ \zeta^{ii'}_i + \zeta^{ii'}_{i'} \Bigr|.
\ee
Now fix a pair $\{i,i'\}$ with $i \not = i'$. 
Then by Bernstein inequality, we have, for any $\tilde{t} > 0$,
\bee\label{T1:bern}
\p\Big(\Big|\sum_{i^{*} = 1\atop i^{*} \neq
  i,{i'}}^{n}\zeta^{ii'}_{i^{*}}\Big|>\tilde{t} \Big)\leq 2 \exp\Big(-\frac{\tilde{t}^2}{2\sigma^2_1+\frac{M}{3}\tilde{t}}\Big)
\ee
where the variance proxy $\sigma_1^2$ is bounded as
\begin{equation*}
  \begin{split}
\sigma_1^2 &= \sum_{i^* = 1\atop i^*f\neq i,{i'}}^{n} \text{Var}(\zeta^{ii'}_{i^{*}})  
= \frac{1}{p_{i^*}^2p_{i'}^2}\sum_{i^* = 1\atop i^*\neq i,{i'}}^{n}\big\{ \text{Var}({p_{ii^*}} - {a_{ii^*}}) \cdot \text{Var}({p_{{i'}i^{*}}} - {a_{i'i^*}})\big\} 
\leq \frac{n}{16p_{i^*}^2p_{i'}^2}
\leq\frac{1}{16c_0^4n^3},
\end{split}
\end{equation*}
and $M$ is any constant bigger than $\max_{i^*} |\tilde{\zeta}^{ii'}_{i^{*}}|$. In particular 
\bee\nonumber
|\zeta^{ii'}_{i^{*}}| &= \Big | \BL\frac{p_{ii^*}}{p_{i^*}} - \frac{a_{ii^*}}{p_{i^*}} \BR\BL\frac{p_{i^*i'}}{p_{i'}} - \frac{a_{i^{*}{i'}}}{p_{i'}} \BR\Big |
\leq \frac{1}{c_0^2n^2} 
\ee
and we can take $M = (c_0 n)^{-2}$. Let $\vartheta_n = n^{-3/2} \log^{1/2}{n}$.
Plugging the above bounds for $\sigma_1^2$ and $M$ into
Eq.~\eqref{T1:bern}, we have, for any $C_1 > 0$, 
\begin{equation*}
  \begin{split}
\p\Big(\Big|\sum_{i^* = 1\atop i^* \neq
  i,{i'}}^{n}\tilde{\zeta}^{ii'}_{i^{*}}\Big|>C_1 \vartheta_n \Big)
  &
\leq 2\exp\Big\{-\frac{\big(C_1 \vartheta_n\big)^2}{
  \frac{2}{16c_0^4n^3} +\frac{1}{3c_0^2 n^2}\big(C_1
  \vartheta_n\big)}\Big\}
= 2\exp\Big(-\frac{C_1\log n}{\frac{1}{8c_0^4}+\frac{\log^{1/2}{n}}{n^{1/2}3c_0^2C_1}}\Big)
\precsim n^{-8C_1c_0^4}.
\end{split}
\end{equation*}
Now choose $C_1$ such that $8C_1 c_0^4 > 5$. 
We then have, by a union bound over all pairs $\{i,i'\}$ with $i \not = i'$, that
\bee\nonumber
\p\Big(\max_{i \neq {i'}}\Big|\sum_{i^* = 1\atop i^* \neq
  i,{i'}}^{n}\tilde{\zeta}^{ii'}_{i^{*}}\Big|>C_1
n^{-3/2}\log^{1/2}{n}\Big)
&\precsim n^{2 - 8C_1c_0^4} 
\precsim n^{-3}.
\ee
We thus conclude
\begin{equation*}
  \begin{split}
\max_{i \neq {i'}}\Big|\sum_{i^{*}= 1}^{n}\frac{(p_{ii^*} - a_{ii^*})(p_{i^*i'} - a_{i^{*}{i'}})}{p_{i^*}p_{i'}} \Big|
\leq \max_{i \neq {i'}}\Big|
\sum_{i^* = 1\atop i^* \neq i,{i'}}^{n}\zeta^{ii'}_{i^{*}}
\Big| + 2 M 
= \op(n^{-3/2}\log^{1/2}{n}).
\end{split}
\end{equation*}
A similar argument to Eq.~\eqref{T22:res1} then yields
\bee\label{T22:res2}
\|\hat{\M W}^2 - \M W^2\|_{\max,\text{off}} 
& \leq \|\MD^{(2)}_1\|_{\max} + \|\M W(\M W - \hat{\M W})\|_{\max} + \delta^{(2,1)}_{2} + \delta^{(2,2)}_{2}  
+\max_{i \neq {i'}}\Big|\sum_{i^{*}=
  1}^{n}\BL\frac{p_{ii^*} - a_{ii^*}}{p_{i^*}}
\BR\BL\frac{p_{i^*i'} - a_{i^{*}{i'}}}{p_{i'}} \BR
\Big| 
\\
&=\op(n^{-3/2} \log^{1/2}{n}).
\ee

\noindent\textbf{Step 3 (Bounding $\|\hat{\M W}^{t} - \M W^{t}\|_{\max},  t\geq 3$):}
We first consider $t = 3$. We have  
\bee\label{T2:s3f1}
\|\HW^3 - \M W^3\|_{\max} 
\leq \|\big(\HW^2 - \W^2\big)\HW\|_{\max} + \|\W^2\big(\HW - \W\big)\|_{\max}
\ee
For the first term on the RHS of Eq.~\eqref{T2:s3f1} we have, by
Eq.~\eqref{T22:res1}, Eq.~\eqref{T22:res2} and Lemma~\ref{l:dense:rate}, that
\bee\label{T2:s3f2}
\|(\HW^2 - \W^2)\HW\|_{\max}
\leq n \|\HW\|_{\max}\|\HW^2 - \W^2\|_{\max,\text{off}} 
+ \|\HW\|_{\max}\|\HW^2 - \W^2\|_{\max,\text{diag}}
= \op(n^{-3/2}\log^{1/2}{n}).
\ee
For the second term in the RHS of Eq.~\eqref{T2:s3f1} we use the
same argument as that for bounding $\|(\HW - \W)\W\|_{\max}$ in
{Step 2}. In particular we have
\bee\label{T2:s3f3}
\|\W^2(\HW - \W)\|_{\max} = 
\op(n^{-3/2} \log^{1/2}{n}).
\ee
Combining Eq.~\eqref{T2:s3f1} through Eq.~\eqref{T2:s3f3} yields
\bee\label{T2:s3res1}
\|\HW^3 - \W^3\|_{\max} = 
\op(n^{-3/2} \log^{1/2}{n}).
\ee
The case when $t = 4$ is analogous. More specifically,
\bee\label{T2:s3res2}
\|\HW^4 - \W^4\|_{\max}
&= \|(\HW^2 - \W^2)\HW^2\|_{\max} + \|\W^2(\HW^2 - \W^2)\|_{\max}
\\ &\leq n\|\HW^2 - \W^2\|_{\max,\text{off}}(\|\HW^2\|_{\max}+\|\W^2\|_{\max}) 
+\|\HW^2 - \W^2\|_{\max,\text{diag}}(\|\HW^2\|_{\max}+\|\W^2\|_{\max}) 
\\
& = \op(n^{-3/2}\log^{1/2}{n})
\ee
We now consider a general $t \geq 5$. We start with the decomposition
\bee\nonumber
\hat{\M W}^{t} - \M W^{t} 
&= (\hat{\M W}^{t} - \M W^2\hat{\M W}^{t -2}) + (\M W^2 \hat{\M W}^{t - 2} - \M W^{t})
= \underbrace{(\HW^2 - \W^2)\HW^{t -2}}_{\M\Delta^{(3)}_1} + \underbrace{\W^2(\HW^{t - 2} - \W^{t - 2})}_{\M\Delta^{(3)}_2},
\ee 
We now have, by Lemma~\ref{l:dense:rate},
Eq.~\eqref{T22:res1}, and Eq.~\eqref{T22:res2}, that
\bee\label{T23:res1}
\|\M \Delta^{(3)}_1\|_{\max} 
&\leq n \|\HW^2 - \W^2\|_{\max,\text{off}}\|\HW^{t -2}\|_{\max} 
+ \|\HW^2 - \W^2\|_{\max,\text{diag}}\|\HW^{t -2}\|_{\max}  
= \op(n^{-3/2}\log^{1/2}{n}).
\ee
Once again, by Lemma~\ref{l:dense:rate}, we have
\bee\label{T23:res2}
\|\M \Delta^{(3)}_2\|_{\max} &\leq n \|\M W^2\|_{\max}\|\hat{\M W}^{t - 2} - \M W^{t - 2}\|_{\max} 
\precsim \|\hat{\M W}^{t - 2} - \M W^{t - 2}\|_{\max}.
\ee
Combining Eq.~\eqref{T23:res1} and Eq.~\eqref{T23:res2}, we have
\bee\label{T23:ite}
&\|\HW^{t} - \W^{t}\|_{\max} 
\precsim \op(n^{-3/2}\log^{1/2}{n}) + \|\hat{\M W}^{t - 2} - \M W^{t - 2}\|_{\max}
\ee
As $t$ is finite, iterating the above argument yields
\bee\nonumber
\|\HW^{t} - \W^{t}\|_{\max} 
\precsim \begin{cases}
\op(n^{-3/2}\log^{1/2}{n}) + \|\hat{\M W}^{4} - \M
W^{4}\|_{\max} &
 \text{(when $t$ is even)} \\
\op(n^{-3/2} \log^{1/2}{n}) + \|\hat{\M W}^3 - \M W^3\|_{\max} 
&
 \text{(when $t$ is odd)}
\end{cases}
\ee
Recalling Eq.~\eqref{T2:s3res1} and Eq.~\eqref{T2:s3res2}, we conclude that
\bee\nonumber
\|\HW^{t} - \W^{t}\|_{\max} = \op(n^{-3/2} \log^{1/2}{n}) \quad \text{for all $t \geq 3$}
\ee
as desired. \qed
\subsection{Proof of Theorem~\ref{T3} (Dense regime)}
We will only present the proof of bounding $\|\tilde{\mathbf{M}}_0 -
\mathbf{M}_0 \|_{\F}$ and $\|\tilde{\mathbf{M}}_0 -
\mathbf{M}_0 \|_{\infty}$ here as the rates for $\|\M M_0\|_{\F}$ and $\|\M M_0\|_{\infty}$ are derived
in the proof of Theorem~\ref{c1}, Eq.~\eqref{MMFF}.
Under the dense regime of Assumption~\ref{am:s} and for a sufficiently
large $n$, the diameter of $\mathcal{G}$ is $2$ with high probability,
see e.g., Corollary 10.11 in \cite{bollobas2001random}. Therefore,
with high probability, all entries of $\hat{\M W}^2$ are
positive and hence
\bee\nonumber
\tilde{\mathbf{M}}_0 &= \log \Bigl\{\frac{2|\M A|}{\kappa \gamma}\sum_{t =
  t_L}^{t_U}(L - t)\cdot\Big(\M D_{\M A}^{-1}\hat{\M W}^t\Big)\Bigr\} 
\ee
is well-defined with high probability.
Next recall the definition of $\M M_0$ given in Eq.~\eqref{truth}. We
have
\bee\label{T7:df1}
\tilde{\M M}_0 - \M M_0 &= 
\log\Big\{\underbrace{{|\M A|}{}\sum_{t = t_L}^{t_U}(L - t)\cdot\Big(\M D_{\M A}^{-1}\hat{\M W}^t\Big)}_{\M I_{\M A}}\Big\} 
- \log\Big\{\underbrace{{|\M P|}{}\sum_{t = t_L}^{t_U}(L - t)\cdot\Big(\M D_{\M P}^{-1}{\M W}^t\Big)}_{\M I_{\M P}}\Big\}.
\ee
By the mean value theorem, the absolute value of $ii'$th entry in $\tilde{\M M}_0 - \M M_0$ is 
\bee\nonumber
\frac{1}{\alpha_{ii'}}\underbrace{\Bigg|\sum_{t = t_L}^{t_U}(L -
  t)\cdot\Big\{|\M A|
  \cdot\Big(\frac{\hat{w}_{ii'}^{(t)}}{d_i}\Big)-|\M P| \cdot\Big(\frac{{w}_{ii'}^{(t)}}{p_i}\Big)\Big\}\Bigg|}_{I^{ii'}_{\M A} - I^{ii'}_{\M P}}
\ee
where $\alpha_{ii'}\in (I^{ii'}_{\M A},I ^{ii'}_{\M P})$ and $I^{ii'}_{\M A}$
and $I^{ii'}_{\M P}$ are the $ii'$th entry of $\M I_{\M A}$ and $\M
I_{\M P}$, respectively. We therefore have
\begin{gather*}
\|\tilde{\M M}_0 - \M M_0\|_{\max,\text{off}} \leq\max_{i \not =
  i'}\Big(\frac{1}{\alpha_{ii'}}\Big)\cdot \|\M I_{\M A} - \M I_{\M
  P}\|_{\max,\text{off}}, \\
\|\tilde{\M M}_0 - \M M_0\|_{\max,\text{diag}} \leq\max_{i}\Big(\frac{1}{\alpha_{ii}}\Big)\cdot \|\M I_{\M A} - \M I_{\M P}\|_{\max,\text{diag}}.
\end{gather*}
We now bound $\alpha_{ii'}^{-1}$, $\|\M I_{\M A} - \M I_{\M
  P}\|_{\max,\text{off}}$ and $\|\M I_{\M A} - \M I_{\M
  P}\|_{\max,\text{diag}}$.

\noindent\textbf{Step 1 (Bounding $\max_{ii'} \alpha_{ii'}^{-1}$):} 
As $\alpha_{ii'} \in (I_{\M A}^{ii'}, I_{\M P}^{ii'})$ we have
\bee\label{T3:S11}
\max_{i,i'}\frac{1}{\alpha_{ii'}}\leq \max_{i,i'}\Big\{\frac{1}{I^{ii'}_{\M A}},\frac{1}{I^{ii'}_{\M P}}\Big\}.
\ee
We first bound $(I_{\M P}^{ii'})^{-1}$. In particular
\bee\nonumber
\frac{1}{I_{\M P}^{ii'}} = \frac{1}{\sum_{t = t_L}^{t_U}(L -
  t)|\M P|\cdot\Big(\frac{{w}_{ii'}^{(t)}}{p_i}\Big)}\leq
\frac{p_i}{|\M P|w^{(t_L)}_{ii'}}
\ee
as $L - t \geq 1$ for all $t_L \leq t \leq t_U \leq L - 1$. 
Now there exists constants $c_0$ and $c_1$ such that $c_0\leq p_{ii'}\leq c_1$ for all $(i,i')\in[n]^2$. Therefore, by Lemma~\ref{l:dense:rate}, we have 
\bee\label{T3:S12}
\max_{i,i'}\frac{1}{I^{ii'}_{\M P}} \precsim \frac{n}{n^2\cdot \frac{1}{n}} = 1.
\ee
We then bound $(I^{ii'}_{\M A})^{-1}$. We have
\bee\nonumber
\frac{1}{I^{ii'}_{\M A}} = \frac{1}{\sum_{t = t_L}^{t_U}(L -
  t)|\M A|\cdot\Big(\frac{\hat{w}_{ii'}^{(t)}}{d_i}\Big)} \leq
\frac{d_i}{|\M A| \hat{w}_{ii'}^{(t_L)}}
\ee
First suppose $i \not = i'$. Then by Theorem~\ref{T2}, we have
$$\max_{i \neq {i'}}|w^{(t_L)}_{ii'} - \hat{w}^{(t_L)}_{ii'}| =
\op(n^{-3/2} \sqrt{\log n})$$
which implies, whp, that
\bee\label{pf:d:ii'}
0\leq {w^{(t_L)}_{ii'} - \max_{i \neq {i'}}|w^{(t_L)}_{ii'} - \hat{w}^{(t_L)}_{ii'}|} & & \text{for all } 
i \neq {i'}.
\ee
 As $\min_{i,i'}w^{(t_L)}_{ii'} \asymp \max_{i,i'}w^{(t_L)}_{ii'}\asymp n^{-1}$ we also have
 $\hat{w}_{ii'}^{(t_L)} \asymp n^{-1}$ for all $i \not = i'$. 
 Hence, by Lemma~\ref{l:dense:rate}, we have
\bee\label{T3:S13}
\max_{i \neq {i'}}\frac{1}{I^{ii'}_{\M A}} &\leq \max_{i\neq
  i'}\frac{d_i}{|\M A| \big\{\hat{w}^{(t_L)}_{ii'}}
\precsim \max_i d_i \cdot \max_{i} \frac{1}{n d_i} \cdot \max_{i \neq {i'}}\frac{1}{\hat{w}^{(t_L)}_{ii'}}
= \op(1).
\ee
Now suppose that $i = i'$. Then for $t_L = 2$, we have
\bee\nonumber
\frac{1}{\hat{w}_{ii}^{(2)}} = \frac{1}{\sum_{i' =
    1}^{n}\hat{w}_{ii'}\hat{w}_{i'i}} &=
\frac{1}{\sum_{i' = 1}^n\frac{a_{ii'}}{d_{i'}}\frac{a_{i'i}}{d_i}} 
= \frac{d_i}{\sum_{i' = 1}^n a_{ii'}/d_{i'}}.
\ee
Once again, by Lemma~\ref{l:dense:rate}, we have
\bee\nonumber
\max_{i}\frac{1}{\hat{w}_{ii}^{(2)}} &\leq \frac{\max_i d_i}{\frac{1}{\max_i
    d_i}\sum_{i' = 1}^n a_{ii'}}
    \leq {(\max_i d_i)^2}  \cdot\max_i
    \frac{1}{d_i} = \op(n), 
\ee
and hence
\bee
\nonumber \max_{i}\frac{1}{I^{ii}_{\M A}}&\precsim \max_{i}d_i \cdot \max\frac{1}{nd_i}  \cdot \max_i\frac{1}{\hat{w}_{ii}^{(2)}} = \op(1).
\ee
If $t_L \geq 3$, then 
$\max_{i}|{w}_{ii}^{(t_L)} - \hat{w}_{ii}^{(t_L)}| =
\op\big(n^{-3/2}\log^{1/2}{n}\big)$
(see Theorem~\ref{T2}) and hence, using the same argument as that for deriving
Eq.~\eqref{T3:S13}, we also have $\max_{i}({I^{ii}_{\M A}})^{-1} =
\op(1)$. In summary, we have 
\bee\label{T3:S1f}
\max_{i,i'}\frac{1}{\alpha_{ii'}} = \op(1).
\ee

\noindent\textbf{Step 2 (Bounding $\|\M I_{\M A} - \M I_{\M P}\|_{\max,\mathrm{off}}$):} 
We start with the inequality
\bee\label{T2:S22}
\max_{i \neq {i'}}|I_{\M A}^{ii'} - I_{\M P}^{ii'}| 
&= \max_{i \neq
  {i'}}\Big|\sum_{t = t_L}^{t_U}(L - t)\Big\{|\M
A|\Big(\frac{\hat{w}_{ii'}^{(t)}}{d_i}\Big)-|\M P|\Big(\frac{{w}_{ii'}^{(t)}}{p_i}\Big)\Big\}\Big|
\leq\sum_{t = t_L}^{t_U} (L - t) \max_{i \neq
  {i'}}\Bigl||\M A|\Big(\frac{\hat{w}_{ii'}^{(t)}}{d_i}\Big)-|\M P| \Big(\frac{{w}_{ii'}^{(t)}}{p_i}\Big)\Big|.
\ee
We now bound each of the summand in the RHS of the above
display. Consider a fixed value of $t \geq 2$. 
We have
\bee\label{T3:S21}
&\max_{i \neq {i'}}\Big||\M
A|\frac{\hat{w}_{ii'}^{(t)}}{d_i}- |\M P|
\frac{{w}_{ii'}^{(t)}}{p_i}\Big| 
\leq (|\M A| - |\M P|) \cdot \max_{i \neq
  {i'}}\Big| \frac{w^{(t)}_{ii'}}{d_i}\Big| + |\M P| \max_{i \neq
  {i'}}\Big| \frac{(p_i - d_i) w^{(t)}_{ii'}}{p_i d_i} \Big| + |\M A| \max_{i \neq {i'}}\Big|\frac{\hat{w}_{ii'}^{(t)} - w_{ii'}^{(2)}}{d_i}\Big|.
\ee
By Lemma~\ref{l:dense:rate}, the first term in RHS of Eq.~\eqref{T3:S21} is bounded as
\bee\label{T3:S23}
&(|\M A| - |\M P|) \max_{i \neq {i'}} \Big|
\frac{w^{(t)}_{ii'}}{d_i}\Big| 
\leq n\max_{i} |d_i - p_i|
\cdot \max_{i,{i'}} w_{ii'}^{(t)} \cdot \max_{i}\frac{1}{d_i} 
=
\op(n^{-1/2}\log^{1/2}{n}).
\ee
The second term in the RHS of Eq.~\eqref{T3:S21} is also bounded by
Lemma~\ref{l:dense:rate} as
\bee\label{T3:S24}
|\M P|\max_{i \neq {i'}}\Big|\frac{(p_i - d_i)w^{(t)}_{ii'}}{p_i d_i} \Big| & \leq |\M P|\cdot \max_{i} \Big|\frac{p_i - d_i}{p_id_i}\Big|\cdot \max_{i,i'} w^{(t)}_{ii'}
= \op(n^{-1/2}\log^{1/2}{n})
\ee
The third term is bounded by Theorem~\ref{T2} and Lemma~\ref{l:dense:rate} as
\bee\nonumber
|\M A|  \max_{i \neq {i'}}\Big|\frac{\hat{w}_{ii'}^{(t)} -
  w_{ii'}^{(t)}}{d_i}\Big|
 \leq |\M A| \max_{\i \neq {i'}}|\hat{w}_{ii'}^{(t)} - w^{(t)}_{ii'}|\cdot \max_{i}\frac{1}{d_i}
& = 
\op\big(n^{-1/2}\log^{1/2}{n}\big).
\ee
Collecting the above terms and summing over the finite values of ${\color{black}t_L \leq t \leq t_U}$, we obtain
\bee\label{T3:S2f}
\|\M I_{\M A} - \M I_{\M P}\|_{\max,\text{off}} =
\op\big(n^{-1/2}\log^{1/2}{n}\big).
\ee

\noindent\textbf{Step 3 (Bounding $\|\M I_{\M A} - \M I_{\M P}\|_{\max,\mathrm{diag}}$):} 
With a similar argument as {Step 2}, we consider 
\bee\label{T3:S33}
\|\M I_{\M A} - \M I_{\M P}\|_{\max,\text{diag}} 
\leq  \sum_{t =
  2}^{t_U} (L - t)\max_{i}\Big||\M A|\Big(\frac{\hat{w}_{ii}^{(t)}}{d_i}\Big)- |\M P|\Big(\frac{{w}_{ii}^{(t)}}{p_i}\Big)\Big|.
\ee
We once again bound each summand in the above display. Similar to Eq.~\eqref{T3:S21}, we have
\bee\label{T3:S31}
&\max_{i}\Big||\M A|  \frac{\hat{w}_{ii'}^{(t)}}{d_i}-
|\M P|  \frac{{w}_{ii'}^{(t)}}{p_i}\Bigr| 
\leq (|\M A| - |\M
P|)  \max_{i}\Big| \frac{w^{(t)}_{ii}}{d_i}\Big|
+ |\M P|  \max_{i}\Big|\frac{(p_i - d_i)w^{(t)}_{ii}}{p_i d_i}
\Big| + |\M A|  \max_{i}\Big|\frac{\hat{w}_{ii}^{(t)} - w_{ii}^{(t)}}{d_i}\Big|.
\ee
The first two terms in the RHS of Eq.~\eqref{T3:S31} is bounded via
$\op(n^{-1/2}\log^{1/2}{n})$; see the arguments for
  Eq.~\eqref{T3:S23} and Eq.~\eqref{T3:S24}. For the third term, we
  consider the cases $t = 2$ and $t > 2$ separately. For $t = 2$, we have
\bee\nonumber
|\M A| \max_{i}\Bigl|\frac{\hat{w}_{ii}^{(2)} -
  w_{ii}^{(2)}}{d_i}\Big| &\leq |\M A| \max_{i}|\hat{w}_{ii}^{(2)} - w^{(2)}_{ii}|\cdot \max_{i}\frac{1}{d_i}
= \op(1).
\ee
In contrast, for $t > 2$, we have
\bee\nonumber
|\M A| \max_{i}\Bigl|\frac{\hat{w}_{ii}^{(t)} -
  w_{ii}^{(t)}}{d_i}\Big| = 
\op\big(n^{-1/2} \log^{1/2}{n}\big)
\ee
Combining the above terms, we have
\bee\label{T3:S3f}
\|\M I_{\M A} - \M I_{\M P}\|_{\max,\text{diag}} = \begin{cases}
  \op(1) & \text{if $t_L = 2$}, \\
  \op(n^{-1/2} \log^{1/2}{n}) & \text{if $t_L \geq 3$}. \end{cases}
\ee
\textbf{Step 4 (Bounding $\|\tilde{\M M}_0 - \M M_0\|_{\F}$):}
In summary, Eq.~\eqref{T3:S1f}, Eq.~\eqref{T3:S2f} and Eq.~\eqref{T3:S3f} imply
\bee
\|\tilde{\M M}_0 - \M M_0\|_{\max,\text{off}} &\leq\max_{i \not = i'} \Big(\frac{1}{\alpha_{ii'}}\Big)\cdot \|\M I_{\M A} -
\M I_{\M P}\|_{\max,\text{off}} 
= \op(n^{-1/2}\log^{1/2}{n}), 
\ee
\bee
\|\tilde{\M M}_0 - \M M_0\|_{\max,\text{diag}}
&\leq\max_{i}\Big(\frac{1}{\alpha_{ii}}\Big)\cdot \|\M I_{\M A} - \M
I_{\M P}\|_{\max,\text{diag}} 
= \op(1).
\ee
We thus conclude that
\bee
\|\tilde{\M M}_0-\M M_0\|_{\F} 
&\leq \BL n^2 \|\tilde{\M M}_0 -
\M M_0\|^2_{\max,\text{off}} + n\cdot \|\tilde{\M M}_0 - \M
M_0\|^2_{\max,\text{diag}}  \BR^{1/2} 
= \op(n^{1/2}\log^{1/2}{n}),
\\
\|\tilde{\M M}_0-\M M_0\|_{\infty} 
&\leq n \|\tilde{\M M}_0 -
\M M_0\|_{\max,\text{off}} +  \|\tilde{\M M}_0 - \M
M_0\|_{\max,\text{diag}}  
= \op(n^{1/2}\log^{1/2}{n}).
\ee
 as desired.
 \qed.


\subsection{Proof of Theorem~\ref{c1} (Dense Regime)}
Recall that $n_k$ denote the number of vertices assigned to block $k$. 
For ease of exposition we shall assume that
\bee\label{t:t:1}
n_k = n\pi_{k}, \quad\text{ for all }k\in [K].
\ee
The case where the $n_k$ are random with $\mathbb{E}[n_k] = n \pi_k$ is, except for a few slight
modifications, identical.
Note that $\tilde{\M M}_0$ and $\M M_0$ are symmetric matrices. Indeed
\bee\nonumber
(\M D_{\M A}^{-1}\hat{\M W}^t)^\T = (\M D_{\M A}^{-1}\M A\M D_{\M
  A}^{-1}\cdots \M D_{\M A}^{-1}\M A\M D_{\M A}^{-1})^\T = \M D_{\M
  A}^{-1}\hat{\M W}^t
\ee
and similarly for $(\M D_{\M P}^{-1} \M W^{t})^{\top} = (\M D_{\M P}^{-1} \M W^{t})$. Recall from Eq.~\eqref{theta} that $\M P = \M Z \M B \M Z^{\top}$ has a $K \times K$ blocks structure. Now let $\M B = \M V \bm{\Upsilon} \M V^{\top}$ be the eigendecomposition of $\M B$ where $\M V$ is a $K \times \mathrm{rank}(\mathbf{B})$ matrix with orthornormal columns, i.e., $\M V^{\top} \M V = \mathbf{I}$ and $\bm{\Upsilon}$ is a diagonal matrix containing the {\em non-zero} eigenvalues of $\mathbf{B}$. 
We then have, for any $t \geq 1$, 
\begin{equation}
\label{eq:pfc1_a}
    \M D_{\M P}^{-1} \M W^{t} = \M D_{\M P}^{-1} (\M Z \M B \M Z^{\top} \M D_{\M P}^{-1}\bigr)^{t} = 
\M D_{\M P}^{-1} (\M Z \M V \bm{\Upsilon} \M V^{\top} \M Z^{\top} \M D_{\M P}^{-1}\bigr)^{t} = 
\M D_{\M P}^{-1} \M Z \M V (\bm{\Upsilon} \M V^{\top} \M Z^{\top} \M D_{\M P}^{-1} \M Z \M V)^{t-1} \bm{\Upsilon} \M V^{\top} \M Z^{\top} \M D_{\M P}^{-1}
\end{equation}
Now $\M Z \M D_{\M P}^{-1} \M Z$ is a $K \times K$ {\em diagonal} matrix and is positive definite. Therefore 
$\M V^{\top} \M Z^{\top} \M D_{\M P}^{-1} \M Z \M V$ is also positive definite. Let $\mathcal{D} := \bm{\Upsilon} \M V^{\top} \M Z^{\top} \M D_{\M P}^{-1} \M Z \M V$. Note that $\mathcal{D}$ is invertible and all of its eigenvalues are real-valued as it is similar to the matrix $(\M V^{\top} \M Z^{\top} \M D_{\M P}^{-1} \M Z \M V)^{1/2} \bm{\Upsilon} (\M V^{\top} \M Z \M D_{\M P}^{-1} \M Z \M V)^{1/2}$. We then have
$$\M M_0' := \frac{2|\M P|}{\kappa \gamma}\sum_{t = t_L}^{t_U}(L - t) \M D_{\M P}^{-1}{\M W}^t = \frac{2|\M P|}{\kappa \gamma} 
\M D_{\M P}^{-1} \M Z \M V \underbrace{\Bigl(\sum_{t=t_L}^{t_U} (L-t) \mathcal{D}^{t-1} \Bigr)}_{f(\mathcal{D})} \bm{\Upsilon} \M V^{\top} \M Z^\top \M D_{\M P}^{-1}.$$
Recall that $t_L \leq t_U \leq L$. 
Then $f(\mathcal{D})$ is singular if and only if there is an eigenvalue $\lambda$ of $\mathcal{D}$ such that $\sum_{t=t_L}^{t_U} (L - t) \lambda^{t-1} = 0$. As the set of roots for any fixed polynomial equation has measure $0$, we conclude that $f(\mathcal{D})$ is invertible for almost every $\mathcal{D}$, or equivalently that the set of matrix $\mathbf{B}$ whose induced $f(\mathcal{D})$ is singular has Lebesgue measure $0$ in the space of $K \times K$ symmetric matrices. Note that $f(\mathcal{D})$ is guaranteed to be invertible whenever $\mathbf{B}$ is positive semidefinite. 

We thus assume, without loss of generality, that $f(\mathcal{D})$ is invertible. Then $f(\mathcal{D}) \bm{\Upsilon}$ is also invertible and hence
$\mathbf{V} f(\mathcal{D}) \bm{\Upsilon} \mathbf{V}^{\top}$ is a symmetric $K \times K$ matrix with {\em distinct} rows. Indeed, suppose that there exists two rows $k$ and $k'$ of 
$\mathbf{V} f(\mathcal{D}) \bm{\Upsilon} \mathbf{V}^{\top}$ that are the same. Let $V_k$ and $V_{k'}$ be the $k$th and $k'$th row vectors for $V$. Then the $k$th and $k$th rows of $\mathbf{V} f(\mathcal{D}) \bm{\Upsilon} \mathbf{V}^{\top}$ are $V_k f(\mathcal{D}) \bm{\Upsilon} \mathbf{V}^{\top}$ and $V_{k'} f(\mathcal{D}) \bm{\Upsilon} \mathbf{V}^{\top}$. We then have
\begin{equation*}
\begin{split}
V_k f(\mathcal{D}) \bm{\Upsilon} \mathbf{V}^{\top} = V_{k'} f(\mathcal{D}) \bm{\Upsilon} \mathbf{V}^{\top} &\Longrightarrow (V_k - V_{k'}) f(\mathcal{D}) \bm{\Upsilon} \mathbf{V}^{\top} \mathbf{V} \bm{\Upsilon} f(\mathcal{D})^{\top} (V_{k} - V_{k'})^{\top} = 0 \\ &\Longrightarrow (V_k - V_{k'}) f(\mathcal{D}) \bm{\Upsilon}^2 f(\mathcal{D})^{\top} (V_{k} - V_{k'})^{\top} = 0
\end{split}
\end{equation*}
which is only possible if $V_{k} - V_{k'} = 0$
as $f(\mathcal{D}) \bm{\Upsilon}^2 f(\mathcal{D})^{\top}$ is positive definite. As $\mathbf{B}$ has distinct rows, the rows of $\mathbf{V}$ are also distinct and hence $V_{k} - V_{k'} = 0$ if and only if $k = k'$. 
The matrix
$\M M_0'$
thus have the same $K \times K$ block structure as $\M P$. Recall that the entries of $\mathbf{M}_0$ are the logarithm of the entries of $\mathbf{M}_0'$ and hence $\mathbf{M}_0$ is of the form
\bee\nonumber
\M M_0 &= \begin{pmatrix}
\xi_{11} \bds 1_{n\pi_1} \bds 1_{n\pi_1}^\T &\dots & \xi_{1K} \bds 1_{n\pi_1} \bds 1_{n\pi_K}^\T\\
\vdots & \vdots & \vdots\\
\xi_{K1} \bds 1_{n\pi_K} \bds 1_{n\pi_1}^\T & \dots & \xi_{KK} \bds 1_{n\pi_K} \bds 1_{n\pi_K}^\T
\end{pmatrix} 
= 
\M Z  \bm{\Xi} \M Z^\T,
\ee
where $\bm{\Xi} = (\xi_{ii'})_{K
  \times K}$ is a symmetric matrix with distinct rows. The non-zero eigenvalues
of $\M M_0$ coincides with that of
\bee\label{trans}
\M Z^{\top} \M Z \bm{\Xi} = n \mathrm{diag}(\pi_1,
\pi_2, \dots, \pi_K) \bm{\Xi}.
\ee As $\text{diag}(\pi_1,\dots,\pi_K)$ is fixed, it is left to study $\bm{\Xi}$. Since the entries in $\M P$  are $\Theta(1)$, we have $|\M P| = \Theta(n^2)$, and entries in $\M D_{\M P}^{-1}$ and $\M W^t$ are of order $\Theta(n^{-1})$ (see Lemma \ref{l:dense:rate}). Therefore each entry of $\M M_{0}'$ is of order $\Theta(1)$. Indeed, with some more careful book-keeping, one can show that under Eq.~\eqref{t:t:1} each entry of $\M M_{0}'$  can take on one of $K^2$ possible values (with these values not depending on $n$). Thus $\bm{\Xi}$ is a fixed matrix not depending on $n$
and hence, by Eq.~\eqref{trans},
all non-zero singular values of $\M M_0$ grows at order $n$ since $\text{diag}(\pi_1,\pi_2,\dots,\pi_K)\cdot \bm{\Xi}$ is fixed. In summary we have
\bee\nonumber
\sigma_{d}(\M M_0)=\Theta(n).
\ee 
In addition as $\M M_0$ has rank at most $K$, we have 
\bee\label{MMFF}
\|\M M_0\|_{\F} \asymp n, \quad \|\M M_0\|_{\infty}\asymp n
\ee 
Therefore, by the Davis-Kahan Theorem \cite{davis70,yu2015useful} and Theorem \ref{T3} (dense regime), we have
\bee\label{fnormbound}
\min_{\M T \in \mathbb{O}_{d}} \|\hat{\bds{\mathcal{F}}} \cdot \mathbf{T} - \mathbf{U}\|_{\F} & \leq  \frac{\|\tilde{\M M}_0 - \M M_0\|_\F}{\sigma_d(\M M_0)} = \op(n^{-1/2}\log^{1/2}{n}).
\ee
To show the $2\rightarrow \infty$ norm concentration of $\hat{\bds{\mathcal{F}}}$, we need to further bound $\|\M U\|_{\twoinf}$. Since $\M M_0 $ shares the exact same $K\times K$ block structure as  $\mathbf{B}$ (with $\mathbf{B}$ assumed to have distinct rows), we  can follow the same argument as that in Section 5.1 in \cite{lei2019unified} and show that $[\M U]_i = [\M U]_{i'}$ if and only if $\bds k(i) = \bds k(i')$ which also implies
$\|\M U\|_{\twoinf} \precsim n^{-1/2}$.
Summarizing above results, by  Theorem \ref{T3} (dense regime) and Theorem 4.2 in \cite{cape2019two}, we have
\bee\label{twoinfnormbound}
\min_{\M T \in \mathbb{O}_{d}} \|\hat{\bds{\mathcal{F}}} \cdot \mathbf{T} - \mathbf{U}\|_{\twoinf} \leq 14\Bigg(\frac{\|\tilde{\M M}_0 - \M M_0\|_{\infty}}{\sigma_d(\M M_0)}\Bigg)\|\M U\|_{\twoinf} = \mathcal{O}_{\mathbb{P}}\big(n^{-1}\log^{1/2}{n}\big),
\ee
with high probability. We thus have
\bee\nonumber
&n^{1/2}\cdot \min_{\M T \in \mathbb{O}_{d}} \|\hat{\bds{\mathcal{F}}} \cdot \mathbf{T} - \mathbf{U}\|_{\twoinf}
 = \op\bigl( n^{-1/2}\log^{1/2}{n} \bigr).
\ee
Hence, by 
combining Lemma 5.1 and the arguments in Section 5.1 of \cite{lei2019unified}, we can show that running $K$-medians or $K$-means on the rows of $\hat{\mathbf{U}}$ will, with high probability, correctly recover the memberships of {\em every} nodes in $\mathcal{G}$.
\qed

\subsection{Strong recovery for DCSBM}
\label{sec:dcsbmsr}
We now extend the proof of Theorem~\ref{c1} to the setting of DCSBM. Recall that the edge probabilities matrix of a DCSBM is of the form
$\M P = \bm{\Theta} \M Z \M B \M Z^{\top} \bm{\Theta}$ where $\bm{\Theta} = \mathrm{diag}(\theta_1, \dots, \theta_n)$ contains the degree 
correction factors. Substituting the above $\M P$ into Eq.~\eqref{eq:pfc1_a} we obtain
\begin{equation}
\label{eq:pfdcsbm_a}
\M D_{\M P}^{-1} \M W^{\top} = \M D_{\M P}^{-1}  (\bm{\Theta} \M Z\M B \M Z^{\top} \bm{\Theta} \M D_{\M P}^{-1})^{t} = \M D_{\M P}^{-1} \bm{\Theta} \M Z \M V (\bm{\Upsilon} \M V^{\top} \M Z^{\top} \bm{\Theta} \M D_{\M P}^{-1} \bm{\Theta} \M Z \M V)^{t-1} \bm{\Upsilon}  \M V^{\top} \M Z^{\top} \bm{\Theta} \M D_{\M P}^{-1}
\end{equation}
The matrix $\M Z^{\top} \bm{\Theta} \M D_{\M P}^{-1} \bm{\Theta} \M Z$ is now a $K \times K$ positive definite diagonal matrix and hence $\M V^{\top} \M Z^{\top} \bm{\Theta} \M D_{\M P}^{-1} \bm{\Theta} \M Z \M V$ is also positive definite. Let $\mathcal{D} = \bm{\Upsilon} \M V^{\top} \M Z^{\top} \bm{\Theta} \M D_{\M P}^{-1} \bm{\Theta} \M Z \M V$. We then have
$$\M M_0' := \frac{2|\M P|}{\kappa \gamma} \sum_{t=t_L}^{t_U} (L - t)\M D_{\M P}^{-1} \M W^{t} = \frac{2|\M P|}{\kappa \gamma}  \M D_{\M P}^{-1} \bm{\Theta} \M Z \M V f(\mathcal{D}) \bm{\Upsilon} \M V^{\top} \M Z^\top \bm{\Theta} \M D_{\M P}^{-1}$$
Once again the set of matrices $\M B$ for which $f(\mathcal{D})$ is singular has Lebesgue measure $0$ in the space of $K \times K$ matrices and hence we can assume, without loss of generality, that $f(\mathcal{D})$ is invertible. Using the same reasoning as that in the proof of Theorem~\ref{c1} we conclude that $\mathbf{M}'_0$ is of the form
$$\mathbf{M}_0' = \M D_{\M P}^{-1} \bm{\Theta} \M Z \bm{\Xi}' \M Z^{\top} \M\Theta\M D_{\M P}^{-1}$$
where $\bm{\Xi}'$ is a $K \times K$ matrix with distinct rows. Now for any vertices $i$ we have
$$p_i = \sum_{j} p_{ij} = \sum_{j} \theta_i \theta_j b_{\tau(i), \tau(j)} = \theta_{i} \sum_{k=1}^{K} \sum_{\tau(j) = k} \theta_j b_{\tau(i), k}
$$
and hence $\theta_{i}/p_{i} = \theta_{i'}/p_{i'}$ whenever $i$ and $i'$ belong to the same community. The $i$th entry of $\M D_{\M P}^{-1} \bm{\Theta}$ is $\theta_i/p_i$ and 
thus
$$ \M D_{\M P}^{-1} \bm{\Theta} \M Z = \M Z \tilde{\bm{\Theta}}$$
for some $K \times K$ positive definite diagonal matrix $\tilde{\bm{\Theta}}$. In summary we have
$$\mathbf{M}'_0 = \M D_{\M P}^{-1} \bm{\Theta} \M Z \bm{\Xi} \M Z^{\top} \M D_{\M P}^{-1} = \M Z \tilde{\bm{\Theta}} \bm{\Xi}' \tilde{\bm{\Theta}} \M Z^{\top}.$$
where $\tilde{\bm{\Theta}} \bm{\Xi}' \tilde{\bm{\Theta}}$ is a $K \times K$ matrix with distinct rows. Again recall that the entries of $\mathbf{M}_0$ are the logarithm of the entries of $\mathbf{M}_0'$ and hence $\mathbf{M}_0$ is of the form
\bee\nonumber
\M M_0 &=  
\M Z  \bm{\Xi} \M Z^\T,
\ee
for some $K\times K$ matrix $\bm{\Xi}$. Let $\hat{\mathcal{F}}$ be the truncated SVD of $\tilde{\mathbf{M}}_0$ and $\M U$ be the leading singular vectors of $\mathbf{M}_0$.
We can then follow the remaining steps in the proof of Theorem~\ref{c1} to show that 
\begin{gather*}
\min_{\M T \in \mathbb{O}_{d}} \|\hat{\bds{\mathcal{F}}} \cdot \mathbf{T} - \mathbf{U}\|_{\F}  \leq  \frac{\|\tilde{\M M}_0 - \M M_0\|_\F}{\sigma_d(\M M_0)} = \op(n^{-1/2}\log^{1/2}{n}), \\
\min_{\M T \in \mathbb{O}_{d}} \|\hat{\bds{\mathcal{F}}} \cdot \mathbf{T} - \mathbf{U}\|_{\twoinf} \leq 14\Bigg(\frac{\|\tilde{\M M}_0 - \M M_0\|_{\infty}}{\sigma_d(\M M_0)}\Bigg)\|\M U\|_{\twoinf} = \mathcal{O}_{\mathbb{P}}\big(n^{-1}\log^{1/2}{n}\big),
\end{gather*}
{\color{black}noting that the involved high-probability bounds still hold for DCSBM as long as $\max_i \theta_i /\min_i \theta_i = \mathcal{O}(1)$ and $\theta_i \in (0,1)$, because we are using entry-wise arguments in the main proofs}. Thus clustering the rows of $\hat{\mathcal{F}}$ will, with high probability, exactly recover the memberships of every nodes in $\mathcal{G}$. \qed

\section{Proofs under the sparse regime}\label{app:sparsepf}
The approach used in the proofs of our main theorems under
the sparse regime is similar to that in the dense regime. However, if
we simply use the same technique as in Section~\ref{app:densepf} then
we only obtain a convergence rate of $\op\{(n\rho_n)^{-3/2}\log^{1/2}{n}\}$ for $\|\hat{\M W}^2 - \hat{\M W}^2\|_{\max,\off}$ and
$\|\hat{\M W}^t - \hat{\M W}^t\|_{\max}$ (with $t\geq 3$), which is too loose.  More specifically the bounds for 
$\|\tilde{\M M}_0- \M M\|_\F$ and $\|\tilde{\M M}_0 - \M M\|_{\infty}$ 
Eq.~\eqref{pf:d:ii'} is currently valid only when $\rho_n
 = \omega(n^{-1/3} \log^{1/3}{n})$. 
Before getting into the formal proofs, we first state 
Lemma~\ref{l:sparse:rate} and Lemma~\ref{l:hoA} as the main
technical results for bounding $\|\hat{\M W}^2 - \hat{\M
W}^2\|_{\max,\off}$ and $\|\hat{\M W}^t - \hat{\M W}^t\|_{\max}$ for $t\geq
3$ under the sparse regime. We summarize the motivation behind these lemmas
below. 
\begin{itemize}
\item Lemma~\ref{l:sparse:rate} is an analogue of
  Lemma~\ref{l:dense:rate} and is used
  repeatedly for bounding several important terms that frequently appear in our proofs.
\item In Section~\ref{app:densepf} we show that the bound for $\|\hat{\M W}^t - \hat{\M
W}^t\|_{\max}$ when $t\geq 3$ is of the same order as that for $\|\hat{\M W}^2 - \hat{\M
  W}^2\|_{\max,\off}$. For the sparse regime these bounds are generally of different magnitude as $n\rho_n$ increases. 
  This difference is the main distinguish feature 
between the two regimes.  Therefore, we derive a more accurate bound for $\|\hat{\M W}^t - \hat{\M
W}^t\|_{\max}$ when $t \geq 3$ in Step~4 and Step~5 of the proof of
Theorem~\ref{T3} (sparse regime) presented below. The main challenge is in controlling the term $\|\mathbf{A}^{t}\|_{\mathrm{max}}$ as
given in Lemma~\ref{l:hoA}. Thus Lemma~\ref{l:hoA} is
 the main technical contribution of this section and might be of independent interest. 
\end{itemize}

For ease of exposition we will only present the proof of Lemma \ref{l:hoA} in this
section; the proofs of the other lemmas are deferred to Section~\ref{app:lemma}.
\begin{lemma}
  \label{l:sparse:rate} 
Under Assumption~\ref{am:s} we have, for any $t \geq 1$ and
sufficiently large $n$, 
\begin{align}
&\|\M D_{\M A}\|_{} = \max_{1\leq i \leq n} d_i = \op(n\rho_n),
\\
&\|\M D_{\M A}^{-1}\|_{} = \max_{1\leq i \leq n} 1/d_i = \op\{(n\rho_n)^{-1}\},
\\ &\|\M D_{\M A} - \M D_{\M P}\|_{} =\max_{i}{|d_i - p_i|}  = \op(\sqrt{n\rho_n\log n}),
\\
&\|\M D^{-1}_{\M A} - \M D^{-1}_{\M P}\|_{} = \max_{i}{|d_i^{-1} -
     p_i^{-1}|} = \op\{(n\rho_n)^{-3/2}\sqrt{\log n}\},
\\
&\label{l:s:hwt}
\|\hat{\M W}^{t}\|_{\max} = \op\{(n\rho_n)^{-1}\},
\\
& \max_{i,{i'}}w^{(\eta)}_{ii'},\min_{i,{i'}}w^{(\eta)}_{ii'} \asymp 1/n,
\end{align}
\end{lemma}

\begin{remark}
\upshape{The bounds for $\|\hat{\M W}^t\|_{\max}$ given in
  Eq.~\eqref{l:s:hwt} is generally sub-optimal for $t \geq
  2$. Nevertheless we use this bound purely as a stepping stone in proving
  Theorem~\ref{T3}. Once Theorem~\ref{T3} is established
  we can improve the bound for $\|\hat{\M W}^{t}\|_{\mathrm{max}}$ by applying triangle inequality incorporated with the tight bounds of
   $ \|\M W^t\|_{\max}$ and $\|\hat{\M W}^t - \M W^t\|_{\max}$.
  }
\end{remark}

\begin{lemma}\label{l:hoA}
  Under Assumption~\ref{am:s}, suppose $\rho_n$ satisfies $\rho_n\rightarrow 0$ and,
  \bee\label{am:fd}
n^{-1/2} \log^{\beta_2}{n} \precsim\rho_n 
\ee
for some $\beta_2>1/2$. Then we have
\bee\label{T5:6}
&\big\|\M A^2\big\|_{\max,\text{off}}= \op(n\rho_n^2),
\\
&\big\|\M A^t\big\|_{\max} = \op\big\{(n\rho_n)^{t/2} + n^{t - 1}\rho_n^t\big\},
\ee
for any fixed $t\geq 3$. Furthermore, if $\rho_n\rightarrow 0$ and $n^{\frac{2-t}{t}}\precsim \rho_n$ for some $t \geq 3$,
then above bounds can be sharpened to $$\big\|\M A^t\big\|_{\max} = \op\big( n^{t - 1}\rho_n^t\big).$$
\end{lemma}
\begin{proof}
  Suppose $\rho_n$ satisfies Eq.~\eqref{am:fd} for some $\beta_2 > 1/2$. Define
  \bee
  \ &\zeta_{i^*}^{ii'} := a_{ii^*}a_{i'i^*}, \quad {\zeta}_{ii'} := \max_{i,{i'}}\sum_{i^{*}\neq
    i,{i'}}a_{ii^*}a_{i'i^*} = \max_{i,{i'}}\sum_{i^{*}\neq i,{i'}} \zeta_{i^*}^{ii'}.\ee
  Given $i \neq {i'}$, the $\{\zeta^{ii'}_{i^*}\}_{i^* \in[n],i^* \neq i,{i'}}$ are a set of independent Bernoulli variables with
$c_2^2\rho_n^2\leq \p(\zeta^{ii'}_{i^*} = 1) \leq c_3^2\rho_n^2.$
A similar argument to that for deriving Eq.~\eqref{A6:1} yields (recall that $n^{-1/2} \log^{\beta_2}{n} \precsim \rho_n$)
\bee\label{A6:2}
\log\Big\{\p\Big(\frac{{\zeta}_{ii'}}{n-1}\leq \frac{c_2^2}{2}\rho_n^2\Big)\Big\}
&\leq \log\Big\{\p\Big({\zeta}_{ii'}\leq\frac{1}{2}\E{\zeta}_{ii'}\Big)\Big\}
\precsim -C_1 n\rho_n^2
\precsim-C_1\log^{2\beta_2}{n}
\ee
for all $i\neq {i'}$; here $C_1 \geq 0$ is a constant not
depending on $n$ or $\rho_n$. Eq.~\eqref{A6:2} together with a union
bound then implies
\bee\nonumber
\p\Big\{\max_{i\neq{i'}} {\zeta}_{ii'}\leq
\frac{c_2^2}{2}n\rho_n^2\Big\}
\leq n^2
\max_{i\neq{i'}}\Big\{\p\Big({\zeta}_{ii'}\leq
\frac{c_2^2}{2}n\rho_n^2\Big)\Big\} \precsim \exp\Big(2\log n - C_1\log^{2\beta_2}{n}\Big)
\rightarrow 0
\ee
as $n \longrightarrow \infty$. We thus have $\max_{i\neq{i'}}\sum_{i^*\neq
  i,{i'}}a_{ii^*}a_{i'i^*} = \op(n\rho_n^2)$ 
and hence
\bee\nonumber
\|\M A^2\|_{\max,\text{off}}
&= \max_{i\neq{i'}}\sum_{i^*=1}^na_{ii^*}a_{i'i^*}
\leq  \max_{i\neq{i'}}\sum_{i^*\neq i,{i'}}a_{ii^*}a_{i'i^*} + 2 
= \op(n\rho_n^2).
\ee
We next consider the case when $t \geq 3$ and $\rho_n$ satisfies $n^{\frac{2-t}{t}}\precsim \rho_n$. We have
\bee\nonumber
\|\M A^{t}\|_{\max} \leq  \|\M P^t\|_{\max} + \|\M A^{t} - \M P^t\|_{\max}.
\ee
Under Assumption~\ref{am:s} we have
\bee\nonumber
\|\M P^{t}\|_{\max} 
\leq n\|\M P^{t - 1}\|_{\max}\|\M P\|_{\max}
\leq n^2\|\M P^{t - 2}\|_{\max}\|\M P\|_{\max}^2
\leq \dots \leq n^{t - 1}\|\M P\|^t_{\max} 
= \mathcal{O}(n^{t - 1}\rho_n^t).
\ee 
We now focus on bounding $\|\M A^{t} - \M P^{t}\|_{\max}$. Consider the following expansion for $\M A^{t} - \M P^{t}$
\bee\label{l:hoA:1}
\M A^{t} - \M P^{t} 
&=(\M A^{t - 1} - \M P^{t - 1})(\M A - \M P) + \M P^{t - 1}(\M A - \M P) 
+ \sum_{b_0 = 1}^{t - 1}\M A^{t - 1 - b_0}(\M A - \M P)\M P^{b_0}.
\ee
Let $\mathbf{E} = \M A - \M P$.
Applying the same expansion to $\M A^{t - 1} - \M P^{t - 1}, \dots ,\M A^{2} - \M P^{2}$, we obtain
\bee\label{l:hoA:2}
\M A^{t} - \M P^{t}  
&= (\M A^{t - 1} - \M P^{t - 1}) \M E + \M P^{t - 1} \M E
+ \sum_{b_0 = 1}^{t - 1}\M A^{t - 1 - b_0}(\M A - \M P)\M P^{b_0}
\\
&= \Big\{\sum_{b_1 = 0}^{t - 2}\M A^{t - 2 - b_1} \M E \M P^{b_1}\Big\} \M E+ \M P^{t - 1} \M E
+ \sum_{b_0 = 1}^{t - 1}\M A^{t - 1 - b_0} \M E \M P^{b_0}
\\
& = (\M A^{t - 2} - \M P^{t -2}) \M E^2 
+ \sum_{b_1 = 1}^{t - 2}\M A^{t - 2 - b_1} \M E \M P^{b_1} \M E
+ \sum_{b_0 = 1}^{t - 1}\M A^{t - 1 - b_0} \M E \M P^{b_0}
+\sum_{b' =  1}^2\M P^{t - b'}\M E^{b'}
\\
&= \dots =
\M E^{t} + \underbrace{\sum_{c = 0}^{t - 1}\sum_{b = 1}^{t - c - 1} \M A^{t - c - b - 1}\M E \M P^{b}\M E^c}_{\M L_1} + \underbrace{\sum_{b' =  1}^{t - 1}\M P^{t - b'}\M E^{b'}}_{\M L_2}.
\ee
Now for any summand appearing in $\M L_1$, if both $c \not = 0$ and $t -
c - b - 1 \not = 0$ then
\bee\label{l:diff:L1}
\|\M A^{t - c - b - 1}\M E\M P^{b}\M E^c\|_{\max} 
&\leq \|\mathbf{A}\|_{1} \cdot \|\M A^{t - c - b - 2} \M E \M P^{b} \M E^c\|_{\max}
\\
&\leq \dots \leq\|\mathbf{A}\|_{1}^{t - c - b - 1}  \|\M E \M P^{b} \M E^c\|_{\max}
\\
&\leq \|\mathbf{A}\|_{1}^{t - c - b - 1}  \|\M E \M
P^{b} \M E^{c - 1}\|_{\max}( \|\mathbf{A}\|_{1} + \|\M P\|_{1})
\leq \|\M A\|_{1}^{t - c - b - 1}  \|\M E \M
P^{b}\|_{\max} (\|\M A\|_{1} + \|\M P\|_{1})^{c}.
\ee
Here $\|\mathbf{M}\|_{1}$ denote the maximum of the absolute column sum of a matrix
$\mathbf{M}$. The bound in Eq.~\eqref{l:diff:L1} also holds when
$c = 0$ or $t -c - b - 1 = 0$. A similar argument yields
\begin{equation}
  \begin{split}
\label{l:diff:L1:2}
\|\M E \M P^b\|_{\max}
&\leq \|\M P\|_{\mathrm{max}}  (\|\M A\|_{1} +
\|\M P\|_{1})  \|\M P\|_{1}^{b-1}.
\end{split}
\end{equation}
Observe that $\|\M A\|_{1} = \max_{i} d_i$ and $\|\M 
P\|_{1} = \max_{i} \sum_{j} p_{ij}$. Combining Eq.~\eqref{l:diff:L1}, Eq.~\eqref{l:diff:L1:2} and
Lemma~\ref{l:sparse:rate}, we have
\bee\nonumber
\|\M L_1\|_{\max} &\leq \sum_{a = 0}^{t - 1}\sum_{b = 1}^{t - c - 1} \|\M A^{t - c - b - 1} \M E \M P^{b} \M E^c\|_{\max}
= \op(n^{t - 1}\rho_n^{t}).
\ee
Similarly, we also have $\|\M L_2\|_{\max} = \op(n^{t -
  1}\rho_n^{t})$. We therefore have
\bee\label{l:pdiff:main}
\|\M A^{t}\|_{\max} &\leq  
\|\M P^t\|_{\max} + \|\M E^{t}\|_{\max} + \|\M L_1\|_{\max} + \|\M L_2\|_{\max}
= \|\M E^{t}\|_{\max} + \op(n^{t - 1}\rho_n^{t}).
\ee
Finally we bound $\|\M E^t\|_{\max}$. Note that $\|\mathbf{M}\|_{\max}\leq \|\mathbf{M}\|_2$ 
for any matrix $\mathbf{M}$. 
Now, under Assumption \ref{am:s}, the maximal expected degree of
$\mathcal{G}$ is of order $n\rho_n\succsim n^{1 -
  \beta}\gg\log^4n$. We can thus apply the spectral norm concentration in \cite{lu2013spectra} to obtain
\bee\label{lupeng}
\|\M E^{t}\|_{\mathrm{max}} \leq \|\M A - \M P\|_{2}^{t} = \mathcal{O}_{\mathbb{P}}\bigl((n\rho_n)^{t/2}\bigr).
\ee
We thus have, after a bit of algebra, that $\|\M E^{t}\|_{\max} = \op(n^{t - 1}\rho_n^t)$ whenever $n^{\frac{2 - t }{t}}\precsim \rho_n$. Combining Eq.~\eqref{l:pdiff:main}, we have $\|\M A^t\|_{\max} = \op(n^{t - 1}\rho_n^t)$ whenever $n^{\frac{2-t}{t}}\precsim \rho_n\prec 1$ holds.
\end{proof}
\subsection{Proof of Theorem~\ref{T3} (Sparse Regime) 
}\label{sec:pf:ts:main}
The proof is organized as follows. In {Step 1} through {Step 4},
we bound $\|\hat{\M W}^t - \M W^t\|_{\max}$ under the
general sparse condition as specified in Assumption~\ref{am:s}. These
arguments are generalizations of the corresponding arguments in the proof
of Theorem~\ref{T2} in Section~\ref{a:p:t2}.
In {Step 5}, we provide an improved bound for $\|\hat{\M W}^t
- \M W^t\|_{\max}$ when $t \geq 3$ and $\tfrac{t-3}{t-1} > \beta$. For ease of exposition we omitted some of the more mundane technical details
from the current proof and refer the interested reader to Section~\ref{app:c:ts:dt}.

\noindent\textbf{Step 1 (Bounding $\|\hat{\M W} - \M W\|_{\max}$):} 
Similar to {Step 1} in the proof of Theorem~\ref{T2}, we have
\bee
  \label{T1s:order_1_0:s}
  \hat{\M W} - \M W &= \M A \M D_{\M A}^{-1} - \M P \M D_{\M P}^{-1} 
=
  \underbrace{\M A \M D_{\M P}^{-1} \M D_{\M A}^{-1}(\M D_{\M P} - \M D_{\M A})}_{\M \Delta^{(1)}_1} + \underbrace{(\M A - \M P)\M D_{\M P}^{-1}}_{\M \Delta^{(1)}_2}
\ee
and hence
\bee\nonumber
&\|\M \Delta^{(1)}_1\|_{\max} = \op\big\{(n\rho_n)^{-3/2} \log^{1/2}{n}\big\}, \quad
&\|\M \Delta^{(1)}_2\|_{\max} = \op\big\{(n\rho_n)^{-1}\big\}.
\ee
We therefore have 
\begin{equation}
  \label{T21:res1}
\begin{split}
\|\hat{\M W} - \M W\|_{\max} &\leq \|\M \Delta^{(1)}_1\|_{\max} + \|\M
\Delta^{(1)}_2\|_{\max} = \op\big\{(n\rho_n)^{-3/2} \log^{1/2}{n} \big\} + \op\big\{(n\rho_n)^{-1}\big\} 
= \op\big\{(n\rho_n)^{-1}\big\}.
\end{split}
\end{equation}
See Section~\ref{app:ts:dt:1} for additional details in deriving the above inequalities.

\noindent\textbf{Step 2 (Bounding $\|\hat{\M W}^2 - \M
  W^2\|_{\max,\mathrm{diag}}$):}
Similar to {Step 2} in the proof of Theorem~\ref{T2},
\bee
\hat{\M W}^2 - \M W^2 &= \underbrace{(\hat{\M W} - \M W )\M W}_{\M\Delta^{(2)}_1} + \underbrace{\hat{\M W} (\hat{\M W} - \M W)}_{\M\Delta^{(2)}_2}, \\
\M\Delta^{(2)}_1 &= \underbrace{\big\{\M A \M D_{\M P}^{-1} \M D_{\M A}^{-1}(\M D_{\M P} - \M D_{\M A}) - \M A \M D_{\M P}^{-2}(\M D_{\M P} - \M D_{\M A})\big\}\M W}_{\M \Delta^{(2,1)}_{1}} 
+ \underbrace{\M A \M D_{\M P}^{-2}(\M D_{\M P} - \M D_{\M A})}_{\M \Delta^{(2,2)}_{1}}\M W 
+  \underbrace{(\M A - \M P)\M D_{\M P}^{-1} \M W}_{\M \Delta^{(2,3)}_{1}}.
\ee
The bounds in Section~\ref{app:ts:dt:2} imply
\begin{gather*}
    \|\M \Delta^{(2)}_{1}\|_{\max}=
\op\big(n^{-3/2}\rho_n^{-1/2}\log^{1/2}{n}\big), \quad \|\M W(\M W -
\hat{\M W})\|_{\max}  = \op(n^{-3/2} \rho_n^{-1/2} \log^{1/2}{n}).
\end{gather*}
Furthermore, we also have
\bee
\|\M \Delta^{(2)}_2\|_{\max} &= \|(\M W - \hat{\M W})(\M W - \hat{\M W}) - \M W(\M W - \hat{\M W})\|_{\max} 
\leq   \|(\M W - \hat{\M W})^2\|_{\max} + \|\M W(\M W - \hat{\M W})\|_{\max}. 
\ee
Replacing $\|\cdot\|_{\max}$ with $\|\cdot\|_{\max,\diag}$ in
Eq.~\eqref{t:w:w-w2}, and following the same argument as that for
Eq.~\eqref{t:w:w-w2} and Eq.~\eqref{t:d:d223}, we have, by Lemma~\ref{l:sparse:rate},
\bee
  \label{T4:res2}
\|(\M W - \hat{\M W})^2\|_{\max,\text{diag}} = \op\{(n\rho_n)^{-1}\},
\ee
and thus
$
\|\M W^2 - \hat{\M W}^2\|_{\max,\diag} = \op\{(n\rho_n)^{-1}\}.
$

\noindent\textbf{Step 3 (Bounding $\|\hat{\M W}^2 - \M W^2\|_{\max,\mathrm{off}}$):} Similar to Eq.~\eqref{t:w:w-w2},
\bee
  \label{t:w-w2:s}
\|(\M W - \hat{\M W})^2\|_{\max,\text{off}}
&=  \max_{i\neq i'}\Big|\sum_{i^* = 1}^{n}\Big(\frac{p_{ii^*}}{p_{i^*}} - \frac{a_{ii^*}}{d_{i^*}}\Big)\Big(\frac{p_{i^*i'}}{p_{i'}} - \frac{a_{i^{*}{i'}}}{d_{i'}}\Big)\Big|
\\
&\leq \underbrace{\max_{i\neq {i'}}\Big|\sum_{i^{*}= 1}^n \BL \frac{a_{ii^*}}{p_{i^*}} - \frac{a_{ii^*}}{d_{i^*}}
\BR\BL\frac{p_{i^*i'}}{p_{i'}} - \frac{a_{i^{*}{i'}}}{p_{i'}} \BR\Big|}_{\delta^{(2,1)}_{2,\text{off}}}
 +\underbrace{\max_{i\neq {i'}}\Big|\sum_{i^* = 1}^{n}\Big(\frac{p_{ii^*}}{p_{i^*}} - \frac{a_{ii^*}}{d_{i^*}}\Big)\BL \frac{a_{i^{*}{i'}}}{p_{i'}} - \frac{a_{i^{*}{i'}}}{d_{i'}}\BR\Big|}_{ \delta^{(2,2)}_{2,\text{off}}}
\\
&+\underbrace{\max_{i\neq {i'}}\Big|\sum_{i^{*}= 1}^{n}\BL\frac{p_{ii^*}}{p_{i^*}} - \frac{a_{ii^*}}{p_{i^*}} \BR\BL\frac{p_{i^*i'}}{p_{i'}} - \frac{a_{i^{*}{i'}}}{p_{i'}} \BR\Big|}_{ \delta^{(2,3)}_{2,\text{off}}}.
\ee
We then have the bounds
\bee
{}&  \delta^{(2,1)}_{2,\text{off}} = \op\bigl\{(n\rho_n)^{-2}{\log
    n}\bigr\}, \quad \delta^{(2,2)}_{2,\text{off}} =
  \op\bigl\{(n\rho_n)^{-2}{\log n}\bigr\}, 
  \quad 
   \delta^{(2,3)}_{2,\text{off}} = \op\big\{n^{-3/2}\rho_n^{-1} \log^{1/2}{n}\big\}.
\ee
For succinctness we only derive the bound for
$\delta^{(2,1)}_{2,\text{off}}$ here. The bounds for
$\delta^{(2,2)}_{2,\text{off}}$ and
$\delta^{(2,3)}_{2,\text{off}}$ are derived similarly; see Section~\ref{app:ts:dt:3} for more details.

\begin{claim} Under the setting of Theorem~\ref{T3} we have $\delta^{(2,1)}_{2,\text{off}} =
  \op\big\{(n\rho_n)^{-2}{\log n}\big\}$.
  \end{claim}
  \begin{proof}
We start by writing
\begin{equation}
  \begin{split}
\sum_{i^* = 1}^n\BL\frac{a_{ii^*}}{p_{i^*}} - \frac{a_{ii^*}}{d_{i^*}}
\BR\BL\frac{p_{i^*i'}}{p_{i'}} - \frac{a_{i^{*}{i'}}}{p_{i'}} \BR 
&=
\frac{1}{p_{i'}}\sum_{i^* = 1}^na_{ii^*} \Bigl\{\frac{(d_{i^*} -
  p_{i^*})^2}{p_{i^*}^2 d_{i^*}} + \frac{p_{i^*} -
  d_{i^*}}{p_{i^*}^2}\Bigr\} \bigl(a_{i^*i'} - p_{i^*i'}\bigr),
\end{split}
\end{equation}
and hence
\bee
\delta_{2,\text{off}}^{(2,1)}
&\leq \max_{i\neq i'}\Bigl|\underbrace{\frac{1}{p_{i'}}\sum_{i^* = 1}^na_{ii^*}\frac{(d_{i^*} - p_{i^*})^2}{p_{i^*}^2d_{i^*}}
  \big(p_{i^*i'} - a_{i^*i'}\big)}_{\xi_{ii'}}\Bigr| 
+ \max_{i\neq i'}\Bigl|\underbrace{\frac{1}{p_{i'}}\sum_{i^* = 1}^na_{ii^*}\BL\frac{d_{i^*} - p_{i^*}}{p_{i^*}^2}
\BR\big(p_{i^*i'} - a_{i^*i'}\big)}_{\zeta_{ii'}}\Bigr|.
\ee
For the term $\xi_{ii'}$, by Lemma~\ref{l:sparse:rate}, we have
\begin{equation*}
\begin{split}
 \max |\xi_{ii'}| 
 & \leq \Big(\max_i\frac{1}{p_i}\Big)
  \Big\{\max_{i^*}\Big|\frac{(d_{i^*} -
    p_{i^*})^2}{p_{i^*}^2d_{i^*}}\Big|\Big\}
    \Big(2 + \max_{i\neq
    i'}\sum_{i^* = 1\atop i^* \neq i,i'}^n|a_{ii^*}||p_{i^*i'} -
  a_{i^*i'}| \Big)
  \\ &\leq \op\big\{(n\rho_n)^{-3} \log n\big\} 
  \Bigl[2 +
  \max_{i\neq i'}\Big\{i^*:a_{ii^*} = a_{i'i^*} = 1\Big\}
   +
  \|\mathbf{P}\|_{\max} \cdot \max_{i\neq i'}\Big\{i^*:a_{ii^*} = 1, a_{i'i^*} = 0\Big\} \Bigr]
  \\ & \leq \op\big\{(n\rho_n)^{-3} \log n\big\}  \Bigl(2 + \max_{i} d_i +
  \|\mathbf{P}\|_{\max}  \max_{i} d_i \Bigr)
= \op\big\{(n\rho_n)^{-2}{\log n}\big\}.
  \end{split}
\end{equation*}
For the term $\zeta_{ii'}$, first let $e_{ii'} = a_{ii'} -
p_{ii'}$. Then expanding $d_{i^{*}} - p_{i^{*}}$ as $\sum_{i^{**}}
e_{i^*i^{**}}$, we have
\bee
\zeta_{ii'} =\frac{1}{p_{i'}}\sum_{i^* = 1}^n
\frac{a_{ii^*}}{p_{i^*}^2}\Bigl(\sum_{i^{**}\not \in \{i',i\}} -
e_{i^*i^{**}} e_{i^*i'} - \sum_{i^{**} \in \{i',i\}}e_{i^*i^{**}}e_{i^*i'}\Bigr).
\ee
Now, for fixed $i,i'$, since $a_{ii} = p_{ii} = 0$ for all $i$, we have
\bee
\label{d211}
\frac{1}{p_{i'}}\sum_{i^* = 1}^n
\frac{a_{ii^*}}{p_{i^*}^2}\sum_{i^{**} \not \in
  \{i',i\}}e_{i^*i^{**}} e_{i^*i'} 
&= \frac{1}{p_{i'}}\sum_{i^* \not \in \{i,i'\}}\sum_{i^{**} \not \in
  \{i',i,i^*\}}\frac{1}{p_{i^*}^2}
a_{ii^*}e_{i^*i^{**}}e_{i^*i'} \\
&= \frac{1}{p_{i'}}\sum_{(i^*, i^{**}) \in \mathcal{T}(i,i')}
\Bigl(\frac{a_{ii^*}e_{i^*i'}}{p_{i^*}^2} +
\frac{a_{ii^{**}}e_{i^{**}i'}}{p_{i^{**}}^2} \Bigr)e_{i^*i^{**}} := \mathcal{S}(\bds a_i, \bds a_{i'}).
\ee
Here $\mathcal{T}(i,i') = \big\{(i^*,i^{**})|i^* < i^{**}, i^* \not \in
\{i,i'\}, i^{**} \not \in \{i,i'\}\}$ and
$\bm{a}_{i} = (a_{i1}, \dots, a_{in})$. 
Let us now, in addition to conditioning on $\mathbf{P}$, also
condition on both $\bds{a}_i$ and $\bds{a}_{i'}$. 
Then the sum for $\mathcal{S}(\bds{a}_{i}, \bds{a}_{i'})$ in
Eq.~\eqref{d211} is a sum of {\em independent}, mean zero random
variables, i.e., once we conditioned on $\mathbf{P}, \bds{a}_{i}$,
and $\bds{a}_{i'}$, the $e_{i^{*}i^{**}}$ for $(i^*,i^{**}) \in
\mathcal{T}(i,i')$ are independent. 
We therefore have
\bee\label{var:d112}
\mathrm{Var}\Big(\frac{1}{p_{i'}} \sum_{\mathcal{T}(i,i')} \Bigl(\frac{a_{ii^*}e_{i^*i'}}{p_{i^*}^2} +
\frac{a_{ii^{**}}e_{i^{**}i'}}{p_{i^{**}}^2} \Bigr)e_{i^*i^{**}} 
  \Big| \bds a, \bds a'\Big)
&= \frac{1}{p_{i'}^2} \sum_{\mathcal{T}(i,i')} \Bigl(\frac{a_{ii^*}e_{i^*i'}}{p_{i^*}^2} +
\frac{a_{ii^{**}}e_{i^{**}i'}}{p_{i^{**}}^2}\Bigr)^2 \mathrm{Var}[e_{i^{*},i^{**}}] 
\\
&\precsim (n\rho_n)^{-6} \cdot (n \rho_n) \cdot (d_{i} + d_{i'})
\ee
Furthermore, we also have
$$\Bigl|\frac{1}{p_{i'}} \Bigl(\frac{a_{ii^*}e_{i^*i'}}{p_{i^*}^2} +
\frac{a_{ii^{**}}e_{i^{**}i'}}{p_{i^{**}}^2} \Bigr)e_{i^*i^{**}}\Bigr| \precsim (n\rho_n)^{-3}.$$
Therefore, by Bernstein inequality, for any $c' > 0$ there exists a
constant $C' > 0$ such that, for all sufficiently large $n$,
\bee\label{d112:bern}
\ &\p\Bigl\{\Bigl|\mathcal{\mathcal{S}}(\bds{a}_{i}, \bds{a}_{i'})\Bigr|\geq
C'(n\rho_n)^{-5/2} (d_{i} + d_{i'})^{1/2} \log^{1/2}{n}\mid \bds a_{i},
\bds a_{i'}\Bigr\}
\leq n^{-c'}
\ee
We can now uncondition with respect
to $\bds{a}_i$ and $\bds{a}_{i'}$. More specifically, for any $t$, we have
\begin{equation*}
  \begin{split}
\p\Bigl\{\Bigl|\mathcal{S}(\bds{a}_{i}, \bds{a}_{i'})\Bigr|\geq
\tilde{t} \Bigr\}
&\leq \p\Bigl\{\Bigl|\mathcal{S}(\bds{a}_{i}, \bds{a}_{i'})\Bigr|\geq
\tilde{t} \mid \max \{d_i, d_{i'}\} < C n
\rho_n\Bigr\}
\p\Bigl(\max\{d_i,d_{i'}\} < C n \rho_n \Bigr)  + \p\big(\max \{d_i,d_{i'}\} \geq C n \rho_n\big),
\end{split}
\end{equation*}
where $C$ is any arbitrary positive constant. Now let $c$ be arbitrary. Then by
Lemma~\ref{l:sparse:rate} and
Eq.~\eqref{d112:bern}, together with
taking $\tilde{t} = C' (n \rho_n)^{-2} \log^{1/2}{n}$ for some constant $C'$,
we have
\begin{equation*}
\p\Bigl\{\Bigl|\mathcal{S}(\bds{a}_{i}, \bds{a}_{i'})\Bigr|\geq
C' (n \rho_n)^{-2} \log^{1/2}{n} \Bigr\} \leq 2n^{-c} .
\end{equation*}
A union bound over the $\tbinom{n}{2}$ possible choices of
$\bds{a}_{i}$ and $\bds{a}_{i'}$ yields
\bee\label{d112:f}
\max_{i \not = i'} \Bigl|\mathcal{S}(\bds{a}_{i}, \bds{a}_{i'})\Bigr| = \op\bigl\{(n \rho_n)^{-2} \log^{1/2}{n}\bigr\}.
\ee
Finally we also have 
\bee
\max_{i,i'}\Big|\frac{1}{p_{i'}}\sum_{i^* = 1}^n
\frac{a_{ii^*}}{p_{i^*}^2} \sum_{i^{**} \in \{i,i'\}}e_{i^*i^{**}}
e_{i^*i'}\Big|
&\precsim \max_i d_i \cdot \Big(\max_i 1/p_i\Big)^3 
=\op\{(n\rho_n)^{-2}\}.
\ee
In summary, we have $\max_{ii'} \zeta_{ii'} =
\op((n\rho_n)^{-2}\log^{1/2}{n})$ and hence
\bee
 \delta_{2,\text{off}}^{(2,1)}
  \leq \max_{ii'} \xi_{ii'} +
  \max_{ii'} \zeta_{ii'} = \op\big\{(n\rho_n)^{-2} \log n\big\}.
\ee
\end{proof}
Given the previous claim together with an argument similar to that for
Eq.~\eqref{T22:res1}, we obtain
\begin{equation}
  \label{T22s:res2}
  \begin{split}
\|\M W^2 - \hat{\M W}^2\|_{\max,\off} 
&\leq \|\MD^{(2)}_1\|_{\max} + \|\M W(\M W - \hat{\M W})\|_{\max} + \delta^{(2,1)}_{2,\off} + \delta^{(2,2)}_{2,\off} + \delta^{(2,3)}_{2,\off} 
\\
&=\op\Big[\max\Big\{(n\rho_n)^{-2} \log n,n^{-3/2}\rho_n^{-1}\sqrt{\log n}\Big\}\Big]
\\ & =\begin{cases}
\op(n^{-3/2}\rho_n^{-1}\log^{1/2}{n}) & \text{if $0\leq \beta <1/2$}
\\
\op((n\rho_n)^{-2} \log n) & \text{otherwise.}
\end{cases}
\end{split}
\end{equation}
\textbf{Step 4 (Bounding $\|\hat{\M W}^{t} - \M W^{t}\|_{\max}$ for $t
\geq 3$):} The following argument is a generalization of the
argument in \textbf{Step 3} of
Section~\ref{a:p:t2}. We first consider $t = 3$. We have  
\bee\label{T2s:s3f1}
\|\HW^3 - \M W^3\|_{\max} 
&\leq \|\big(\HW^2 - \W^2\big)\HW\|_{\max} + \|\W^2\big(\HW - \W\big)\|_{\max}
\ee
For the first term in the RHS, by Eq.~\eqref{T4:res2} and Lemma~\ref{l:sparse:rate}, we have
\bee\label{T2s:s3f2}
\|(\HW^2 - \W^2)\HW\|_{\max}
&\leq \max_{i}d_i\cdot \|\HW\|_{\max}\|\HW^2 - \W^2\|_{\max,\text{off}} \\
&= \op(1)\|\HW^2 - \W^2\|_{\max,\text{off}} 
=\op\Bigl(\max\Bigl\{(n\rho_n)^{-2}\log n,n^{-3/2}\rho_n^{-1}\log^{1/2}{n}\Bigr\}\Bigr).
\ee
For the second term in the RHS of Eq.~\eqref{T2s:s3f1}, we apply the
similar technique for $\M \Delta^{(2,1)}_{1}$ in {Step 2} and deduce 
\bee\label{T2s:s3f3}
\|\W^2(\HW - \W)\|_{\max} =\op(1) \|\M W^2 - \hat{\M W}^2\|_{\max,\off}.
\ee
Combining Eq.~\eqref{T2s:s3f1}, Eq.~\eqref{T2s:s3f2} and Eq.~\eqref{T2s:s3f3}, we have
\bee\label{T2s:s3res1}
\|\HW^3 - \W^3\|_{\max} 
= \op\Big(\max\Big\{(n\rho_n)^{-2}\log
n,n^{-3/2}\rho_n^{-1}\log^{1/2}{n}\Big\}\Big).
\ee
Now for $t = 4$, we have 
\bee\label{T4:s3f1}
\HW^4 - \W^4 &= (\HW^2 - \W^2)\HW^2 + \W^2(\HW^2 - \W^2)
=(\HW^2 - \W^2)\HW^2 - (\W^2 - \HW^2)^2
+ \HW^2(\HW^2 - \W^2).
\ee
For $(\HW^2 - \W^2)\HW^2$, by Eq.~\eqref{T2s:s3f2} and Lemma~\ref{l:sparse:rate}
\bee\label{T4:s3f2}
\|(\HW^2 - \W^2)\HW^2\|_{\max} 
&\leq \max_i d_i \cdot \|\HW\|_{\max}\cdot \|(\HW^2 - \W^2)\HW\|_{\max}
\\& = \op(n \rho_n \cdot (n \rho_n)^{-1})\cdot\op(1) \|\M W^2 -
\hat{\M W}^2\|_{\max,\off} = \op(1) \|\M W^2 - \hat{\M W}^2\|_{\max,\off}.
\ee
Similarly, we have $\|\HW^2(\HW^2 - \W^2)\|_{\max} = \op(1) \|\M W^2 - \hat{\M W}^2\|_{\max,\off}$. Furthermore,
\bee\label{T4:s3f3}
\|(\HW^2 - \W^2)^2\|_{\max} 
&\leq n \|\HW^2 - \W^2\|_{\max,\text{off}}^2 + \|\HW^2 - \W^2\|_{\max,\text{diag}}^2
=\op\big((n\rho_n)^{-2}\log n\big).
\ee
Combining Eq.~\eqref{T4:s3f1}, Eq.~\eqref{T4:s3f2} and Eq.~\eqref{T4:s3f3}, we obtain
\bee\label{T2s:s3res2}
\|\HW^4 - \W^4\|_{\max} 
&\leq \mathcal{O}_{\mathbb{P}}(1) \|(\HW^2 - \W^2)\HW^2\|_{\max} + \|(\W^2 - \HW^2)^2\|_{\max}  
\\ &= \op\Big(\max\Big\{(n\rho_n)^{-2}\log
n,n^{-3/2}\rho_n^{-1} \log^{1/2}{n} \Big\}\Big) .
\ee
For any $t \geq 5$, we can write $\hat{\M W}^{t} - \M W^{t}$ as
\begin{equation*}
\hat{\M W}^{t} - \M W^{t} = \underbrace{(\HW^2 - \W^2)\HW^{t -2}}_{\M\Delta^{(3)}_1} + \underbrace{\W^2(\HW^{t - 2} - \W^{t - 2})}_{\M\Delta^{(3)}_2}.
\end{equation*}
We first consider $\M\Delta^{(3)}_1$. By Lemma~\ref{l:sparse:rate} and Eq.~\eqref{T2s:s3f2}
\bee\label{T23s:res1}
\|\M \Delta^{(3)}_1\|_{\max} 
&\leq \max_i d_i \cdot \|\HW\|_{\max}\cdot\|(\HW^2 - \W^2)\HW^{t-3}\|_{\max}
\\
&\leq \dots
\leq \Big(\max_i d_i  \|\HW\|_{\max}\Big)^{t - 3} \|(\HW^2-\W^2)\HW\|_{\max}
=\op\Bigl(\max\Bigl\{(n\rho_n)^{-2}\log
n,n^{-3/2}\rho_n^{-1}\log^{1/2}{n}\Bigr\}\Bigr).
\ee
For $\M \Delta^{(3)}_2$, by Lemma~\ref{l:sparse:rate}
\bee\label{T23s:res2}
\|\M \Delta^{(3)}_2\|_{\max} 
&\leq n\cdot \|\M W^2\|_{\max}\|\hat{\M W}^{t - 2} - \M W^{t - 2}\|_{\max}
 = \mathcal{O}(1)\|\hat{\M W}^{t - 2} - \M W^{t - 2}\|_{\max}.
\ee
Similar to Eq.~\eqref{T23:ite}, we obtain
\bee
\|\HW^{t} - \W^{t}\|_{\max}
&= 
\op\Big(\max\Big\{(n\rho_n)^{-2}\log
n,n^{-3/2}\rho_n^{-1} \log^{1/2}{n} \Big\}\Big) + 
\begin{cases} 
\mathcal{O}(1)\|\hat{\M W}^{4} - \M W^{4}\|_{\max} & \text{if $t$ is even} \\
\mathcal{O}(1)\|\hat{\M W}^3 - \M W^3\|_{\max}, & \text{if $t$ is odd}
\end{cases}
\ee
Eq.~\eqref{T2s:s3res1} and
Eq.~\eqref{T2s:s3res2} then implies
\bee\nonumber
\|\HW^{t} - \W^{t}\|_{\max} 
&= \op\Big(\max\Big\{(n\rho_n)^{-2}\log
n,n^{-3/2}\rho_n^{-1} \log^{1/2}{n}\Big\}\Big)
=\begin{cases}
\op(n^{-3/2}\rho_n^{-1}\log^{1/2}{n}) & \text{if $0\leq \beta <1/2$}
\\
\op((n\rho_n)^{-2}\log n) & \text{otherwise.}
\end{cases}
\ee
\textbf{Step 5 (Bounding $\|\hat{\M W}^t - \M W^t\|_{\max}$ for $t
  \geq 4$ and $\tfrac{t-3}{t-1} \geq \beta$):}
Similar to Eq.~\eqref{l:hoA:1} and Eq.~\eqref{l:hoA:2} in the proof of
Lemma~\ref{l:hoA}, we write
\bee
  \label{Ts:ho:main}
\hat{\M W}^{t} - \M W^{t} 
&= (\hat{\M W}- \M W)^{t} 
+ \sum_{r = 0}^{t
  - 1}\sum_{s = 1}^{t - r - 1} \hat{\M W}^{t - r - s - 1}(\hat{\M W} -
\M W)\M W^{s}(\hat{\M W} - \M W)^r  + \sum_{r' =  1}^{t - 1}\M W^{t - r'}(\hat{\M W} - \M W)^{r'}.
\ee
In {Step 2}, we shown $\|(\hat{\M W} - \M W)\M W\|_{\max} = \op
\big(n^{-3/2}\rho_n^{-1/2}\sqrt{\log n}\big)$. Let $d_{\max} = \max_{i} d_i$. A similar argument to
that for $L_1$ in the proof of Lemma~\ref{l:hoA} yields, for $0 \leq r
\leq t-1$ and $1 \leq s \leq t - r - 1$, that
\bee\nonumber
\|\hat{\M W}^{t - r - s - 1}(\hat{\M W} - \M W)\M W^{s}(\hat{\M W} - \M W)^r\| 
&\leq \bigl(d_{\max}  \|\hat{\M W}\|_{\max}\Bigr)^{t - r - s - 1} \|(\hat{\M W} - \M W)\M W^{s}\|_{\max}
(n\|\M W\|_{\max} + d_{\max} \|\hat{\M W} - \M W\|_{\max})^r
\\ 
&= \op\Big(\|(\hat{\M W} - \M W)\M W^{s}\|_{\max}\Big) 
\\
&= \op\Big(n^{s -1}\|(\hat{\M W} - \M W)\M W\|_{\max}  \|\M W\|_{\max}^{s -1}\Big)   
\\ &= \op\Big(\|(\hat{\M W} - \M W)\M W\|_{\max} \Big) 
= \op
\Big(n^{-3/2}\rho_n ^{-1/2}\log^{1/2}{n}\Big),
\ee
where the second inequality follows from
Lemma~\ref{l:sparse:rate} and Eq.~\eqref{T21:res1} and the last
inequality follows from Lemma~\ref{l:sparse:rate}. As $t$ is
finite, we have
\bee\nonumber
\Big\|\sum_{r = 0}^{t - 1}\sum_{s = 1}^{t - r - 1} \hat{\M W}^{t - r - s - 1}(\hat{\M W} - \M W)\M W^{s}(\hat{\M W} - \M W)^r\Big\|_{\max}
=\op
\Big(n^{-3/2} \rho_n^{-1/2}\log^{1/2}{n}\Big).
\ee
Similarly,
\bee\nonumber
\Big\|\sum_{r' =  1}^{t - 1}\M W^{t - r'}(\hat{\M W} - \M W)^{r'} \Big\|_{\max} = \op
\big(n^{-3/2} \rho_n^{-1/2} \log^{1/2}{n}\big).
\ee
Recalling Eq.~\eqref{Ts:ho:main}, we obtain
\bee\label{Ts:ho:main1}
\|\hat{\M W}^{t} - \M W^{t}\|_{\max} 
\leq \|(\hat{\M W} - \M W)^t\|_{\max} + \op
\big(n^{-3/2} \rho_n^{-1/2}\log^{1/2}{n}\big).
\ee
Now we focus on $(\hat{\M W} - \M W)^t$. We start with the polynomial expansion
\begin{equation*}
\begin{split}
(\hat{\M W} - \M W)^t &= \Big\{\M A(\M D_{\M A}^{-1} - \M D_{\M P}^{-1}) + (\M A - \M P) \M D_{\M P}^{-1}\Big\}^t
= \Big\{(\M A - \M P) \M D_{\M P}^{-1}\Big\}^t 
+\sum_{\begin{subarray}\ \bds c\in \{1,2\}^t \\ \bds c\neq (1,\cdots,1)\end{subarray}}\prod_{r = 1}^t\M \Xi_{c_r}.
\end{split}
\end{equation*}
Here $c_r$ represents the $r$th coordinate of $\bds{c} = (c_1, 
\dots, c_t) \in \{1,2\}^{t}$ and 
\bee\nonumber
\M \Xi_{{c}_r} =
\begin{cases}
(\M A - \M P)\M D_{\M P}^{-1}, & \text{if $c_r = 1$}
\\
\M A (\M D_{\M A}^{-1} - \M D_{\M P}^{-1}), & \text{if $c_r = 2$.}
\end{cases}
\ee
We note that there are $2^t - 1$ distinct $\bds{c} \not =
(1,1,\dots,1)$. Now for any given
$\bds{c}$, let $r^{*} = r^*(\bds{c})$ be the smallest value of $r$
such that $c_r = 2$. We emphasize that $r^{*}$ depends on $\bds{c}$;
however, for simplicity of notation we make this dependency
implicit. We further denote
\bee\nonumber
(\M \Xi^{1}_{c_r},\M \Xi^2_{c_r}) := \begin{cases}
(\M A \M D_{\M A}^{-1}, - \M A\M D_{\M P}^{-1}),& \text{if $c_r = 1$,} \\
(\M A \M D_{\M P}^{-1}, - \M P\M D_{\M P}^{-1}),& \text{if $c_r = 2$},
\end{cases}
\ee
so that $\M \Xi_{c_r} = \M \Xi_{c_r}^1 + \M \Xi_{c_r}^2$. Given $\bds c$, we could write
\begin{equation}
  \label{Ts:ho:xis}
\begin{split}
  \prod_{r = 1}^t\M \Xi_{c_r} &= \prod_{j < r^*} (\M \Xi_{c_j}^1
  + \M \Xi_{c_j}^2) \cdot \M \Xi_{c_{r^*}} \cdot \prod_{k > r^*} (\M \Xi_{c_k}^1 + \M \Xi_{c_k}^{2})
=\sum_{\bds{m} \in \{1,2\}^{t}} \underbrace{\prod_{j < r^*}
  \M \Xi_{c_j}^{m_j} \cdot \M \Xi_{c_{r^*}} \cdot \prod_{k > r^*} \M \Xi_{c_k}^{m_k}}_{\bds{\Im}_{\bds{m}, \bds{c}}}
\end{split}
\end{equation}
Now for any $c_r \in \{1,2\}$ and $m \in \{1,2\}$, 
the $ii'$th entry of $\M \Xi^m_{c_r}$ is
$\xi^{c_r,m}_{ii'} \cdot \vartheta^{c_r,m}_{ii'}$ where
\bee\label{Ts:ho:defxitheta}
(\xi^{c_r,m}_{ii'}, \vartheta^{c_r,m}_{ii'}) = 
\begin{cases}
(a_{ii'},1/d_{i'}) & \text{if $c_r = 1$ and $m = 1$},\\
(a_{ii'},-1/p_{i'}) & \text{if $c_r = 1$ and $m = 2$},
\\
(a_{ii'},1/p_{i'}) & \text{if $c_r = 2$ and $m = 1$},
\\
(p_{ii'},-1/p_{i'}) & \text{if $c_r = 2$ and $m = 2$}.
\end{cases}
\ee
Using the above notations, we can now write the $ii'$th entry of $\bds{\Im}_{\bds{m},\bds{c}}$ as
\bee\nonumber
 \Im^{\bds{m},\bds{c}}_{ii'} 
= \sum_{\bds{i} \in
  \{1,\dots,n\}^{t-1}} \Bigl( \prod_{j < r^*}  \xi^{c_j,m_j}_{i_{j-1}
    i_j}  \vartheta^{c_j,m_j}_{i_{j-1} i_j}\Bigr) 
\cdot
  a_{i_{r^* - 1}i_{r_*}}  \Bigl(\frac{1}{d_{i_{r^*}}} - \frac{1}{p_{i_{r^*}}}\Bigr)\Bigl(\prod_{k > r^*} \xi^{c_k,m_k}_{i_{k-1}
    i_k} \cdot \vartheta^{c_k,m_k}_{i_{k-1} i_k}\Bigr),
    \ee
where, with a slight abuse of notation, we denote $\bm{i} = (i_1,
\dots, i_{t-1}) \in \{1,\dots,n\}^{t-1}$, $i_0 = i$ and $i_t = i'$. 
We thus have
\begin{equation*}
  \begin{split}
|\Im^{\bds{m},\bds{c}}_{ii'}| 
&\leq \sum_{\bds{i}}
\Big|a_{i_{r^* - 1}i_{r^*}} \prod_{j \not =
  r^*} \xi^{c_j,m_j}_{i_{j-1}i_j} \Bigr|
\cdot \Bigl|\Bigl(\frac{1}{d_{i_{r^*}}} -
\frac{1}{p_{i_{r^*}}}\Bigr)  \prod_{j \not = r^*}
\vartheta^{c_j,m_j}_{i_{j-1}i_j} \Bigr|
\\
&\leq \|\M D_{\M A}^{-1}\|^{s(\bds{c},\bds{m})}  \|\M
D_{\M P}^{-1}\|^{t - 1 - s(\bds{c},\bds{m})} 
\cdot \|\M D^{-1}_{\M A} - \M
D_{\M P}^{-1}\|  
 \Bigl(\sum_{\bds{i}} a_{i_{r^*} - 1i_{r^*}} \prod_{j \not =
  r^*} \xi^{c_j,m_j}_{i_{j-1}i_j} \Bigr).
\end{split}
\end{equation*}
Here $s(\bds{c}, \bds{m})$ is the number of indices $r$ with $r
\not = r^*$ and $(c_r, m_r) = (1,1)$. 
Now, by Lemma~\ref{l:sparse:rate} and Assumption~\ref{am:s}, we have
\bee\label{Ts:ho:mainIm}
\|\bds{\Im}_{\bds{m},\bds{c}}\|_{\max} 
&= \op\Bigg\{(n\rho_n)^{-1/2
  -t}(\log n)^{1/2} \cdot \max_{ii'}  \Bigl(\sum_{\bds{i}} a_{i_{r^*} - 1i_{r^*}} \prod_{j \not =
  r^*} \xi^{c_j,m_j}_{i_{j-1}i_j}\Big) \Bigg\}.
\ee
with the convention $i_0 = i$ and $i_{t} = i'$. 
Now define a matrix-valued function $\upxi_{\M A, \M P}(\cdot)$ by
\bee\nonumber
\upxi_{\M A, \M P}(c_r, m_r)=\begin{cases}
\M A & \text{if $(c_r, m_r) \in \{(1,1),(1,2),(2,1)\}$}, \\
\M P & \text{if $(c_r, m_r) = (2,2)$},
\end{cases}
\ee
Also define
\bee\nonumber
\upxi_{\M A,\M P}^{\bds{\Im}_{\bds{m},\bds{c}}} = \Bigl\{\prod_{j < r^*} \upxi_{\M A,\M
  P}(c_j, m_j)\Bigr\} \cdot \M A \cdot \Bigl\{\prod_{k > r^*}
\upxi_{\M A,\M P}(c_k, m_k)\Bigr\}.
\ee 
Then by the definition of the $\xi^{c_r,m}_{ii'}$ in
Eq.~\eqref{Ts:ho:defxitheta}, we have 
\bee
&\|\upxi_{\M A,\M P}^{\bds{\Im}_{\bds{m},\bds{c}}}\|_{\max} 
= \max_{ii'} \Bigl(\sum_{\bds{i}} a_{i_{r^*} - 1i_{r^*}}
  \prod_{j \not = r^*} \xi^{c_j,m_j}_{i_{j-1}i_j} \Bigr).
\ee
We now consider two cases to bound $\|\upxi_{\M A,\M P}^{\bds{\Im}_{\bds{m},\bds{c}}}\|_{\max}$.

\noindent\textbf{Case 1:} Suppose that for the given $\bds{m},\bds{c}$,
there exists at least one index $r \not = r^{*}$ with $(c_r, m_r) =
2$, i.e., the matrix $\M P$ appears at least once among the collection
of 
$\upxi_{\M A, \M P}(c_r, m_r)$ for $r \not = r^{*}$. Then $\upxi_{\M A,\M P}^{\bds{\Im}_{\bds{m},\bds{c}}}$ must have the form
\begin{equation}\label{Ts:ho:stru}
\begin{aligned}
&\underbrace{\cdots\cdots}_{g(\bds{c},\bds{m}) \,\, \text{matrices}} \ \M
P \M A\underbrace{\cdots\cdots}_{t - 2 - g(\bds{c},\bds{m})
  \,\, \text{matrices}}, \quad \text{or} \quad
\underbrace{\cdots\cdots}_{g(\bds{c},\bds{m}) \,\, \text{matrices}}\ \M A \M P
\underbrace{\cdots\cdots}_{t - 2 - g(\bds{c},\bds{m}) \,\,\text{matrices}}.
\end{aligned}
\end{equation}
Note that it is possible that $g(\bds{c},\bds{m}) = 0$ or $g(\bds{c},\bds{m}) = t - 2$ as $\M A$
could be either the first or last matrix in the product
$\upxi_{\M A,\M P}^{\bds{\Im}_{\bds{m},\bds{c}}}$.
Consider the first form in Eq.~\eqref{Ts:ho:stru}. We have
\begin{equation}
  \begin{split}
\|\upxi_{\M A,\M P}^{\bds{\Im}_{\bds{m},\bds{c}}}\|_{\max} 
&=
\|\underbrace{\cdots\cdots}_{g(\bds{c},\bds{m}) \,\, \text{matrices}} \ \M P
\M A\underbrace{\cdots\cdots}_{t - 2 - g(\bds{c},\bds{m}) \,\, \text{ matrices}}\|_{\max}
\\
& \overset{(1)}{\leq}
\begin{cases}
n\max_{i,i'} p_{ii'}
\cdot\|\underbrace{\cdots\cdots}_{g(\bds{c},\bds{m}) - 1 \,\, \text{matrices}}
\M P \M A\underbrace{\cdots\cdots}_{t - 2 - g(\bds{c},\bds{m}) \,\,
  \text{matrices}}\| 
  \\(\text{if first matrix is }\M P)
\\
\max_i d_i \cdot\|\underbrace{\cdots\cdots}_{g(\bds{c},\bds{m}) - 1
  \,\, \text{matrices}} \M P \M A\underbrace{\cdots\cdots}_{t - 2 -
  g(\bds{c},\bds{m}) \,\, \text{matrices}}\| 
  \\(\text{if first matrix is }\M A)
\end{cases}
\\
&\overset{(1)}{=}
\op\Big(n\rho_n\|\underbrace{\cdots\cdots}_{g(\bds{c},\bds{m} - 1) \,\,
  \text{matrices}} \ \M P \M A\underbrace{\cdots\cdots}_{t - 2 -
  g(\bds{c},\bds{m}) \,\, \text{matrices}}\|\Big)
  \\
  &\overset{(1)}{\leq}\dots
\overset{(1)}{\leq}\op\Big((n\rho_n)^{g(\bds{c},\bds{m})}\|\M
P\M A\underbrace{\cdots\cdots}_{t - 2 - g(\bds{c},\bds{m}) \,\, \text{matrices}}\|\Big)
\\
&\overset{(2)}{\leq}
\op\Big((n\rho_n)^{g(\bds{c},\bds{m})+1}\|\M P\M
A\underbrace{\cdots\cdots}_{t - 3 - g(\bds{c},\bds{m}) \,\,
  \text{matrices}}\|\Big)
  \\
  &\overset{(2)}{\leq}\dots
\overset{(2)}{\leq}\op\Big((n\rho_n)^{t-2}\|\M P \M A\|_{\max}\Big)
\leq\op\Big((n\rho_n)^{t-2}\max_{i,i'}p_{ii'} \max_i d_i \Big)
= \op(n^{t-1}\rho_n^{t}).
\end{split}
\end{equation}
In the above inequality, all relationships with $\overset{(1)}{\leq}$
and $\overset{(2)}{\leq}$ are when we removed either $\M P$ or $\M A$
from before and after the term $\M P \M A$, respectively. An identical argument also yields
\bee\label{Ts:ho:Im:C1}
\|\upxi_{\M A,\M P}^{\bds{\Im}_{\bds{m},\bds{c}}}\|_{\max} = \op(n^{t-1}\rho_n^{t})
\ee
for the second form in Eq.~\eqref{Ts:ho:stru}.
\par
\noindent\textbf{Case 2:} Suppose now that, for all $r \not = r^{*}$, we have
$(c_r, m_r) \not = (2,2)$, i.e., $\xi_{\mathbf{A}, \mathbf{P}}(c_r,
m_r) = \M A$ for all $r \not = r^{*}$. Then $\upxi_{\M A,\M
  P}^{\bds{\Im}_{\bds{m},\ba}} = \M A^t$ and, by the fact that $\frac{t-3}{t-1} > \beta$ already yields
$n^{\frac{2-t}{t}}\precsim \rho_n$ hold, we have by Lemma~\ref{l:hoA}, 
\bee\label{Ts:ho:Im:C2}
\|\upxi_{\M A,\M P}^{\bds{\Im}_{\bds{m},\ba}}\|_{\max} = \|\M A^t\|_{\max} = \op(n^{t-1}\rho_n^{t}).
\ee
Combining Eq.~\eqref{Ts:ho:xis}, Eq.~\eqref{Ts:ho:mainIm},
Eq.~\eqref{Ts:ho:Im:C1} and Eq.~\eqref{Ts:ho:Im:C2},
we have
\bee\label{Ts:ho:main4}
\Big\|\prod_{r = 1}^t\M \Xi_{c_r}\Big\|_{\max} 
 &\leq \sum_{\bds{m}
  \in \{1,2\}^{t}}\|
\bds{\Im}_{\bds{m},\bds{c}}\|_{\max} 
\\
&= 
\op\Big\{2^t\|\upxi_{\M A,\M P}^{\bds{\Im}_{\bds{m},\bds{c}}}\|_{\max}(n\rho_n)^{-1/2 -t} \log^{1/2}{n}\Big\} 
\\ &= \op\Big((n\rho_n)^{-1/2 -t}\log^{1/2}{n}\Big) \cdot \op\Big(n^{t-1}\rho_n^{t}\Big)
= \op\bigl\{n^{-3/2}\rho_n^{-1/2}\log^{1/2}{n}\bigr\}.
\ee
We then have, by Eq.~\eqref{Ts:ho:main1},
\bee\label{stepfinal}
\|\hat{\M W}^{t} - \M W^{t}\|_{\max} 
&\leq \|(\hat{\M W} - \M W)^t\|_{\max} + \op
\bigl(n^{-3/2} \rho_n^{-1/2}\log^{1/2}{n}\bigr)
\\
&\leq\|\{(\M A - \M P) \M D_{\M P}^{-1}\}^t
\|_{\max}+\sum_{\substack{\bds{c}\in \{1,2\}^t \\ \bds{c}\neq
    (1,\cdots,1)}}
\Big\|\prod_{r = 1}^t\M \Xi_{c_r}\Big\|_{\max}
+\op
\bigl(n^{-3/2}\rho_n^{-1/2}
\log^{1/2}{n}
\bigr)
\\
&\leq \big\|\bigl\{(\M A  - \M P) \M D_{\M P}^{-1}\bigr\}^t\big \|_{\max} + \op
\bigl(n^{-3/2} \rho_n^{-1/2}
\log^{1/2}{n}\bigr).
\ee
The last inequality in the above display follows from the fact that the number of distinct
$\bm{c}$ is $2^t - 1$ which is finite and does not depend on $n$.
Finally we focus on $\|\bigl\{(\M A - \M P) \M D_{\M
P}^{-1}\bigr\}^t \|_{\max}$. An argument similar to that for
$\|(\M A - \M P)^t\|_{\max}$ in the proof of Lemma~\ref{l:hoA} yields
\bee\label{apdpt}
\|\bigl\{(\M A - \M P) \M D_{\M P}^{-1}\bigr\}^t \|_{\max} 
&\leq\|\M A - \M P\|^t_2\cdot\|\M D_{\M P}^{-1}\|_2^t
\\
&\precsim (n\rho_n)^{-t}\|\M A - \M P\|^t_2 && \text{(By Lemma \ref{l:sparse:rate})}
\\
& = \mathcal{O}_{\mathbb{P}}\big\{(n\rho_n)^{-t/2}\big\} && (\text{By Eq.~\eqref{lupeng}}).
\ee
It could be directly checked that the rate $(n\rho_n)^{-t/2}\precsim n^{-3/2}\rho_n^{-1/2}\log^{1/2}{n}$ when $\frac{t-3}{t-1} > \beta$. Thus, together with Eq.~\eqref{Ts:ho:main4}, we conclude
\bee\label{Wttdiff}
&\|\hat{\M W}^{t} - \M W^{t}\|_{\max}  = \op
\big(n^{-3/2}\rho_n^{-1/2}\log^{1/2}{n}\big)
\ee
for $t\geq 4$ and $\frac{t-3}{t-1} > \beta$, as desired.\qed
\subsection{Proof of Theorem~\ref{T3} (Sparse Regime)}\label{pf:T3}
We will only present the proof of bounding $\|\tilde{\mathbf{M}}_0 -
\mathbf{M}_0 \|_{\F}$ and $\|\tilde{\mathbf{M}}_0 -
\mathbf{M}_0 \|_{\infty}$ here as the rates for $\|\M M_0\|_{\F}$ and $\|\M M_0\|_{\infty}$ are derived
in the proof of Theorem~\ref{c1} (c.f. Eq.~\eqref{MMFF:s}). 
For sufficient large $t_U > 0$ such that, w.h.p. $\tilde{\M M}_0$ is well defined, we use the same notations of $\M I_{\M A}, \M I_{\M P}$ as defined in Eq.~\eqref{T7:df1}. Then we also have
\bee\label{t:s:m:res1}
&\|\tilde{\M M}_0 - \M M_0\|_{\max,\text{off}} \leq\max_{i, i'}\Big(\frac{1}{\alpha_{ii'}}\Big)\cdot \|\M I_{\M A} - \M I_{\M P}\|_{\max,\text{off}},
\\
&\|\tilde{\M M}_0 - \M M_0\|_{\max,\text{diag}} \leq\max_{i,i'}\Big(\frac{1}{\alpha_{ii'}}\Big)\cdot \|\M I_{\M A} - \M I_{\M P}\|_{\max,\text{diag}},
\ee
where $\alpha_{ii'}\in (\min\{I^{ii'}_{\M A},I ^{ii'}_{\M P}\},\max\{I_{\M A}^{ii'},\, I_{\M P}^{ii'}\})$, $I^{ii'}_{\M A}, I^{ii'}_{\M P}$ are $ii'$th entries of $\M I_{\M A}, \M I_{\M P}$. We now extend the argument in the dense regime to the sparse regime. 

\noindent{\textbf{Step 1 (Bounding $\max_{i,i'}\alpha_{ii'}^{-1}$):}} 
We also have $\alpha_{ii'}$ is between $I_{\M A}^{ii'}$ and $I_{\M P}^{ii'}$ and $\max_{i,i'}\frac{1}{\alpha_{ii'}}\leq \max_{i,i'}\Big\{\frac{1}{I^{ii'}_{\M A}},\frac{1}{I^{ii'}_{\M P}}\Big\}.$ By Lemma~\ref{l:sparse:rate} and Assumption~\ref{am:s}, we have
\bee\nonumber
\max_{i,i'}\frac{1}{I_{\M P}^{ii'}} &= \max_{i,i'}\frac{1}{\sum_{t = t_L}^{t_U}(L - t)\Big(\sum_{ii'}{p_{ii'}}\Big)\cdot\Big(\frac{{w}_{ii'}^{(t)}}{p_i}\Big)}
\leq \max_{i,i'}\frac{p_i}{(L-t_L)(\sum_{ii'}p_{ii'})(w^{(t_L)}_{ii'})}
\precsim \frac{n\rho_n}{n^2\rho_n\cdot \frac{1}{n}} = 1.
\ee
We also have $\frac{1}{I^{ii'}_{\M A}} \leq \frac{d_i}{(L-t_L)(\sum_{ii'}a_{ii'})(\hat{w}_{ii'}^{(t_L)})}$ and we consider off-diagonal and diagonal cases separately. 

\textbf{(1) \text{When $i\neq i'$:}} By Theorem~\ref{T3} and Lemma~\ref{l:sparse:rate}, we have $\min_{i,i'}w^{(t_L)}_{ii'}, \max_{i,i'}w^{(t_L)}_{ii'}\asymp n^{-1}$ for fixed $2 \leq t_L$ and 
\bee\label{wwtww}
\max_{i \neq {i'}}|w^{(t_L)}_{ii'} - \hat{w}^{(t_L)}_{ii'}|/\min_{i,i'}w^{(t_L)}_{ii'}
&= \begin{cases}
\op\Big(n^{-3/2}\rho_n^{-1}\log^{1/2}{n}/(1/n)\Big) = \mathrm{o}_{\M P}(1) &
(\text{when } t_L \geq 2 \text{ and }\beta <\frac{1}{2})
\\
\op\Big(n^{-3/2}\rho_n^{-1/2}\log^{1/2}{n}/(1/n)\Big) = \mathrm{o}_{\M P}(1) 
&
(\text{when } t_L \geq 4 \text{ and }\beta<\frac{t_L - 3}{t_L-1})
\end{cases}
\ee
which implies 
\bee\nonumber
0< {\min_{i,i'}w^{(t_L)}_{ii'} - \max_{i \neq {i'}}|w^{(t_L)}_{ii'} - \hat{w}^{(t_L)}_{ii'}|}\asymp\min_{i,i'}w^{(t_L)}_{ii'} = \op(1/n).
\ee
Furthermore we have
\bee\nonumber
\max_{i \neq {i'}}\frac{1}{w^{(t_L)}_{ii'} - \max_{i \neq {i'}}|w^{(t_L)}_{ii'} - \hat{w}^{(t_L)}_{ii'}|} 
= \frac{1}{\min _{i\neq i'}\big(w^{(t_L)}_{ii'}\big) - \max_{i \neq {i'}}|w^{(t_L)}_{ii'} - \hat{w}^{(t_L)}_{ii'}|} 
= \op(n), \\
\max_{i \neq {i'}}\frac{1}{I^{ii'}_{\M A}} 
\leq \max_{i\neq i'}\frac{d_i}{(L-t_L)(\sum_{ii'}a_{ii'})(\hat{w}_{ii'}^{(t_L)})} 
\leq \max_id_i \cdot \max_{i}\frac{1}{nd_i} 
\cdot \max_{i \neq {i'}}\frac{1}{w^{(t_L)}_{ii'} - \max_{i \neq {i'}}|w^{(t_L)}_{ii'} - \hat{w}^{(t_L)}_{ii'}|}
 = \op(1).
\ee
\textbf{(2) \text{When $i= i'$:}} Similarly, when $t_L = 2$, we have $\max_{i}\frac{1}{\hat{w}_{ii}^{(2)}} = \max_{i}\frac{d_i}{\sum_{i' = 1}^{n}a_{ii'}/d_{i'}} \leq {(\max_i d_i)^2} \cdot \max_i \frac{1}{d_i} = \op(n\rho_n)$. Thus
\bee\nonumber
\max_{i}\frac{1}{I^{ii}_{\M A}}\precsim \max_{i}d_i\cdot\max\frac{1}{nd_i} \cdot \max_i\frac{1}{\hat{w}_{ii}^{(2)}} = \op(\rho_n).
\ee
When $t_L \geq 3$, since $\|\hat{\M W}^t - \M W^t\|_{\max,\text{diag}}
= \op\big(n^{-3/2}\rho_n^{-1/2}\sqrt{\log n}\big)$, a similar argument as the $i \neq i'$ case shows
\bee\nonumber
\max_{i}\frac{1}{I^{ii}_{\M A}} &\leq  \max_{i}\frac{d_i}{(L-t_L)(\sum_{ii'}a_{ii})[(\hat{w}_{ii}^{(t_L)}-w^{(t_L)}_{ii})+w^{(t_L)}_{ii}]} 
= \op(1).
\ee
In summary we have both $\max_{i,i'}\frac{1}{I^{ii'}_{\M A}} = \op(1)$ and $\max_{i,i'}\frac{1}{I^{ii'}_{\M P}} = \op(1)$ under the conditions of Theorem~\ref{T3} (sparse regime). We thus conclude
\bee\label{t:s:m:res2}
\max_{i,i'}\frac{1}{\alpha_{ii'}} = \op(1).
\ee
\noindent{\textbf{Step 2 (Bound of $\|\M I_{\M A} - \M I_{\M P}\|_{\max,\text{off}}$ and $\|\M I_{\M A} - \M I_{\M P}\|_{\max,\text{diag}}$):}} 
We consider two cases for $(t_L, \beta)$

\noindent{\bf\textit{Case 1 ($t_L \geq 2$ and $\beta<0.5$}):} For $t \geq 2$ we have
\bee\label{ts:m:apo}
&\noindent{\textbf{(a).}}
\max_{i, {i'}}\Big|
\Big(\sum_{ii'}a_{ii'}-\sum_{ii'}p_{ii'}\Big)\frac{w^{(t)}_{ii'}}{d_i}\Big|
\leq n\max_{i \neq {i'}}|d_i - p_i|\cdot \max_{i,{i'}}
w_{ii'}^{(t)}\cdot \max_{i}\frac{1}{d_i} 
=
\op((n\rho_n)^{-1/2}\sqrt{\log n}),
\\
&\noindent{\textbf{(b).}}\max_{i,
  {i'}}\Big|\sum_{ii'}p_{ii'}\Big(\frac{1}{d_i} -
\frac{1}{p_i}\Big)w^{(t)}_{ii'}\Big| 
=
\Big|\sum_{ii'}p_{ii'}\Big|\cdot \max_{i \neq {i'}}\Big|\frac{p_i -
  d_i}{p_id_i}\Big|\cdot \max_{i,i'} w^{(t)}_{ii'} =
\op\big((n\rho_n)^{-1/2}\sqrt{\log n}\big),
\\
&\noindent{\textbf{(c).}}\max_{i \neq
  {i'}}\Big|\Big(\sum_{ii'}a_{ii'}\Big)\frac{\hat{w}_{ii'}^{(t)} -
  w_{ii'}^{(t)}}{d_i}\Big| 
 \leq n\max_i d_i\cdot \max_{i \neq
  {i'}}|\hat{w}_{ii'}^{(t)} - w^{(t)}_{ii'}|\cdot
\max_{i}\frac{1}{d_i} 
= \op\big(n^{-1/2}\rho_n^{-1}\sqrt{\log n}\big),
\\
&\noindent{\textbf{(d).}}\max_{i}\Big|\Big(\sum_{ii'}a_{ii'}\Big)\frac{\hat{w}_{ii}^{(t)} - w_{ii}^{(t)}}{d_i}\Big|
 \leq n\max_i d_i\cdot \max_{i}|\hat{w}_{ii}^{(t)} - w^{(t)}_{ii}|\cdot \max_{i}\frac{1}{d_i} 
= \op\big(\rho_n^{-1}\big),
\ee
by Lemma~\ref{l:sparse:rate}, Theorem~\ref{T3} and Eq.~\eqref{T10:extend}. Thus a similar argument as the proof under dense regime gives
\bee\label{t:s:m:res3}
&\|\M I_{\M A} - \M I_{\M P}\|_{\max,\text{off}} =
\op\big(n^{-1/2}\rho_n^{-1}\sqrt{\log n}\big), \quad
&\|\M I_{\M A} - \M I_{\M P}\|_{\max,\text{diag}} = \op\big(\rho_n^{-1}\big).
\ee
\noindent{\bf\textit{Case 2 ($t_L \geq 4$ and $\beta<\frac{t_L - 3}{t_L - 1}$ holds}):} We note that $\textbf{(a)}$ and $\textbf{(b)}$ in Eq.~\eqref{ts:m:apo} do not change for $t\geq t_L \geq 3$. By Eq.~\eqref{T5:6} we further have
\bee\label{t:s:m:res4}
\max_{i,{i'}}\Big|\Big(\sum_{ii'}a_{ii'}\Big)\frac{\hat{w}_{ii'}^{(t)}
  - w_{ii'}^{(t)}}{d_i}\Big| =\op\big\{(n\rho_n)^{-1/2}\sqrt{\log n}\big\}
\ee
for all $t \geq t_L \geq 3$, which implies, in this case, 
\bee\nonumber
\|\M I_{\M A} - \M I_{\M P}\|_{\max,\text{off}}, \|\M I_{\M A} - \M
I_{\M P}\|_{\max,\text{diag}} = \op\big\{(n\rho_n)^{-1/2}\sqrt{\log n}\big\}.
\ee
\\
\par
\noindent{\textbf{Step 3 (Bounding $\|\tilde{\M M}_0 - \M M_0\|_{\F}$ and $\|\tilde{\M M}_0 - \M M_0\|_{\infty}$):}} 
From Eq.~\eqref{t:s:m:res1}, Eq.~\eqref{t:s:m:res2}, Eq.~\eqref{t:s:m:res3} and Eq.~\eqref{t:s:m:res4}, we obtain
\bee\nonumber
\|\tilde{\M M}_0 - \M M_0\|_{\max,\text{off}}
&= \begin{cases}
\op\big(n^{-1/2}\rho_n^{-1}\sqrt{\log n}\big) & \text{when } t_L \geq 2 \text{ and }\beta <\frac{1}{2},
\\
\op\big\{(n\rho_n)^{-1/2}\sqrt{\log n}\big\}& \text{when } t_L \geq 4 \text{ and }\beta<\frac{t_L - 3}{t_L - 1},
\end{cases}
\\
\|\tilde{\M M}_0 - \M M_0\|_{\max,\text{diag}}
&= \begin{cases}
\op\big(\rho_n^{-1}\big) & \text{when } t_L \geq 2 \text{ and }\beta <\frac{1}{2},
\\\op\big\{(n\rho_n)^{-1/2}\sqrt{\log n}\big\}& \text{when } t_L \geq 4 \text{ and }\beta<\frac{t_L - 3}{t_L - 1}.
\end{cases}
\ee
and hence
\bee\nonumber
\|\tilde{\M M}_0-\M M_0\|_{\F} 
&\leq  \BL n^2 \|\tilde{\M M}_0 - \M M_0\|^2_{\max,\text{off}} + n \|\tilde{\M M}_0 - \M M_0\|^2_{\max,\text{diag}}  \BR^{1/2}
= \begin{cases}
\op\big(n^{1/2}\rho_n^{-1}\log^{1/2}{n}\big) & \text{when } t_L \geq 2 \text{ and }\beta <\frac{1}{2},
\\\op\big(n^{1/2}\rho_n^{-1/2}\log^{1/2}\big)& \text{when } t_L \geq 4 \text{ and }\beta<\frac{t_L - 3}{t_L - 1},
\end{cases} \\
\label{M0Minfty}
\|\tilde{\M M}_0-\M M_0\|_{\infty}
&\leq   n\|\tilde{\M M}_0 - \M M_0\|_{\max,\text{off}} +  \|\tilde{\M M}_0 - \M M_0\|_{\max,\text{diag}}  
= \begin{cases}
\op\big(n^{1/2}\rho_n^{-1} \log^{1/2}{n} \big) & \text{when } t_L \geq 2 \text{ and }\beta <\frac{1}{2},
\\\op\big(n^{1/2}\rho_n^{-1/2}\log^{1/2}{n}\big)& \text{when } t_L \geq 4 \text{ and }\beta<\frac{t_L - 3}{t_L - 1}.
\end{cases}
\ee
as desired \qed.
\subsection{Proof of Theorem~\ref{c1} (Sparse Regime)}\label{pfthm4}
Same as the dense regime, we shall assume $n_k = n\pi_{k}, \text{ for all }k\in [K]$. From Eq.~\eqref{truth}, one key observation is that when we write (1) $\M B = \rho_n \M B_0$ and $\M B_0$ is a constant matrix; (2) $n_k = n\pi_{k}, \text{ for all }k\in [K]$; the ${\M M}_0$ built on $\M B$ or $\M B_0$ are identical and do not depend on $n$. We therefore have
\bee\nonumber
\M M_0	&= \begin{pmatrix}
\xi_{11} \bds 1_{n\pi_1} \bds 1_{n\pi_1}^\T &\dots & \xi_{1K} \bds 1_{n\pi_1} \bds 1_{n\pi_K}^\T\\
\vdots & \vdots & \vdots\\
\xi_{K1} \bds 1_{n\pi_K} \bds 1_{n\pi_1}^\T & \dots & \xi_{KK} \bds 1_{n\pi_K} \bds 1_{n\pi_K}^\T
\end{pmatrix} 
= 
\M \Theta \M \Xi \M \Theta^\T
\ee
where $\M \Xi := (\xi_{ii'})_{K \times K}$ is a fixed matrix. The remaining steps in the proof of Theorem \ref{c1} (dense regime) still hold, e.g., 
\bee\label{MMFF:s}
\|\M M_0\|_\F = \Theta(n), \quad \|\M M_0\|_{\infty} =\Theta(n), \quad \sigma_{d}(\M M_0)=\Theta(n), \quad \|\M U\|_{\twoinf} \precsim n^{-1/2}.
\ee
Applying Theorem \ref{T3} (sparse regime) and similar to Eq.~\eqref{fnormbound} and Eq.~\eqref{twoinfnormbound}, one has
\begin{align}
\label{spectral1a:sparse}
\min_{\M T \in \mathbb{O}_{d}} \|\hat{\bds{\mathcal{F}}} \cdot \mathbf{T} - \mathbf{U}\|_{\F} 
& \leq \frac{\|\tilde{\M M}_0 - \M M_0\|_\F}{\sigma_d(\M M_0)}
=\begin{cases}
\op\big\{(n\rho_n)^{-1/2}\rho_n^{-1/2}\log^{1/2}{n}\big\} & \text{if  } t_L \geq 2 \text{ \& }\beta <\frac{1}{2},
\\\op\big\{(n\rho_n)^{-1/2}\log^{1/2}{n}\big\}& \text{if } t_L \geq 4 \text{ \& }\beta<\frac{t_L - 3}{t_L - 1},
\end{cases} \\
\label{twoinf:sparse}
\min_{\M T \in \mathbb{O}_{d}} \|\hat{\bds{\mathcal{F}}} \cdot \mathbf{T} - \mathbf{U}\|_{\twoinf} 
&\leq 14\Bigg(\frac{\|\tilde{\M M}_0 - \M M_0\|_{\infty}}{\sigma_d(\M M_0)}\Bigg)\|\M U\|_{\twoinf}
=\begin{cases}
\op\big(n^{-1}\rho_n^{-1}\log^{1/2}{n}\big) & \text{if } t_L \geq 2 \text{ \& }\beta <\frac{1}{2},
\\\op\big(n^{-1}\rho_n^{-1/2}\log^{1/2}{n}\big)& \text{if } t_L \geq 4 \text{ \& }\beta<\frac{t_L - 3}{t_L - 1}.
\end{cases}
\end{align}
Finally it is easy to check that if $t_{L}\geq 2$ and $\beta < 1/2$ then
\bee\nonumber
n^{1/2} \min_{\M T \in \mathbb{O}_{d}} \|\hat{\bds{\mathcal{F}}} \cdot \mathbf{T} - \mathbf{U}\|_{\twoinf} &\precsim (n\rho_n)^{-1/2}\rho_n^{-1/2}\log^{1/2}{n}
\precsim n^{\beta - 1/2} \log^{1/2}{n}
\rightarrow 0
\ee
with high probability. When $t_L \geq 4$ and $\beta<\frac{t_L - 3}{t_L - 1}$ then
\bee\label{n12123}
n^{1/2}\cdot \min_{\M T \in \mathbb{O}_{d}} \|\hat{\bds{\mathcal{F}}} \cdot \mathbf{T} - \mathbf{U}\|_{\twoinf} &\precsim (n\rho_n)^{-1/2} \log^{1/2}{n}
\rightarrow 0
\ee
with high probability.
Similar to the dense regime, the above $2 \to \infty$ results together with the technical arguments in \cite{lei2019unified} show that applying the $K$-medians algorithm to cluster the rows of $\hat{\bds{\mathcal{F}}}$ will recover the memberships for all nodes with high probability. 
\qed

\subsection{Proof of Remark~\ref{rk:improved}}\label{sec:rk6:pf}
We now show that the condition $\beta < \frac{t_L - 2}{t_L}$, which is weaker than that assumed in Theorem~\ref{c1}, is actually sufficient to achieve exact recovery under the sparse regime of $\rho_n \succsim n^{-\beta}$ for $\beta \in (0,1]$. We will follow the same proof strategy as that presented earlier with the main changes being that we perform more careful book-keeping when bounding the important terms in Section~\ref{sec:pf:ts:main} and Section~\ref{pf:T3}.

More specifically all the arguments in Section~\ref{sec:pf:ts:main} are still valid up to (and including) Eq.~\eqref{apdpt}.
The rate shown in Eq.~\eqref{apdpt} is, however, only negligible compared to the term $\mathcal{O}_{P}(n^{-3/2} \rho_n^{-1/2} \log^{1/2}{n})$ in Eq.~\eqref{stepfinal}for $\frac{t_L - 3}{t_L - 1} > \beta$, i.e., if we only assume $\frac{t_L - 2}{t_L} > \beta$ then we have to include the error term in Eq.~\eqref{apdpt} by replacing Eq.~\eqref{Wttdiff} the following bound
\bee\label{WtWnew}
\|\hat{\M W}^{t} - \M W^{t}\|_{\max} =
\op\Bigg\{\frac{\log^{1/2}{n}}{n^{3/2} \rho_n^{1/2}} + (n\rho_n)^{-t/2}\Bigg\}.
\ee
Noting above equation implies that $\|\hat{\M W}^t - \M W^t\|_{\max}=o_{\mathbb{P}}(1)$ as implied by the following Eq.~\eqref{followbound}, then similar to the discussion in Section \ref{sec:dense}, $\tilde{\M M}_0$ is thus well-defined with high probability. We can then follow the same arguments as that presented in
Section~\ref{pf:T3}, taking care to replace the bound in Eq.~\eqref{Wttdiff} with that in Eq.~\eqref{WtWnew}. 
Eq.~\eqref{WtWnew} to the Proof of Theorem \ref{T3} in Section~\ref{pf:T3}. In particular Eq.~\eqref{wwtww} now becomes
\bee\label{followbound}
&\max_{i \neq {i'}}|w^{(t_L)}_{ii'} - \hat{w}^{(t_L)}_{ii'}|/\min_{i,i'}w^{(t_L)}_{ii'}
=
\op\Bigg\{\frac{\log^{1/2}{n}}{n^{1/2} \rho_n^{1/2}} + n\cdot(n\rho_n)^{-t_L/2}\Bigg\}
= o_{\mathbb{P}}(1)
\ee
as     
$(\log^{1/2}{n})/(n^{1/2} \rho_n^{1/2}) \precsim \log^{1/2} n^{(\beta - 1)/2}\rightarrow 0$ 
and $n\cdot(n\rho_n)^{-t_L/2} \precsim n^{\{2 - (1 - \beta)t_L\}/2} \rightarrow 0$ for $\beta < \frac{t_L - 2}{t_L}$. 
Meanwhile Eq.~\eqref{t:s:m:res4} becomes
$$
\max_{i,{i'}}\Big|\Big(\sum_{ii'}a_{ii'}\Big)\frac{\hat{w}_{ii'}^{(t)}
  - w_{ii'}^{(t)}}{d_i}\Big| =\op\Bigg\{\frac{\log^{1/2}{n}}{n^{1/2} \rho_n^{1/2}} + n\cdot(n\rho_n)^{-t_L/2}\Bigg\}
$$
and in turn Eq.~\eqref{M0Minfty} becomes
$$
\|\tilde{\M M}_0-\M M_0\|_{\infty}=\op\Bigg\{\frac{n^{1/2}\log^{1/2}{n}}{ \rho_n^{1/2}} + n^2\cdot(n\rho_n)^{-t_L/2}\Bigg\}.
$$
Finally, using the above bound $\|\tilde{\M M}_0-\M M_0\|_{\infty}$, we can proceed with the same proof of Theorem~\ref{c1} in Section \ref{pfthm4} and obtain, in place of Eq.~\eqref{twoinf:sparse}, the bound
\bee
\label{eq:hat_F_alternative}
 &\min_{\M T \in \mathbb{O}_{d}} \|\hat{\bds{\mathcal{F}}} \cdot \mathbf{T} - \mathbf{U}\|_{\twoinf} 
\leq 14\Bigg(\frac{\|\tilde{\M M}_0 - \M M_0\|_{\infty}}{\sigma_d(\M M_0)}\Bigg)\|\M U\|_{\twoinf}
=\op\Bigg\{\frac{\log^{1/2}{n}}{ n^{3/2}\rho_n^{1/2}} + (n\rho_n)^{-t_L/2}\Bigg\}.
\ee
Eq.~\eqref{n12123} then becomes
\bee
n^{1/2}\cdot \min_{\M T \in \mathbb{O}_{d}} \|\hat{\bds{\mathcal{F}}} \cdot \mathbf{T} - \mathbf{U}\|_{\twoinf} 
\precsim  n^{-1/2}\rho_n^{-1/2}\log^{1/2}n + n^{-t_L/2 + 1}\rho_n^{-{t_L/2}}
\precsim  o(1) + n^{\beta t_L/2 - t_L/2 + 1}
 = o(1),
\ee
The technical arguments in \cite{lei2019unified} once again show that clustering the rows of $\hat{\bds{\mathcal{F}}}$ using $K$-medians will, with high probability recovers the memberships of all nodes.
\qed
\subsection{Strong recovery for DCSBM (Sparse Regime)}\label{sec:sparsedcebm}
The exact same arguments as that presented in Section~\ref{sec:dcsbmsr} also apply to the sparse regime where $\rho_n \succsim n^{-\beta}$ for some $\beta \in [0,1)$. Indeed, the matrix $\mathbf{M}'_0$ and $\mathbf{M}_0$ in Section~\ref{sec:dcsbmsr} does not depend on the sparsity $\rho_n$. Rather $\rho_n$ only affects the convergence rate of $\tilde{\mathbf{M}}_0$ to $\mathbf{M}_0$ which in turns lead to a slower convergence rate for $\hat{\bds{\mathcal{F}}}$ to $\mathbf{U}$. This convergence rate is, however, still the same as that in Theorem~\ref{c1} provided that the $\theta_i$ are homogeneous (so that $\max_{i} \theta_i / \min_{i} \theta_i \precsim 1$). Clustering the rows of $\hat{\bds{\mathcal{F}}}$ will thus, with high probability, recover the membership of all nodes. 

\section{Additional proofs and derivations}
\label{app:lemma}
\subsection{Proof of Lemma~\ref{l1}}
\textbf{(1):} The sum of $i$th row of $\M A \M D_{\M A}^{-1}$ is 
\bee\nonumber
\sum_{i=1}^{n} \frac{a_{ii'}}{d_{i'}} = \frac{\sum_{i=1}^{n}a_{ii'}}{\sum_{i=1}^{n}a_{ii'}} = 1.
\ee So $\M 1_n\tp \cdot \hat{\M W}  =\M 1_n\tp \cdot \M A \M D_{\M A}^{-1}  = \M 1_n\tp$ and
\bee\nonumber
\M 1_n\tp \cdot \hat{\M W}^{t} &= \M1_n \tp \underbrace{\M A \M D_{\M A}^{-1} \dots \M A \M D_{\M A}^{-1}}_{t\text{ items}} 
\\ &= (\M1_n\tp \cdot \M A \M D_{\M A}^{-1})\underbrace{\M A \M D_{\M A}^{-1} \dots \M A \M D_{\M A}^{-1}}_{t-1\text{ items}}
= \M1_n\tp\underbrace{\M A \M D_{\M A}^{-1} \dots \M A \M D_{\M A}^{-1}}_{t-1\text{ items}} 
= \cdots 
= \M 1_n\tp.
\ee
With similar argument we have $\M 1_n\tp\M W^t  = \M 1_n\tp$.

\noindent{\textbf{(2):}} The $i$th element of $\hat{\M W}\bds{d}$ is 
\bee\nonumber
\sum_{j = 1}^n\frac{a_{ii'}}{d_{i'}}d_{i'} = d_i
\ee
and hence $\hat{\M W}\bds{d} = \bds{d}$. This implies
$\hat{\M W}^t \bds{d} = \bds{d}$. The same argument yields $\M W^t \bds{p} = \bds{p}$.\qed
\subsection{Proof of Lemma~\ref{l:dense:rate}}
\noindent{\textbf{(1) $\|\M D_{\M A}\|_{}$, $\|\M D^{-1}_{\M A}\|_{}$: }}With Chernoff bound, since $c_0\leq p_{ii'}\leq c_1$ for any fixed $i$ we can get
\bee\nonumber
&\p\BL \frac{d_i}{n} > \frac{3c_1}{2}\BR\leq \p\Big(d_i > \frac{3}{2}p_i\Big) \precsim \exp\big(-C_0\cdot n\big)
\\
&\p\BL \frac{1}{d_i}\Big/\frac{1}{n} > \frac{2}{c_0}\BR \leq \p\Big({d_i} \leq \frac{p_i}{2}\Big) \precsim \exp\big(-C_1\cdot n\big)
\ee
for some $C_0,C_1 > 0$ by taking appropriate constant in general Chernoff bound. So we have
\bee\nonumber
&\p\BL \frac{\max_{1\leq i \leq n}d_i}{n} 
\leq \frac{3c_1}{2}\BR\precsim n\exp\big(-C_0\cdot n\big) \longrightarrow 0
\\
&\p\Big\{\Big(\max_{1\leq i \leq n}\frac{1}{d_i}\Big)\Big/\frac{1}{n} \leq \frac{2}{c_0}\Big\} \precsim n\exp\big(-C_1\cdot n\big) \longrightarrow 0
\ee
as $n \rightarrow \infty$. We therefore have $\|\M D_{\M A}\| = \max_{1\leq i \leq n}d_i = \op(n),\|\M D_{\M A}^{-1}\| = \max_{1\leq i \leq n} 1/d_i = \op(1/n)$.

\noindent{\textbf{(2) $\|\M D_{\M A} - \M D_{\M P}\|_{}$: }}For a given $i$ the $a_{i1},\dots,a_{in}$ are independent and $\E a_{ii'} = p_{ii'}$ for any $1\leq i,{i'} \leq n$. Also $|a_{ii'} - p_{ii'}|\leq 1$. By Bernstein inequality, we have
\bee\nonumber
\p\Big(\Big|\sum_{j = 1}^{n}a_{ii'} - \sum_{j = 1}^{n}p_{ii'}\Big|> \tilde{t}\Big) \leq 2 \exp\Big(-\frac{\tilde{t}^2}{2\sigma_0^2n + \frac{2}{3}\tilde{t}}\Big),
\ee
where $\sigma_0^2 = \frac{1}{n}\sum_{j = 1}^{n}\text{Var}(a_{ii'} - p_{ii'}) = \frac{1}{n}\sum_{j=1}^{n}p_{ii'}(1-p_{ii'}).$ So we have $\sigma_0^2\leq \frac{1}{n}\sum_{j=1}^{n}(p_{ii'} + 1-p_{ii'})^2/4 = 1/4$ and 
\bee\label{l2:bern1}
\p\Big(\Big|\sum_{j = 1}^{n}a_{ii'} - \sum_{j = 1}^{n}p_{ii'}\Big|> \tilde{t}\Big) \leq 2 \exp\Big(-\frac{\tilde{t}^2}{\frac{1}{2}n + \frac{2}{3}\tilde{t}}\Big).
\ee
Take $\tilde{t} = c\sqrt{n \log n}$ in Eq.~\eqref{l2:bern1}, we have
\bee\nonumber
\p\Big(\Big|\sum_{j = 1}^{n}a_{ii'} - \sum_{j = 1}^{n}p_{ii'}\Big|>
c\sqrt{n \log n}\Big)
&\leq 2 \exp\Big(-\frac{c^2 n \log
  n}{\frac{1}{2}n + \frac{2}{3}c\sqrt{n \log n}}\Big)
\precsim\exp\Big(-2c^2 \log n\Big) 
= \Big(\frac{1}{n}\Big)^{2c^2}.
\ee
Combining all the events among $i \in \{1,\dots,n\}$, we get
\bee\nonumber
\p\BL \|\M D_{\M A} - \M D_{\M P}\|_{} > c\sqrt{n \log n}\BR 
&=
\p\BL\max_{1 \leq i\leq n}\Big|\sum_{j = 1}^{n}a_{ii'} - \sum_{j =
  1}^{n}p_{ii'}\Big|> c\sqrt{n \log n} \BR 
\\ &= \p\BL \bigcup_{i = 1}^{n}\Big\{\Big|\sum_{j = 1}^{n}a_{ii'} -
\sum_{j = 1}^{n}p_{ii'}\Big|> c\sqrt{n \log n}\Big\}\BR 
\precsim n\cdot \BL \frac{1}{n}\BR^{2c^2} 
= \BL \frac{1}{n}\BR^{2c^2-1},
\ee
which implies $\|\M D_{\M A} - \M D_{\M P}\|_{} = \op(\sqrt{n \log (n)})$.

\noindent{\textbf{(3) $\|\M D_{\M A}^{-1} - \M D_{\M P}^{-1}\|_{}$: }}
With the results in \textbf{(1)} and \textbf{(2)}, we have
\bee\nonumber
\|\M D_{\M A}^{-1} - \M D_{\M P}^{-1}\|_{} &= \max_i \Big| \frac{1}{d_i}- \frac{1}{p_i}  \Big| = \max_i \Big|\frac{p_i - d_i}{d_i p_i} \Big|
\leq \|\M D_{\M A} - \M D_{\M P}\|_{}\cdot \|\M D_{\M
  A}^{-1}\|_{}\cdot\|\M D_{\M P}^{-1}\|_{}
 = \op\big(n^{-3/2}\sqrt{\log n}\big).
\ee

\noindent{\textbf{(4) $\|\hat{\M W}^{t}\|_{\max}$, $\max_{i,{i'}} w^{(t)}_{ii'} \asymp 1/n$ and $\min_{i,{i'}} w^{(t)}_{ii'} \asymp 1/n$: }}By result in \textbf{(1)} and $\M A, \M P$ are bounded by 1 elementwisely, it is easy to see that
\bee\label{T1:order_0}
\|\hat{\M W}\|_{\max} = \|\M A \M D_{\M A}^{-1}\|_{\max} &\leq \|\M A\|_{\max} \cdot \|\M D_{\M A}^{-1}\|_{} 
\leq   \max_{1\leq i\leq n}1/d_{i} 
= \op(n^{-1}).
\ee
For any $t\geq 1$,
\bee\nonumber
\|\hat{\M W}^{t}\|_{\max} &\leq n\|\hat{\M W}^{t - 1}\|_{\max} \|\hat{\M W}\|_{\max} 
\leq n^2 \|\hat{\M W}^{t - 2}\|_{\max} \|\hat{\M W}\|_{\max}^2
\leq \dots \leq n^{t - 1}\|\hat{\M W}\|_{\max}^{t}
 = \op(n^{-1}).
\ee
Since $c_0 \leq p_{ii'} \leq c_1$, we could also see 
\bee\label{lemma:rate:wt}
\frac{c_0^t}{c_1^t n} \leq w^{(t)}_{ii'}\leq \frac{c_1^t}{c_0^t n}
\ee
for any given $t\geq 1$ and $1 \leq i,{i'}\leq n$, which implies $\min_{i,i'}w^{(t)}_{ii'}, \max_{i,i'}w^{(t)}_{ii'}\asymp n^{-1}$.\qed
\subsection{Proof of Lemma~\ref{l:sparse:rate}}
\noindent{\textbf{(1)} Bounding \textbf{ $\|\M D_{\M A}\|_{}$, $\|\M D^{-1}_{\M A}\|_{}$: }}
As $c_2\rho_n\leq p_{ii'}\leq c_3\rho_n$ and $\E d_i = p_i$, we have by Chernoff bounds that
\bee\label{A6:1}
\log\Big\{\p\BL \frac{d_i}{n} >\frac{3c_3}{2}\rho_n\BR\Big\}
&\leq \log\Big\{\p\Big(d_i > \frac{3}{2}p_i\Big)\Big\}\leq -C_0n\rho_n, 
\\
\log\Big\{\p\BL \frac{n}{d_i} > \frac{1}{2c_2\rho_n}\BR\Big\}
&\leq\log\BL \p\Big({d_i} \leq \frac{p_i}{2}\Big)\BR \leq -C_0n\rho_n, 
\ee
for any $i \in \{1,2,\dots,n\}$ where $C_0 > 0$ is a constant not depending on $i$ or $n$. Eq.~\eqref{A6:1} together with a union bound imply
\bee
\label{A6:Da}
&\p\BL \frac{\max_{i}d_i}{n} > \frac{3c_3\rho_n}{2}\BR\leq  n \max_i \Big\{\p\BL {d_i}> \frac{3p_i}{2}\BR\Big\}
\leq \exp\bigl(\log n - C_0 n \rho_n \bigr) \longrightarrow 0,
\\
&\p\BL\max_{i}\frac{n}{d_i} > \frac{1}{2c_2\rho_n}\BR \leq n\max_i\Big\{ \p\Big({d_i} \leq \frac{p_i}{2}\Big)\Big\}
\leq \exp\bigl(\log n - C_0 n \rho_n \bigr) \longrightarrow 0,
\ee
as $n \longrightarrow \infty$. So $\|\M D_{\M A}\| = \max_{i}d_i = \op(n\rho_n),\|\M D_{\M A}^{-1}\| = \max_{i} 1/d_i = \op\bigl((n\rho_n)^{-1}\bigr)$.

\noindent{\textbf{(2)} Bounding $\|\M D_{\M A} - \M D_{\M P}\|_{}, \|\M D^{-1}_{\M A} - \M D^{-1}_{\M P}\|_{}$:} For a given vertex $i$ the $\{a_{i1},\dots,a_{in}\}$ are independent Bernoulli random variables with $\E a_{ii'} = p_{ii'}$. Therefore, by Bernstein inequality, we have
\bee\label{l1s:bern}
\p\Big(\Big|\sum_{i' = 1}^{n}a_{ii'} - \sum_{i' = 1}^{n}p_{ii'}\Big|> \tilde{t}\Big) \leq 2 \exp\Big(-\frac{\tilde{t}^2}{2\sigma_0^2n + \frac{2}{3}\tilde{t}}\Big),
\ee
where $\sigma_0^2 = n^{-1} \sum_{i'} \text{Var}(a_{ii'} -
p_{ii'}) = n^{-1} \sum_{i'} p_{ii'}(1-p_{ii'})$. As
$\min_{i,i'}p_{ii'} \asymp \rho_n$ and $\max_{i,i'}p_{ii'} \asymp \rho_n$ where $\rho_n
\rightarrow 0$, we have $\sigma_0^2 \asymp \rho_n$. Letting $\tilde{t}
= C(n\rho_n \log n)^{1/2}$ in Eq.~\eqref{l1s:bern}, we have
\bee\nonumber
\frac{\sigma_0^2n}{\tilde{t}} \asymp\BL\frac{\rho_n}{n^{-1} \log n}\BR^{1/2} \rightarrow \infty
\ee
by Assumption~\ref{am:s}. We therefore have
\bee\nonumber
\frac{\tilde{t}^2}{2\sigma_0^2n + \frac{2}{3}\tilde{t}} \asymp
\frac{\tilde{t}^2}{2\sigma_0^2n} \asymp C^2 \log n.
\ee
There thus exists some constant $C > 0$ such that 
\bee\nonumber
\p\Big(\Big|\sum_{j = 1}^{n}a_{ii'} - \sum_{j = 1}^{n}p_{ii'}\Big|>
C(n\rho_n \log n)^{1/2} \Big) \leq 2 \exp\Big(-2 \log n\Big)
\ee
and hence
\bee\label{l5:main}
&\p\Big(\max_{1 \leq i \leq n}\Big\{\Big|\sum_{j = 1}^{n}a_{ii'} -
\sum_{j = 1}^{n}p_{ii'}\Big|\Big\}> C(n\rho_n \log n)^{1/2} \Big) 
\leq
2n \exp\Big(-2 \log n \Big)  \rightarrow 0
\ee
as $n \rightarrow \infty$, which implies $\|\M D_{\M A} - \M D_{\M
  P}\|_{} = \op\bigl((n\rho_n \log n)^{1/2}\bigr)$. Combining the above results we obtain
\bee\nonumber
\|\M D_{\M A}^{-1} - \M D_{\M P}^{-1}\|_{} &= \max_i \Big| \frac{1}{d_i}- \frac{1}{p_i}  \Big| = \max_i \Big|\frac{p_i - d_i}{d_i p_i} \Big|
\leq \|\M D_{\M A} - \M D_{\M P}\|_{}\cdot \|\M D_{\M
  A}^{-1}\|_{}\cdot\|\M D_{\M P}^{-1}\|_{} 
  =
\op\bigl((n\rho_n)^{-3/2}\log^{1/2}{n}\bigr).
\ee

\noindent{\textbf{(3)} Bounding $\|\hat{\M W}^{t}\|_{\max}$, $\max_{i,{i'}} w^{(t)}_{ii'}$ and $\min_{i,{i'}} w^{(t)}_{ii'}$}
As the elements of $\M A$ and $\M P$ are non-negative and bounded above by $1$, we have
\bee\|\hat{\M W}\|_{\max} &= \|\M A \M D_{\M A}^{-1}\|_{\max} \leq \|\M A\|_{\max}  \|\M D_{\M A}^{-1}\|
 \leq   \max_{i} d_i^{-1} 
= \op((n\rho_n)^{-1}).
\ee
We thus have, for any $t\geq 1$, that
\bee\nonumber
\|\hat{\M W}^{t}\|_{\max} &\leq (\max_{i}d_i) \|\hat{\M W}^{t - 1}\|_{\max} \|\hat{\M W}\|_{\max} 
\leq (\max_{i}d_i)^2\|\hat{\M W}^{t - 2}\|_{\max} \|\hat{\M W}\|_{\max}^2
\leq \dots \leq (\max_{i}d_i)^{t - 1}\|\hat{\M W}\|_{\max}^{t} 
= \op((n\rho_n)^{-1}).
\ee
Finally, as $c_2\rho_n \leq p_{ii'} \leq c_3\rho_n$ for all $1 \leq i,{i'}\leq n$, we have
\bee\nonumber
\frac{c_2^t}{c_3^t n} \leq w^{(t)}_{ii'}\leq \frac{c_3^t}{c_2^t n}
\ee for all $t\geq 1$. As $t$ is finite, this implies $\min_{i,i'}w^{(t)}_{ii'} \asymp n^{-1}$ and $\max_{i,i'}w^{(t)}_{ii'} \asymp n^{-1}$.
\qed

\subsection{Technical details in section~\ref{sec:pf:ts:main}}\label{app:c:ts:dt}
\subsubsection{Bounding  $\| \mathbf{\Delta}^{(1)}_1\|_{\max}$ and $\|\mathbf\Delta^{(1)}_2\|_{\max}$}\label{app:ts:dt:1}
We first write
\bee\nonumber
\| \M \Delta^{(1)}_1 \|_{\max} &= \max_{i,{i'}} \Big| \frac{a_{ii'}}{d_{i'} \cdot p_{i'}}(d_{i'} - p_{i'})\Big| 
\precsim \max_{i'}\frac{1}{n\rho_n}\Big| \frac{d_{i'} - p_{i'}}{d_{i'}}\Big|
\ee
by $|a_{ii'}| \leq 1$ and Assumption~\ref{am:s}. Applying Lemma~\ref{l:sparse:rate}, we obtain
\begin{gather}
\label{T1S:order_1_1:s}
\| \M \Delta^{(1)}_1 \|_{\max}\precsim \frac{1}{n\rho_n}\max_{i'}|d_{i'} - p_{i'}| \cdot \max_{i'}\frac{1}{d_{i'}}
= \op\big((n\rho_n)^{-3/2}\log^{1/2}{n}\big) \\
\label{T1S:order_1_2:s}
\|\M \Delta^{(1)}_2\|_{\max}  = \max_{i,{i'}}\Big| \frac{a_{ii'} - p_{ii'}}{p_{i'}}\Big| 
\leq \max_{i'} \frac{1}{p_{i'}} 
= \op\bigl((n\rho_n)^{-1}\bigr).
\end{gather}
Since $\rho_n \succsim \frac{\log n}{n}$, we have $(n\rho_n)^{-3/2} \log^{1/2}{n} \precsim (n\rho_n)^{-1}$.
\subsubsection{Bounding $\|\mathbf\Delta_{1}^{(2)}\|_{\max}$}\label{app:ts:dt:2}
The $ii'$th element of $\M \Delta^{(2,1)}_{1}$ is given by
\bee\nonumber
&\sum_{i^* = 1}^{n}\Big(\frac{a_{ii^*}}{d_{i^*} p_{i^*}}(p_{i^*} - d_{i^*}) - \frac{a_{ii^*}}{p_{i^*}^2}(p_{i^*} - d_{i^*})\Big)\cdot w_{i^{*}{i'}} 
= \sum_{i^* = 1}^{n} \frac{a_{ii^*}(p_{i^*} - d_{i^*})^2 w_{i^{*}{i'}}}{p_{i^*}^2 d_{i^*}}.
\ee
We therefore have, by Lemma~\ref{l:sparse:rate}, that
\bee\nonumber
\|\M \Delta^{(2,1)}_{1}\|_{\max} 
&\leq \max_{i,{i'}}\sum_{i^* = 1}^{n} \frac{a_{ii^*}(p_{i^*} - d_{i^*})^2w_{i^{*}{i'}}}{p_{i^*}^2 d_{i^*}}
\\
&\leq \max_{i}d_i\big(\max_{i}|p_i - d_i|\big)^2
\cdot \BL\max_{i}\frac{1}{p_i}\BR^2\BL\max_{i}\frac{1}{d_i}\BR(\max_{i,{i'}}w_{ii'})
= \op(n^{-2}\rho_n^{-1} \log n)
\ee
Similar to Eq.~\eqref{T1:d12} and Eq.~\eqref{T1:d22}, we also have
\bee\nonumber
\|\M \Delta^{(2,2)}_{1}\|_{\max} 
&= \max_{i,{i'}}\BL \Big|\sum_{i^* = 1}^{n}\frac{a_{ii^*}}{p_{i^*}^2}(p_{i^*} - d_{i^*})w_{i^{*}{i'}}\Big|\BR
\leq \max_i d_i (\max_{i} |p_i - d_i|) \cdot \max_{i,{i'}}w_{ii'}  \BL \max_{i}\frac{1}{p_i}\BR^2 =\op\big(n^{-3/2}\rho_n^{-1/2}\log^{1/2}{n}\big) \\
\|\M \Delta^{(2,3)}_{1}\|_{\max} 
&=\max_{i,{i'}}\Big|\sum_{i^*}\frac{a_{ii^*}- p_{ii^*}}{p_{i^*}}w_{i^{*}{i'}}\Big|=
\op\big(n^{-3/2}\rho_n^{-1/2}\log^{1/2}{n}\big).
\ee
By Assumption~\ref{am:s}, we have $n^{-3/2}\rho_n^{-1/2} \log^{1/2}{n} \succsim n^{-2}\rho_n^{-1}{\log n}$. Combining the above results we obtain \bee\label{T1:order_2_2:s}
\|\M \Delta_{1}^{(2)}\|_{\max} &\leq \|\M \Delta_{1}^{(2,1)}\|_{\max} + \|\M \Delta_{1}^{(2,2)}\|_{\max} + \|\M \Delta_{1}^{(2,3)}\|_{\max} 
= \op
\big(n^{-3/2}\rho_n^{-1/2} \log^{1/2}{n}\big).
\ee
\subsubsection{Bounding  ${{\delta}}^{(2,2)}_{2,\off}$ and ${\delta}^{(2,3)}_{2,\off}$}\label{app:ts:dt:3}
\noindent{\textbf{Step 1 (Bounding  ${{\delta}}^{(2,2)}_{2,\off}$): }}By Lemma \ref{l:sparse:rate}, we have
\bee\nonumber
\delta^{(2,2)}_{2,\off} 
&= \max_{i,i'}\Big|\sum_{i^* = 1}^{n}\Big(\frac{p_{ii^*}}{p_{i^*}} - \frac{a_{ii^*}}{d_{i^*}}\Big)\BL \frac{a_{i^{*}{i'}}}{p_{i'}} - \frac{a_{i^{*}{i'}}}{d_{i'}}\BR\Big|
\\
&\leq \max_{i,i'}\Big|\sum_{i^* = 1}^{n}\Big(\frac{p_{ii^*}}{p_{i^*}} - \frac{a_{ii^*}}{d_{i^*}}\Big) {a_{i^{*}{i'}}}\Big|\cdot \max_{i'}\Big|\frac{1}{p_{i'}} - \frac{1}{d_{i'}}\Big| 
=\op(n^{-3/2}\rho_n^{-3/2}\log^{1/2}{n}) \cdot \max_{i,i'}\Big|\sum_{i^* = 1}^{n}\Big(\frac{p_{ii^*}}{p_{i^*}} - \frac{a_{ii^*}}{d_{i^*}}\Big) {a_{i^{*}{i'}}}\Big|.
\ee
Now we focus on $\sum_{i^* = 1}^{n}\Big(\frac{p_{ii^*}}{p_{i^*}} - \frac{a_{ii^*}}{d_{i^*}}\Big) {a_{i^{*}{i'}}}$. We have
\bee\label{t:s:d224}
\sum_{i^* = 1}^{n}\Big(\frac{p_{ii^*}}{p_{i^*}} - \frac{a_{ii^*}}{d_{i^*}}\Big) {a_{i^{*}{i'}}} 
&= \sum_{i^* = 1}^{n}\Big(\frac{d_{i^*}}{p_{i^*}} + 1 - \frac{d_{i^*}}{p_{i^*}}\Big)\Big(\frac{p_{ii^*}d_{i^*} -a_{ii^*}p_{i^*} }{p_{i^*}d_{i^*}}\Big) {a_{i^{*}{i'}}} 
\\
&= \sum_{i^* = 1}^{n}\Big(\frac{p_{ii^*}d_{i^*} -a_{ii^*}p_{i^*}}{p_{i^*}^2}\Big) {a_{i^{*}{i'}}} 
+\sum_{i^* = 1}^n\Big(1 - \frac{d_{i^*}}{p_{i^*}}\Big)\cdot\Big(\frac{p_{ii^*}d_{i^*} -a_{ii^*}p_{i^*} }{p_{i^*}d_{i^*}}\Big)a_{i^*i'}
\\
&=\sum_{i^*= 1}^n\sum_{\iss = 1}^n\Big(\frac{ a_{i^*\iss}p_{ii^*} - p_{i^*\iss}a_{ii^*}}{p_{i^*}^2}\Big)a_{i^*i'} 
+\sum_{i^* = 1}^n\frac{(p_{ii^*}d_{i^*} - a_{ii^*}p_{i^*})(p_{i^*} - d_{i^*})a_{i^*i'}}{p_{i^*}^2d_{i^*}}.
\ee
For the first term on the RHS of Eq.~\eqref{t:s:d224}, we have
\bee \nonumber
\sum_{i^* = 1}^n\sum_{\iss = 1}^n\Big(\frac{ a_{i^*\iss}p_{ii^*} - p_{i^*\iss}a_{ii^*}}{p_{i^*}^2}\Big)a_{i^*i'} 
&= \sum_{i^* = 1\atop i^*\neq i,i'}^n\sum_{\iss= 1\atop \iss \neq i,i'}^n\Big(\frac{ a_{i^*\iss}p_{ii^*} - p_{i^*\iss}a_{ii^*}}{p_{i^*}^2}\Big)a_{i^*i'}
\\
&=\sum_{i^* = 1\atop i^*\neq i,i'}^n\frac{a_{i^*i'}}{p_{i^*}^2}\sum_{\iss = 1\atop\iss \neq i,i'}^na_{i^*\iss}p_{ii^*} - p_{i^*\iss}a_{ii^*} 
\\
&= \sum_{(i^*,i^{**})\in\mathcal{T}(i,i')}\Big(\frac{a_{i^*i'}p_{ii^*}}{p_{i^*}^2} + \frac{a_{\iss i'}p_{ii^{**}}}{p_{\iss}^2}\Big)a_{i^*\iss}
 - \Big(\frac{a_{i^*i'}a_{ii^*}}{p_{i^*}^2} + \frac{a_{\iss i'}a_{ii^{**}}}{p_{\iss}^2}\Big)p_{i^*\iss}
\\ &=: \mathcal{S}_2(\bds a_i,\bds a_{i'}).
\ee
Here $\mathcal{T}(i,i') = \{(i^*,i^{**})|i^* < i^{**}, i^* \not \in
\{i,i'\}, i^{**} \not \in \{i,i'\}\}$ and
$\bm{a}_{i} = (a_{i1}, \dots, a_{in})$.
\par
Conditioning on the event that $\|\M D_{\M A}\| = \max_i d_i \leq cn\rho_n$ with some $c>0$ and $\bds{a}_i,\bds{a}_{i'}$, it is easy to check that the summands in 
$\mathcal{S}_2(\bds a_i,\bds a_{i'})$
are {\em independent} mean $0$ random variables with $\mathrm{Var}[\mathcal{S}_2(\bds a_i,\bds a_{i'})] \asymp n^{-2}$. Similar arguments to Eq.~\eqref{var:d112}-Eq.~\eqref{d112:f} yield
\bee\nonumber
\max_{i,i'}\Big|\sum_{i^*,\iss= 1}^n\Bigg(\frac{ a_{i^*\iss}p_{ii^*} - p_{i^*\iss}a_{ii^*}}{p_{i^*}^2}\Bigg)a_{i^*i'}\Big| = \op(n^{-1}).
\ee
In addition, by Lemma \ref{l:sparse:rate}, we also have
\bee\nonumber
\max_{i,i'}\Big|\sum_{i^* = 1}^n\frac{(p_{ii^*}d_{i^*} -
  a_{ii^*}p_{i^*})(p_{i^*} - d_{i^*})}{p_{i^*}^2d_{i^*}}a_{i^*i'}\Big|
&= \op\{(n\rho_n)^{-1/2}\log^{1/2}{n}\},
\ee
as $\max_{i,i'}|p_{ii'}d_i - a_{ii'}p_i|\leq \max_{i,i'}p_{ii'}d_i + p_i= \op(n\rho_n)$. Combining the above results we obtain
\bee\nonumber
\delta^{(2,2)}_{2,\off} = \op((n\rho_n)^{-2} \log n).
\ee
\par
\noindent\textbf{Step 2 (Bounding ${{\delta}}^{(2,3)}_{2,\off}$): } We first note that
\bee\label{T1s:off-d}
{\delta}^{(2,3)}_{2,\off} &= \max_{i \neq {i'}}\Big|\sum_{i^{*}= 1}^{n}\BL\frac{p_{ii^*}}{p_{i^*}} - \frac{a_{ii^*}}{p_{i^*}} \BR\BL\frac{p_{i^*i'}}{p_{i'}} - \frac{a_{i^{*}{i'}}}{p_{i'}} \BR\Big| 
= \max_{i \neq {i'}}\Big|\sum_{i^{*}= 1}^{n}\frac{1}{p_{i^*} p_{i'}}\BL p_{ii^*} - a_{ii^*} \BR\BL p_{{i'}i^{*}}-a_{i'i^*} \BR\Big|.
\ee
Denote $\tilde{\zeta}^{ii'}_{i^{*}} := \frac{1}{p_{i^*} p_{i'}}\big( p_{ii^*} - a_{ii^*} \big)\big( p_{{i'}i^{*}}-a_{i'i^*} \big)$. We then have
\bee\label{T1s:bern_mean}
\sum_{i^{*}= 1}^{n}\BL\frac{p_{ii^*}}{p_{i^*}} - \frac{a_{ii^*}}{p_{i^*}} \BR\BL\frac{p_{i^*i'}}{p_{i'}} - \frac{a_{i^{*}{i'}}}{p_{i'}} \BR 
= \Big(\sum_{i^* = 1\atop i^* \neq i,{i'}}^{n}\tilde{\zeta}^{ii'}_{i^{*}} \Big)+ \tilde{\zeta}^{ii'}_i + \tilde{\zeta}^{ii'}_{i'}.
\ee
Next observe that the $\{\tilde{\zeta}^{ii'}_{i^{*}}\}$ for $i^* \not \in \{i,i'\}$ are independent mean $0$ random variables. Hence
\bee
\frac{1}{n-2}\sum_{i^* = 1\atop i^*\neq i,{i'}}^{n} \text{Var}(\tilde{\zeta}^{ii'}_{i^{*}})
& = \frac{1}{n-2}\sum_{i^* = 1\atop i^*\neq i,{i'}}^n\text{Var}\Big\{\BL\frac{p_{ii^*}}{p_{i^*}} - \frac{a_{ii^*}}{p_{i^*}} \BR\BL\frac{p_{i^*i'}}{p_{i'}} - \frac{a_{i^{*}{i'}}}{p_{i'}} \BR\Big\}
\\
& = \frac{1}{(n-2)}\sum_{i^* = 1\atop i^*\neq i,{i'}}^{n}\frac{1}{p_{i^*}^2p_{i'}^2} \text{Var}({p_{ii^*}} - {a_{ii^*}}) \cdot \text{Var}({p_{{i'}i^{*}}} - {a_{i'i^*}}) 
\asymp \frac{1}{n^4\rho_n^2}.
\ee
We also have
\bee\label{T1s:bern}
\max_{i \not = i'} |\tilde{\zeta}^{ii'}_{i^{*}}| &= \Big | \BL\frac{p_{ii^*}}{p_{i^*}} - \frac{a_{ii^*}}{p_{i^*}} \BR\BL\frac{p_{i^*i'}}{p_{i'}} - \frac{a_{i^{*}{i'}}}{p_{i'}} \BR\Big |
\leq \Bigl(\frac{1}{n\rho_n c_2}\Bigr)^2
\precsim (n\rho_n)^{-2}.
\ee
Therefore, by Bernstein inequality, for any given $c' > 0$ there exists a constant $C'>0$ not depending on $n$ such that
\bee\nonumber
&\p\Big(\max_{i \neq {i'}}\Big|\sum_{i^* = 1\atop i^* \neq
  i,{i'}}^{n}\tilde{\zeta}^{ii'}_{i^{*}}\Big|>C'n^{-3/2}\rho_n^{-1} \log^{1/2}{n}\Big) \leq n^{-c'} 
\ee
In summary we have 
$$\max_{i \neq
  {i'}}\Big|\sum_{i^* = 1\atop i^* \neq
  i,{i'}}^{n}\tilde{\zeta}^{ii'}_{i^{*}}\Big| =
\op(n^{-3/2}\rho_n^{-1}\log^{1/2}{n}).$$ 
By Eq.~\eqref{T1s:bern}, we also have $\max_{i,{i'}}|\tilde{\zeta}^{ii'}_{i^{*}}| = \mathcal{O}((n\rho_n)^{-2})$ and thus
\bee\label{T4:s1finals}
\delta^{(2,3)}_{2,\off} &= \max_{i \neq {i'}}\Big|\sum_{i^{*}= 1}^{n}\BL\frac{p_{ii^*}}{p_{i^*}} - \frac{a_{ii^*}}{p_{i^*}} \BR\BL\frac{p_{i^*i'}}{p_{i'}} - \frac{a_{i^{*}{i'}}}{p_{i'}} \BR \Big| 
\\ &\leq \max_{i \neq {i'}}\Big| \sum_{i^* = 1\atop i^* \neq i,{i'}}^{n}\tilde{\zeta}^{ii'}_{i^{*}} \Big|+ \max_{i,{i'}}\Big|\tilde{\zeta}^{ii'}_i\Big| + \max_{i,{i'}}\Big|\tilde{\zeta}^{ii'}_{i'}\Big| 
= \op(n^{-3/2}\rho_n^{-1}\log^{1/2}{n}).
\ee

\subsection{Proof of Theorem~\ref{T1}}
To give a detailed analysis for all components in $C_{ii'}$, we firstly denote $C_{ii'}^{(t)}$ as the times that the structure,
\bee\label{structure3}
\dots,v_i,\underbrace{\dots\dots}_{t - 1 \text{terms}},v_{i'},\dots
\ee 
with a fixed $t$ appears among all random paths in $\bigcup_{i = 1}^n \mathcal{L}_i$. As $C_{ii'}$ counts all structures defined in Eq.~\eqref{node2vec:ii'}, we have
\bee\nonumber
C_{ii'} = \sum_{t = t_L}^{t_U}C_{ii'}^{(t)} + \sum_{t = t_L}^{t_U}C_{i'i}^{(t)}.
\ee
Let $C_{k,i,i'}^{(t)}$ denote the number of times the following structure appears among all random paths in $\bigcup_{i = 1}^{n} \mathcal{L}_i$,
\bee\label{structure1}
\underbrace{\dots\dots}_{k \text{ nodes}},v_i,\underbrace{\dots\dots}_{t - 1\text{ nodes}},v_{i'},\dots
\ee
We then have
\bee\label{structure2}
C_{ii'}^{(t)} = \sum_{i^* = 0}^{L-t-1}C_{i^*,i,i'}^{(t)}.
\ee
Let $\{R_i\}_{i \geq 1}$ represents a stationary simple random walk on $\mathcal{G}$. Since all random paths are stationary and independent simple random walks over $\mathcal{G}$, the strong law of large numbers for Markov chain implies
\bee\label{T1:SLLN}
\lim_{r \rightarrow \infty} \frac{1}{r} C^{(t)}_{i^*,i,i'}
&= \p(R_{i^*+1} = v_i)\cdot \p(R_{i^*+t+1} = v_{i'}|R_{i^*+1} = v_i)
=S_i\cdot \p(R_{t + 1} = v_{i'}|R_{1} = v_i)
= \frac{d_i}{2|\M A|} \cdot \hat{w}^{(t)}_{i'i}
\ee
almost surely. 
Furthermore we also have
\bee\nonumber
\lim_{r \rightarrow \infty} \frac{1}{r} C^{(t)}_{ii'} &= \lim_{r \rightarrow \infty} \frac{1}{r} \sum_{i^* = 0}^{L-t-1}C_{i^*,i,i'}^{(t)}
= \frac{(L - t)d_i}{2|\M A|}  \hat{w}^{(t)}_{i'i}, \\
\lim_{r \rightarrow \infty} \frac{1}{r} C_{ii'} &= 
\lim_{r \rightarrow \infty} \frac{1}{r} \Bigl(\sum_{t = t_L}^{t_U}C_{ii'}^{(t)} + \sum_{t =
  t_L}^{t_U}C_{i'i}^{(t)}\Bigr)
 = \sum_{t = t_L}^{t_U}(L -
t)\Big(\frac{d_i}{2|\M A|}\hat{w}_{i'i}^{(t)}+\frac{d_{i'}}{2|\M
  A|}\hat{w}_{ii'}^{(t)}\Big).
\ee
almost surely. 
Combining the above two convergences, we have
\bee
\label{T1:Ci}
\lim_{r \rightarrow \infty} \frac{1}{r} \sum_{i = 1}^{n}C_{ii'}
= \sum_{t = t_L}^{t_U}(L - t)\Big(\frac{1}{2|\M A|}\sum_{i = 1}^{n}{d_i}\hat{w}_{i'i}^{(t)}+\frac{d_{i'}}{2|\M A|}\sum_{i = 1}^n\hat{w}_{ii'}^{(t)}\Big)
= \sum_{t = t_L}^{t_U}(L - t)\Big(\frac{d_{i'}}{2|\M A|} + \frac{d_{i'}}{2|\M A|}\Big) 
= 
\frac{\gamma d_{i'}}{|\M A|}
\ee
almost surely, 
where we denote $\gamma:= \frac{(2L - t_L - t_U)(t_U -
  t_L + 1)}{2}$. Note that the second equality in Eq.~\eqref{T1:Ci} is
due to Lemma~\ref{l1}. Similar reasoning yields
\bee
\lim_{r \rightarrow \infty} \frac{1}{r} \sum_{i' = 1}^{n}C_{i'i} = 
\frac{\gamma d_i}{|\M A|}, \quad
\lim_{r \rightarrow \infty} \frac{1}{r} \sum_{i = 1}^n \sum_{i' = 1}^n C_{ii'} = 2 \gamma
\ee
almost surely. 
Now for $(t_L, t_U)$ satisfying that all entries in $\sum_{t = t_L}^{t_U}\hat{\M W}^t$ are positive, the $ii'$th entry
in $\tilde{\M M}(\M C,\kappa)$ satisfies
\bee\label{T1:final}
\log \left(\frac{C_{ii'} \cdot(\sum_{i,i'}C_{ii'})}{\kappa \sum_{i}C_{ii'} \cdot \sum_{i'}C_{ii'}}\right)
&= \log \left(\frac{(C_{ii'}/r)\cdot(\kappa \sum_{i,i'}C_{ii'}/r)}{\sum_{i}(C_{ii'}/r) \cdot \sum_{i'}(C_{ii'}/r)}\right) 
\\
&\xrightarrow{a.s.}\log\Big[\frac{|\M A|}{\kappa \gamma}\sum_{t = t_L}^{t_U}(L -t)\Big(\frac{\hat{w}^{(t)}_{i'i}}{d_{i'}}+\frac{\hat{w}_{ii'}^{(t)}}{d_i}\Big)\Big] 
\\
&=\log\Big[\frac{2|\M A|}{\kappa \gamma}\sum_{t = t_L}^{t_U}(L - t)\Big(\frac{\hat{w}_{ii'}^{(t)}}{d_i}\Big)\Big],
\ee
where the last equality is because $\hat{\M W}^{t}$ is a
transition matrix that satisfies the detailed balance
condition. Writing Eq.~\eqref{T1:final} in matrix form, we obtain 
\bee\nonumber
\tilde{\M M}(\M C,k) \xrightarrow{a.s.} &\log\Big[\frac{2|\M A|}{\kappa \gamma} \sum_{t = t_L}^{t_U}(L - t)\M D_{\M A}^{-1}\hat{\M W}^t\Big] 
\ee 
as desired.
\qed
\section{Additional figures}\label{sec:addfigure}
This section contains the figures of all additional simulation results for Section \ref{sec:simu} and Section \ref{sec:disc} in the main paper.
\begin{figure*}
\begin{subfigure}{.23\columnwidth}
\includegraphics[width=\columnwidth]{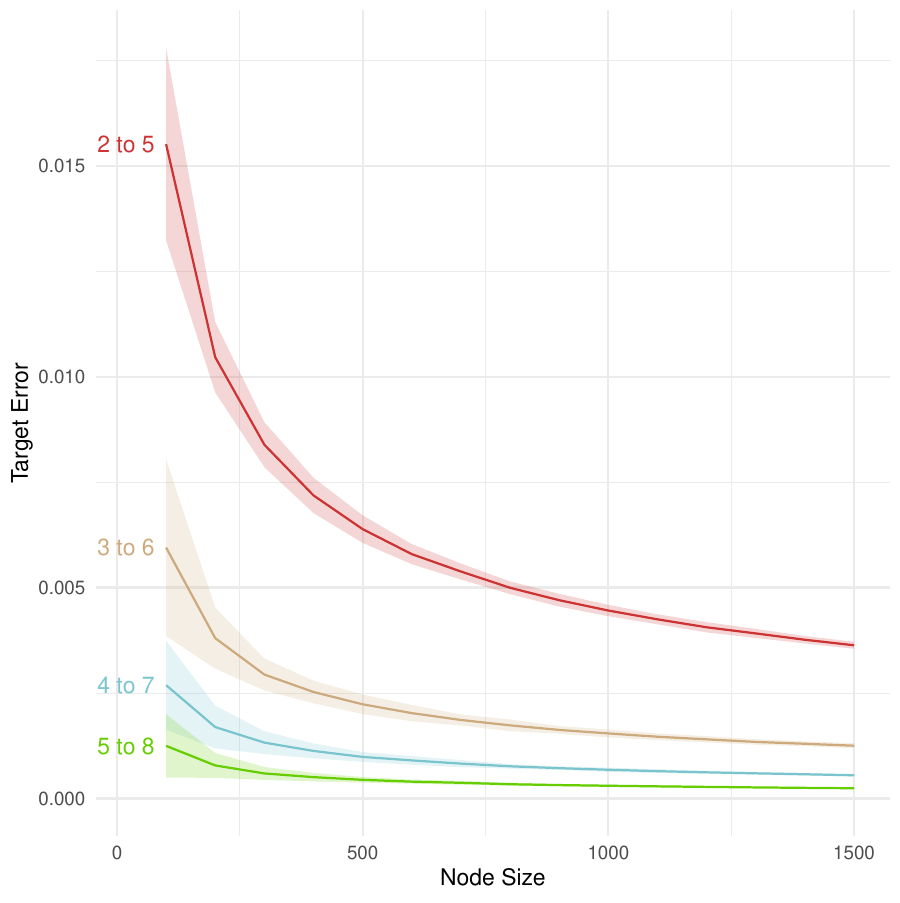}%
\caption{$\rho_n = 1$}
\end{subfigure}\hfill
\begin{subfigure}{.23\columnwidth}
\includegraphics[width=\columnwidth]{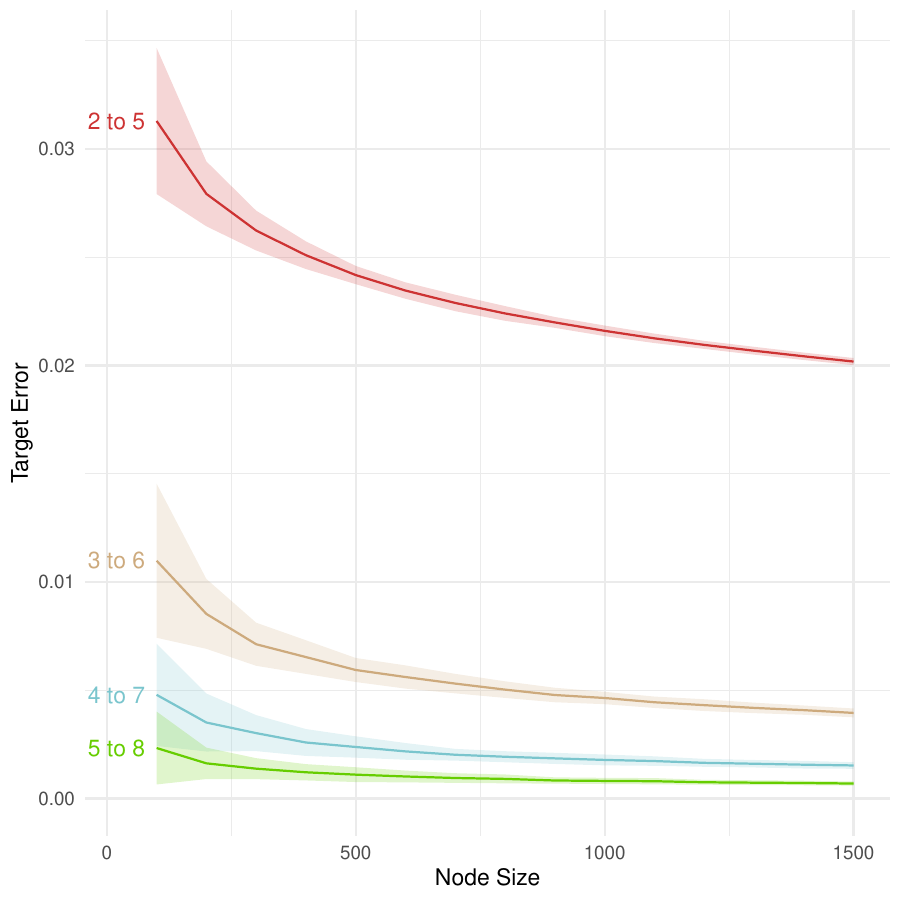}%
\caption{$\rho_n = 3n^{-1/3}$}
\end{subfigure} 
\begin{subfigure}{.23\columnwidth}
\includegraphics[width=\columnwidth]{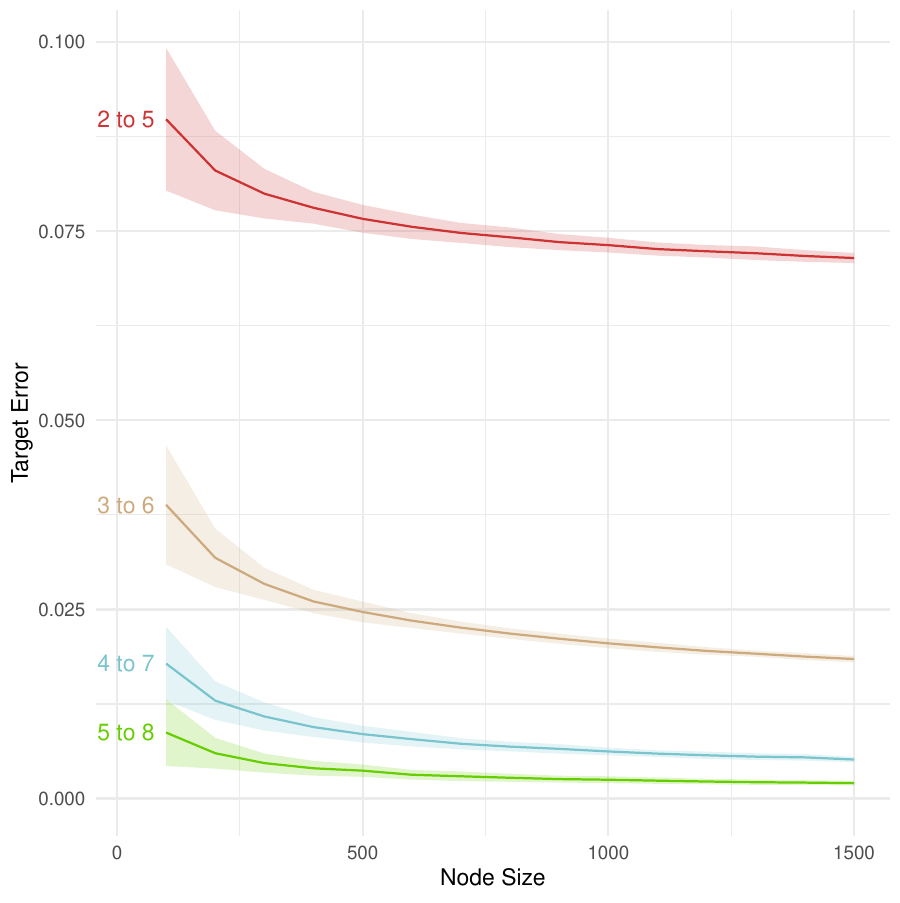}%
\caption{$\rho_n = 3n^{-1/2}$}
\end{subfigure}\hfill
\begin{subfigure}{.23\columnwidth}
\includegraphics[width=\columnwidth]{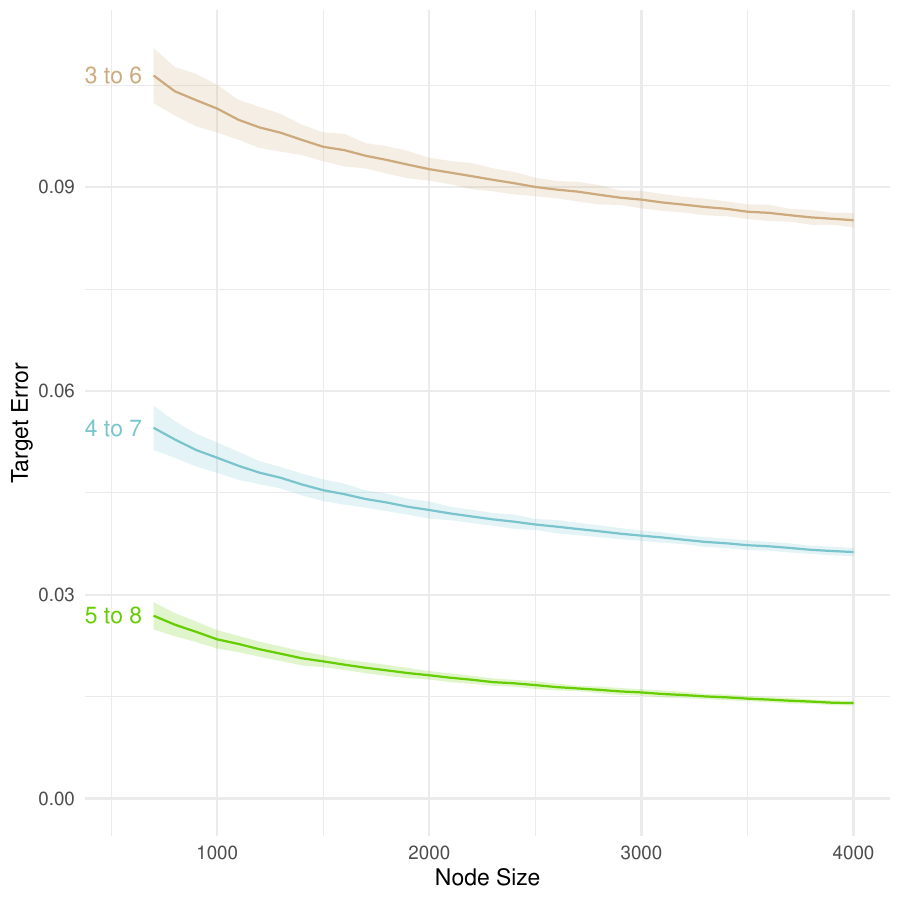}%
\caption{$\rho_n = 15n^{-2/3}$}
\end{subfigure}
\caption{Sample means and $95\%$ empirical confidence
    intervals for $\varepsilon_1(\tilde{\mathbf{M}}_0)$ based on $100$ Monte Carlo replicates for different
    settings of $n,\rho_n, (t_L, t_U)$. Here we set $t_U - t_L = 3$.}
\label{f:rate:2}
\end{figure*}
\begin{figure*}
\begin{subfigure}{.23\columnwidth}
\includegraphics[width=\columnwidth]{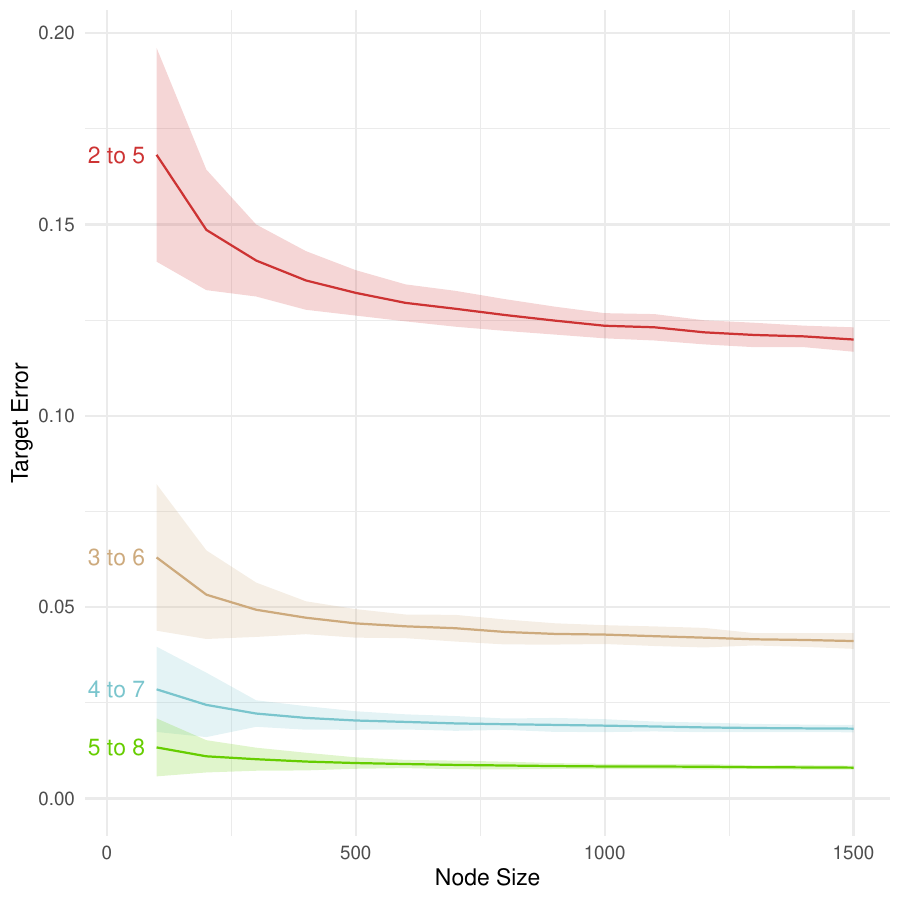}%
\caption{$\rho_n = 1$}
\end{subfigure}\hfill
\begin{subfigure}{.23\columnwidth}
\includegraphics[width=\columnwidth]{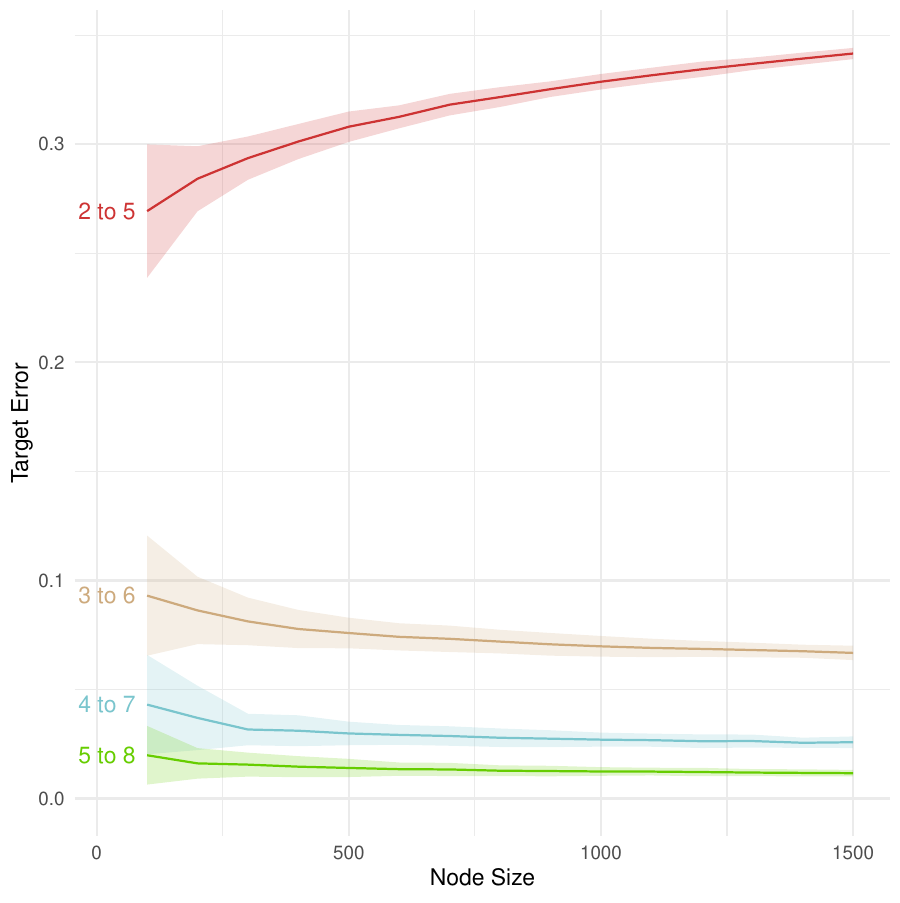}%
\caption{$\rho_n = 3n^{-1/3}$}
\end{subfigure}
\begin{subfigure}{.23\columnwidth}
\includegraphics[width=\columnwidth]{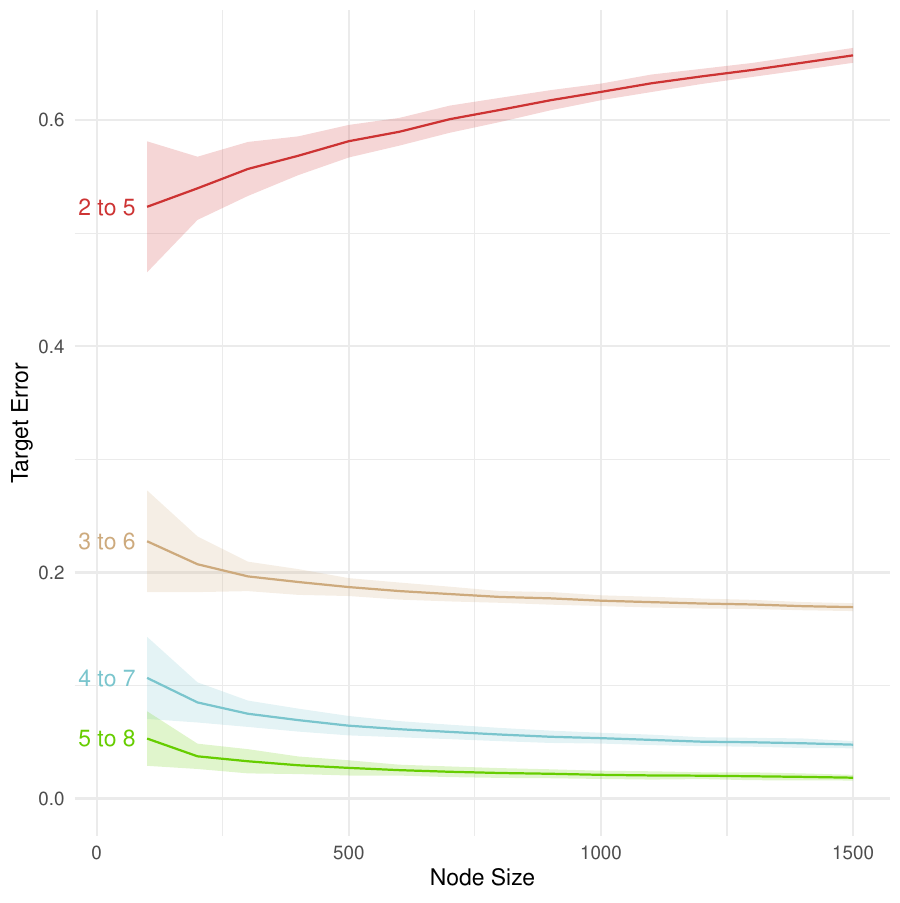}%
\caption{$\rho_n = 3n^{-1/2}$}
\end{subfigure}\hfill
\begin{subfigure}{.23\columnwidth}
\includegraphics[width=\columnwidth]{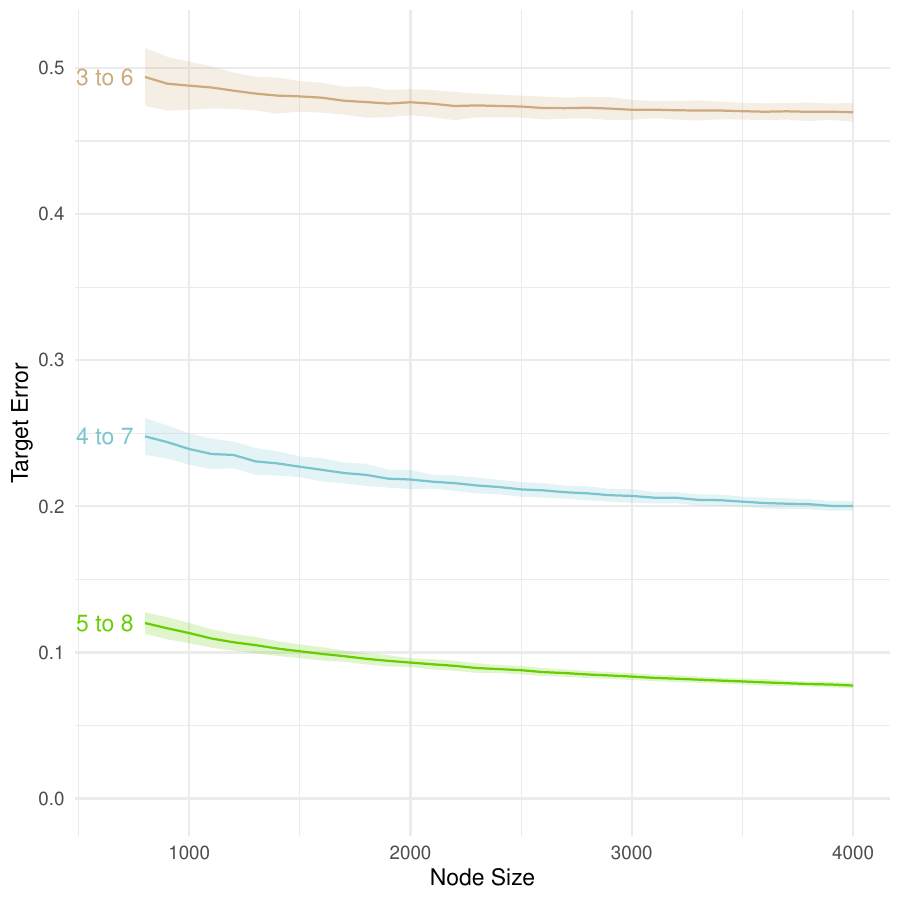}%
\caption{$\rho_n = 15n^{-2/3}$}
\end{subfigure}
\caption{Sample means and $95\%$ empirical confidence
    intervals for $\varepsilon_2(\tilde{\mathbf{M}}_0)$ based on $100$
    Monte Carlo replicates for different values of  of $n,\rho_n$, and
    $(t_L, t_U)$ with $t_L - t_U = 3$.}
\label{f:rate:4}
\end{figure*}

\begin{figure*}
\centering
\begin{subfigure}{.24\columnwidth}
\includegraphics[width=\columnwidth]{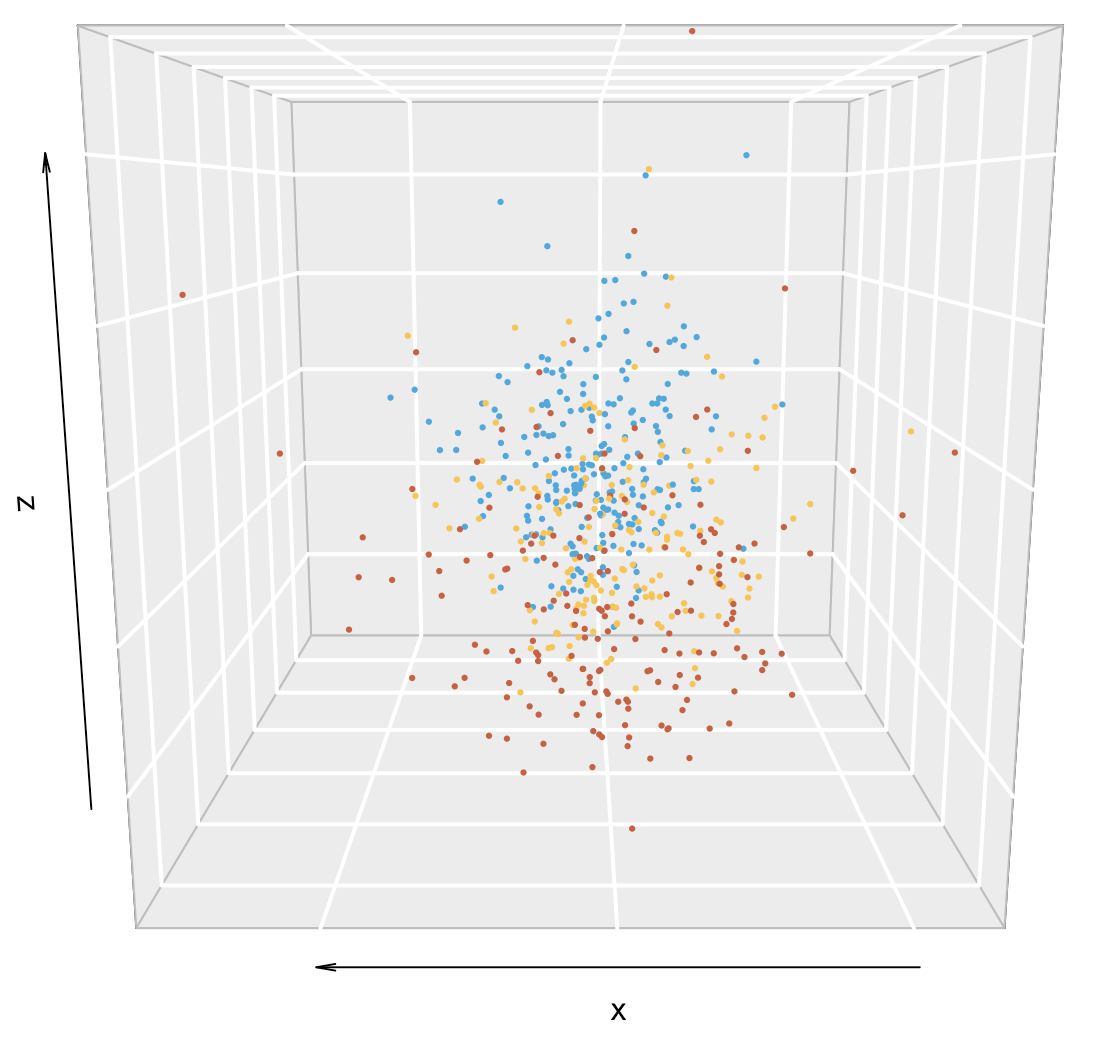}%
\caption{$t_U = 5$, accuracy $=0.72$}
\end{subfigure}
\begin{subfigure}{.24\columnwidth}
\includegraphics[width=\columnwidth]{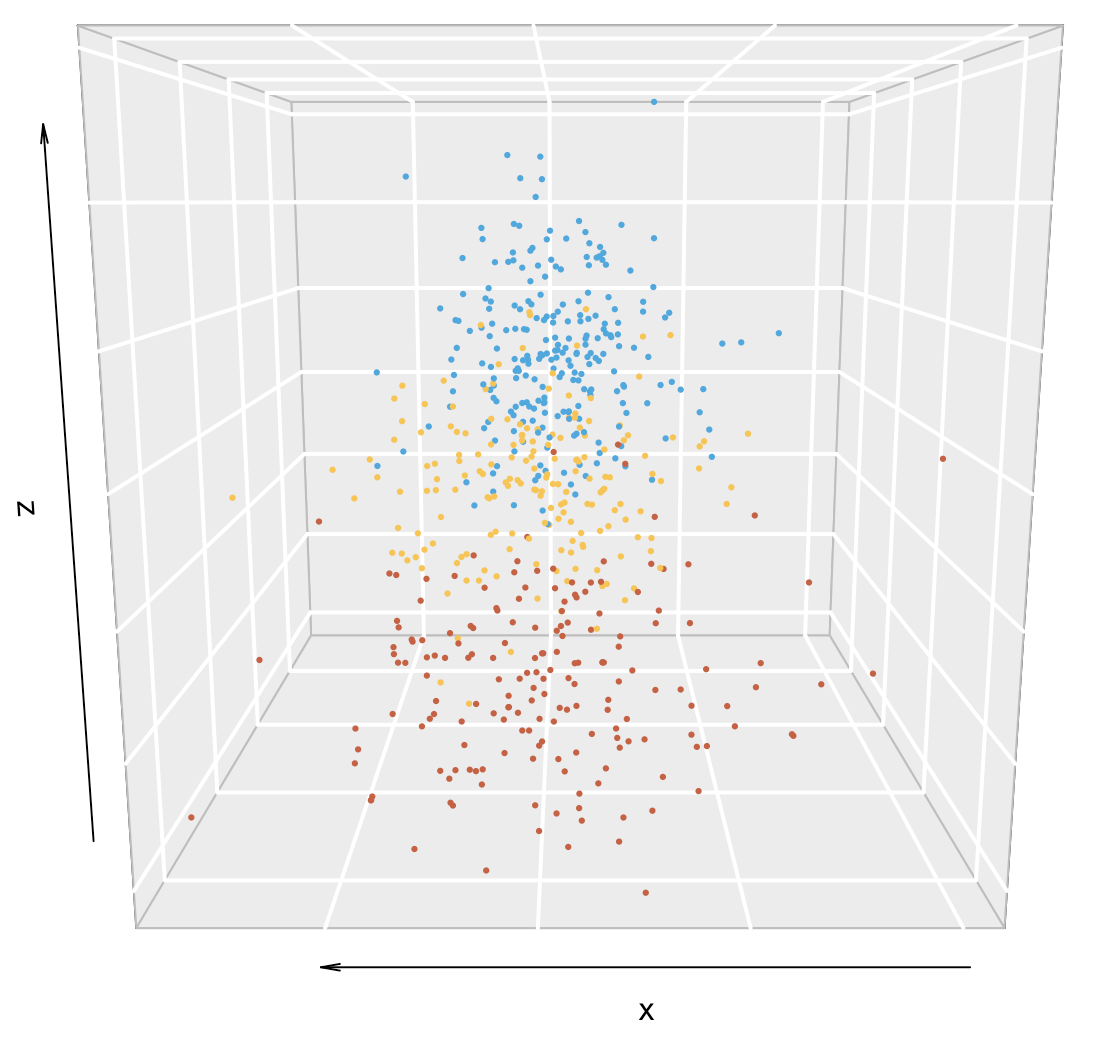}%
\caption{$t_U = 6$, accuracy $=0.80$}
\label{12:3}
\end{subfigure}
\begin{subfigure}{.24\columnwidth}
\includegraphics[width=\columnwidth]{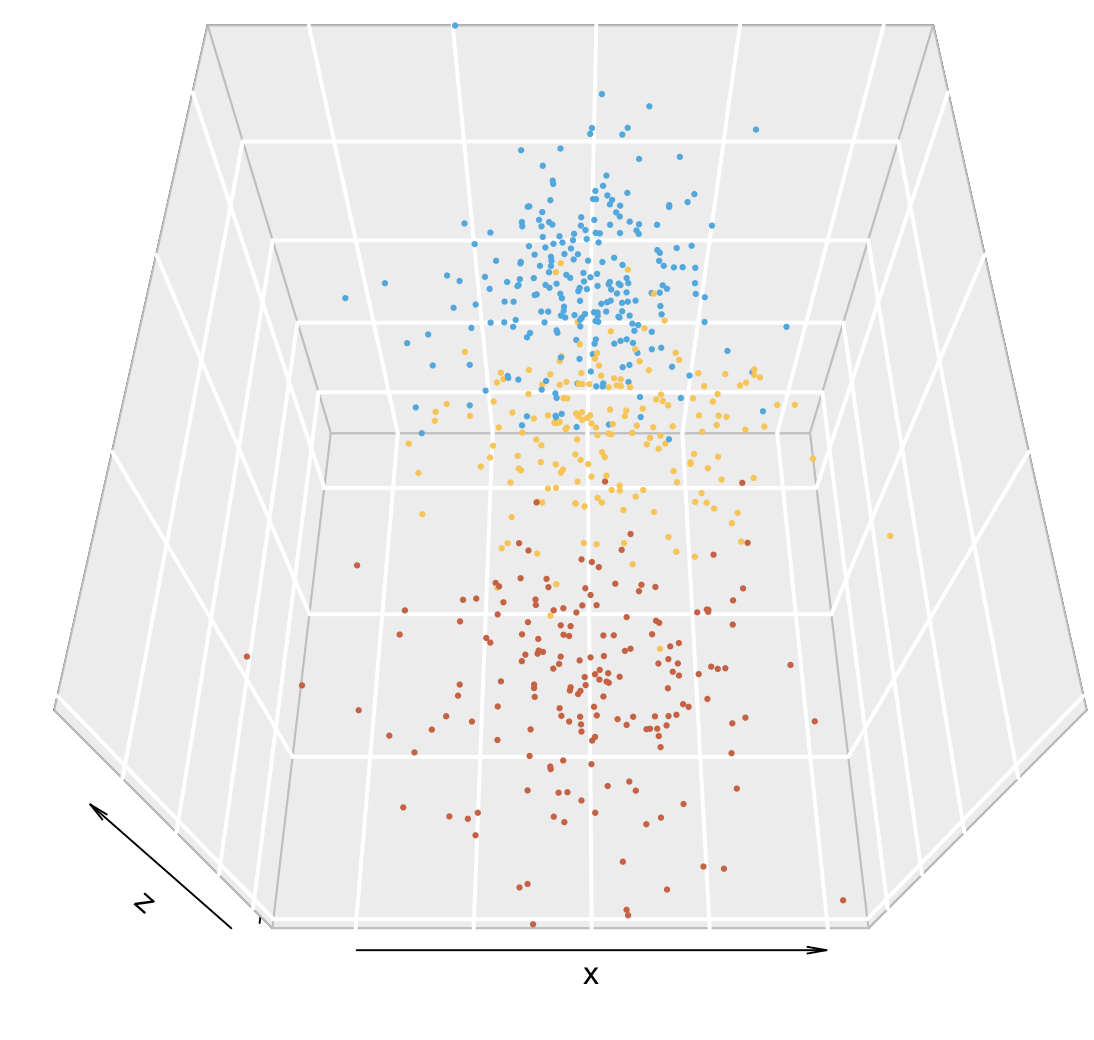}%
\caption{$t_U = 8$, accuracy $=0.88$}
\end{subfigure}
\begin{subfigure}{.24\columnwidth}
\includegraphics[width=\columnwidth]{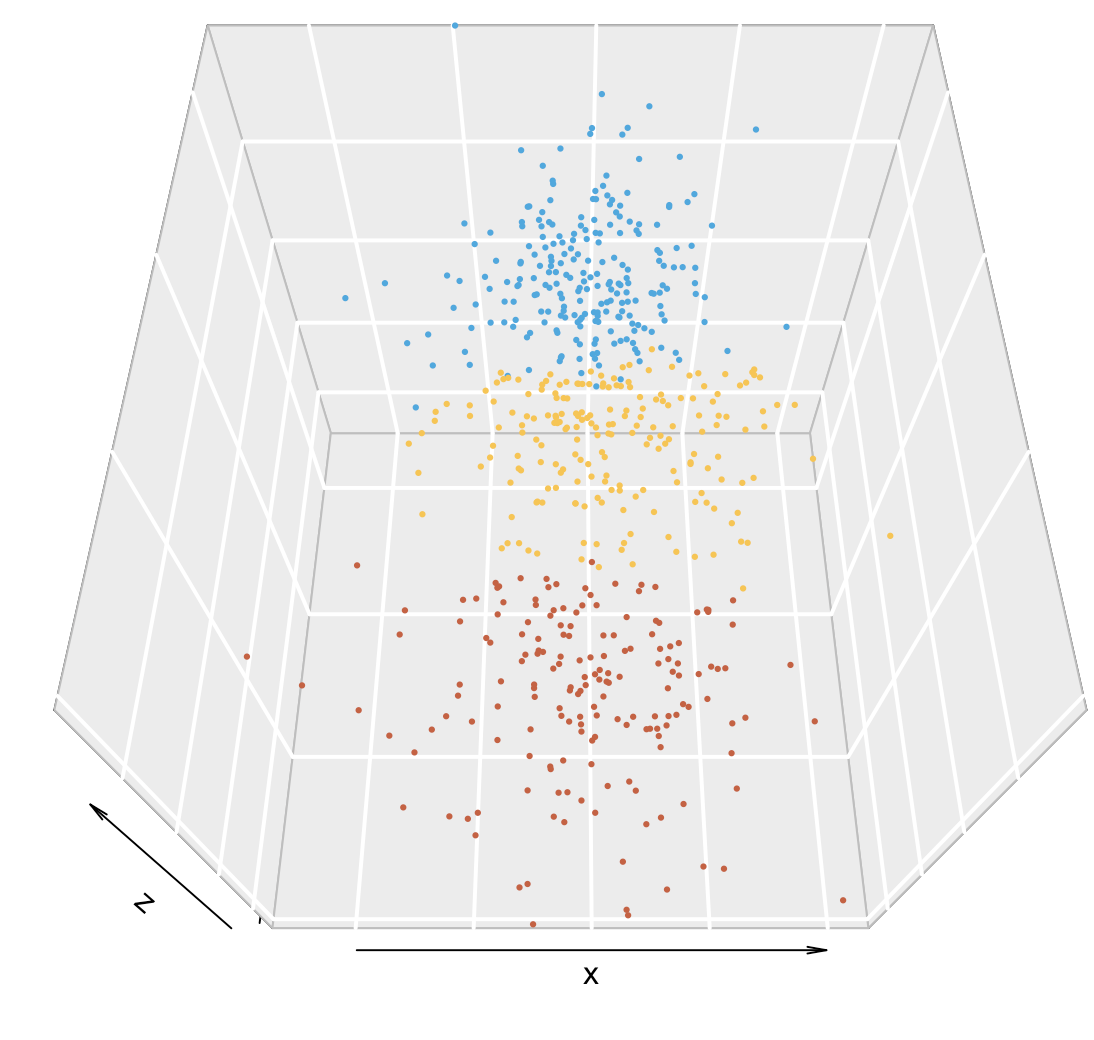}%
\caption{$t_U = 8$, recovery result}
\end{subfigure}

\begin{subfigure}{.24\columnwidth}
\includegraphics[width=\columnwidth]{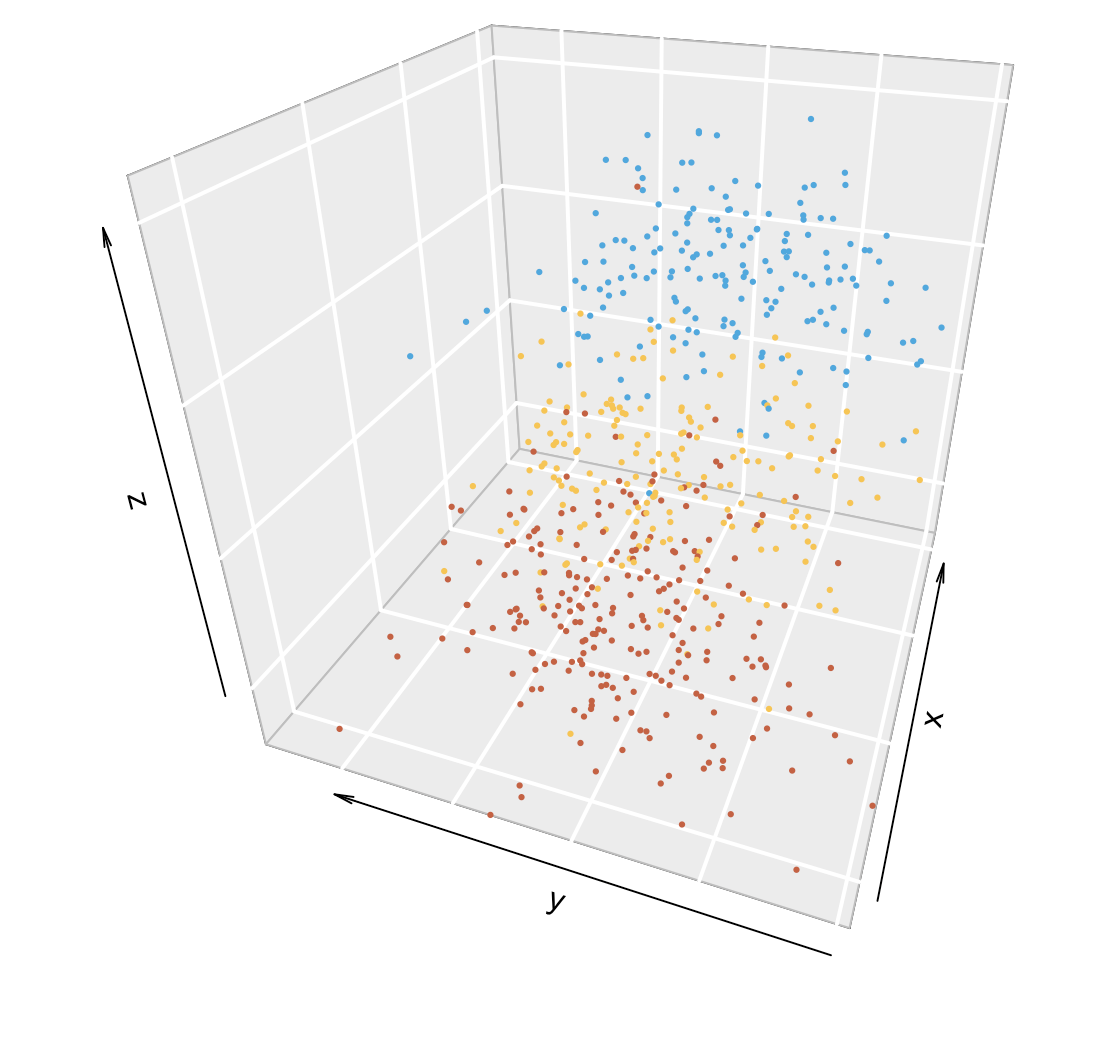}%
\caption{$t_U = 5$, accuracy $=0.62$}
\end{subfigure}
\begin{subfigure}{.24\columnwidth}
\includegraphics[width=\columnwidth]{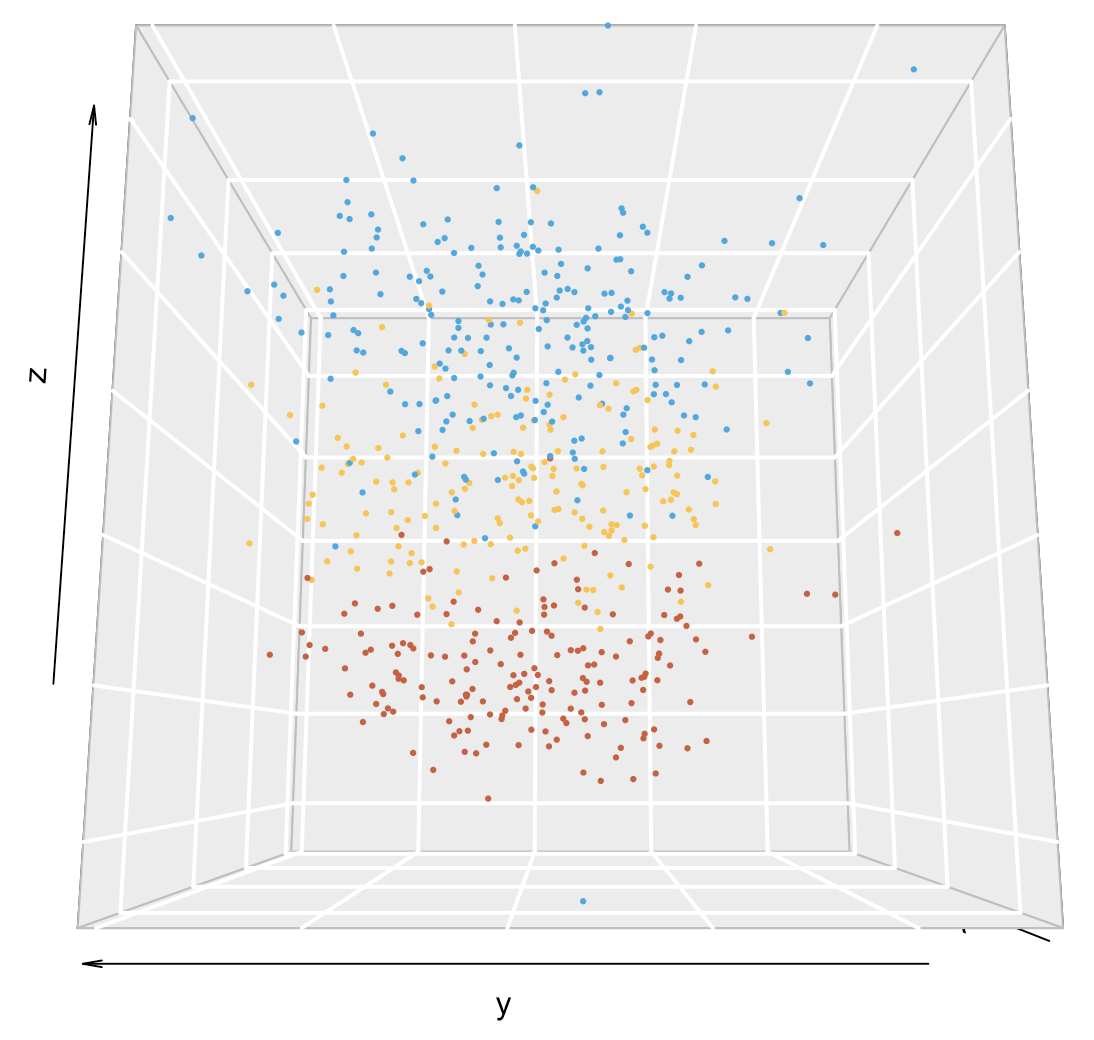}%
\caption{$t_U = 6$, accuracy $=0.62$}
\label{12:3}
\end{subfigure}
\begin{subfigure}{.24\columnwidth}
\includegraphics[width=\columnwidth]{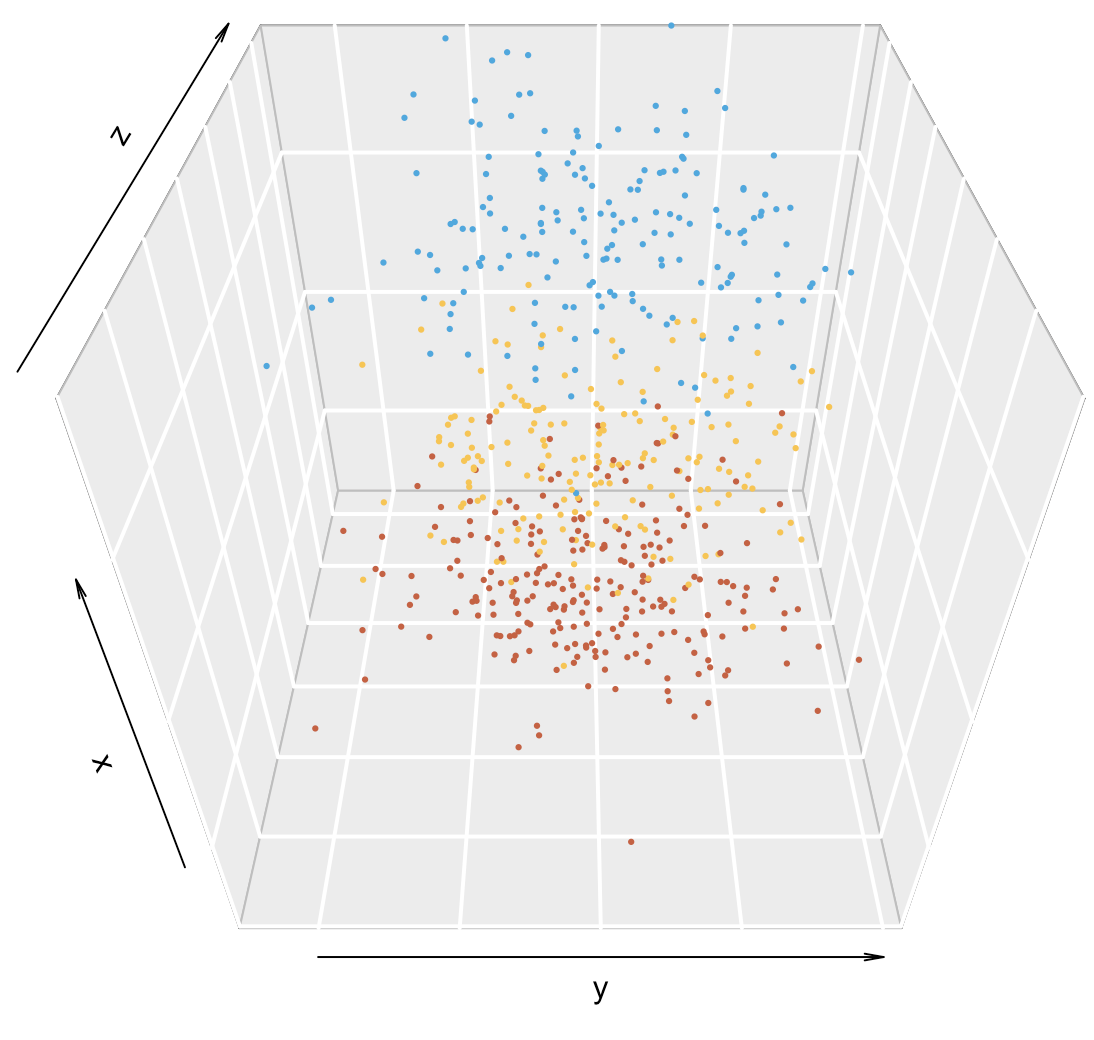}%
\caption{$t_U = 8$, accuracy $=0.61$}
\end{subfigure}
\begin{subfigure}{.24\columnwidth}
\includegraphics[width=\columnwidth]{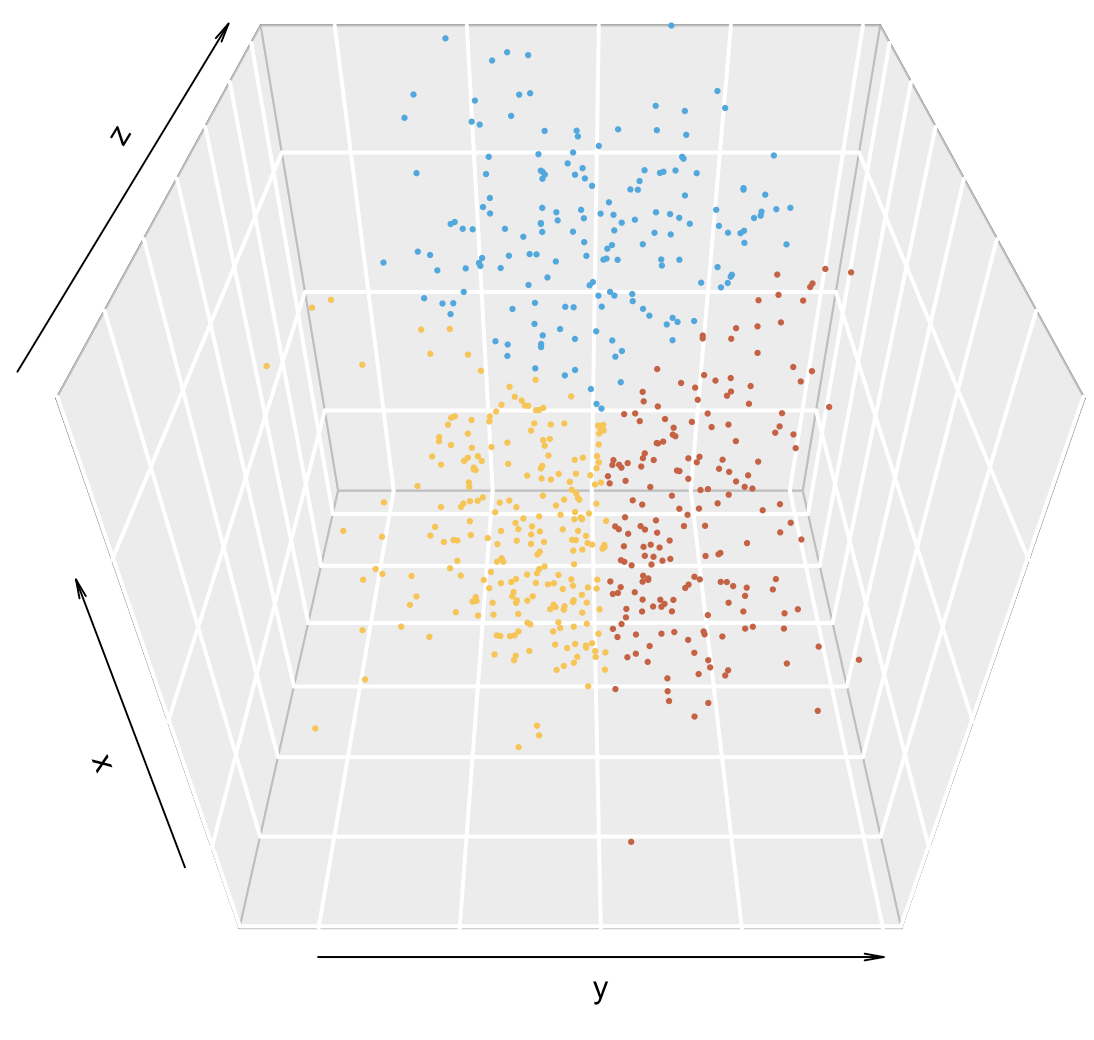}%
\caption{$t_U = 8$, recovery result}
\end{subfigure}
\caption{Visualizations of the SVD-based node2vec embeddings (first row) and original node2vec embeddings (second row) for different choices of $t_U$. The plots are for a single realization of a DCSBM graph on $n = 600$ vertices with block probabilities matrix $\M B_1$ (see eq.~\eqref{eq:Bmat_simulations}), sparsity $\rho_n = 3n^{-1/2}$, and block assignment probabilities $\bm{\pi} = (0.3, 0.3, 0.4)$. The embeddings in panels (a)--(c) and (e)--(g) are colored using the true membership assignments while the embeddings in panels (d) and (h) are colored using the $K$-means clustering.
Accuracy of the recovered memberships (by $K$-means clustering) are also reported for panels 
(a)--(c) and (e)--(g).}
\label{f:embd:dcsbm}
\end{figure*}

\begin{figure*}
\centering
\begin{subfigure}{.32\columnwidth}
\includegraphics[width=\columnwidth]{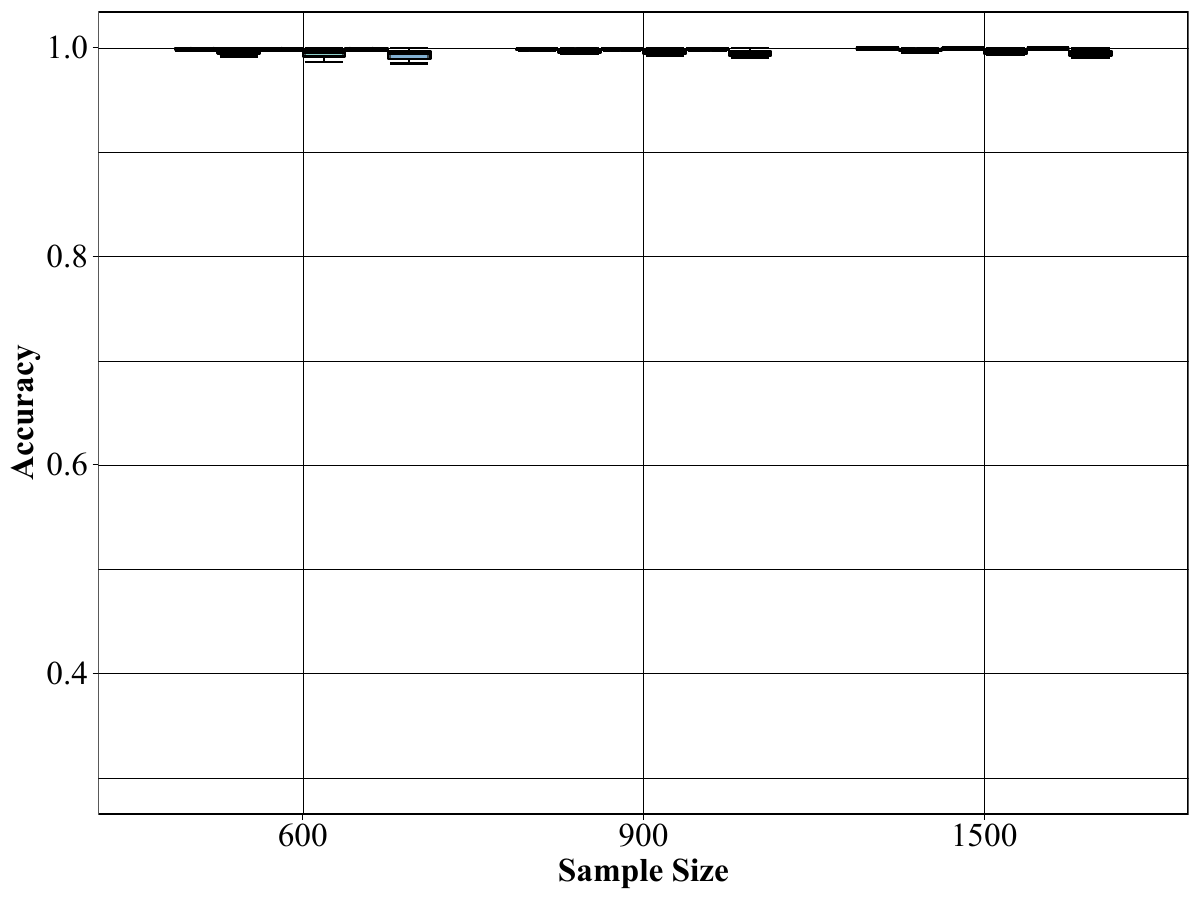}%
\caption{$\rho_n = 1$}
\label{dense:1}
\end{subfigure}
\begin{subfigure}{.32\columnwidth}
\includegraphics[width=\columnwidth]{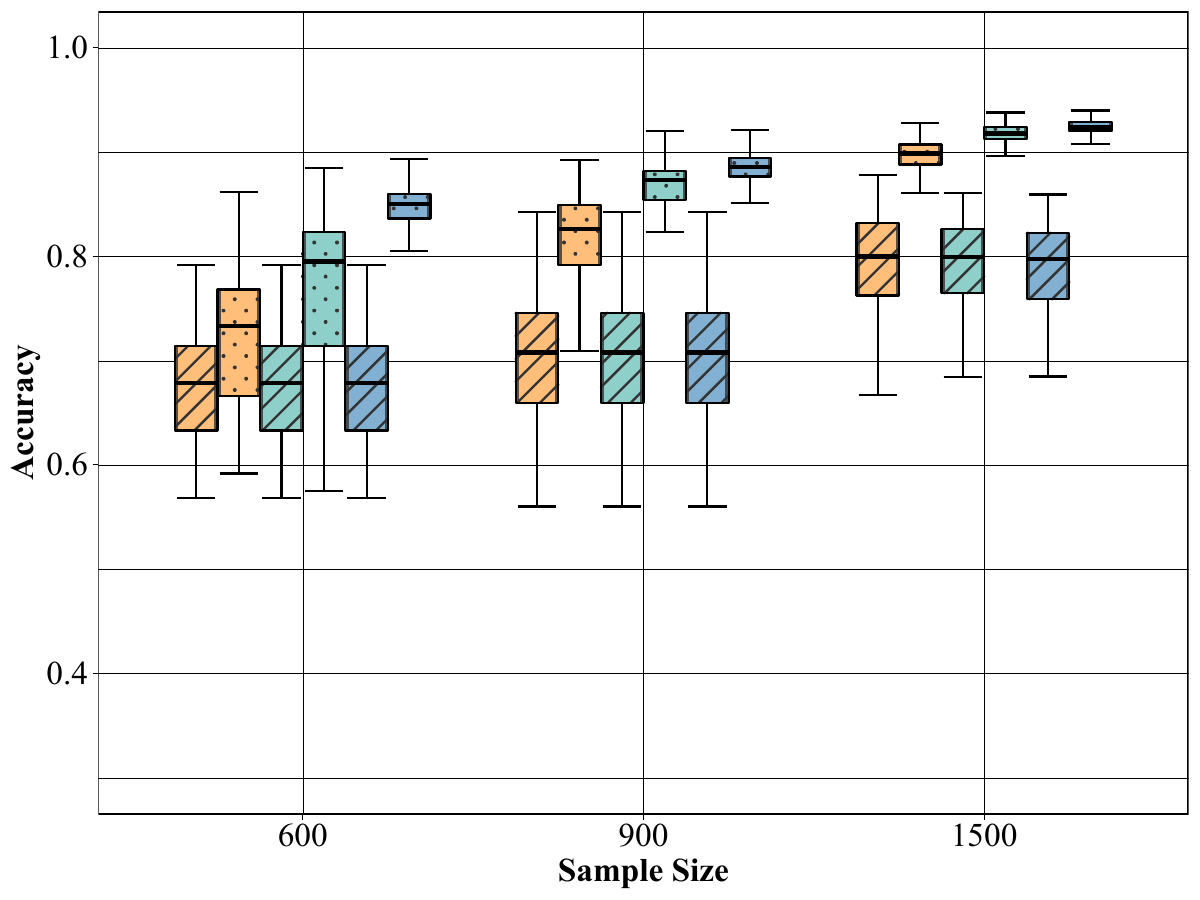}%
\caption{$\rho_n = 3n^{-1/2}$}
\end{subfigure}
\begin{subfigure}{.32\columnwidth}
\includegraphics[width=\columnwidth]{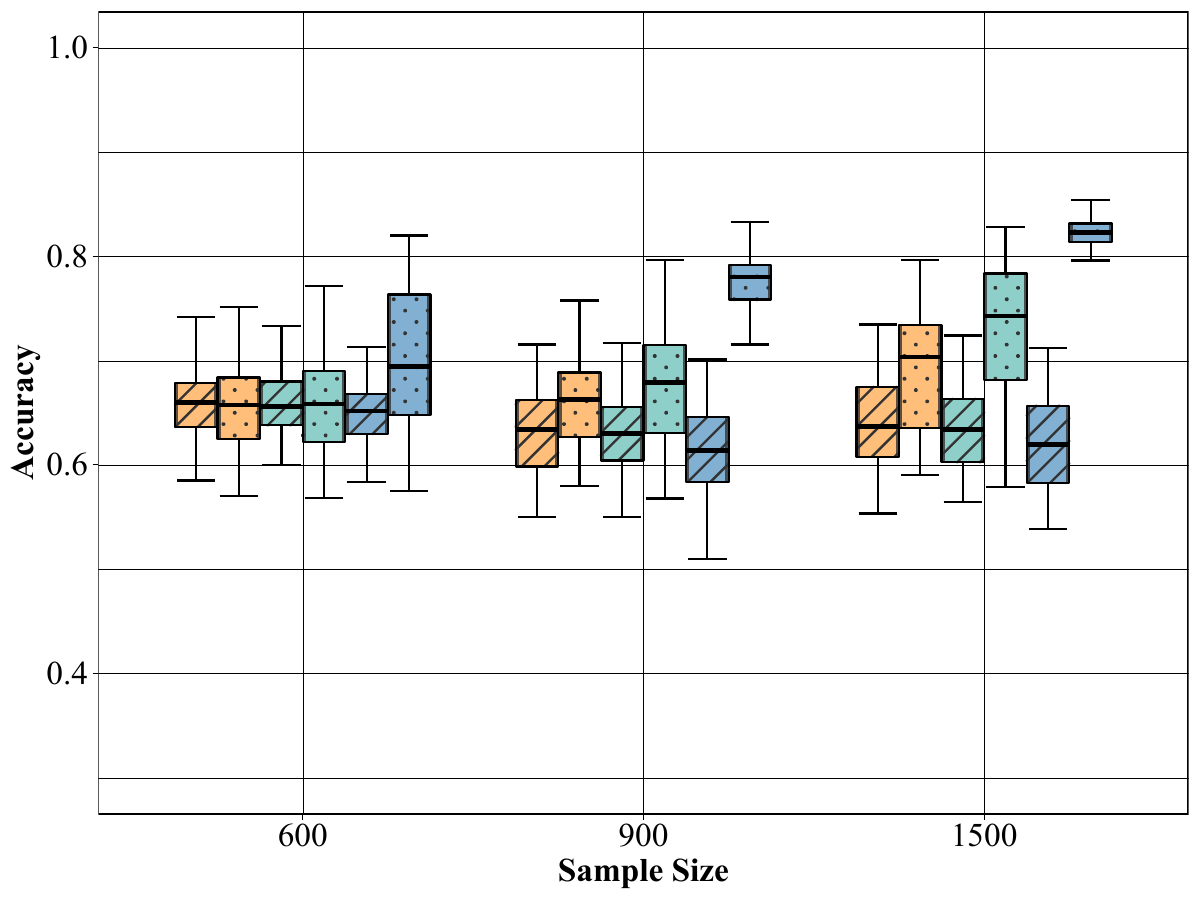}%
\caption{$\rho_n = 6n^{-2/3}$}
\end{subfigure}
\par
\centering
\begin{subfigure}{.32\columnwidth}
\includegraphics[width=\columnwidth]{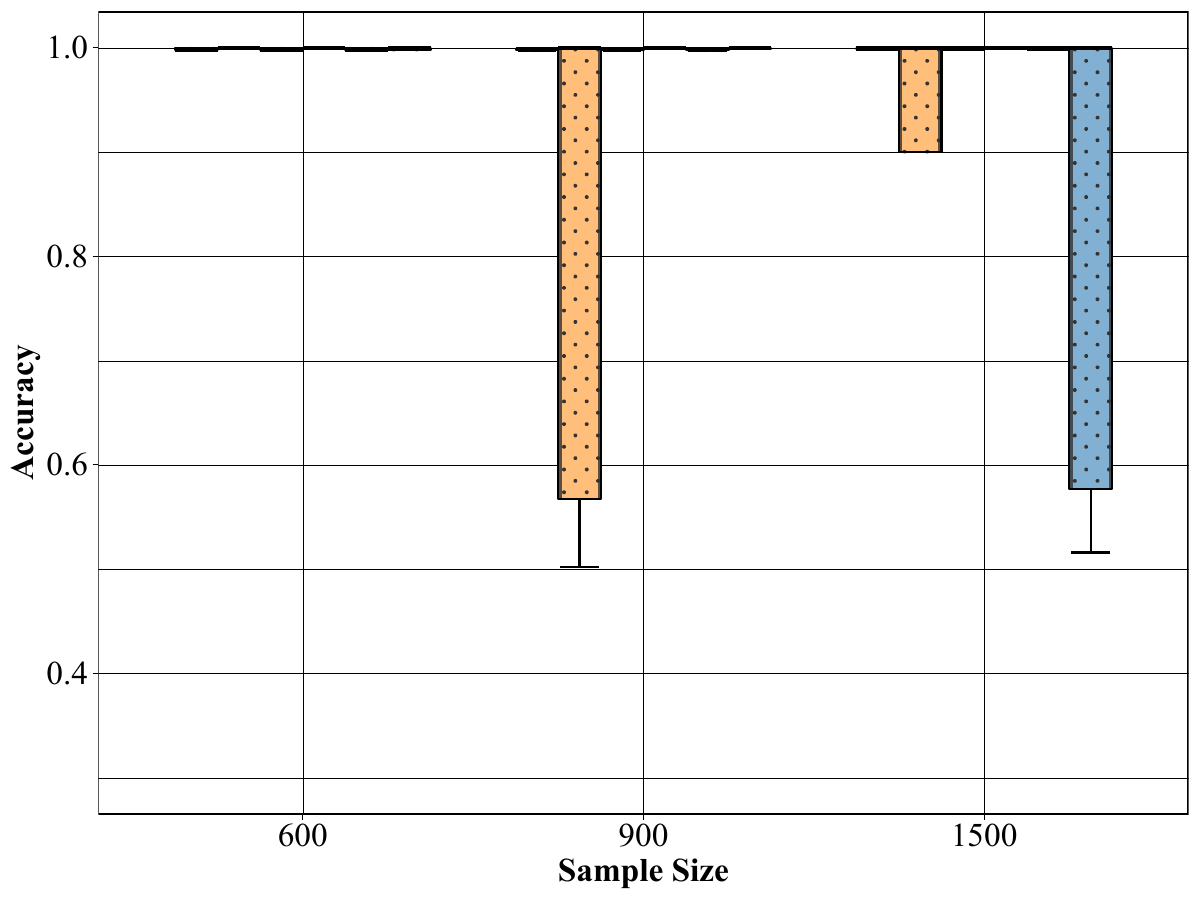}%
\caption{$\rho_n = 1$}
\label{dense:4}
\end{subfigure}
\begin{subfigure}{.32\columnwidth}
\includegraphics[width=\columnwidth]{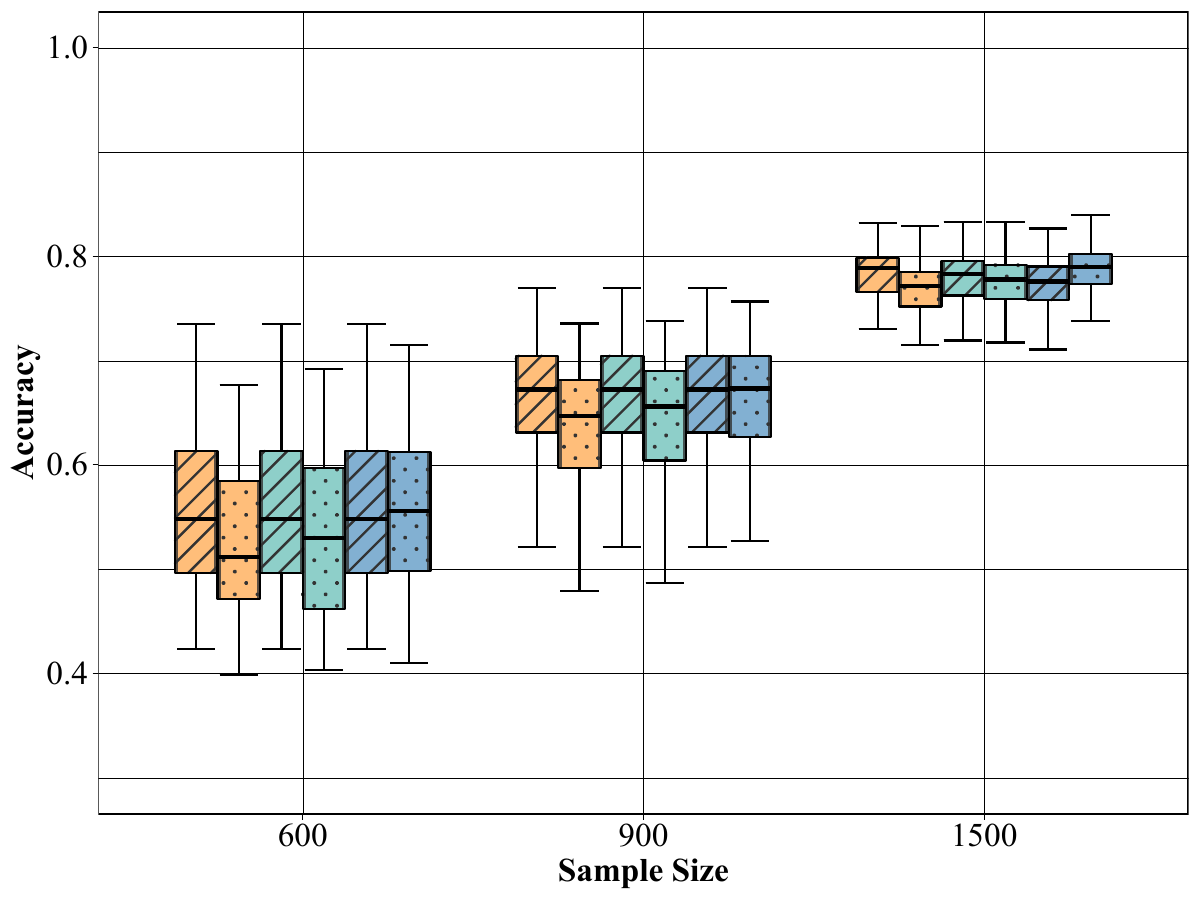}%
\caption{$\rho_n = 3n^{-1/2}$}
\end{subfigure}
\begin{subfigure}{.32\columnwidth}
\includegraphics[width=\columnwidth]{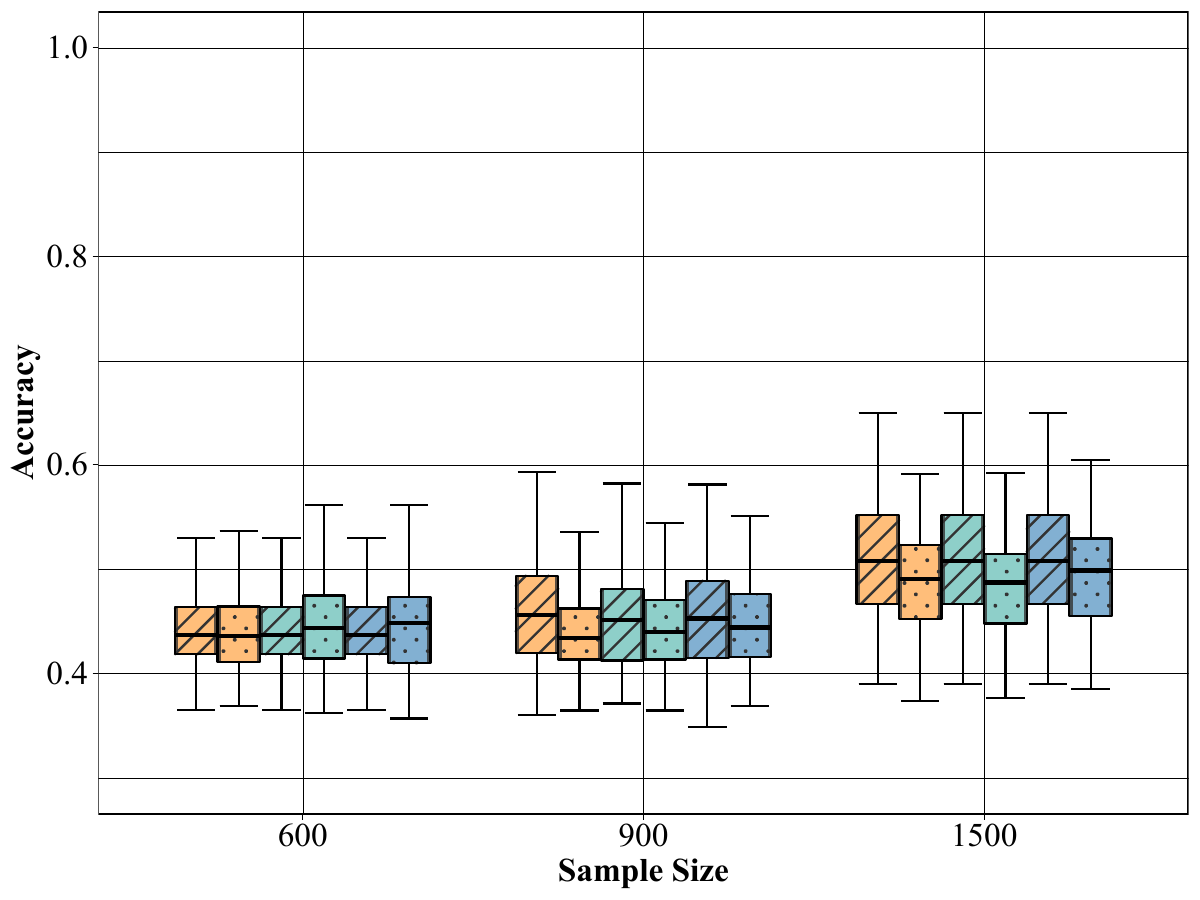}%
\caption{$\rho_n = 6n^{-2/3}$}
\label{subfig:special}
\end{subfigure}
\caption{Community detection accuracy of node2vec followed by
    $K$-means for DCSBM graphs. The
    boxplots of the accuracy for each value of $n, \rho_n$ and $
    t_U$ are based on $100$ Monte Carlo replications. Boxplots with the slash pattern (resp. dot pattern) summarized the results for the original (resp. SVD-based) node2vec. Different colors (yellow, green, blue) represent the algorithms implemented for different choices of $t_U \in \{5,6,8\}$. The first and second row
    plot the results when the block probabilities for the DCSBM is $\M B_1$ and $\M
    B_2$ (see Eq.~\eqref{eq:Bmat_simulations}) respectively.}
\label{box:2}
\end{figure*}

\begin{figure*}[htbp]
\centering
\begin{subfigure}{.38\columnwidth}
\includegraphics[width=\columnwidth]{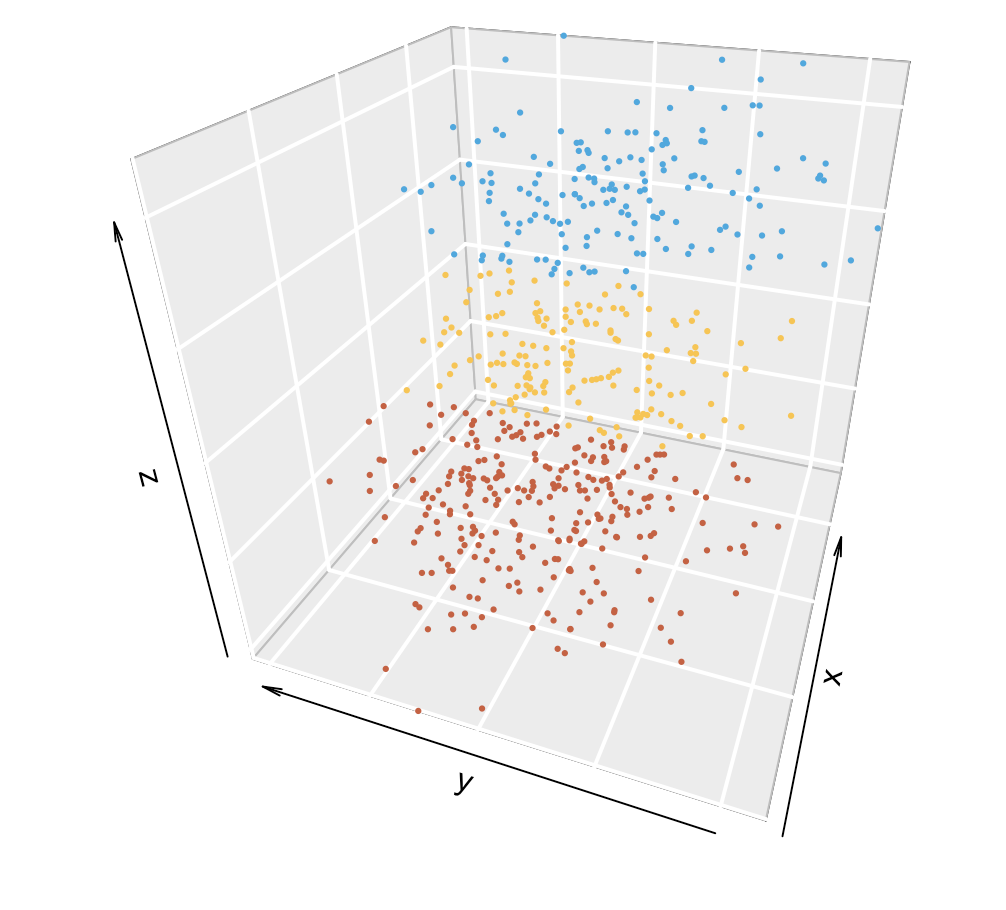}%
\end{subfigure}
\caption{The colors of embeddings represent the recovered memberships of corresponding vertices, by applying the Gaussian mixture model-based (GMM) clustering on the embeddings shown in Fig.~\ref{f:embd:sbm} (g). Comparing Fig. \ref{fig:improve} and Fig.~\ref{f:embd:sbm} (g), one can see that GMM clustering correctly recovers most of vertices' memberships (accuracy of $0.84$).}
\label{fig:improve}
\end{figure*}

\begin{figure*}
\centering
\begin{subfigure}{.32\columnwidth}
\includegraphics[width=\columnwidth]{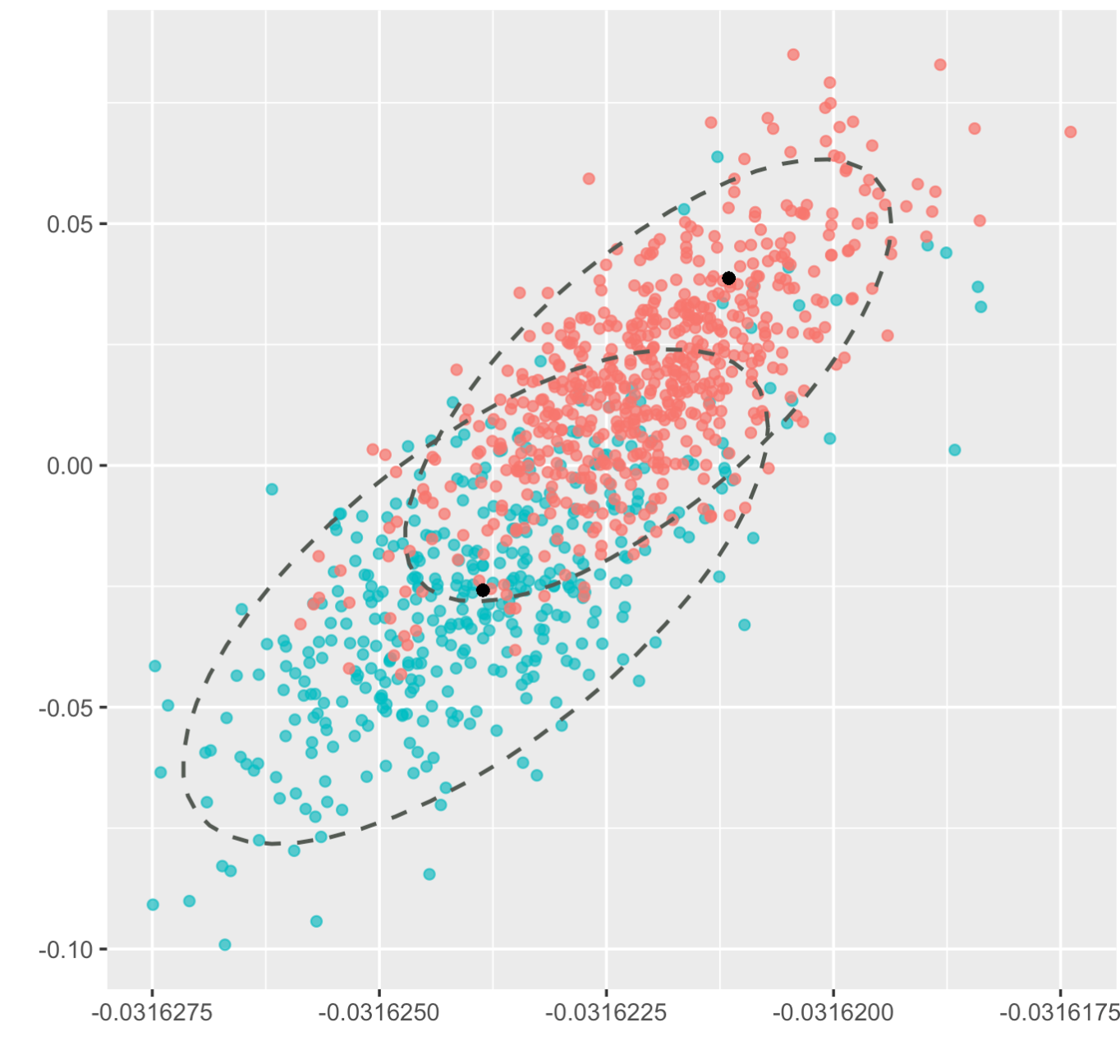}%
\caption{$n = 1000$}
\label{dense:1}
\end{subfigure}
\begin{subfigure}{.32\columnwidth}
\includegraphics[width=\columnwidth]{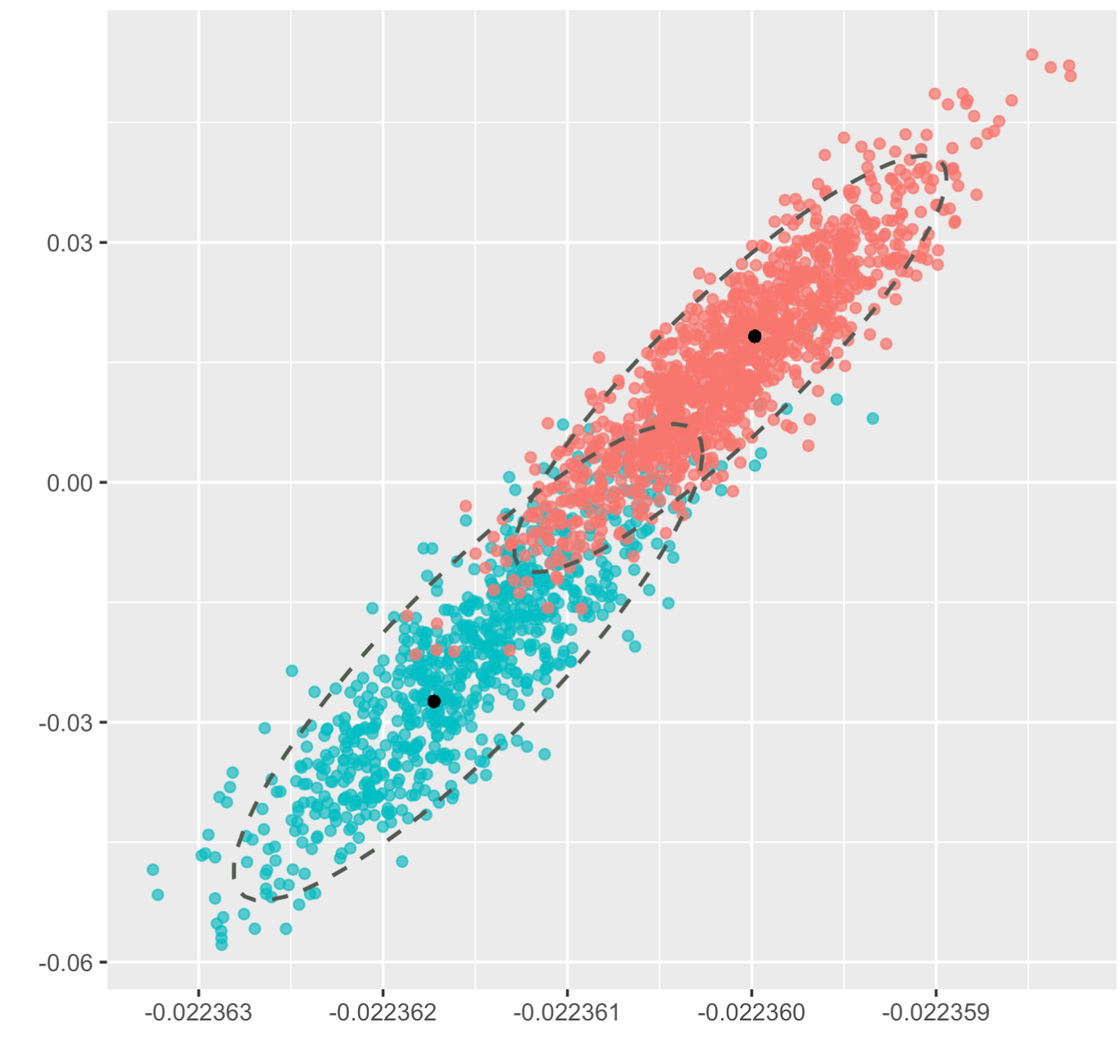}%
\caption{$n = 2000$}
\end{subfigure}
\begin{subfigure}{.32\columnwidth}
\includegraphics[width=\columnwidth]{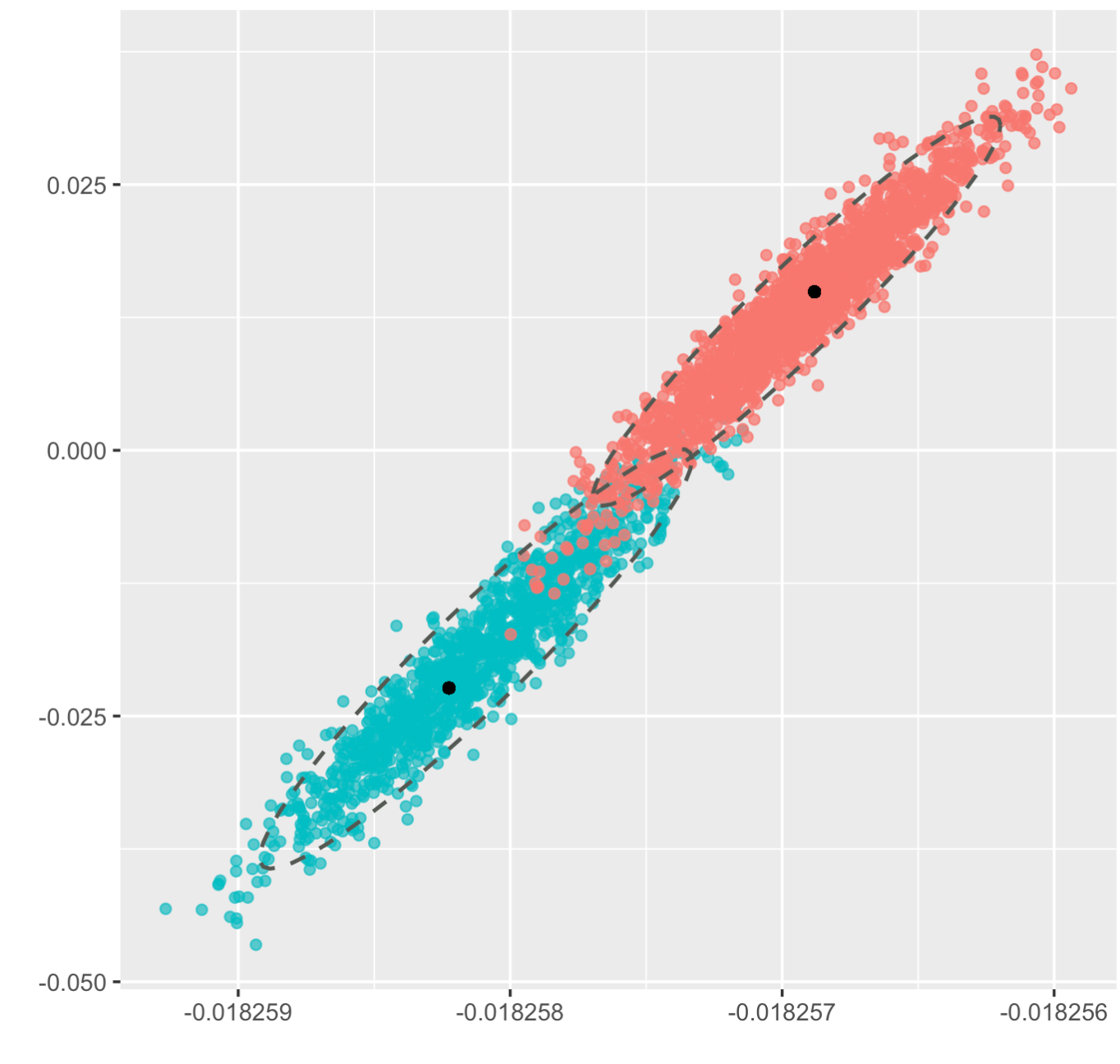}%
\caption{$n = 3000$}
\end{subfigure}
\caption{Scatter plots of node2vec/DeepWalk embeddings
    for a two-blocks SBM with $\mathbf{B}$ and
    $\bds{\pi}$ as defined in Eq.~\eqref{Bpidef} as $n$ varies. The points are colored according to their community membership. The dashed ellipses are the $95\%$ level curves for the {\em block-conditional} empirical distributions. The two black points are the two distinct embedding vectors obtained by factorizing ${\M M}_0$; note that these points had been transformed by the appropriate orthogonal matrices so as to align them with the node2vec/DeepWalk embedding obtained from the observed graphs.}\label{f:embd:sbm}
\end{figure*}

\end{document}